\documentclass[12pt]{article}
\usepackage{amsmath, amssymb}
\usepackage{graphicx}
\usepackage{enumerate}
\usepackage{natbib}
\usepackage{url} 

\usepackage{xcolor}
\usepackage{bm}
\usepackage{subcaption}

\newcommand{\blind}{0}

\newcommand{\MUp}{\bm{T}^{(\bm{U}')}}
\newcommand{\bdcoeff}{b}
\newcommand{\Uset}{\mathbb G}
\newcommand{\Uele}{\mathcal G}
\newcommand{\MUb}{\bm{T}^{(\breve{\bm{U}})}}

\newcommand{\SigB}{{\bm{\Sigma}_\B}}

\newcommand{\cst}{c_1}
\newcommand{\cstt}{c_2}
\newcommand{\csttt}{c_3}
\newcommand{\cstttt}{c_4}

\newcommand{\Cst}{ C }
\newcommand{\bd}{B}
\newcommand{\Mb}{\delta}

\newcommand{\limp}{\rightarrow_p}
\newcommand{\limd}{\rightarrow_d}
\def\T{{ \mathrm{\scriptscriptstyle T} }}

\def\E{\mathbb{E}}
\def\P{\mathbb{P}}
\def\R{\mathbb{R}}

\newcommand{\noise}{\gamma}
\newcommand{\noiseMat}{\Gamma}

\def\A{\textsc{a} }
\def\B{\textsc{b} }
\def\C{\textsc{c} }

\def\a{\textsc{a}} 
\def\b{\textsc{b}} 

\def\T{{ \mathrm{\scriptscriptstyle T} }}
\def\de{\overset{\Delta}{=}}

\def\cRate{\alpha_n }

\usepackage{amsmath}

\DeclareMathOperator*{\argmin}{arg\,min}

\def\bSig\mathbf{\Sigma}

\newcommand{\id}{t}

\def\R{\mathbb{R}}

\newcommand{\nlog}{M}
\newcommand{\mest}{m}

\newcommand{\gradient}{\mathcal G}
\newcommand{\lossW}{\mathcal W}

\newcommand{\lp}{\nu}

\newcommand{\minibatch}{\mathcal S}

\newcommand{\penala}{\textbf{P}}
\newcommand{\penalb}{\mathcal P}
\newcommand{\lam}{\lambda}
\newcommand{\Lam}{\Lambda}

\newcommand{\floor}[1]{\left\lfloor #1 \right\rfloor}

\DeclareMathAlphabet{\mathlcal}{U}{dutchcal}{m}{n}
\newcommand{\mi}{\text{-}}

\newcommand{\AC}{\mathcal A}

\newcommand{\BDset}{\mathbb B(\ThetaSet)}
\newcommand{\InitSet}{\mathcal O}
\newcommand{\InitSetN}{\mathcal O'}

\newcommand{\MU}{\bm{T}^{(\bm{U})}}
\newcommand{\MA}{\bm{T}^{(\mathcal A)}}
\newcommand{\bdd}{S}

\newcommand{\ThetaSet}{\Omega_0}
\newcommand{\ThetaSett}{\widetilde{\Omega}_0}
\newcommand{\ThetaSetN}{\Omega_0'}
\newcommand{\TransformMat}{\bm{T}}

\newcommand{\LSG}{\mathcal L_n}

\newtheorem{definition}{Definition}

\newtheorem{theorem}{Theorem}

\newtheorem{lemma}{Lemma} 
\newtheorem{corollary}{Corollary}
\newtheorem{proof}{Proof}

\newtheorem{proposition}{Proposition}
\newtheorem{condition}{Condition}

\newcommand{\InitSetNn}{\mathcal O'_n}
\newcommand{\emptySet}{\phi}
\newcommand{\ThetaSetU}{\Omega_0^{(\bm{U})}}

\newcommand{\csthes}{ c }
\graphicspath{{figures_supplement/}{figures/}}

\usepackage[title,toc,titletoc]{appendix}
\begin{document}

\def\spacingset#1{\renewcommand{\baselinestretch}%
{#1}\small\normalsize} \spacingset{1}


\if0\blind
{
  \title{\bf Additive-Effect Assisted Learning}
  

\author{Jiawei Zhang\thanks{Dr. Bing Zhang Department of Statistics, University of Kentucky. Email: zhan4362@umn.edu}, Yuhong Yang\thanks{School of Statistics, University of Minnesota. Email: yangx374@umn.edu}, and Jie Ding\thanks{School of Statistics, University of Minnesota. Email: dingj@umn.edu}}

  \maketitle
} \fi

\if1\blind
{
  \bigskip
  \bigskip
  \bigskip
  \begin{center}
    {\LARGE\bf Additive-Effect Assisted Learning}
\end{center}
  \medskip
} \fi

\bigskip

\tableofcontents

\newpage
%
%






\abstract{
It is quite popular nowadays for researchers and data analysts holding different datasets to seek assistance from each other to enhance their modeling performance. We consider a scenario where different learners hold datasets with potentially distinct variables, and their observations can be aligned by a nonprivate identifier. Their collaboration faces the following difficulties: First, learners may need to keep data values or even variable names undisclosed due to, e.g., commercial interest or privacy regulations; second, there are restrictions on the number of transmission rounds between them due to e.g., communication costs. To address these challenges, we develop a two-stage assisted learning architecture for an agent, Alice, to seek assistance from another agent, Bob. In the first stage, we propose a privacy-aware hypothesis testing-based screening method for Alice to decide on the usefulness of the data from Bob, in a way that only requires Bob to transmit sketchy data. Once Alice recognizes Bob's usefulness, Alice and Bob move to the second stage, where they jointly apply a synergistic iterative model training procedure. With limited transmissions of summary statistics, we show that Alice can achieve the oracle performance as if the training were from centralized data, both theoretically and numerically.
}

\textbf{keywords:}{
assisted learning,
additive effects,
generalized linear model,
hypothesis testing,
decentralized learning
}



\maketitle
\section{Introduction}
\label{sec_intro}

With the rapid development of data-generating/collection devices, there are various learning scenarios where data from diverse sources need to be integrated to address personalized tasks. Problems of interest include distributed learning from low-energy-cost devices~\citep{konevcny2016federated,diao2020heterofl}, modularized learning~\citep{jacob2017better,DingCollabParam}, 
fusion learning~\citep{tang2016fused,tang2019fusion, shen2020fusion, tang2021poststratification, luo2021leveraging}, and multimodal data integration~\citep{cai2016structured, yu2020optimal, xue2020integrating}.

This paper considers an important practical scenario where a learner Alice (denoted by ``\a'') seeks assistance from other learners, such as Bob (denoted by ``\b''), to fit a model jointly.
\A and \B have access to datasets that contain the same cohort of individuals. As illustrated in Figure~\ref{datsplit}, the two datasets can be collated according to a certain identifier, e.g., timestamps in website traffic data and patient IDs in medical data. 
  \begin{figure}[ht]
	\centering
	\vspace{-0.1cm}
	\includegraphics[width=0.7\linewidth]{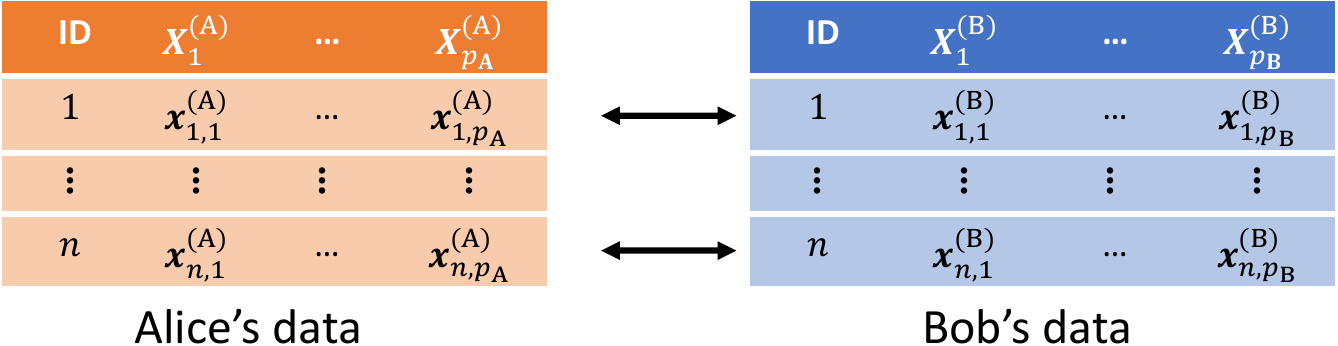}
	\vspace{-0.1cm}
	\caption{Illustration of data held by two learners after removing unmatched cases, where the rows(cases) are collated according to a common identifier, e.g., timestamp, username, or case ID.}
	\label{datsplit}
\end{figure}
\A hopes that \B may help improve her modeling performance with his side information. This modeling problem has two major challenges. The first is \textbf{learner selection}.  It is possible that \b's data cannot improve \a's model performance. In this case, switching from \B to other learners quickly is clearly desirable. Thus, \A prefers a screening procedure to evaluate whether a particular learner is helpful before formal joint learning starts. 
The second is \textbf{joint training}. Suppose \B cannot directly share his original data with \A for the following reasons: \textbf{i)} \B only provides query services that return summary statistics while keeping the original dataset local due to regulations or business concerns; \textbf{ii)} the sample size is so large that \A cannot afford the transmission cost. The above reasons motivate the problem of how to efficiently obtain a model, if possible, with its performance close to the one trained on the centralized data pooled from \A and \b.
Some real-world examples are given below.


\textit{Example 1:} 
Two medical institutions \A and \B possess data on a set of common patients. Researchers in \A want to predict the outcomes of the patients (e.g., length of ICU stay), and incorporating the data from \B may potentially enhance the prediction accuracy. However, due to data privacy concerns, \B is unable to directly share the data with \a.

\textit{Example 2:}  
A sensor \A in a smart transportation system is used to identify cars exceeding the speed limit. It has images of cars captured from the front angle. Additionally, it can establish a remote connection to a monitor \B that captures side images of cars. However, due to the significant costs associated with transmitting images, pooling the images together is not feasible for \a. 

\textit{Example 3:} \label{exp_internet_introduction}
An internet service provider \A aims to protect against internet attacks. Some information of internet traffic from a user \B can potentially be crucial. However, \A is unable to measure or access this information since they are protected by \B via encryptions. Moreover, due to privacy policies, \A cannot directly access the data collected by~\b.

There has been a large body of studies on distributed learning \citep{bekkerman2011scaling, bertsekas2015parallel} that address challenges in model fitting where data are held separately by different learners. For instance, the works by \citet{du2004privacy}, \citet{dean2012large}, \citet{gong2015privacy}, \citet{gratton2018distributed}, \citet{hu2019fdml}, \citet{hu2019learning}, and \citet{diao2022gal}.
A popular distributed learning method is federated learning \citep{konevcny2016federated, yang2019federated}. It was originally designed for learners with the same set of covariates from different subjects \citep{konevcny2016federated} and was later extended to vertical federated learning \citep{chen2020vafl, cheng2021secureboost, khan2022communication} for the scenario where learners hold different covariates (as in Figure~\ref{datsplit}).
Another solution to tackling the distributed learning problem is assisted learning \citep{xian2020assisted, wang2022parallel, diao2022gal}. 
It targets scenarios where learners have sufficient computational resources but aim to reduce the number of communication rounds. This differs from 
federated learning, which typically relies on frequent transmissions between the learners to reduce the overall computation loads. 
Apart from predictive modeling, there has been increasing attention in recent years to statistical inference techniques, such as hypothesis testing and confidence interval estimation, as well as understanding the statistical properties of the trained models \citep{el2015slepian, rosenblatt2016optimality, jordan2019communication, dobriban2021distributed}. 

In our learning problem, there are two key factors.
\textbf{1. Undisclosed sets of variables held by the learners:} The learners do not reveal what variable they possess nor their values due to, e.g.,  privacy or commercial considerations.
\textbf{2. Assistance cost associated with transmission: }
The transmissions between the learners may incur significant expenses, including communication costs due to bandwidth limitation, as well as time and other resource costs (e.g., when \B is a consulting company, it may charge \A a service fee based on the number of interactions).
To our best knowledge, existing works do not directly address the following challenges arising from the above factors:
\begin{enumerate}
\item\ The cost associated with the collaboration will make \A carefully weigh the value of the assistance from \b. 
This assessment is, however, challenging due to the undisclosed set of variables and their values held by \b. 
Therefore, \A needs a new mechanism to decide the merit of \b's assistance.
In contrast, the existing studies typically consider settings where different parties have already reached an agreement on collaboration, with a focus on training efficiency and privacy protection in the subsequent modeling procedures.

\item\  
Even when \A  decides to collaborate with \b,
it is still challenging to obtain a good model training algorithm and establish its efficacy in both statistical and computational aspects.
First, to control the overall cost of training, it is desirable for the algorithm to
converge relatively fast with respect to the number of transmission rounds. 
Second, it is important to understand the statistical properties of the trained model within a limited number of transmission rounds, given the randomness of data. 
In particular, we may need to quantify model uncertainty and compute prediction intervals. 
However, the focus in the literature is on the algorithmic convergence on a fixed dataset only, ignoring the large sample convergence from a statistical perspective.

\end{enumerate}

To tackle the challenges, we propose an assisted learning architecture that consists of the following two components. 
\begin{itemize}
\item[(S1)] \textbf{Initializing connection.} \A assesses \b's capability in assisting her learning task by a statistical testing based on a small set of sketchy data from \b. If \B passes the test, \A continues to \textbf{S2}.
This test properly addresses the learner selection problem in the following three aspects: \textbf{1)} Unlike traditional variable selection or model comparison tools~\citep{DingOverview} that require pooling the two datasets together or rely on fully trained models from \b, our method only needs \B to send a sketchy dataset to \a. This feature makes our approach easily implementable in decentralized settings. \textbf{2)} The sketchy data can be conveniently adjusted to satisfy local differential privacy \citep{kasiviswanathan2011can, duchi2013local, duchi2018minimax} according to the privacy budget, thereby safeguarding the data of  \b. 
\textbf{3)}
 The modeling goal of \A is undisclosed during the testing procedure, which is appealing for \A to seek assistance from many learners in the open internet world, some of whom may be scammers with fake data.

\item[(S2)]	\textbf{Assisted training}.  \B assists \A in virtually fitting a joint model by communicating task-relevant statistics. After finishing the model training, \A and \B can separately operate on their local parts to jointly obtain prediction results. 
The proposed training method has the following two key properties that help reduce the number of transmission rounds, which is crucial in the assisted learning setting: 
\textbf{1)} It only requires \A and \B to iteratively exchange some linear combinations of their data and minimize their local losses, without the need for extensive hyperparameter (e.g., learning rates) tuning. 
    In contrast, other applicable decentralized algorithms including gradient descent-based methods \citep{yang2021model, das2021multi,liu2022fedbcd}, coordinate descent Newton methods \citep{bradley2011parallel, bian2019parallel}, and CoCoA \citep{smith2018cocoa} typically require learners to jointly train models with different hyperparameter values, which may greatly increase the number of transmissions between \A and \b. 
    \textbf{2)} Our estimator has an exponential convergence rate with respect to the number of transmission rounds, converging to a loss minimizer that achieves the same prediction performance as the oracle estimator trained on the directly combined data from \A and \B in hindsight. This result holds true even in scenarios where the covariates from \A and \B are linearly dependent. It can occur when \A and \B share some common covariates, and they are unable to remove duplicated covariates due to a lack of knowledge regarding the covariates held by each other. In such cases, other algorithms often only guarantee a polynomial convergence rate (e.g., \citet{richtarik2016distributed}, \citet{peng2016arock}, \citet{mahajan2017distributed}). 
Additionally, we present theoretical results on the statistical properties of our proposed method, which lead to valid prediction~intervals. 

\end{itemize}

The outline of the paper is given below.  
Section~\ref{sec_procedure} describes the proposed assisted learning architecture and methods and presents their theoretical properties.
Section~\ref{sec_GLM} focuses on the application of the proposed methods to generalized linear models. 
Sections~\ref{sec_exp} and \ref{Sec_real_data} provide simulation results and real data examples, respectively.
Section~\ref{sec_con} summarizes the paper. Technical proofs and additional simulation studies are included in the supplementary document.


\section{Additive-Effect Assisted Learning}\label{sec_procedure}
Suppose there are two learners, \A and \b, and \A seeks assistance from \b. 
Let $Y\in\mathbb Y\subseteq \mathbb R$  denote the response of interest to \a, and let $\bm{y} = (y_1,\dots,y_n)^\T$ denote an i.i.d.\ sample. 
Let $X^{(\A)} \in \R^{p_\A}$ and $X^{(\B)} \in \R^{p_\B}$ denote the associated covariates of \A and \b, respectively. 
For any random vector $W \in \R^q$ of covariates, let $\bm W \in \R^{n \times q}$ denote the design matrix whose $i$th row $\bm w_i$ is an independent observation of $W$. Let $\bm w$ denote a generic observation of $W$. For example, for \a's covariates $X^{(\A)}$, we use the notations $\bm{X}^{(\A)}$, $\bm{x}^{(\A)}_i$, and $\bm{x}^{(\A)}$.
Assume that the covariates in $X^{(\A)}$ are linearly independent in the sense that there does not exist a nonzero deterministic vector $\bm{z}$ such that $\bm{z}^\T X^{(\A)} = 0$ holds almost surely, and the same holds for $X^{(\B)}$.  Assume that $X^{(\A)} = ({X^{(a)}}^\T, {X^{(c)}}^\T)^\T$ and $X^{(\B)} = ({X^{(b)}}^\T, {X^{(c)}}^\T)^\T$, and the covariates in $X = ({X^{(a)}}^\T, {X^{(b)}}^\T, {X^{(c)}}^\T)^\T$ are linearly independent. 
Note that $X^{(c)}$ is the set of shared covariates, which is allowed to be empty. It is also important to point out that the above partition of the covariates is not known by \A and \b, and the covariates they hold may be arbitrarily ordered.
Denote the dimensions of $X$, $X^{(\A)}$, $X^{(\B)}$, $X^{(a)}$, $X^{(b)}$, and $X^{(c)}$ by $p$, $p_\A$, $p_\B$, $p_a$, $p_b$, and $p_c$, respectively. Each categorical covariate is represented by $d-1$ indicator variables, where $d$ is the number of categories of the covariate.
\begin{figure}[!ht]
	\centering
	\includegraphics[width=0.3\linewidth]{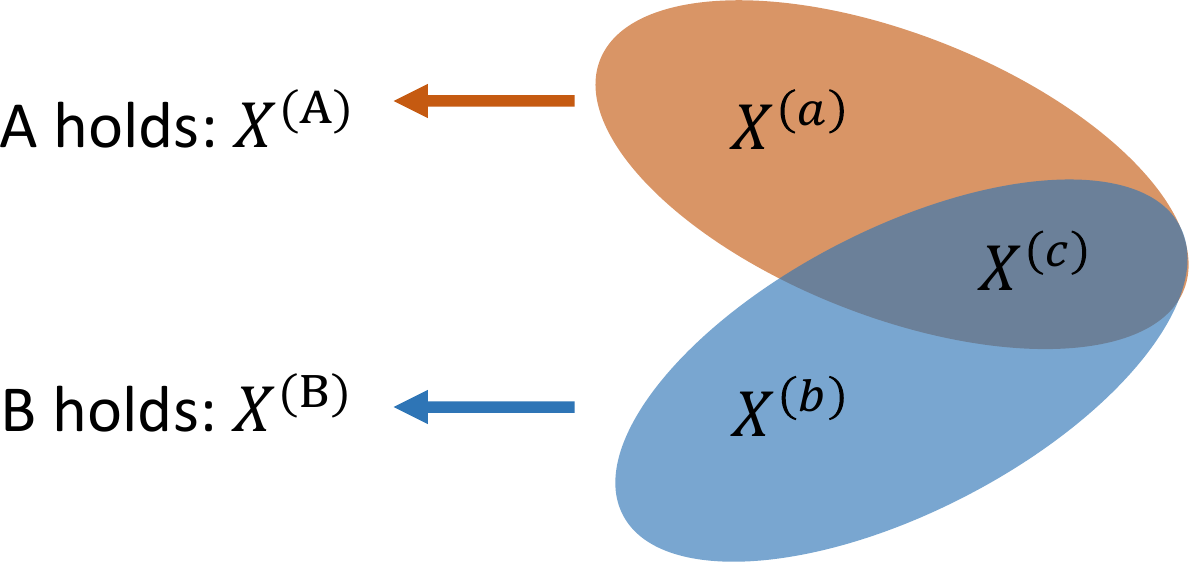}
	\caption{An illustration of the data held by \A and \b. }
	\label{fig_dataSimplified}
\end{figure}

The goal of \A is to predict future $Y$ based on the associated covariates information. To that end,  \A  targets the performance of an M-estimator 
$
  \check{\bm{\beta}}_n \de \argmin_{\bm{\beta}\in \R^p}\ \nlog_n(\bm{\beta})
$, where
\begin{equation}\label{eq_mestimator}
   \nlog_n(\bm{\beta}) \de \frac{1}{n}\sum_{i=1}^n \mest(y_i,{\bm{x}_i}^\T \bm{\beta} ), 
\end{equation}
and $\mest:\mathbb Y\times\R\rightarrow\R$ is a loss function. The loss function has an additive effect nature in the sense that the total effect of the infeasible linear combination $\bm{x}^\T_i\bm{\beta}$ can be expressed through a sum of the effects of the covariates held by \A and \B separately. This additivity is crucial for our assisted learning methodology. 
Our additive-effect assisted learning (AE-AL) procedure is summarized as follows. In step \textbf{S1} ``initializing connection,'' 
\B first transmits $\bm{X}^{(\B)}\bm{U}$ to \a, where $\bm{U}$ is a randomly generated matrix. 
Then, \A performs a statistical test that compares the model fitted on $(\bm{y},\bm{X}^{(\A)})$ with the one on $(\bm{y},\bm{X}^{(\A)},\bm{X}^{(\B)}\bm{U})$. Once the test shows that there is a significant difference between the two models, \A confirms \b's usefulness and proceeds to the next step.
 In step \textbf{S2} ``assisted training,'' \A and \B iteratively exchange and update their locally estimated parameters. They terminate the process according to a stop criterion.
The obtained outputs $\hat {\bm{\beta}}^{(\A)}_n$ and $\hat {\bm{\beta}}^{(\B)}_n$ are held by \A and \b, respectively. 
For ``future prediction,'' \A and \B hold new observations $\widetilde{\bm{x}}^{(\A)}$ and $\widetilde{\bm{x}}^{(\B)}$, respectively. \B sends the output $(\widetilde{\bm{x}}^{(\B)})^\T\hat {\bm{\beta}}^{(\B)}_n$  to \a, who then produces a combined prediction result.
The above procedure is illustrated in Figure~\ref{fig_concept}. 
\begin{figure}[!ht]
	\centering
	\includegraphics[width=0.5\linewidth]{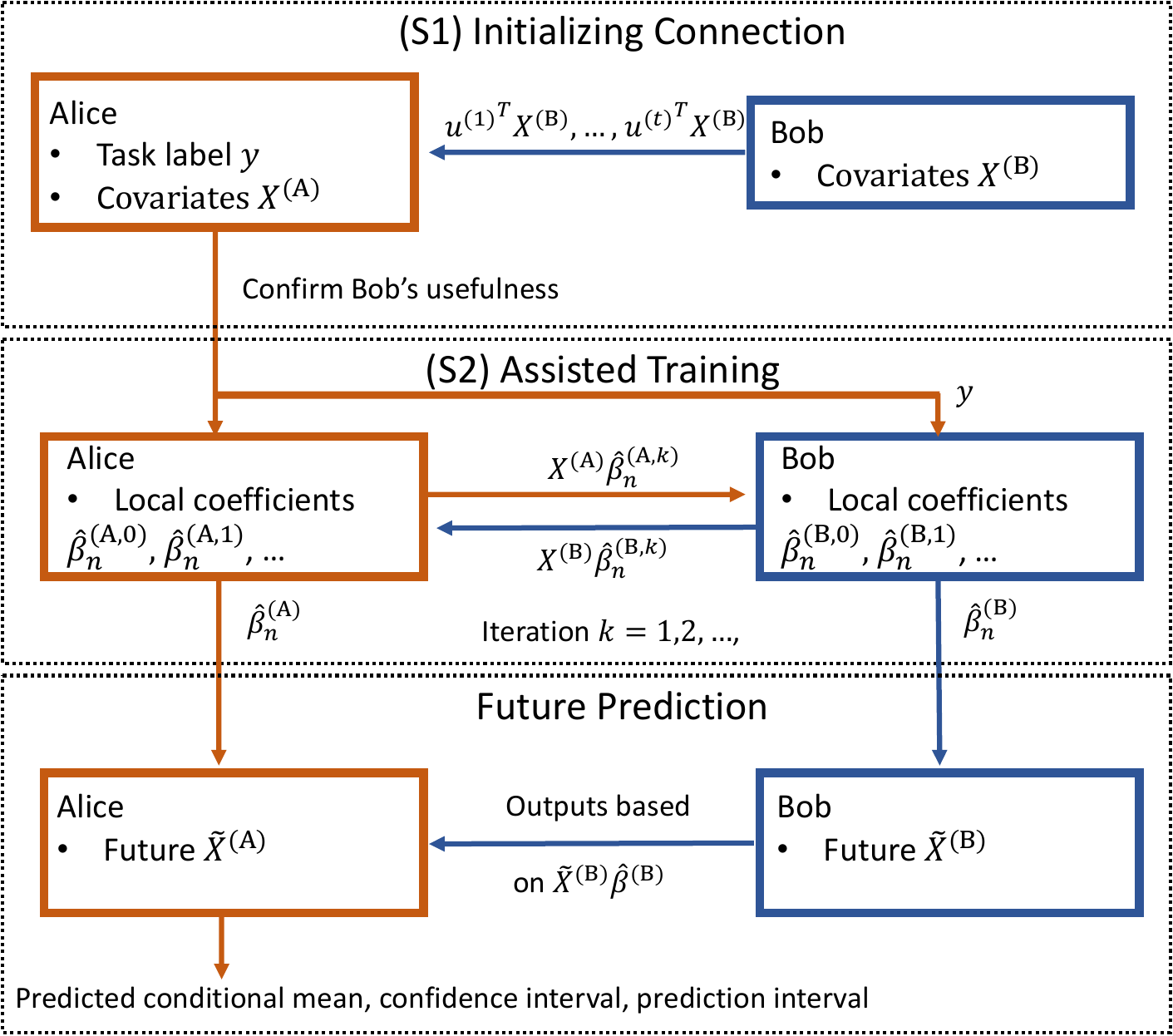}
	\caption{A graphical summary of the additive-effect assisted learning (AE-AL).}
	\label{fig_concept}
\end{figure}
We next introduce the detailed procedures and their theoretical results in the following sections. 

\subsection{Initializing Connection}\label{sec_test} 
In this subsection, we first introduce the procedure for conducting statistical testing during the ``initializing connection'' stage in Section~\ref{subsub_testProc}, and then discuss its privacy protection in Section~\ref{sec_priv}. 

\subsubsection{Testing Procedure}\label{subsub_testProc}
In 
step \textbf{S1}, \A assesses the usefulness of \B from a hypothesis testing perspective, with 
\begin{align}
    &\textrm{H}_0:\ \bm{\beta}^{(b)}_* = \bm{0},\quad
\textrm{H}_1:\  \bm{\beta}^{(b)}_* \neq \bm{0},\label{eq_H0H1}
 \end{align}
 where
$\bm{\beta}_*=
 ({\bm{\beta}^{(a)}_*}^\T,
 {\bm{\beta}^{(b)}_*}^\T,{\bm{\beta}^{(c)}_*}^\T)^\T
 \de \underset{\bm{\beta}\in\mathbb R^{p}}{\argmin}\  \nlog_*(\bm{\beta} )$, $ \nlog_*(\bm{\beta})\de \E\bigl(m(Y,X^\T\bm{\beta})\bigr)$,
 $\E(\cdot)$ is the expectation with respect to the joint distribution of $(Y,X)$, and $\bm{\beta}^{(a)}_*$, $\bm{\beta}^{(b)}_*$, $\bm{\beta}^{(c)}_*$ correspond to $X^{(a)}$, $X^{(b)}$, $X^{(c)}$, respectively. 
Clearly, $\textrm{H}_0$ indicates that \B is not helpful. 
For a  function $\bm{\alpha}\mapsto G(\bm{\alpha})$, let $\nabla G(\bm{\alpha})$ and $\nabla^2 G(\bm{\alpha})$ denote its gradient and Hessian matrix, respectively. 
For a function $(\bm{\alpha}_1,\bm{\alpha}_2)\mapsto G(\bm{\alpha}_1,\bm{\alpha}_2)$, let $\nabla_{\bm{\alpha}_1}G(\bm{\alpha}_1,\bm{\alpha}_2)$ and $\nabla_{\bm{\alpha}_1}^2G(\bm{\alpha}_1,\bm{\alpha}_2)$ denote its gradient and Hessian matrix with respect to $\bm{\alpha}_1$, respectively. Let $\lp \de \bm{x}^\T\bm{\beta}$ be the combined linear effect.

To perform the test, \B first sends \A a small set of sketchy data. Then, \A tests the hypotheses in \eqref{eq_H0H1} by calculating a Wald-type test statistic that compares the model fitted on $(\bm{y},\bm{X}^{(\A)})$ with the one that includes the data transmitted from \b.
 The detailed testing procedure is as follows.
\\
\textbf{Step 1: }\B generates a matrix 
$
    \bm{U} = (\bm{u}^{(1)},\dots,\bm{u}^{(\id)}) \in \R^{p_\B\times\id},
$
where each $\bm{u}^{(j)}$ ($j=1,\dots,t$) is a unit vector transformed from an i.i.d.\ sample from the standard normal distribution with $\id\leq p_\B$. Then, \B sends $ \bm{X}^{(\B)}\bm{U}$ to \a. \\
\textbf{Step 2: }\A first fits the model based on $(\bm{y},\bm{X}^{(\bm{U})})$, where \label{page_betaU_def}
$\bm{X}^{(\bm{U})}\de (\bm{X}^{(\A)},\bm{X}^{(\B)}\bm{U})$.
Next, \A calculates a Wald test statistic $W_{n,\id} = n\bigl(\hat{\bm{\beta}}^{(\bm{U},\id)}_n\bigr)^\T \hat{\bm{V}}_{\id}^{-1}\hat{\bm{\beta}}^{(\bm{U},\id)}_n$, 
where  $\hat{\bm{\beta}}^{(\bm{U},\id)}_n$ denotes the $\id$ elements of  $\hat{\bm{\beta}}^{(\bm{U})}_n$ corresponding to $\bm{X}^{(\B)}\bm{U}$, 
$\hat{\bm{\beta}}^{(\bm{U})}_n\de \arg\min_{\bm{\beta}^{(\bm{U})}\in \mathbb R^{p_\A+\id}}\nlog_n^{(\bm{U})}(\bm{\beta}^{(\bm{U})})$,   
$\nlog_n^{(\bm{U})}(\bm{\beta}^{(\bm{U})})\de n^{-1}\sum_{i=1}^n\mest(y_i, ({\bm{x}^{(\bm{U})}_i})^\T\bm{\beta}^{(\bm{U})})$, $\hat{\bm{V}}_{\id}$ denotes the lower right
 $\id\times \id$ block of $\hat{\bm{V}}\de \hat{\bm{V}}_1^{-1}\hat{\bm{V}}_2\hat{\bm{V}}_1^{-1}$ with
 $\hat{\bm{V}}_1\de \nabla^2\nlog^{(\bm{U})}_n(\hat{\bm{\beta}}^{(\bm{U})}_n)$,
 and 
$\hat{\bm{V}}_2\de n^{-1}\sum_{i=1}^n\nabla\mest_i^{(\bm{U})}(\hat{\bm{\beta}}^{(\bm{U})}_n)\bigl(\nabla\mest_i^{(\bm{U})}(\hat{\bm{\beta}}^{(\bm{U})}_n)\bigr)^\T$ 
 with  $\mest_i^{(\bm{U})}(\bm{\beta}^{(\bm{U})})\allowbreak\de \mest(y_i, ({\bm{x}^{(\bm{U})}_i})^\T\bm{\beta}^{(\bm{U})})$.
 If $W_{n,\id}$ is larger than the upper $\alpha$-quantile of chi-squared distribution with degrees of freedom $\id$, where $\alpha$ is the significance level,  \A concludes that the data from \B is useful. Otherwise, \A declines the assistance from \b.

In practice, if the number of observations is large, \A and \B may construct the test statistic based on a relatively small subset dataset.
In other words, instead of  $\bm{X}^{(\B)}\bm{U}$, \B may send $\widetilde{\bm{X}}^{(\B)}\bm{U}$ to \a, where $\widetilde{\bm{X}}^{(\B)}$ consists of a subset of the rows of $\bm{X}^{(\B)}$.

Next, we present the theoretical properties of $W_{n,\id}$.
Let 
\begin{align*}
\nlog_*^{(\bm{U})}(\bm{\beta}^{(\bm{U})})
\de&\ \E\bigl(\mest\bigl(Y,{X^{(\bm{U})}}^\T \bm{\beta}^{(\bm{U})} \bigr)\mid\bm{U}\bigr),\\
\bm{\beta}^{(\bm{U})}_*\de &\ \argmin_{\bm{\beta}^{(\bm{U})}\in\mathbb R^{p_\A + \id}}  \nlog_*^{(\bm{U})}(\bm{\beta}^{(\bm{U})}),\\
\bm{V}_1\de &\ \E \bigl(\nabla^2_{\bm{\beta}^{(\bm{U})}} \mest(Y,{X^{(\bm{U})}}^\T\bm{\beta}^{(\bm{U})}_*)\mid\bm{U}\bigr),\\
\bm{V}_2\de&\ \E \bigl(\nabla_{\bm{\beta}^{(\bm{U})}} \mest(Y,{X^{(\bm{U})}}^\T\bm{\beta}^{(\bm{U})}_*)\bigl(\nabla_{\bm{\beta}^{(\bm{U})}} \mest(Y,{X^{(\bm{U})}}^\T\bm{\beta}^{(\bm{U})}_*)\bigr)^\T\mid\bm{U}\bigr).
\end{align*}

Let $\|\cdot\|_{\infty}$ denote the sup norm of a vector or matrix, and let $\|\cdot\|_2$ denote the $l_2$ norm of a vector.
\begin{condition}\label{cond_designMatFullrank}
$\P(\text{the design matrix $\bm{X}$ has full column rank})\rightarrow 1\text{ as }n\rightarrow\infty.$
\end{condition}

\begin{condition}\label{A1} 
The function $\bm{\beta}\mapsto M_*(\bm{\beta})$ is differentiable.
The population-based minimizer $\bm{\beta}_*$ exists and is the unique solution to the equation $\nabla\nlog_*(\bm{\beta})=\bm{0}$.
\end{condition}


\begin{condition}\label{A2}
The minimizer
 $\bm{\beta}^{(\bm{U})}_*$ of  $\bm{\beta}^{(\bm{U})}\mapsto\nlog_*^{(\bm{U})}(\bm{\beta}^{(\bm{U})})$ over $\R^{p_\A + \id}$ exits and is the unique solution to the equation $\nabla\nlog_*^{(\bm{U})}(\bm{\beta}^{(\bm{U})}_*) = \bm{0}$ almost surely.  
 Additionally,  $\|\hat{\bm{\beta}}^{(\bm{U})}_n - \bm{\beta}^{(\bm{U})}_*\|_2\limp 0$ as $n\rightarrow\infty$.  
\end{condition}
\begin{condition}\label{A3} 
The matrices $\bm{V}_1$ and $\bm{V}_2$ are positive definite almost surely,  and $\|\hat{\bm{V}}_1 - \bm{V}_1\|_{\infty}\limp 0$ and  $\|\hat{\bm{V}}_2 - \bm{V}_2\|_{\infty}\limp 0$ as $n\rightarrow\infty$. 
Additionally, $\sup_{\bm{\beta}^{(\bm{U})}\in\LSG}\|\nabla^2\nlog^{(\bm{U})}_n(\bm{\beta}^{(\bm{U})}) - \bm{V}_1\|_{\infty}\limp 0$ as $n\rightarrow \infty$ under $\textup{H}_0$, where $\LSG$ is the line segment connecting $\hat{\bm{\beta}}^{(\bm{U})}_n$ and $\bm{\beta}^{(\bm{U})}_*$. 
\end{condition}

\begin{condition}\label{cond_exchInteDiff}
$
\nabla\nlog_*^{(\bm{U})}(\bm{\beta}^{(\bm{U})}_*)=\E\bigl(\nabla_{\bm{\beta}^{(\bm{U})}}  \mest(Y,{X^{(\bm{U})}}^\T\bm{\beta}^{(\bm{U})}_*)\mid \bm{U}\bigr)
$ 
almost surely.
\end{condition}
 \label{page_condset1}
 Condition~\ref{cond_designMatFullrank} is imposed on the pooled dataset with 
 $X = ({X^{(a)}}^\T, {X^{(b)}}^\T, {X^{(c)}}^\T)^\T$ where redundant variables are dropped. We emphasize that to guarantee this condition, it only requires \A and \B to make their own covariates $X^{(\A)} = ({X^{(a)}}^\T, {X^{(c)}}^\T)^\T$ and $X^{(\B)} = ({X^{(b)}}^\T, {X^{(c)}}^\T)^\T$ linearly independent, respectively.
 For instance, this can be achieved by \A and \B removing their linearly dependent columns of their local datasets.
 The other conditions can be satisfied for GLMs and robust regression with some commonly used loss functions. The details can be found in Section~\ref{eq_a_discussion_about_conditions} of the supplement.  
 


\begin{theorem}\label{thm_Wn}
Assume that Conditions~\ref{cond_designMatFullrank}-\ref{cond_exchInteDiff} hold. As
 $n\rightarrow\infty$, we have
\begin{itemize}
\item[(1)] under $\textup{H}_0$, $W_{n,\id}$ converges weakly to a chi-squared distribution with degrees of freedom~$\id$;
\item[(2)]under $\textup{H}_1$, $W_{n,\id}$ goes to infinity in probability.
\end{itemize}
\end{theorem}
When $\id = p_\B$, the test is equivalent to the Wald test \citep{white1982maximum} that compares the nested models with and without $X^{(b)}$. 
\label{page_Wald_alternatives}A detailed discussion about the potential alternatives to the Wald test can be found in Section~\ref{sec_proofOfTh3} of the supplement.
Due to the randomness of $\bm{U}$, the testing result may vary across different realizations of this matrix. To assess the robustness of the testing results with respect to different realizations of $\bm{U}$, we present experimental results in Section~\ref{subsec_robustness_U} of the supplement. They indicate that the value of $\bm{U}$ does not significantly affect the testing result. Nevertheless, it can affect testing results when the improvement of modeling performance from incorporating the data of \B is moderate.

When the covariates are highly correlated and the sample size is relatively small, \A may want to include a penalty term when fitting the model in Step~2, in order to stabilize the results and mitigate numerical issues. Namely, \A obtains $\hat{\bm{\beta}}^{(\bm{U})}_n$ by minimizing 
$\nlog_n^{(\bm{U})}(\bm{\beta}) + \lam\cdot\penalb(\bm{\beta}^{(\bm{U})})$, where $\penalb: \R^{p_\A + \id} \rightarrow \R^+$ and $\lam$ is a positive tuning parameter. 
In Section~\ref{sec_penal_th_sim} of the supplement, 
we provide theoretical analysis and numerical studies to illustrate that the penalty term may have a moderate impact on the testing when $\lam$ is not excessively large.


\subsubsection{Privacy Protection}\label{sec_priv}
Recall that the ``initializing connection'' step requires \B to send sketchy data $\bm{X}^{(\B)}\bm{U}$ to \a. Although $\bm{U}$ is randomly generated, its protection over $\bm{X}^{(\B)}$ can be limited in some scenarios. For example, if a row of $\bm{X}^{(\B)}$ consists entirely of zeros, its linear combination will always be zero,  potentially revealing information about specific individuals. Furthermore, when $\id = p_\B$, and \A possesses knowledge of the covariance matrix $\SigB$ of $X^{(\B)}$, \A may 
 estimate $\bm{X}^{(\B)}$ based on the sample covariance matrix of $\bm{X}^{(\B)}\bm{U}$, which approximately equals $\bm{U}^\T\SigB\bm{U}$. Also, it is hard to quantify and adjust the amount of the privacy protection.  
 
 To alleviate the aforementioned issues where \A or other potential attackers may gain access to the raw data of \b,
we consider local differential privacy \citep{kasiviswanathan2011can, duchi2013local, duchi2018minimax}.
 For the released outputs from each row of a dataset, this approach applies a 
randomizing mechanism to prevent potential attackers from accurately determining the original data values. Details can be found in Section~\ref{sec_localdiff_def} of the supplement. 
The local differential privacy has the following two key properties that differentiate it from  the standard (central) differential privacy \citep{dwork2006calibrating, dwork2016calibrating, shokri2015privacy, dong2022gaussian}, making it suitable for AE-AL:
 \begin{itemize}
 \item \textbf{Decentralized protection mechanism.} The standard differential privacy needs a 
trusted central server to hold the pooled dataset and implement a randomized mechanism to the aggregated outputs. Conversely, local differential privacy provides protection to each individual observation without imposing such a requirement. It aligns well with the AE-AL's decentralized setting.
\item \textbf{Observation value protection.} The standard differential privacy ensures that the presence or absence of a single individual's observation is hard to infer from the generated output. 
 However, the datasets in AE-AL are known to be from the same set of individuals since they are aligned by a nonprivate ID, and our focus is on protecting the individual data values.
Local differential privacy can achieve this goal by preventing accurate recovery of the data values.
 
 \end{itemize}

 For AE-AL, we assume
 that $\|X^{(\B)}\|_2$ is almost surely upper bounded by a positive constant $\cstt$.
 To enhance the privacy protection of $\bm{x}^{(\B)}_i$ from each individual $i=1,\dots,n$, \B transmits  $\bm{X}^{(\B)}\bm{U} + \bm{\noiseMat}$ instead of $\bm{X}^{(\B)}\bm{U}$, where $\bm{\noiseMat}$ is an $n\times\id$ random matrix with its entries i.i.d.\ sampled from the Laplace distribution with mean zero and scale parameter $2\id\cdot\cstt/\epsilon$.  
Let $\bm{\noiseMat}_i$, $i=1,\dots,n$,  be the $i$th row of $\bm{\noiseMat}$. Let $W_{n,\id}'$ be the test statistic calculated from the testing procedure in Section~\ref{subsub_testProc} with $\bm{X}^{(\B)}\bm{U}$ replaced by $\bm{X}^{(\B)}\bm{U} + \bm{\noiseMat}$. 
This modified testing procedure enjoys the following properties:
\begin{proposition} \label{prop_local_diff_privacy}
\textbf{\textup{(1)}}
The procedure preserves the $\epsilon$-local differential privacy of  ${\bm{x}_i^{(\B)}}$ for all $i=1,\dots,n$.
\textbf{\textup{(2)}}
 Given Conditions~\ref{cond_designMatFullrank}-\ref{cond_exchInteDiff} and that the function $\lp \mapsto \mest(y, \lp)$ is strictly convex for each given $y$ in the support of $Y$, the results in Theorem~\ref{thm_Wn} hold for $W_{n,\id}'$ with any $\epsilon>0$.
\end{proposition}
The result \textbf{(1)} follows from the properties of the Laplace mechanism and the sequential composition property of local differential privacy (see, e.g., \citet{dwork2010differential} and \citet{xiong2020comprehensive}).
The result \textbf{(2)} is proved by treating $\bm{\noiseMat}$ as the observations of an additional set of covariates from \b, which do not contribute to the improvement of the model performance. 
The detailed proof is presented in Section~\ref{sec_prop1Proof} of the supplement. If $X^{(\B)}$ contains covariates with unbounded ranges, which leads to an unbounded $l_2$ norm, \B may set a threshold and exclude observations with an $l_2$ norm exceeding the threshold when providing assistance.

\subsection{Assisted Training}\label{almeth}
We first present the AE-AL training algorithm along with its theoretical properties, then highlight its difference compared with other training methods. The simulation studies are presented in Sections~\ref{subsec_assisted_train}. 

\subsubsection{Training Algorithm}\label{subsubsec_trainAlgori}

In the ``assisted training'' stage, both learners iteratively transmit linear combinations of their covariates columns to build a joint model virtually. Each learner locally holds their part of the model parameters. We assume that  \B has acquired $Y$ from \A at the beginning of this stage. When privacy concerns arise,  \A may send the response after applying a random transformation. More details about this approach can be found in Section~\ref{sec_masking the response} of the supplement. \label{page_ref_masking}

Denote the estimated coefficients held by \A and \B in the $k$th assistance round by $ \hat{\bm{\beta}}^{(\A,k)}_n$ and $ \hat{\bm{\beta}}^{(\B,k)}_n$, respectively.
For $k\geq 1$, denote the loss functions for \A and \B locally in the $k$th assistance round by
\begin{align}
	\nlog_n^{(\B,k)}(\bm{\beta}^{(\B)}) &\de \frac{1}{n}\sum_{i=1}^n \mest(y_i, {\bm{x}^{(\B)}_i}^\T\bm{\beta}^{(\B)} + \hat{\bm{\lp}}^{(\A,k-1)}_{n,i}),\label{eq_assistedLossB}\\
		\nlog_n^{(\A,k)}(\bm{\beta}^{(\A)}) &\de \frac{1}{n}\sum_{i=1}^n \mest(y_i, {\bm{x}^{(\A)}_i}^\T\bm{\beta}^{(\A)} + \hat{\bm{\lp}}^{(\B,k)}_{n,i})\label{eq_assistedLossA},
\end{align}
 where $\hat{\bm{\lp}}^{(\A,k)}_{n,i}$ is the $i$th element of $\bm{X}^{(\A)} \hat{\bm{\beta}}^{(\A,k)}_n$ and $\hat{\bm{\lp}}^{(\B,k)}_{n,i}$ is the $i$th element of $\bm{X}^{(\B)} \hat{\bm{\beta}}^{(\B,k)}_n$.
The training procedure is as follows.
\begin{enumerate}
\item \A estimates the initial value $\hat{\bm{\beta}}^{(\A,0)}_n$ by fitting the model based on $(\bm{y}, \bm{X}^{(\A)})$, namely, by minimizing the loss in Equation~\eqref{eq_assistedLossA} with $\hat{\bm{\beta}}^{(\B,0)}_n = 0$, and sends  $\hat{\bm{\lp}}^{(\A,0)}_n$ to \b.
\item 
\B obtains $\hat{\bm{\beta}}^{(\B,1)}_n = \argmin_{\bm{\beta}^{(\B)}\in \R^{p_\B}}\nlog_n^{(\B,1)}(\bm{\beta}^{(\B)})$ and sends $\hat{\bm{\lp}}^{(\B,1)}_n$ to \a.
\item \A obtains $\hat{\bm{\beta}}^{(\A,1)}_n=	\argmin_{\bm{\beta}^{(\A)}\in \R^{p_\A}}\allowbreak \nlog_n^{(\A,1)}(\bm{\beta}^{(\A)})$ and sends $\hat{\bm{\lp}}^{(\A,1)}_n$ to \b.
\item This process continues until a stop criterion is met. \A and \B obtain the fitted model coefficients denoted by $\hat{\bm{\beta}}^{(\A)}_n$ and $\hat{\bm{\beta}}^{(\B)}_n$, respectively.
\end{enumerate}
Possible stop criteria include stopping when $\|\hat{\bm{\lp}}^{(\A,k)}_n + \hat{\bm{\lp}}^{(\B,k)}_n - \hat{\bm{\lp}}^{(\A,k-1)}_n - \hat{\bm{\lp}}^{(\B,k-1)}_n\|_2$ or $\|\hat{\bm{\beta}}^{(\A,k)}_n - \hat{\bm{\beta}}^{(\A,k-1)}_n\|_2$ is less than a tolerance level, or when $k$ is larger than a preset limit that is based on the computation budget.

For future prediction at $\lp$ level,  \A and \B first receive new covariate values $\tilde{\bm{x}}^{(\A)}$ and $\tilde{\bm{x}}^{(\B)}$, respectively.
Then, \B sends $(\tilde{\bm{x}}^{(\B)})^\T\hat{\bm{\beta}}^{(\B)}_n$ to \a, and \A obtains the prediction 
$(\tilde{\bm{x}}^{(\A)})^\T \hat{\bm{\beta}}^{(\A)}_n + (\tilde{\bm{x}}^{(\B)})^\T \hat{\bm{\beta}}^{(\B)}_n$. 
To construct a confidence interval,  \A and \B first obtain the estimated variances $\hat\sigma_\A^2$ and $\hat\sigma_\B^2$, respectively, where $\hat\sigma_\A^2\de(\tilde{\bm{x}}^{(\A)})^\T\big(\hat{\bm{V}}_1^{(\A)}\big)^{-1}\hat{\bm{V}}_2^{(\A)}\big(\hat{\bm{V}}_1^{(\A)}\big)^{-1}\tilde{\bm{x}}^{(\A)}$, $\hat{\bm{V}}_1^{(\A)} \de n^{-1}\sum_{i=1}^n\allowbreak \nabla^2_{\bm{\beta}^{(\A)}}\mest(y_i,  {\bm{x}^{(\A)}_i}^\T\hat{\bm{\beta}}^{(\A)}_n + {\bm{x}^{(\B)}_i}^\T\hat{\bm{\beta}}^{(\B)}_n)$, 
$\hat{\bm{V}}_2^{(\A)} \de n^{-1}\sum_{i=1}^n \nabla_{\bm{\beta}^{(\A)}}\mest(y_i, {\bm{x}^{(\A)}_i}^\T\hat{\bm{\beta}}^{(\A)}_n + {\bm{x}^{(\B)}_i}^\T\hat{\bm{\beta}}^{(\B)}_n)\bigl(\nabla_{\bm{\beta}^{(\A)}}\mest( \allowbreak  y_i,{\bm{x}^{(\A)}_i}^\T\hat{\bm{\beta}}^{(\A)}_n + {\bm{x}^{(\B)}_i}^\T\hat{\bm{\beta}}^{(\B)}_n)\bigr)^\T$, and $\hat\sigma_\B^2$ is defined in a similar way. 
Let $z_{\alpha}$ denote the upper $\alpha$-quantile of the standard normal distribution. Next, \B sends $\hat\sigma_\B^2$ to \a, and  \A obtains a conservative $(1-\alpha)$-confidence interval  by
\begin{equation}\label{eq_conf}
    (\tilde{\bm{x}}^{(\A)}_i)^\T \hat{\bm{\beta}}^{(\A)}_n + (\tilde{\bm{x}}^{(\B)}_i)^\T \hat{\bm{\beta}}^{(\B)}_n \pm z_{1-\alpha/4}(\hat\sigma_\A + \hat\sigma_B),\end{equation}
which is based on the Bonferroni correction and asymptotic normality of the estimated parameter (as will be shown). If the prediction at the original $Y$ level is of interest, the specific relationship between $Y$ and $\lp$ is needed. This will be studied in case of generalized linear models (GLMs) later. 


Next, we present the statistical properties of the AE-AL estimator as $n\rightarrow\infty$.
Let
$$
\TransformMat \de 
\begin{pmatrix}
\bm{I}_{p_a\times p_a} & & & \\
& & \bm{I}_{p_b\times p_b} &\\
&\bm{I}_{p_c\times p_c}& &\bm{I}_{p_c\times p_c}
\end{pmatrix},
$$
where each $\bm{I}$ is the identity matrix with its dimension indicated by the subscript, and  $\hat{\bm{\beta}}^{(k)}_n\de \TransformMat \cdot \bigl((\hat{\bm{\beta}}^{(\a,k)}_n)^\T, (\hat{\bm{\beta}}^{(\b,k)}_n)^\T\bigr)^\T$. 
The matrix $\TransformMat$  is used solely for theoretical analysis and is not required to be known by \A and \B 
 when they conduct local updates in practice. 
Define the oracle estimator $\check{\bm{\beta}}_n  \de \arg\min_{\bm{\beta}\in\mathbb R^p}\nlog_n(\bm{\beta})$, which is based on the practically infeasible combined data from \A and \b. 
\begin{condition}\label{cond_strictConvexloss}
The function $\lp \mapsto \mest(y, \lp)$ is strictly convex and twice differentiable for each  $y$ in the essential support of $Y$.
\end{condition}
\begin{condition}\label{cond_pdHessian}
The matrix $\nabla^2\nlog_*(\bm{\beta})$ is positive definite for each $\bm{\beta}\in\R^p$.
\end{condition}
\begin{condition}\label{cond_oraExist}
The oracle estimator $\check{\bm{\beta}}_n$ uniquely exists with probability going to one as $n\rightarrow\infty$. Also, $\|\check{\bm{\beta}}_n-\bm{\beta}_*\|_2\limp 0$
as $n\rightarrow\infty$.
\end{condition}
\begin{condition}\label{cond_sUconv} 
There exists an open ball $\ThetaSet\subset\R^p$  centered at $\bm{\beta}_*$, such that as $n\rightarrow\infty$, $ \sup_{\bm{\beta}\in \ThetaSet}\left|\nlog_n(\bm{\beta})-\nlog_*(\bm{\beta})\right|  \limp 0$, $ \sup_{\bm{\beta}\in \ThetaSet}\left\|\nabla\nlog_n(\bm{\beta})-\nabla\nlog_*(\bm{\beta})\right\|_{\infty}  \limp 0$, and $ \sup_{\bm{\beta}\in \ThetaSet}\allowbreak  \left\|\nabla^2 \nlog_n(\bm{\beta})-\nabla^2\nlog_*(\bm{\beta})\right\|_{\infty}  \limp 0$.
 \end{condition}
\begin{condition}\label{cond_AsymNorm}
The oracle estimator 
$\sqrt{n}(\check{\bm{\beta}}_n - \bm{\beta}_*)\rightarrow_d\mathcal N(\bm{0},\Sigma)$
 as $n\rightarrow\infty$, where $\Sigma$ is a positive semi-definite matrix.
\end{condition}
\label{page_condset2} 
We will discuss the scenarios where the above conditions are satisfied in Section~\ref{eq_a_discussion_about_conditions} of the supplement.
In Section~\ref{sec_GLM}, we will provide more concrete conditions that fulfill the above requirements for GLMs.

For constants $0 < q_1 < 1$, $0<q_2<\infty$, 
define $A_n(q_1,q_2)$ to be the event that
 $\check{\bm{\beta}}_n$ and $\hat{\bm{\beta}}^{(k)}_n$ uniquely exist, and 
 $\|\hat{\bm{\beta}}^{(k)}_n - \check{\bm{\beta}}_n\|_2\leq
q_1^k\cdot q_2$ for all $k\in\mathbb N$.
\begin{theorem}\label{probcon}
Assume that Conditions~\ref{cond_designMatFullrank}, \ref{A1}, and \ref{cond_strictConvexloss}-\ref{cond_sUconv} hold.

\begin{itemize}
    \item[(1)] There exist a neighborhood $\InitSet$ of $\bm{\beta}_*$ and constants $0 < q_1 < 1$ and $0<q_2<\infty$ that depend on $\InitSet$,  such that when the initial value
	$\hat{\bm{\beta}}^{(0)}_n \in \InitSet$,  $\P(A_n(q_1,q_2))\rightarrow 1$ as $n\rightarrow\infty$. 
\item[(2)] If Condition~\ref{cond_AsymNorm} also holds, there exists a constant $q_3>0$ that depends on $q_1$, such that when  $k \geq q_3\cdot \log(n)$, we have $\sqrt{n}(\hat{\bm{\beta}}^{(k)}_n - \bm{\beta}_*)\rightarrow_d\mathcal N(\bm{0},\Sigma)$ as $n\rightarrow\infty$.
\end{itemize}
 \end{theorem}
The theorem above gives a uniform characterization of the behavior of $\|\hat{\bm{\beta}}^{(k)}_n - \check{\bm{\beta}}_n\|_2$ for all $k$. 
\label{page_challenges_algorithm}Next, we discuss some unique challenges in deriving the above results.
First, 
although the convergence of $\hat{\bm{\beta}}^{(\A,k)}_n$ and $\hat{\bm{\beta}}^{(\B,k)}_n$ on a fixed dataset as $k\rightarrow\infty$ has been well studied from an optimization point of view \citep[see, e.g.,][]{luo1992convergence,luo1993error, tseng1993dual,tseng2001convergence, tseng2009coordinate,
saha2013nonasymptotic, 
zhao2018pathwise,
zhao2014accelerated}, their convergence results typically involve some unknown data-dependent constants.
Consequently, the conclusions and proofs from these previous studies do not apply to our context where we need to study the required number of iterations $k$ as the sample size $n$ increases.
Second, standard techniques regarding M-estimators (Theorems 5.7 and 5.23 in \citet{van2000asymptotic}) focus on the estimators that minimize the loss function. 
 In contrast, our setting requires the investigation of modeling performance with a finite
 number of iterations. To address these challenges, 
 we will derive the convergence result based on population-level constants and quantify the relationship between the iteration number $k$ and the sample size $n$.

Furthermore, one can quantify the relationship between $\P(A_n(q_1,q_2))$ and $n$ by assuming the following condition.
\begin{condition}\label{cond_finiteSample}
There exists a constant $\Cst >0$, a function $\bd: \R^{+}\rightarrow\R^{+}$,  and a positive sequence $\cRate\rightarrow 0$ as $n\rightarrow\infty$,
such that 
for each $\Mb>0$, $\P\bigl(\sup_{\bm{\beta}\in \ThetaSet}\left|\nlog_n(\bm{\beta})-\nlog_*(\bm{\beta})\right|\geq\Mb\bigr) \leq\bd(\Mb)\cdot \cRate$, $\P\bigl(\sup_{\bm{\beta}\in \ThetaSet}\left\|\nabla^2 \nlog_n(\bm{\beta})-\nabla^2\nlog_*(\bm{\beta})\right\|_{\infty}  \geq\Mb\bigr) \leq\bd(\Mb)\cdot \cRate$, and $\P(\bm{X}$ has full column rank $)\geq 1 -\cRate\cdot\Cst$.

\end{condition}

Let $\BDset$ denote the boundary set of the ball $\ThetaSet$. We define $\lambda^{(\textrm{min})}_* \de \inf_{\bm{\beta}\in \ThetaSet}\lambda^{(\textrm{min})}\bigl(\nabla^2\nlog_*(\allowbreak\bm{\beta}) \bigr)$ where  $\lambda^{(\textrm{min})}(\cdot)$ returns the smallest eigenvalue of a matrix.
\begin{corollary}\label{coro_finite_sample_convergence}
 Assume that Conditions~\ref{A1},  \ref{cond_strictConvexloss}-\ref{cond_sUconv}, and \ref{cond_finiteSample} hold.
We have
 $
 \P\bigl(A_n(q_1,q_2)\bigr)\geq 1-\bigl(2\bd(\Mb) + \Cst\bigr)\cRate\text{ as $n\rightarrow\infty$},$
 where $\delta$ is the minimum of the following three terms: $\bigl( \min_{\bm{\beta}\in\BDset}\nlog_*(\bm{\beta}) - \nlog_*(\bm{\beta}_*)\bigr)/4$, $\lambda^{(\textrm{min})}_*/8p$,  and 
 $$\sup_{\bm{\beta}^{(1)}, \bm{\beta}^{(2)} \in\ThetaSet}|\nlog_*(\allowbreak\bm{\beta}^{(1)})- \nlog_*(\bm{\beta}^{(2)})|/4.$$
\end{corollary}

 In the following Section~\ref{sec_GLM} for GLMs, we will provide more concrete forms of $\cRate$, $\bd(\cdot)$, and $\Cst$.
 This result can be extended to the case with a changing $p$ under stronger assumptions. A discussion about this extension is provided in Section~\ref{subsec_changingp} of the supplement. 
 In the scenario where $p$ is relatively large compared with $n$, it may be more preferable to employ penalized loss functions for the model training. Along this line, we provide both theoretical and numerical analyses for AE-AL that incorporate penalty terms in Section~\ref{sec_penal_th_sim} of the supplementary document. 
They indicate that the AE-AL estimator may achieve a performance comparable to that of the oracle estimator when smooth penalty functions, such as the $l_2$ penalty, are employed.
The numerical studies in the supplement also include AE-AL for LASSO regression and the elastic net in scenarios where the number of covariates $p$ exceeds the sample size $n$. The results demonstrate that AE-AL can still yield improvements in model performance in these cases.

\subsubsection{Comparison with Other Training Methods}\label{subsub_comparison_learning}
The AE-AL training algorithm addresses the learning scenario where there is no third-party coordinator and the learners possess possibly different sets of covariates. Therefore, its application scope differs from methods that assume learners have different observations of the same covariates \citep[e.g.,][]{gong2015privacy, shokri2015privacy, jordan2019communication, dobriban2021distributed}) 
and methods that require a trusted central server for encryption \citep[e.g.,][]{yang2019quasi,  xu2021fedv, wenjie2021vertical}). 

In vertical federated learning,  the decentralized learning problem with 
the loss function in \eqref{eq_mestimator} can be addressed by gradient descent-based methods \citep[e.g.,][]{yang2021model, das2021multi}).
Specifically, \A and \B may simultaneously update their coefficients by stochastic gradient descent as follows.
For each iteration, let $\minibatch_k \subset \{1,\cdots,n\}$ denote the index set of the mini-batch used in stochastic gradient descent, with its corresponding observations denoted by $\bigl(\bm{Y}^{(k)}, \bm{X}^{(\A,k)}, \bm{X}^{(\B,k)}\bigr)$. 
Let $\lossW^{(k)}(\bm{\lp})\de n^{-1}\sum_{i\in\minibatch_k} \mest\bigl(y_i,\lp_i\bigr)$, where $\lp_i$ $(i\in\minibatch_k)$ are the elements of the vector $\bm{\lp}$.
First, \B sends $\bm{X}^{(\B,k)}\hat{\bm{\beta}}^{(\B, k-1)}_n$ to \a. Based on this vector, \A computes the gradient 
$\gradient_\A^{(k)} \de \nabla_{\bm{\beta}^{(\A)}} n^{-1}\sum_{i\in\minibatch_k} \mest\bigl(y_i, {\bm{x}^{(\A,k-1)}_i}^\T\allowbreak\bm{\beta}^{(\A)} + {\bm{x}^{(\B,k-1)}_i}^\T\hat{\bm{\beta}}^{(\B, k-1)}_n\bigr)$. 
Second, \A sends $\nabla \lossW^{(k)}(\lp)$ evaluated at $\bm{\lp}^{(k-1)} = {\bm{x}^{(\A,k-1)}_i}^\T\hat{\bm{\beta}}^{(\A, k-1)}_n + {\bm{x}^{(\B,k-1)}_i}^\T\hat{\bm{\beta}}^{(\B, k-1)}_n$ to \b.
Third, \B calculates the gradient $\gradient_\B^{(k)} \de \bigl(\bm{X}^{(\B,k)}\bigr)^\T\nabla \lossW^{(k)}(\bm{\lp}^{(k-1)})$. Finally, \A and \B update $\hat{\bm{\beta}}^{(\A, k-1)}_n$ and $\hat{\bm{\beta}}^{(\B, k-1)}_n$ by $\gradient_\A$ and $\gradient_\B$, respectively.
The above training procedure is commonly known as FedSGD. 
Several modified FedSGD methods have been developed to address data encryption, communication efficiency, and other requirements. One state-of-the-art variant is FedBCD \citep{liu2022fedbcd} which expedites convergence and reduces the number of communication rounds by
allowing \A and \B to perform $Q>1$ local updates before synchronization.  Additionally, FedBCD has a hyperparameter that controls an additional proximal term in the local losses of \A and \b. This term helps prevent potential divergence when the number of local updates $Q$ is large. 
FedSGD and FedBCD can only guarantee convergence rates of $1/\sqrt{k}$ and $k^{-2/3}$ \citep{liu2022fedbcd}, respectively, which are slower than the exponential rate of our AE-AL training algorithm. Moreover, they require parameter tuning that entails communication between the learners, e.g., the selection of the step size.  This results in an increase in the total number of transmission rounds by a factor that is approximately equal to the number of candidate values for the tuning parameter.
Other methods that consider a similar setting include 
\citet{zhang2021asysqn, gu2021privacy, zhang2021secure}. Like AE-AL,  their learners perform local updates and transmit linear combinations of their data. 
However, these methods require the loss function to be strongly convex with respect to the coefficient of $X^{(\A)}$ and $X^{(\B)}$ for 
their theoretical analysis, which is 
violated in our setting when there are overlapping covariates. We will provide more comparisons between AE-AL and related methods in numerical studies.


\section{Generalized Linear Models}\label{sec_GLM}

In this section, we consider 
GLMs with the conditional probability density or mass function of the form
\begin{equation}\label{eq_expFamilyDf}
f(y \mid \bm x^\T \bm{\beta}) = c(y)\exp\left(y\cdot u( \bm{x}^\T\bm{\beta}) - b(u(\bm{x}^\T\bm{\beta}))  \right),
\end{equation}
 where the unknown parameter $\bm{\beta} = ({\bm{\beta}^{(a)}}^\T, {\bm{\beta}^{(b)}}^\T, {\bm{\beta}^{(c)}}^\T)^\T$ with $\bm{\beta}^{(a)}\in \mathbb R^{p_a}$, $\bm{\beta}^{(b)}\in \mathbb R^{p_b}$, and $\bm{\beta}^{(c)}\in\mathbb R^{p_c}$. The intercept can be included in $\bm{\beta}^{(c)}$ with the corresponding covariate in $\bm{x}^{(c)}$ being constant one. The overdispersion parameter, which is omitted in the above expression, is assumed to be known. For instance, for logistic regression, $c(y) = 1$, $u(\bm{x}^\T\bm{\beta}) = \bm{x}^\T\bm{\beta}$, and $b(u(\bm{x}^\T\bm{\beta}))=\log(1 + \exp{(\bm{x}^\T\bm{\beta}}))$. When the model is correctly specified, we have $b'(u(\bm{x}^\T\bm{\beta})) = \E(Y\mid\bm{x}^\T\bm{\beta})$. The link function $g\de (b'\circ u)^{-1}$ is assumed to be monotone and differentiable.
To fit the GLM, we obtain the maximum likelihood estimator by taking $\mest(y,\bm{x}^\T \bm{\beta} ) = -\log\ f(y \mid \bm x^\T \bm{\beta})$. 
Next, we present the theoretical properties of the test statistic in \textbf{S1} and AE-AL estimator in \textbf{S2}. 

\begin{condition}\label{cond_diffu}
The function $\lp \mapsto u(\lp)$ from \eqref{eq_expFamilyDf} is three times continuously differentiable.
\end{condition}

\begin{condition}\label{dc} There exists a function $(y,\bm{x})\mapsto\bdd(y,\bm{x})$ with $\E(\bdd(Y, X)) < \infty$ and an open ball $\ThetaSett\subset\R^p$  centered at $\bm{\beta}_*$, such that when $\bm{\beta}\in\ThetaSett$, $\bigl| \frac{\partial^3}{\partial \bm{\beta}_j \partial \bm{\beta}_k \partial \bm{\beta}_l}  \log f(y\mid \bm{x}^\T\bm{\beta}) \bigr|\leq \bdd(y,\bm{x})$ holds almost surely for each $j,k,l\in\{1,\dots,p\}$, where $\bm{\beta}_j$ with $j\in\{1,\dots,p\}$ denote the $j$th element of $\bm{\beta}$. Also,  
$
\frac{\partial^3}{\partial \bm{\beta}_j \partial \bm{\beta}_k \partial \bm{\beta}_l} \E\bigl( \log f(Y\mid X^\T\bm{\beta})\bigr) = \E\biggl(\frac{\partial^3}{\partial \bm{\beta}_j \partial \bm{\beta}_k \partial \bm{\beta}_l}  \log f(Y\mid X^\T\bm{\beta})\biggr), 
$
for each $\bm{\beta}\in\ThetaSett$.

\end{condition}

\begin{condition}\label{cond_pdI}
The Fisher information matrix 
$\E \bigl(\nabla_{\bm{\beta}} \log f(Y\mid X^\T\bm{\beta})(\nabla_{\bm{\beta}} \log f(Y\mid X^\T\bm{\beta}))^\T\bigr)$ is positive definite for each $\bm{\beta}\in\R^p$.
\end{condition}

\begin{condition}\label{cond_Fbounded}
$\E\bigl\|\nabla_{\bm{\beta}}^2  \log f(Y\mid X^\T\bm{\beta}_*)\bigr\|_{\infty}<\infty,\quad\E\bigl\|\log f(Y\mid X^\T\bm{\beta}_*)\bigr\|_{\infty}<\infty,$
\begin{align*} 
\E \biggl(\sup_{\bm{\beta}\in \ThetaSett}\biggl\|\nabla_{\bm{\beta}} \log f(Y\mid X^\T\bm{\beta})(\nabla_{\bm{\beta}} \log f(Y\mid X^\T\bm{\beta}))^\T\biggr\|_{\infty}\biggr)<\infty.
\end{align*}
\end{condition}


\begin{condition}\label{lem_sbb_inm}
There exists a positive constant $\cst>0$ such that  $\{\bm{\beta}\in\R^p : \|\bm{\beta} - \MU\bm{\beta}^{(\bm{U})}_*\|_2\leq \cst\}\subseteq \ThetaSett$ holds almost surely, where $$
\MU \de 
\begin{pmatrix}
\bm{I}_{p_a\times p_a} & \bm{0}& \bm{0}\\
\bm{0}& \bm{0} &\bm{U}\\
\bm{0}&\bm{I}_{p_c\times p_c}&\bm{0}
\end{pmatrix}.
$$
(Similar to  $\TransformMat$,  $\MU$ is  not required to be known by \A and \b.)
\end{condition}

Conditions~\ref{cond_diffu}, \ref{dc} and \ref{cond_pdI} are required in the classical theory for the maximum likelihood estimators (e.g., in Theorem~7.5.2 of \cite{lehmann2004elements}). 
Condition~\ref{cond_Fbounded} is used to guarantee the convergence of the sandwich covariance estimator 
$\hat{\bm{V}}_1^{-1}\hat{\bm{V}}_2\hat{\bm{V}}_1^{-1}$ \citep{white1982maximum}.
For Condition~\ref{lem_sbb_inm}, we present an illustrative example.
Suppose that $Y$ follows $\mathcal N(X^\T\bm{\beta}_*, \sigma^2)$, where  $\sigma$ is a positive constant and $\E(X) = \bm{0}$. 
Let $\bm{\Sigma}\de \E(XX^\T)$.
Then,  it can be shown that $\MU\bm{\beta}^{(\bm{U})}_* = \MU({\MU}^\T\bm{\Sigma}\MU)^{-1}{\MU}^\T\bm{\Sigma}\bm{\beta}_*$.  Therefore, $\bm{\Sigma}^{1/2}\MU\bm{\beta}^{(\bm{U})}_*$ is the projection of $\bm{\Sigma}^{1/2}\bm{\beta_*}$, where $\bm{\Sigma}^{1/2}\cdot\bm{\Sigma}^{1/2} = \bm{\Sigma}$. Together with the fact that $\bm{\Sigma}$ is a non-degenerate deterministic matrix, $\|\MU\bm{\beta}^{(\bm{U})}_*\|_2$ is bounded almost surely, and Condition~\ref{lem_sbb_inm} is satisfied when the set $\ThetaSett$ is large enough. 


\begin{corollary}\label{prop_suffGLM}
For GLMs, Conditions~\ref{A2}-\ref{cond_exchInteDiff}, and \ref{cond_oraExist}-\ref{cond_AsymNorm} hold given Conditions~\ref{cond_designMatFullrank}, \ref{A1},  \ref{cond_strictConvexloss}, \ref{cond_pdHessian}, and \ref{cond_diffu}-\ref{lem_sbb_inm}, where $\ThetaSet=\ThetaSett$. 
Consequently, the results in Theorems~\ref{thm_Wn} and \ref{probcon} hold given this set of conditions. 
Also, when $\E\bigl(\bigl(\sup_{\bm{\beta}\in \ThetaSet}|\mest(Y,X^\T \bm{\beta} )|\bigr)^2\bigr)$, $\E\bigl(\bigl(\sup_{\bm{\beta}\in \ThetaSet}\|\nabla^2\mest(Y,X^\T \bm{\beta} )\|_{\infty}\bigr)^2\bigr)$, and $\E\bigl(\|XX^\T\|_{\infty}^2\bigr)$ exist, Condition~\ref{cond_finiteSample}, which is required for Corollary~\ref{coro_finite_sample_convergence},  holds with: \\$\bd(\Mb) =  4\Mb^{-2}\cdot \max\biggl(\E\biggl(\bigl(\sup_{\bm{\beta}\in \ThetaSet}|\mest(Y,X^\T \bm{\beta} )|\bigr)^2\biggr)$, \\ $\E\biggl(\bigl(\sup_{\bm{\beta}\in \ThetaSet}\|\nabla^2\mest(Y,X^\T \bm{\beta} )\|_{\infty}\bigr)^2\biggr)\biggr)$, 
$\Cst = \frac{16p^2\cdot \E\bigl(\|XX^\T\|_{\infty}^2\bigr)}{\bigl(\lambda^{(\textrm{min})}(\E(XX^\T))\bigr)^2}$, and $\cRate = 1/n$.
\end{corollary}
Suppose for a new subject, \A and \B observe $\tilde{\bm{x}}^{(\A)}$ and $\tilde{\bm{x}}^{(\B)}$, respectively. With $(\tilde{\bm{x}}^{(\B)})^\T \hat{\bm{\beta}}^{(\B)}_n$ transmitted from \b, \A obtains  $\hat{y} \de {g}^{-1}((\tilde{\bm{x}}^{(\A)})^\T \hat{\bm{\beta}}^{(\A)}_n + (\tilde{\bm{x}}^{(\B)})^\T \hat{\bm{\beta}}^{(\B)}_n)$ (as a point prediction) and a confidence interval of $\E(y\mid \tilde{\bm{x}}^{(\A)}, \tilde{\bm{x}}^{(\B)})$ based on the formula in \eqref{eq_conf} and the monotonicity of the link function $g$.
\clearpage
\section{Simulation Results}\label{sec_exp}
In this section, we evaluate the finite sample performance of the AE-AL procedure through simulation studies, specifically focusing on the logistic regression model. Additional simulation studies are provided in Section~\ref{sec_addexp} of the supplement, which include other GLM models and robust regression.




\subsection{Stage of Initializing Connection: Null distribution} \label{subsec_type1}
To investigate the convergence in distribution of $W_{n,\id}$ under $\textrm{H}_0$, we generate  simulated data from logistic regression models in the following settings.
\begin{itemize}
\item \textbf{Setting 1: }    $\lp =  0.5x_1 +\cdots+  0.5x_6$, where $\A$ holds $x_1,\dots,x_6$ and $\B$ holds $x_7,\dots, x_{12}$,
\item \textbf{Setting 2: }    $\lp =  0.5x_1 +\cdots+  0.5x_8$, where $\A$ holds $x_1,\dots,x_8$ and $\B$ holds $x_5,\dots, x_{12}$,
\item \textbf{Setting 3: }    $\lp =  0.5x_1 + \cdots+  0.5x_{10}$, where $\A$ holds $x_1,\dots,x_{10}$ and $\B$ holds $x_3,\dots, x_{12}$.
\end{itemize}
Consequently, the number of shared covariates $p_c=0$, $4$, and $8$ for the above three settings, respectively. 
Let $V$ be a $p\times p$ matrix with $V_{i,j}=\rho^{|i-j|}$, namely, it represents the AR(1) covariance matrix.
The covariate observations 
$\bm{X} = \tilde{\bm{X}}\sqrt{V}$, where $\tilde{\bm{X}}$ is a $n\times p$ matrix with elements
 i.i.d.\ sampled from $\textrm{Uniform}(0, 1)$.
We consider $\rho = 0$ (corresponds to independent covariates), $0.1$ or $0.25$ and the sample sizes $n = 300$ or $2000$.  Additionally, the sketchy data $\bm{X}^{(\text{B})}\bm{U}$ can either be noise-free or have Laplace noise added with a scale of 0.1 or 0.5.
The simulation is independently replicated 100 times.
Figure~\ref{logisticQQ} presents the Q-Q plot of the p-values of $W_{n,\id}$ under Setting~2 with AR(1)-correlated covariates, $\rho = 0.1$,  Laplace noise with a scale of 0.5, and $\id = 1,\dots,5$. It can be seen that in general, $W_{n,\id}$ approximates its asymptotic distribution well, except that it tends to be conservative when $n=300$ (the sample quantiles of the p-values tend to be slightly larger than the theoretical quantiles).
The results from the other settings are similar and are provided in Section~\ref{subsec_QQlogistic} of the supplement. 

\begin{figure}[tb]
   \centering
\includegraphics[width=1\linewidth]{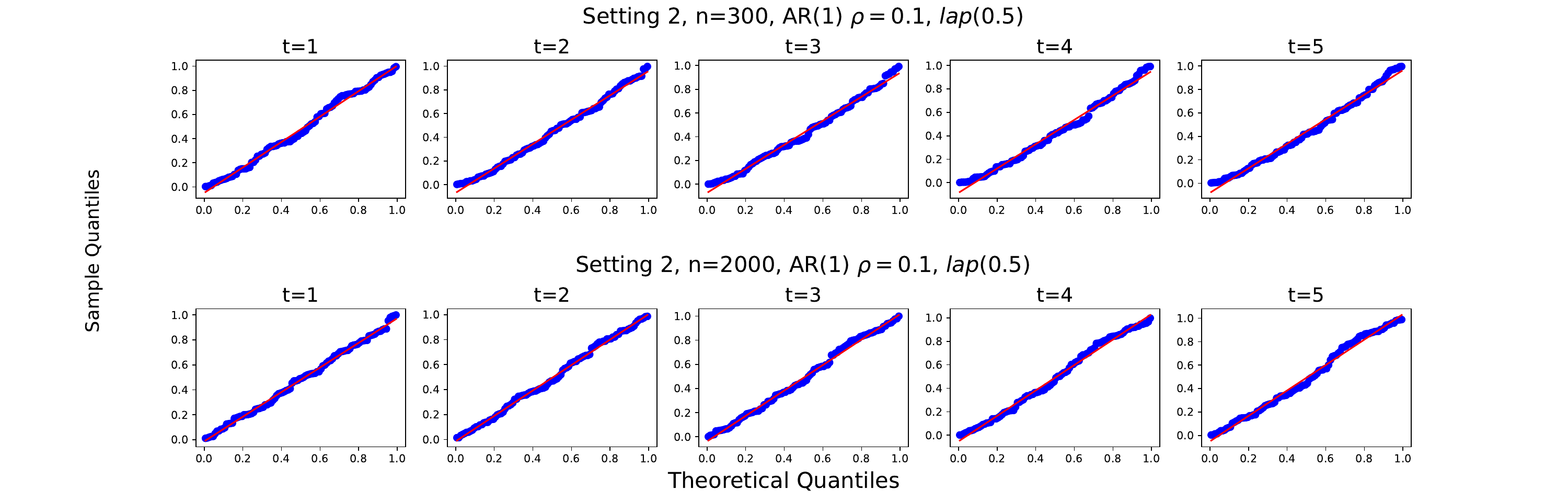}
	\caption{The Q-Q plot of $W_{n,\id}$ from Setting 2 with AR(1) covariates, $\rho = 0.1$,  and the scale of the Laplace noise is 0.5. The red straight line stands for the perfect match between the sample quantiles of the p-values and the theoretical quantiles of $\textrm{Uniform}(0,1)$.
	}
	\label{logisticQQ}
\end{figure}

\subsection{Stage of Initializing Connection: Power}\label{subsubsec_power}
Now, we explore the statistical power of the test under $\textrm{H}_1$. We adopt the settings from Subsection~\ref{subsec_type1} except that the data are simulated from a logistic regression model with  $\lp =\beta_1 x_1 +\cdots+\beta_{12} x_{12}$ and $\beta_1$ to $\beta_{12}$ are i.i.d.\ generated from $\mathcal N(0,0.25)$ in each replication to avoid cherry-picking.

The rejection rates of $W_{n,\id}$  versus $\id$ (the number of columns in the sketchy data) from 100 replications are presented in Figure~\ref{logisticPower}. 
It shows that tests with larger sample sizes ($n$), the number of columns in the sketchy data ($\id$), and a smaller scale of the Laplace noise tend to have a higher power. Moreover, when the total number of covariates is held constant, the power decreases with more shared covariates, which is expected since \B has fewer uniquely useful covariates to assist \a.
\begin{figure}
\centering
   \includegraphics[width=0.85\linewidth]{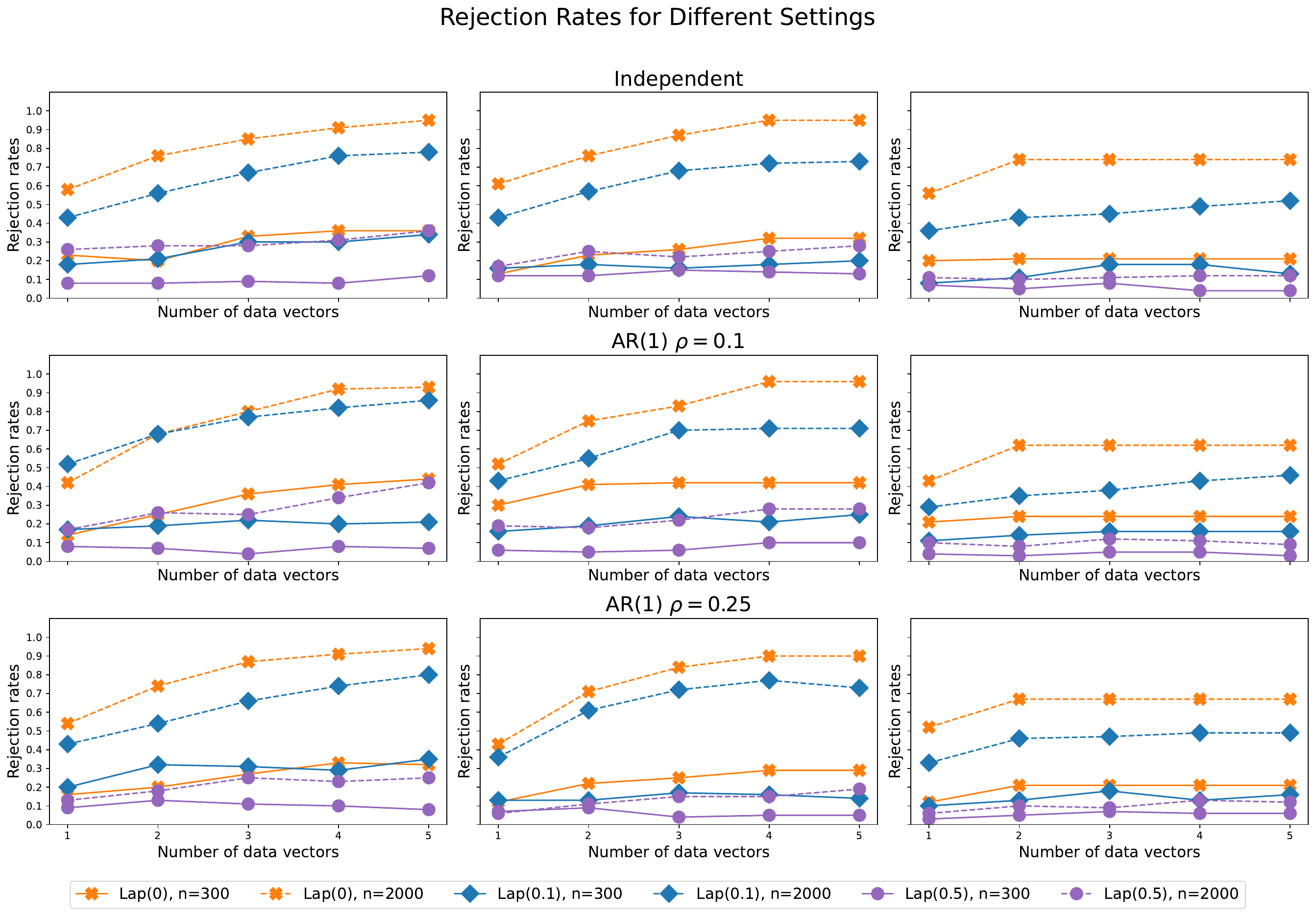}
\caption{The rejection rates of $W_{n,\id}$ under $\textrm{H}_1$ versus $\id$ (the number of columns in the sketchy data) from 100 replications, where $\textrm{Lap}(0)$, $\textrm{Lap}(0.1)$, and $\textrm{Lap}(0.5)$ stand for the sketchy data without noise, with Laplace noise of scale 0.1, and with Laplace noise of scale 0.5, respectively.}\label{logisticPower}
\end{figure}


\subsection{Assisted Training}\label{subsec_assisted_train}
This experiment evaluates the learning performance of the AE-AL method compared to FedSGD and FedBCD.  
We simulate data based on the settings in Subsection~\ref{subsubsec_power} with $\rho = 0.1$ (AR(1) covariates), $n=2000$, and 100 replications. In each replication, we also generate an evaluation set with size $10^6$ from the same distribution. The performance of AE-AL, FedSGD, and FedBCD estimators together with the oracle estimator are evaluated by the averaged area under the ROC curve (AUC) on the evaluation set. For training FedSGD and FedBCD, we employ the decay learning rate strategy in \citet{liu2019communication}. The initial learning rate is determined from a set of 20 candidate values between 0.01 and 5, equally spaced on a logarithmic scale. 
The number of local updates $Q$ for FedBCD is set to 5, 10, or 25. We also apply FedBCD with a proximal term with the tuning parameter setting in \citet{liu2019communication}. The mini-batch size is 32. We present the performance of the FedSGD and FedBCD configurations that achieve the highest AUC after 50 transmission~rounds. 
The AUCs are presented in Figure~\ref{fig_logistic_loglikeli}. It can be seen that the AUC from the AE-AL converges to that of the oracle estimator at a significantly faster rate compared with FedSGD and FedBCD. Notably, it also illustrates the potential of AE-AL to achieve a reasonably good performance without requiring extensive hyperparameter tuning that involves communication between the learners.

\begin{figure}
\centering
   \includegraphics[width=0.9\linewidth]{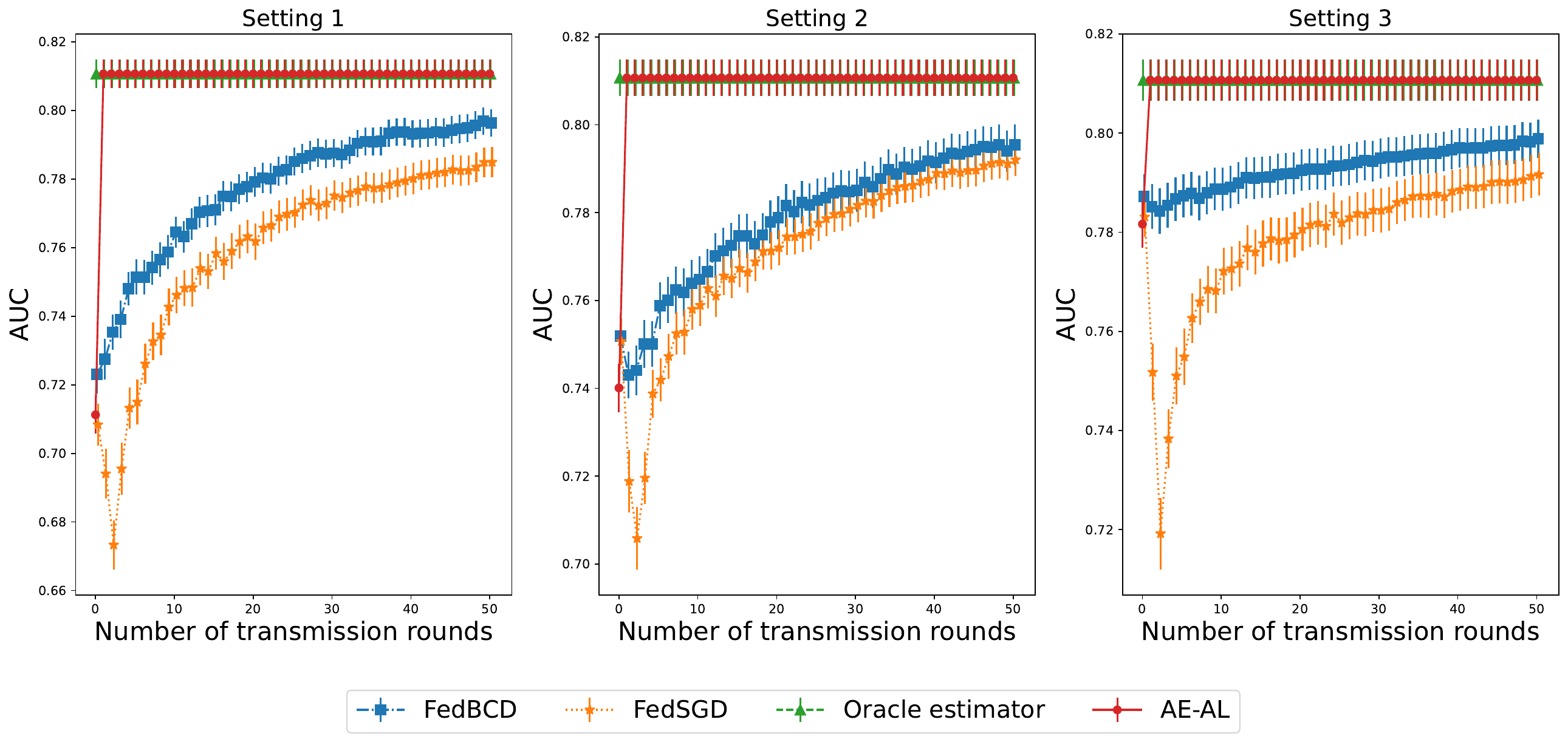}
\caption{
Classification performance of AE-AL, FedBCD, FedSGD, and the oracle estimator. The vertical bars represent the standard errors based on 100 replications.}\label{fig_logistic_loglikeli}
\end{figure}



\section{Data Examples}\label{Sec_real_data}
In the following subsections, we demonstrate the potential applications of AE-AL based on three data examples for classification via logistic regression. We also assess the performance of FedSGD and FedBCD for comparison.
\subsection{Hospital Length of Stay}\label{subsec_hospital_length_of_stay}
Predicting the length of stay of patients may greatly facilitate the resource management and scheduling of the hospital. 
We consider the benchmark dataset MIMIC3 \citep{johnson2016mimic} that consists of patients data collected from different divisions of a hospital. Similar modeling problems to ours have been explored in prior works such as \citet{purushotham2018benchmarking}, 
\citet{harutyunyan2019multitask}, 
and \citet{wang2022parallel}. We follow the data pre-processing procedure of \citet{wang2022parallel}. The details about the obtained dataset can be found in Section~\ref{sec_hosp_data_detail} of the supplement.

Suppose researchers aim to predict whether a patient will have a length of stay exceeding three days based on 10 covariates from the ICU and the remaining covariates from the laboratory. Employing a decentralized model learning approach eliminates the need to merge the two datasets, thereby enhancing data confidentiality protection against unauthorized leakage
\citep[see, e.g.,][]{lohr1994health, edemekong2018health}.

Let \A and \B be the ICU and laboratory, respectively. For \textbf{S1} (initializing a connection), we consider $\id=1$, $2$, or $3$. \B  either sends the original sketchy data or the ones with Laplace noise of scales 0.1 or 0.5. 
To ensure comparable privacy protection with the simulation settings in Sections~\ref{subsec_type1} and \ref{subsubsec_power} of the supplement,
the covariate observations of $X^{(\B)}$ are normalized to have values between zero and one. With the exception of the case where $\id=1$ and the noise has a scale of 0.5, the testing results from all other cases reject the null hypothesis ($\textrm{H}_0$) at the significance level of 0.05. These findings in general support the conclusion that the data from source \B carries significant value and contributes meaningfully to the analysis. In \textbf{S2} (assisted training), to evaluate the performance of the fitted model, we randomly sample half of the observations as the training set and the remaining part as the evaluation set. 
The training of FedSGD and FedBCD follows the parameter settings in Section~\ref{subsec_assisted_train} of the supplement where the range of the initial step sizes have been adjusted to be in the range of $10^{-5}$ to better suit this dataset. We choose the configurations with the best prediction performance on the evaluation set. We also include an additional setting for FedSGD and FedBCD, where instead of sampling a mini-batch for updates, they utilize the entire datasets to calculate the gradients. 
The AUCs of different estimators evaluated on the evaluation set are presented in 
Figure~\ref{fig_Real_data_results}.
It can be seen that the AE-AL estimator has a significantly faster convergence speed than its competitors.


\begin{figure}
    \centering\includegraphics[width=0.98\linewidth]{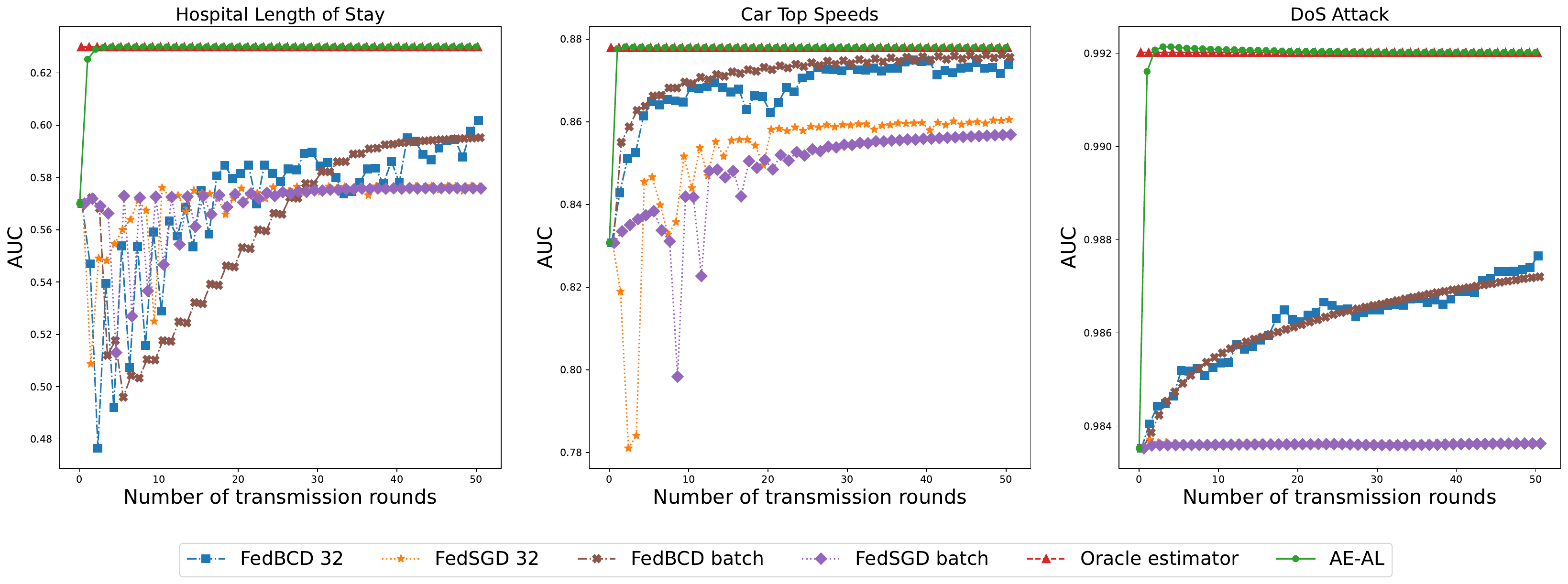}
\caption{
Classification performance of AE-AL, FedBCD, and FedSGD. FedBCD and FedSGD with batch size 32 and their counter parts that utilize the entire dataset for gradient calculation are denoted by ``FedBCD 32'', ``FedSGD 32'', ``FedBCD batch'', and ``FedSGD batch'', respectively.}\label{fig_Real_data_results}
\end{figure}

\subsection{Car Top Speeds}\label{subsec_image}


We consider the data from \cite{yang2015large}, which contains 2965 paired car images taken from side and front views together with the top speed of the cars.  A snapshot of the images is in Figure~\ref{carImage_samplePlots}. 
Suppose a sensor \A in a smart transportation system (e.g., as described in \citet{villanueva2013architecture} and \citet{liu2021new}) wants to pay extra attention to cars with a top speed exceeding 200km/h to prevent accidents. 
It captures the side image of cars and can connect remotely to 
another monitoring camera \B that captures front images of the cars. Due to constraints in data transmission, \A implements the AE-AL to train the classifier without the need for pooling the image data together.

To handle the image data under the framework of AE-AL, we extract covariates by processing the data through neural network models.  This is motivated by the recent advancements in deep transfer learning \citep{weiss2016survey,zhuang2020comprehensive}, e.g., quickly adapting a pre-trained neural network model trained on massive images from a domain to another one.
Specifically, \A and \B pre-process their image data through the ResNet18 neural network \citep{he2016deep} from the ImageNet benchmark dataset~\citep{imagenet_cvpr09} with its last layer replaced by two fully connected layers with dimensions $512 \times 16$ and $16 \times 2$, respectively. 
Then, \A and \B independently retrain this neural network on their respective data for one epoch and extract the outputs of 16 neurons from the second-last layer as their covariates.

We follow the same testing and training procedure in Subsection~\ref{subsec_hospital_length_of_stay}. 
When \B sends the sketchy data without noise, all three tests with $\id=1$, $2$, and $3$ reject $\textrm{H}_0$. Nevertheless, the test with Laplace noise of scale $0.5$ and the ones with noise scale $0.1$ and $\id=1$ fail to reject. This demonstrates a trade-off between the power to detect a helpful assistance provider and privacy protection. 
Figure~\ref{fig_Real_data_results}
illustrates the performances of the trained classifiers, where the initial step sizes of FedBCD and FedSGE are set within the ranges $[10^{-5}, 0.1]$ and $[0.01,0.2]$, respectively. It can be seen that the estimator from AE-AL converges faster than the other methods.

\begin{figure}[!ht]
	\centering
	\vspace{-0.1cm}
\includegraphics[width=0.5\linewidth]{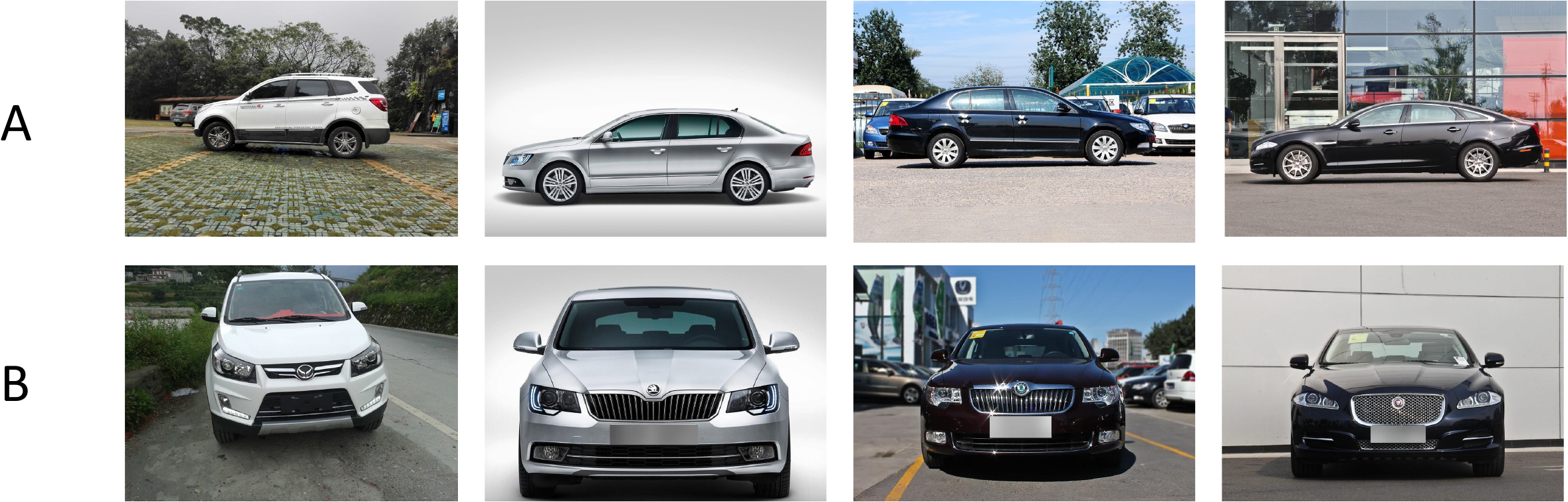}
	\vspace{-0.0cm}
	\caption{An example of car images held by \A and \b.}
	\label{carImage_samplePlots}
\end{figure}

\subsection{DoS Attack}\label{subsec_DOSattack}
The denial-of-service (DoS) attack aims to render a server or its resources inaccessible to legitimate requests by overwhelming it with a large volume of traffic. Suppose an internet service provider \A seeks to predict such attacks based on some traffic information from its user \B to enhance protection. 
 However, \A is unable to measure or access this information since they are protected by \B under e.g., encryption from browsers, HTTP proxies, or intermediary servers. Moreover, \A cannot directly access these data due to privacy policies.


We consider the dataset from \citet{tavallaee2009detailed}, which consists of nine traffic covariates and ten destination IP address-related covariates. The dataset comprises 45,927 observations with DoS attacks and 67,343 observations without attacks.  We randomly sample a subset of 10,000 observations where half of them are from the first group and the remaining ones from the other group. We let \A and \B hold the destination IP address-related covariates and the traffic covariates, respectively. 
The testing and training procedure are the same as the one in Subsection~\ref{subsec_hospital_length_of_stay}, except for FedBCD and FedSGD where their initial step sizes are set within the ranges $[10^{-3}, 1]$ and $[0.01,2]$. The testing results from all settings reject the $\textrm{H}_0$ at a significance level of 0.05, indicating that the additional covariates from \B may improve the modeling performance of \a. The model training performance is presented in 
Figure~\ref{fig_Real_data_results},
where AE-AL achieves the fastest convergence speed in terms of the number of transmission rounds. Among the other methods, FedBCD which utilizes the entire datasets to calculate the gradients has the best performance.

\section{Concluding Remarks}\label{sec_con}
We have proposed a new methodology for assisting a learner 
(e.g., a device, a research lab, or a company) to significantly improve its local learning in the AE-AL framework. 
The developed solution consists of two stages, namely, the stage for \A to test \b's usefulness, and a stage of assisted training by iterative transmissions of summary statistics between \A and \b. After that, \A receives \b's assistance for predictions for future subjects. 

There are some open problems left for future study.
One is extending the methodology to scenarios where the learners' local data can only be partially collated by the data ID. We will need to develop techniques to utilize the missing data, possibly leveraging some recent developments in data imputation~\citep{yu2020optimal,xue2020integrating, diao2021semifl}. 
 Another remaining challenge is to extend the method to the scenarios with more than one assisting learner. Interesting problems include testing the usefulness of another learner \C given \b's data and finding an efficient algorithm for the joint model training with three or more learners.

\section{Funding}
Jiawei Zhang is supported by the National Science Foundation under grant number under CNS-2220286. Jie Ding is supported by the Army Research Laboratory and the Army Research Office under grant number  W911NF2310315.

\clearpage
\begin{appendices}

 \section{Organization of the Supplement}
This supplementary document is organized as follows. We prove Theorem~\ref{thm_Wn}, Theorem~\ref{probcon}, Corollary~\ref{coro_finite_sample_convergence}, 
Proposition~\ref{prop_local_diff_privacy},  and Corollary~\ref{prop_suffGLM} in Sections~\ref{sec_proofOfTh1},~\ref{sec_thm2Proof}, \ref{sec_prof_coro_finite_sample}, \ref{sec_AEAL_local_priv},  and \ref{sec_ProofOfCo1}, respectively. A discussion about Conditions~\ref{cond_designMatFullrank}-\ref{cond_AsymNorm} are given in Section~\ref{eq_a_discussion_about_conditions}. We present  a discussion on using the likelihood ratio test as an alternative to the Wald test for the AE-AL in Section~\ref{sec_proofOfTh3}. We provide theoretical and numerical analyses for the AE-AL that incorporates a penalty term in Section~\ref{sec_penal_th_sim}. A discussion on protecting sensitive information of $Y$ for AE-AL training by a random transformation is included in Section~\ref{sec_masking the response}.
We include 
details of the hospital length of stay dataset in Section~\ref{sec_hosp_data_detail}.
Additional simulation results are presented in Section~\ref{sec_addexp}.

\section{Proof of Theorem~\ref{thm_Wn}}\label{sec_proofOfTh1}

The sketch of our proof is as follows. First, in Section~\ref{sec_H0}, we show the convergence in distribution of the test statistic under $\textrm{H}_0$. Second, in Section~\ref{sec_H1}, we derive the result that the test statistic goes to infinity in probability under $\textrm{H}_1$. 

\subsection{Convergence in distribution under $\textrm{H}_0$}\label{sec_H0}


The proof follows a similar procedure to that of the classical Wald test. However, there is a key difference in that we need to account for the additional randomness introduced by $\bm{U}$.
Let $\bm{\beta}^{(\bm{U},\id)}_*$ be the $\id$ elements of  $\bm{\beta}^{(\bm{U})}_*$ corresponding to $\bm{X}^{(\B)}\bm{U}$.
We first show 
\begin{equation}\label{eq_beta_star_zero}
\bm{\beta}^{(\bm{U},\id)}_* = \bm{0} \text{ holds almost surely}.
\end{equation}
Then, we show
\begin{equation}\label{eq_betaU_normal}
\sqrt{n}\hat{\bm{\beta}}^{(\bm{U},\id)}_n\limd Z^{(\bm{V}_\id)},\text{ as $n\rightarrow\infty$,}
\end{equation}
where $Z^{(\bm{V}_\id)}$ is a random variable that follows $\mathcal N(\bm{0}, \bm{V}_{\id})$ and $\bm{V}_{\id}$ denotes the lower right
 $\id\times \id$ block of $\bm{V}\de \bm{V}_1^{-1}\bm{V}_2\bm{V}_1^{-1}$.

For Equality~\eqref{eq_beta_star_zero},
recall that 

\begin{align*}
 &\bm{\beta}_*=({\bm{\beta}^{(a)}_*}^\T,{\bm{\beta}^{(b)}_*}^\T,{\bm{\beta}^{(c)}_*}^\T)^\T\de  \underset{\bm{\beta}\in\mathbb R^{p}}{\argmin}\  \nlog_*(\bm{\beta}),\quad
    \nlog_*(\bm{\beta})\de \E\bigl(m(Y,X^\T\bm{\beta})\bigr),\\  
    &\bm{\beta}^{(\bm{U})}_*\de \argmin_{\bm{\beta}^{(\bm{U})}\in\mathbb R^{p_\A + \id}}  \nlog_*^{(\bm{U})}(\bm{\beta}^{(\bm{U})}),\quad  \nlog_*^{(\bm{U})}(\bm{\beta}^{(\bm{U})})
\de\E\bigl(\mest\bigl(Y,{X^{(\bm{U})}}^\T \bm{\beta}^{(\bm{U})} \bigr)\mid\bm{U}\bigr).
\end{align*}
Since by definition, $X^{(\bm{U})}=(X^{(\A)},X^{(\B)}\bm{U})$, $X^{(\bm{U})}$ can be obtained from a linear transformation of $X$.  Therefore, 
$$
\nlog_*^{(\bm{U})}(\bm{\beta}^{(\bm{U})}_*) \geq 
 \nlog_*(\bm{\beta}_*).$$
Under $\textrm{H}_0$, $\bm{\beta}^{(b)}_* = \bm{0}$, and we have 
$
\nlog_*^{(\bm{U})}(\bm{\beta}^{(\bm{U})}) = 
 \nlog_*(\bm{\beta}_*)
$
by taking $\bm{\beta}^{(\bm{U})} = \bigl({\bm{\beta}^{(a)}_*}^\T,{\bm{\beta}^{(c)}_*}^\T, \bm{0}\bigr)^\T$. 
By the definition of 
$\bm{\beta}^{(\bm{U})}_*$ and its uniqueness in Condition~\ref{A2}, we obtain Equality~\eqref{eq_beta_star_zero}.

Next, we show the result in \eqref{eq_betaU_normal}. 
We have
 $$
 \nabla \nlog^{(\bm{U})}_n( \bm{\beta}^{(\bm{U})}_*  ) 
= {\MU}^\T\nabla \nlog_n(\MU \bm{\beta}^{(\bm{U})}_*) = 
{\MU}^\T\nabla \nlog_n(\bm{\beta}_*),
 $$
 where the second equality is by Equality~\eqref{eq_beta_star_zero}.
 By the central limit theorem we have
 \begin{equation}\label{eq_convMnbetaStar}
\sqrt{n}\nabla\nlog_n(\bm{\beta}_*)\limd \mathcal N(\bm{0}, \bm{V}_2')
 \end{equation}
where 
$\bm{V}_2'\de \E \bigl(\nabla_{\bm{\beta}} \mest(Y,{X}^\T\bm{\beta}_*)\bigl(\nabla_{\bm{\beta}} \mest(Y,{X}^\T\bm{\beta}_*)\bigr)^\T\bigr).$
By the mean value theorem, we have
$$
\sqrt{n}\nabla\nlog^{(\bm{U})}_n(\hat{\bm{\beta}}^{(\bm{U})}_n) = \sqrt{n}\nabla\nlog^{(\bm{U})}_n(\bm{\beta}^{(\bm{U})}_*) +\nabla^2\nlog^{(\bm{U})}_n(\ddot{\bm{\beta}}_n)\sqrt{n}(\bm{\beta}^{(\bm{U})}_* - \hat{\bm{\beta}}^{(\bm{U})}_n), 
$$
where $\ddot{\bm{\beta}}_n$ is on the line segment connecting $\hat{\bm{\beta}}^{(\bm{U})}_n$ and $\bm{\beta}^{(\bm{U})}_*$. 
Conditional on the event that 
$\nabla^2\nlog^{(\bm{U})}_n(\ddot{\bm{\beta}}_n)$ is positive definite, which by Condition~\ref{A3} holds with probability going to one as $n\rightarrow\infty$. 
By $\nabla\nlog^{(\bm{U})}_n(\hat{\bm{\beta}}^{(\bm{U})}_n) = \bm{0}$, We have
\begin{align*}
 \sqrt{n}(\hat{\bm{\beta}}^{(\bm{U})}_n - \bm{\beta}^{(\bm{U})}_*) 
 =&\
\bigl(\nabla^2\nlog^{(\bm{U})}_n(\ddot{\bm{\beta}}_n)\bigr)^{-1}
 \sqrt{n}\nabla\nlog^{(\bm{U})}_n(\bm{\beta}^{(\bm{U})}_*)\\
=&\  \bigl(\nabla^2\nlog^{(\bm{U})}_n(\ddot{\bm{\beta}}_n)\bigr)^{-1} {\MU}^\T\sqrt{n}\nabla \nlog_n(\bm{\beta}_*).
\end{align*}
By the result in \eqref{eq_convMnbetaStar}, for each $\mathcal A\subseteq \R^{p_\A + \id}$,
 \begin{equation}\label{eq_MNconv1}
 \biggl|\P\biggl(\bigl(\nabla^2\nlog^{(\bm{U})}_n(\ddot{\bm{\beta}}_n)\bigr)^{-1} {\MU}^\T\sqrt{n}\nabla \nlog_n(\bm{\beta}_*)\in\mathcal A\biggr) - \P\biggl(\bigl(\nabla^2\nlog^{(\bm{U})}_n(\ddot{\bm{\beta}}_n)\bigr)^{-1} {\MU}^\T Z^{(\bm{V}_2')}\in \mathcal A\biggr)\biggr|\rightarrow 0,
\end{equation}
 as $n\rightarrow\infty$,
 where $Z^{(\bm{V}_2')}$ is a random variable that follows $\mathcal N(\bm{0}, \bm{V}_2')$.
 By Condition~\ref{A3} and the continuous mapping theorem, 
\begin{equation}\label{eq_MnConv2}
\bigl(\nabla^2\nlog^{(\bm{U})}_n(\ddot{\bm{\beta}}_n)\bigr)^{-1} {\MU}^\T Z^{(\bm{V}_2')}\limd\bm{V}_1^{-1} {\MU}^\T Z^{(\bm{V}_2')}.
\end{equation}
 Also, recall that
 \begin{align}
 \bm{V}_2\de&\ \E \bigl(\nabla_{\bm{\beta}^{(\bm{U})}} \mest(Y,{X^{(\bm{U})}}^\T\bm{\beta}^{(\bm{U})}_*)\bigl(\nabla_{\bm{\beta}^{(\bm{U})}} \mest(Y,{X^{(\bm{U})}}^\T\bm{\beta}^{(\bm{U})}_*)\bigr)^\T\mid\bm{U}\bigr) \nonumber\\
 =&\ \E \bigl(\nabla_{\bm{\beta}^{(\bm{U})}} \mest(Y,{X}^\T\MU\bm{\beta}^{(\bm{U})}_*)\bigl(\nabla_{\bm{\beta}^{(\bm{U})}} \mest(Y,{X}^\T\MU\bm{\beta}^{(\bm{U})}_*)\bigr)^\T\mid\bm{U}\bigr)\nonumber\\
  =&\ \E \bigl({\MU}^\T\nabla_{\bm{\beta}} \mest(Y,{X}^\T\bm{\beta}_*)\bigl(\nabla_{\bm{\beta}} \mest(Y,{X}^\T\bm{\beta}_*)\bigr)^\T\MU\mid\bm{U}\bigr)\nonumber\\
 =&\ {\MU}^\T\bm{V}_2'\MU,\nonumber
 \end{align}
 which indicates 
 \begin{equation}\label{eq_MNconv3}
 \bm{V}_1^{-1} {\MU}^\T  \bm{V}_2'\MU\bm{V}_1^{-1} = \bm{V}_1^{-1} \bm{V}_2\bm{V}_1^{-1}.
 \end{equation}
Combining the results in \eqref{eq_MNconv1} \eqref{eq_MnConv2}, and \eqref{eq_MNconv3}, we have
 \begin{equation}\label{eq_beta_normal_asymp}
  \sqrt{n}(\hat{\bm{\beta}}^{(\bm{U})}_n - \bm{\beta}^{(\bm{U})}_*)\limd Z^{(\bm{V})},\text{ as }n\rightarrow\infty,
\end{equation}
where $Z^{(\bm{V})}\sim \mathcal N(\bm{0}, \bm{V})$.

By \eqref{eq_beta_normal_asymp} and Equality~\eqref{eq_beta_star_zero}, we obtain \eqref{eq_betaU_normal}.

Combining the above result with $\|\hat{\bm{V}} - \bm{V}\|_{\infty}\limp 0$ as $n\rightarrow\infty$ from Condition~\ref{A3}, we conclude that
 $W_{n,\id}$ converges weakly to a chi-squared distribution with degrees of freedom~t.

\subsection{Convergence to infinity under $\textrm{H}_1$}\label{sec_H1}
By Condition~\ref{A3}, $\hat{\bm{V}}$ converges to $\bm{V}$ that is positive definite. Moreover, since the values of the elements in $\bm{U}$ are bounded, it can be shown by Corollary~6.3.8 of  \cite{horn2012matrix} that the eigenvalues of  $\bm{V}_{\id}^{-1}$
are bounded away from zero. By Condition~\ref{A2}, $\hat{\bm{\beta}}^{(\bm{U},\id)}_n\limp \bm{\beta}^{(\bm{U},\id)}_*$, where we recall that $\hat{\bm{\beta}}^{(\bm{U},\id)}_n$ and $\bm{\beta}^{(\bm{U},\id)}_*$ are the elements of $\hat{\bm{\beta}}^{(\bm{U})}_n$ and $\bm{\beta}^{(\bm{U})}_*$ corresponding to $\bm{X}^{(\bm{U})}$, respectively. Therefore, it suffices to show that 
for each $\epsilon\in (0,1)$, there exists $M_\epsilon$, such that 
\begin{equation}\label{eq_betaUstar_lowerbounded}
\P(\|\bm{\beta}^{(\bm{U},\id)}_*\|_2> M_\epsilon) \geq \epsilon,
\end{equation} which guarantees $W_{n,\id} = n\bigl(\hat{\bm{\beta}}^{(\bm{U},\id)}_n\bigr)^\T \hat{\bm{V}}_{\id}^{-1}\hat{\bm{\beta}}^{(\bm{U},\id)}_n\limp \infty$ as $n\rightarrow\infty$.

We first show 
\begin{equation}\label{eq_betaUstar_not_equal_zero}
\bm{\beta}^{(\bm{U},\id)}_*\neq \bm{0} \text{ almost surely}.
\end{equation}
Recall that $\nlog^{(\bm{U})}_*( \bm{\beta}^{(\bm{U})})\de \E\bigl(\mest(Y,{X^{(\bm{U})}}^\T\bm{\beta}^{(\bm{U})})\mid\bm{U}\bigr)$ and $\nlog_*( \bm{\beta})\de \E\bigl(\mest(Y,X^\T\bm{\beta})\bigr)$.
We will show that when 
\begin{equation}\label{eq_assump}
\bm{\beta}^{(\bm{U},\id)}_*= \bm{0},
\end{equation}
we have
\begin{equation}\label{eq_derivUzero}
\nabla \nlog^{(\bm{U})}_*( \bm{\beta}^{(\bm{U})}_*  ) \neq \bm{0},
\end{equation} which contradicts with Condition~\ref{A2}.
Let 
\begin{equation}\label{Mdef}
\MU \de 
\begin{pmatrix}
\bm{I}_{p_a\times p_a} & \bm{0}& \bm{0}\\
\bm{0}& \bm{0} &\bm{U}^{(p_b)}\\
\bm{0}&\bm{I}_{p_c\times p_c}&\bm{U}^{(p_c)}
\end{pmatrix},
\end{equation}
where $\bm{U}^{(p_b)}$ and $\bm{U}^{(p_c)}$ are the rows of $\bm{U}$ corresponding to $\bm{X}^{(b)}$ and $\bm{X}^{(c)}$, respectively.
By Condition~\ref{cond_exchInteDiff}, and conditional on $\bm{U}$, we have 
\begin{align*}
\nabla \nlog^{(\bm{U})}_*( \bm{\beta}^{(\bm{U})}_*  ) 
=& {\MU}^\T\nabla \nlog_*(\MU \bm{\beta}^{(\bm{U})}_*) \\
=&
\begin{pmatrix}
	\nabla_{\bm{\beta}^{(a)}}  \nlog_*(\MU \bm{\beta}^{(\bm{U})}_*) \\
		\nabla_{\bm{\beta}^{(c)}}   \nlog_*(\MU \bm{\beta}^{(\bm{U})}_*) \\
{(\bm{U}}^{(p_b)})^\T \nabla_{\bm{\beta}^{(b)}}   \nlog_*(\MU \bm{\beta}^{(\bm{U})}_*) +
{(\bm{U}}^{(p_c)})^\T \nabla_{\bm{\beta}^{(c)}}   \nlog_*(\MU \bm{\beta}^{(\bm{U})}_*)
\end{pmatrix},
\end{align*}
where $\nabla_{\bm{\beta}^{(a)}}$, $\nabla_{\bm{\beta}^{(c)}}$, and $\nabla_{\bm{\beta}^{(b)}}$ denote the gradients with respect to the first $p_a$, the last $p_c$, and the remaining $p_b$ arguments of $\nlog_*$, respectively.
Next, we show that $\nabla \nlog_*(\MU \bm{\beta}^{(\bm{U})}_*)\neq\bm{0}$.
By Equation~\eqref{eq_assump},
\begin{equation}\label{eq_TBmatrix}
\MU \bm{\beta}^{(\bm{U})}_* =
\MU
\begin{pmatrix}
\bm{\beta}^{(\bm{U},a)}_*\\
\bm{\beta}^{(\bm{U},c)}_*\\
\bm{0}
\end{pmatrix}
=
\begin{pmatrix}
\bm{\beta}^{(\bm{U},a)}_*\\
\bm{0}\\
\bm{\beta}^{(\bm{U},c)}_*
\end{pmatrix},
\end{equation}
where $\bm{\beta}^{(\bm{U},a)}_*$ and  $\bm{\beta}^{(\bm{U},c)}_*$ correspond to $\bm{X}^{(a)}$ and $\bm{X}^{(c)}$, respectively.
Since $\bm{\beta}^{(b)}_*\neq \bm{0}$ under $\textrm{H}_1$, we have $\MU \bm{\beta}^{(\bm{U})}_*\neq \bm{\beta_*}$.
Therefore, combining this result with Condition~\ref{A1}, we obtain
$$
 \nabla \nlog_*(\MU \bm{\beta}^{(\bm{U})}_*)=  \begin{pmatrix}
\nabla_{\bm{\beta}^{(a)}} \nlog_*(\MU \bm{\beta}^{(\bm{U})}_*) \\
\nabla_{\bm{\beta}^{(b)}}  \nlog_*(\MU \bm{\beta}^{(\bm{U})}_*) \\
\nabla_{\bm{\beta}^{(c)}} \nlog_*(\MU \bm{\beta}^{(\bm{U})}_*)
\end{pmatrix}\neq \bm{0}.
$$
If 
$
\nabla_{\bm{\beta}^{(a)}} \nlog_*(\MU \bm{\beta}^{(\bm{U})}_*) \neq \bm{0} 
$
or
$
	\nabla_{\bm{\beta}^{(c)}} \nlog_*(\MU \bm{\beta}^{(\bm{U})}_*)  \neq \bm{0}, 
$
we automatically obtain the result~\eqref{eq_derivUzero}. Otherwise, we have
$\nabla_{\bm{\beta}^{(b)}} \nlog_*(\MU \bm{\beta}^{(\bm{U})}_*)  \neq \bm{0}.
$
Next, we show
\begin{equation}\label{eq_UnablaMstar_neq_zero}
({\bm{U}}^{(p_b)})^\T \nabla_{\bm{\beta}^{(b)}}  \nlog_*(\MU \bm{\beta}^{(\bm{U})}_*) \neq \bm{0}.
\end{equation}
Since the elements of $\bm{U}$ are generated from an i.i.d.\ sample of a continuous distribution, they are almost surely linearly independent, we have
\begin{equation}\label{eq_cont_rand_linear}
\P(\exists \bm{z}\in\mathbb R^{p_b}, \text{ s.t. }\bm{z} \neq \bm{0}\text{ and }{(\bm{U}}^{(p_b)})^\T\bm{z} = \bm{0}) = 0.
\end{equation}
Additionally, by Equation~\eqref{eq_TBmatrix} we have:
\begin{equation}\label{eq_constant_grad}
    \text{$\nabla_{\bm{\beta}^{(b)}} \nlog_*(\MU \bm{\beta}^{(\bm{U})}_*)$ is a constant that does not dependent on $\bm{U}$.}
\end{equation}
Therefore, by combining the results in \eqref{eq_cont_rand_linear} and \eqref{eq_constant_grad}, we conclude that the condition stated in \eqref{eq_UnablaMstar_neq_zero} holds almost surely. Consequently, we obtain Equality~\eqref{eq_betaUstar_not_equal_zero}.

Next, we show the result in \eqref{eq_betaUstar_lowerbounded}. Since $\|\bm{\beta}^{(\bm{U},\id)}_*\|_2\geq 0$ almost surely, 
$$
\P(\|\bm{\beta}^{(\bm{U},\id)}_*\|_2<0)=0.
$$
If there exists $\epsilon'>0$ such that $\P(\|\bm{\beta}^{(\bm{U},\id)}_*\|_2\leq M)\geq \epsilon'$ for all $M>0$,
by the right continuity of the CDF, 
we have $\P(\|\bm{\beta}^{(\bm{U},\id)}_*\|_2\leq 0)\geq \epsilon'$. Thus, $\P(\|\bm{\beta}^{(\bm{U},\id)}_*\|_2=0) = \P(\|\bm{\beta}^{(\bm{U},\id)}_*\|_2\leq 0) - \P(\|\bm{\beta}^{(\bm{U},\id)}_*\|_2< 0) \geq \epsilon'$, and it contradicts with Equality~\eqref{eq_betaUstar_not_equal_zero}. Accordingly, we obtain the result in \eqref{eq_betaUstar_lowerbounded} and finish the proof.


\section{Proof of Theorem~\ref{probcon}}\label{sec_thm2Proof}
Before giving the detailed proof,  we introduce some necessary notations and present a sketch of our derivations. Recall that
$$
\TransformMat \de 
\begin{pmatrix}
\bm{I}_{p_a\times p_a} & & & \\
& & \bm{I}_{p_b\times p_b} &\\
&\bm{I}_{p_c\times p_c}& &\bm{I}_{p_c\times p_c}
\end{pmatrix},
$$
where each $\bm{I}$ is an identity matrix with its dimension indicated by the subscript, and $\bm{X} = ({\bm{X}^{(a)}}^\T, {\bm{X}^{(b)}}^\T, {\bm{X}^{(c)}}^\T)^\T$ and $\bm{X}^{(\A,\B)} = ({\bm{X}^{(a)}}^\T,  {\bm{X}^{(c)}}^\T, {\bm{X}^{(b)}}^\T, {\bm{X}^{(c)}}^\T)^\T$.
Therefore, we have
$$
\bm{X}\TransformMat = \bm{X}^{(\A,\B)}.
$$
To facilitate our derivation, we distinguish between \a's update and \b's update by introducing $k'$ in addition to the original assistance round number $k$, where 
after \a's or \b's update, $k'$ is updated by $k'+1$. In other words, $k'$ indexes each agent's update. Consequently, we have $\floor{k'/2} = k$, where $\floor{x}$ denotes the largest integer less than or equal to $x$.
Let $\nlog^{(\A,\B)}_n(\bm{\beta}^{(\A,\B)}) \de \nlog_n(\TransformMat\bm{\beta}^{(\A,\B)})$, where $\bm{\beta}^{(\A,\B)} = ({\bm{\beta}^{(\A)}}^\T, {\bm{\beta}^{(\B)}}^\T)^\T\in\R^{p_\A +p_\B}$,
and 
$\hat{\bm{\beta}}^{(\A,\B,k')}_n \de 
((\hat{\bm{\beta}}^{(\A,k')}_n)^\T, (\hat{\bm{\beta}}^{(\B,k')}_n)^\T)^\T$, which concatenates the estimated parameters from \A and \B in the $k'$th round.

The sketch of the proof is as follows.
First,  in Section~\ref{subsec_initSet}, we show the existence of a set $\InitSet\subset \R^p$ such that when the initial value satisfies $\TransformMat\hat{\bm{\beta}}^{(\A,\B,0)}_n\in \InitSet$, $$
\P(\TransformMat\hat{\bm{\beta}}^{(\A,\B,k')}_n\in \ThetaSet, \ \forall k'\in \mathbb N)\rightarrow1,
$$
as $n\rightarrow\infty$, where $\ThetaSet$ is from Condition~\ref{cond_sUconv}.
Second, in Section~\ref{subsec_proof_of_lemma9},
we 
obtain a sample-wise convergence result of $\TransformMat\hat{\bm{\beta}}^{(\A,\B,k')}_n$ to $\check{\bm{\beta}}_n$ as $k\rightarrow\infty$. Third, in Section~\ref{subsec_uniform_convergence_parameters}, we extend the sample-wise convergence result to an asymptotic version. Finally, in Section~\ref{sec_proofCorollary1}, we complete the proof by showing the asymptotic normality of the assisted learning estimator.

 

\subsection{Obtaining the set $\InitSet$}\label{subsec_initSet}
In this section, we obtain the set $\InitSet$ in three steps. We first construct a sample-based set $\InitSet_n$ such that when $\TransformMat\hat{\bm{\beta}}^{(\A,\B,0)}_n\in \InitSet_n$,  $\TransformMat\hat{\bm{\beta}}^{(\A,\B,k')}_n\in \InitSet_n$ holds for all $k'\in\mathbb N$. Second, we construct a population-based set $\InitSet$, such that $\P(\InitSet\subseteq\InitSet_n)\rightarrow1$ as $n\rightarrow\infty$. 
Third, we show that $\P(\InitSet_n \subseteq\ThetaSet)\rightarrow 1$ as $n\rightarrow\infty$.

Let $\BDset$ denote the boundary set of the open ball $\ThetaSet$.
Define
$$
\InitSet_n \de 
\biggl\{\bm{\beta}\in\R^p : \nlog_n(\bm{\beta}) \leq \min_{\bm{\beta}\in\BDset}\nlog_*(\bm{\beta})-\bigl( \min_{\bm{\beta}\in\BDset}\nlog_*(\bm{\beta}) - \nlog_*(\bm{\beta}_*)\bigr)/4 \biggr\}.
$$ 
According to Condition~\ref{cond_sUconv},
\begin{equation}\label{eq_OnNotEmpty}
\P(\InitSet_n\neq\emptySet)\rightarrow 1, \text{ as }n\rightarrow\infty,
\end{equation} 
where $\emptySet$ denotes the empty set.
Since for each $k'\in\mathbb N$,
$$\nlog_n^{(\A,\B)}(\hat{\bm{\beta}}^{(\A,\B,k')}_n)\leq \nlog_n^{(\A,\B)}(\hat{\bm{\beta}}^{(\A,\B,k'-1)}_n),$$
when $\TransformMat\hat{\bm{\beta}}^{(\A,\B,0)}_n\in \InitSet_n$,  $\TransformMat\hat{\bm{\beta}}^{(\A,\B,k')}_n\in \InitSet_n$ for all $k'\in\mathbb N$. 

Next, we define
$$
\InitSet \de 
\biggl\{\bm{\beta}\in\ThetaSet: \nlog_*(\bm{\beta}) \leq \min_{\bm{\beta}\in\BDset}\nlog_*(\bm{\beta})-\bigl( \min_{\bm{\beta}\in\BDset}\nlog_*(\bm{\beta}) - \nlog_*(\bm{\beta}_*)\bigr)/2  \biggr\}.
$$ 
By Condition~\ref{cond_sUconv}, 
\begin{equation}\label{eq_OinOn}
\P(\InitSet\subseteq\InitSet_n)\rightarrow1, \text{ as }n\rightarrow\infty.
\end{equation}
Therefore, when $\hat{\bm{\beta}}^{(\A,\B,0)}_n\in \InitSet$,  $\hat{\bm{\beta}}^{(\A,\B,k')}_n\in \InitSet_n$ for all $k'\in\mathbb N$ holds with probability going to one as $n\rightarrow\infty$. 

Finally, we show that $\InitSet_n \subseteq\ThetaSet$ holds with probability going to one as $n\rightarrow\infty$ by contradiction. 
If the statement does not hold, for each $N>0$, there exists $n\geq N$ and $\bm{\beta}^{(n,1)}\notin\ThetaSet$ such that
\begin{equation}\label{eq_lineSeg1}
    \nlog_n(\bm{\beta}^{(n,1)}) \leq \min_{\bm{\beta}\in\BDset}\nlog_*(\bm{\beta})-\bigl( \min_{\bm{\beta}\in\BDset}\nlog_*(\bm{\beta}) - \nlog_*(\bm{\beta}_*)\bigr)/2
\end{equation}
holds with probability bounded away from zero as $n\rightarrow\infty$. 
Since $\ThetaSet$ is an open ball in $\R^p$, there exists $\bm{\beta}^{(n,2)}\in\BDset$ such that $\bm{\beta}^{(n,2)}$ is on the line segment connecting $\bm{\beta}^{(n,1)}$ and $\bm{\beta}_*$ and satisfies 
\begin{equation}\label{eq_lineSeg}
    \nlog_*(\bm{\beta}^{(n,2)}) \geq \min_{\bm{\beta}\in\BDset}\nlog_*(\bm{\beta}).
\end{equation} 
Next, we demonstrate, using Conditions~\ref{cond_designMatFullrank} and \ref{cond_strictConvexloss}, that the following holds:
\begin{equation}\label{eq_Mnbeta_strictly_convex}
\P(\nlog_n(\bm{\beta})\text { is strictly convex})\rightarrow 1 \text{ as }n\rightarrow\infty.
\end{equation}
Here is a detailed explanation of this result.
 We have $\nabla^2\nlog_n(\bm{\beta}) = \bm{X}^\T D_n
\bm{X},
$
where $D_n$ is a diagonal matrix with diagonal elements $
\nabla_{\lp_1}^2m(y_1, \bm{x}_1^\T\bm{\beta}),\dots, \nabla_{\lp_n}^2m(y_n, \bm{x}_n^\T\bm{\beta})$, where $\lp_i = \bm{x}_i^\T\bm{\beta}$ and $\nabla_{\lp_i}^2$ denotes the second order derivative with respect to $\lp_i$. By Condition~\ref{cond_strictConvexloss}, $D_n$ is positive definite. Therefore, together with Condition~\ref{cond_designMatFullrank}, we obtain the result in \eqref{eq_Mnbeta_strictly_convex}. 
 However, when $\P(\nlog_n(\bm{\beta})$ is strictly convex, we have a contradiction with Condition~\ref{cond_sUconv} and Inequalities~\eqref{eq_lineSeg1} and \eqref{eq_lineSeg}, which indicate that $\nlog_n(\bm{\beta}^{(n,2)})$ is larger than both $\nlog_n(\bm{\beta}_*)$ and $ \nlog_n(\bm{\beta}^{(n,1)})$ with a non-dimishing probability.
 Therefore, when $\TransformMat\hat{\bm{\beta}}^{(\A,\B,0)}_n\in \InitSet$,  $\TransformMat\hat{\bm{\beta}}^{(\A,\B,k')}_n\in \ThetaSet$ for all $k'\in\mathbb N$ holds with probability going to one as $n\rightarrow\infty$.

\subsection{Sample-wise convergence}\label{subsec_proof_of_lemma9}

In this section, we show the convergence of $\TransformMat\hat{\bm{\beta}}^{(\A,\B,k')}_n$ to $ \check{\bm{\beta}}_n$ as $k\rightarrow\infty$ for a fixed dataset. The proof consists of three steps.
First, we condition on a fixed dataset and assume that $\check{\bm{\beta}}_n\in\ThetaSet$ and $\nabla^2 \nlog_n(\bm{\beta})$ is positive definite for all $\bm{\beta}\in\ThetaSet$, which holds with probability going to one as $n\rightarrow\infty$
by Conditions~\ref{cond_pdHessian}, \ref{cond_oraExist}, and \ref{cond_sUconv}.
Then, we show that $\hat{\bm{\beta}}^{(\A,\B,k')}_n$ with $k'\in\mathbb N$ are unique. Second, we will show in Lemma~\ref{lem_linConvNlog} that $ \nlog_n(\TransformMat\hat{\bm{\beta}}^{(\A,\B,k'+1)}_n) - 
\nlog_n(\check{\bm{\beta}}_n)\rightarrow 0$ as $k'\rightarrow\infty$ by Lemmas~\ref{lem_ctsEigen}-\ref{lem_bdLPP}.
Third, we will obtain that $\|\TransformMat\hat{\bm{\beta}}^{(\A,\B,k')}_n- \bm{\beta}^{(\A,\B,*)}_n\|_2\rightarrow 0$ as $k'\rightarrow\infty$ by Lemmas~\ref{lem_linConvNlog} and \ref{lem_convG}.

We consider a fixed data $(\bm{y}, \bm{X})$  where $\check{\bm{\beta}}_n\in\ThetaSet$ and $\nabla^2 \nlog_n(\bm{\beta})$ is positive definite for all $\bm{\beta}\in\ThetaSet$. We assume that the initial value $\hat{\bm{\beta}}^{(\A,\B,0)}_n\in\InitSet$. Let 
$$
\TransformMat^{(\A)} \de 
\begin{pmatrix}
\bm{I}_{p_a\times p_a}& \bm{0}_{p_a\times p_c}\\
 \bm{0}_{p_b\times p_a} &\bm{0}_{p_c\times p_c}\\
 \bm{0}_{p_c\times p_a} &\bm{I}_{p_c\times p_c}
\end{pmatrix}, \quad
\TransformMat^{(\B)} \de 
\begin{pmatrix}
\bm{0}_{p_a\times p_b}& \bm{0}_{p_a\times p_c}\\
 \bm{I}_{p_b\times p_b} &\bm{0}_{p_c\times p_c}\\
 \bm{0}_{p_c\times p_b} &\bm{I}_{p_c\times p_c}
\end{pmatrix}, 
$$
and let $\nabla^2_{\bm{\beta}^{(\A)}}$ and $\nabla^2_{\bm{\beta}^{(\B)}}$ denote Hessian matrices respect to the first $p_\A$ and last $p_\B$ arguments of the function. Since both $\TransformMat^{(\A)}$ and $\TransformMat^{(\B)}$ have full column ranks, both $\nabla^2_{\bm{\beta}^{(\A)}} \nlog_n^{(\A,\B)}\bigl(\hat{\bm{\beta}}^{(\A,\B,k')}_n\bigr) = {\TransformMat^{(\A)}}^\T \nabla \nlog_n\bigl(\TransformMat\hat{\bm{\beta}}^{(\A,\B,k')}_n\bigr)\TransformMat^{(\A)}$ and $\nabla^2_{\bm{\beta}^{(\B)}} \nlog_n^{(\A,\B)}\bigl(\hat{\bm{\beta}}^{(\A,\B,k')}_n\bigr) = {\TransformMat^{(\B)}}^\T \nabla \nlog_n\bigl(\TransformMat\hat{\bm{\beta}}^{(\A,\B,k')}_n\bigr)\TransformMat^{(\B)}$ are positive definite.
Therefore, $\hat{\bm{\beta}}^{(\A,\B,k')}_n$ is unique and inside $\ThetaSet$ for all $k'\in\mathbb N$ with probability going to one as $n\rightarrow\infty$. 
Let
\begin{align}
    \lambda^{(\textrm{min})}_n \de \inf_{\bm{\beta}\in \ThetaSet}\lambda^{(\textrm{min})}\bigl(\nabla^2\nlog_n(\bm{\beta}) \bigr),\label{eq_egenPos}\\
    \lambda^{(\textrm{max})}_n \de \sup_{\bm{\beta}\in \ThetaSet}\lambda^{(\textrm{max})}\bigl(\nabla^2\nlog_n(\bm{\beta}) \bigr),\label{eq_egenma}
\end{align}
where $\lambda^{(\textrm{min})}(\cdot)$ and $\lambda^{(\textrm{max})}(\cdot)$ return the smallest and largest eigenvalues of a matrix, respectively. The required technical lemmas are as follows.
\begin{lemma}[Continuity of eigenvalues]\label{lem_ctsEigen}
For two $p\times p$ symmetric matrices $E_1$ and $E_2$, let 
$\lambda_1^{(1)}\leq\cdots\leq \lambda_p^{(1)}$ and $\lambda_1^{(2)}\leq\cdots\leq \lambda_p^{(2)}$ denote their eigenvalues, respectively.
We have 
\begin{align*}
\sum_{i=1}^p|\lambda_i^{(1)} - \lambda_i^{(2)}|^2\leq \|E_1 - E_2\|_F^2,
\end{align*}
where $\|\cdot\|_F$ is the Frobenius norm.

\end{lemma}
\begin{proof}
This result follows from Corollary~6.3.8 and Definition~2.5.1 of  \cite{horn2012matrix}.
\end{proof}

\begin{lemma}[Strong convexity of $\nlog_n(\bm{\beta})$]\label{lem_strCVXg}
We have $\lambda^{(\textrm{min})}_n>0$, and $\nlog_n(\bm{\beta})$ is strongly convex  with the following inequalities hold for each $\bm{\beta}^{(1)}, \bm{\beta}^{(2)}\in\ThetaSet$.
\begin{align}
\nlog_n(\bm{\beta}^{(1)}) -  \nlog_n(\bm{\beta}^{(2)}) - \bigl(\nabla\nlog_n(\bm{\beta}^{(2)})\bigr)^\T(\bm{\beta}^{(1)} - \bm{\beta}^{(2)})
\geq& 
\lambda_n^{(\textrm{min})}/2\cdot\|\bm{\beta}^{(1)} - \bm{\beta}^{(2)}\|^2_2,\label{eq_stcvx1}\\
        \bigl(\nabla \nlog_n(\bm{\beta}^{(1)}) -  \nabla \nlog_n(\bm{\beta}^{(2)})\bigr)^\T(\bm{\beta}^{(1)} - \bm{\beta}^{(2)})
        \geq& \lambda_n^{(\textrm{min})}\|\bm{\beta}^{(1)} - \bm{\beta}^{(2)}\|^2_2.\label{eq_stcvx2}
\end{align}
\end{lemma}
\begin{proof}
By Lemma~\ref{lem_ctsEigen},  the smallest eigenvalue of $\nabla^2 \nlog_n(\bm{\beta})$ is a continuous function of $\bm{\beta}$. Since according to the assumption, $\nabla^2 \nlog_n(\bm{\beta})$
is positive definite for all $\bm{\beta}\in \R^p$, and $\ThetaSet$ is bounded, we have $\lambda^{(\textrm{min})}_n>0$. Thus, by using the properties of strong convexity, we complete the proof.
\end{proof}

\begin{lemma}[Zero partial derivatives]\label{lem_zeroDeri}
 We have
$$
 \nabla_{\bm{\beta}^{(l(k'))}}\nlog_n^{(\A,\B)}(\hat{\bm{\beta}}^{(\A,\B,k')}_n) = 0,\ \forall k'\in\mathbb N,
$$
where
\begin{equation}\label{eq_lk_def}
   l(k')=
\begin{cases}
\A&\ \text{if \A updates in the $k'$th round},\\
\B&\ \text{if \B updates in the $k'$th round}.
\end{cases} 
\end{equation}
\end{lemma}
\begin{proof}
Since $\nlog_n(\bm{\beta})$ is differentiable and $\nabla \nlog_n^{(\A,\B)}(\bm{\beta}^{(\A,\B)}) = \TransformMat^\T \nabla \nlog_n(\bm{\beta})$, it suffices to show that for each $k'\in\mathbb N$, $\hat {\bm{\beta}}^{(\A,\B,k')}_n$ exists. This is guaranteed by the property of $\InitSet$ derived in Section~\ref{subsec_initSet}.
\end{proof}

\begin{lemma}[Bounding the estimated parameters by the loss function]\label{lem_bdPLF}
We have
$$
\nlog_n(\TransformMat\hat {\bm{\beta}}^{(\A,\B,k')}_n) - \nlog_n(\TransformMat\hat{\bm{\beta}}^{(\A,\B,k'+1)}_n) \geq \lambda_n^{(\textrm{min})}/2\cdot \|\TransformMat(\hat {\bm{\beta}}^{(\A,\B,k')}_n - \hat{\bm{\beta}}^{(\A,\B,k'+1)}_n)\|^2_2.
$$
\end{lemma}
\begin{proof}
By Lemma~\ref{lem_zeroDeri}, 
\begin{align*}
 &\bigl(\nabla \nlog_n^{(\A,\B)}(\hat{\bm{\beta}}^{(\A,\B,k'+1)}_n)\bigr)^\T (\hat {\bm{\beta}}^{(\A,\B,k' + 1)}_n - \hat {\bm{\beta}}^{(\A,\B,k')}_n)= 0.
\end{align*}
Therefore, together with Inequality~\eqref{eq_stcvx1}, we have
\begin{align*}
 &\nlog_n(\TransformMat\hat {\bm{\beta}}^{(\A,\B,k')}_n) - \nlog_n(\TransformMat\hat{\bm{\beta}}^{(\A,\B,k'+1)}_n)\\
    =&\nlog_n(\TransformMat\hat {\bm{\beta}}^{(\A,\B,k')}_n) - \nlog_n(\TransformMat\hat{\bm{\beta}}^{(\A,\B,k'+1)}_n) - \bigl(\nabla \nlog_n(\TransformMat\hat {\bm{\beta}}^{(\A,\B,k' + 1)}_n)\bigr)^\T \TransformMat(\hat {\bm{\beta}}^{(\A,\B,k')}_n - \hat {\bm{\beta}}^{(\A,\B,k' + 1)}_n)\\
    \geq& \lambda_n^{(\textrm{min})}/2\cdot\|\TransformMat(\hat{\bm{\beta}}^{(\A,\B,k')}_n - \hat{\bm{\beta}}^{(\A,\B,k'+1)}_n)\|^2_2.
\end{align*}
\end{proof}

\begin{lemma}[Bounding the error by the derivative]\label{lem_bpg}
We have
\begin{equation}\label{eq_bdPG}
   \|\TransformMat\hat{\bm{\beta}}^{(\A,\B,k')}_n - \check{\bm{\beta}}_n\|_2
    \leq 1/\lambda_n^{(\textrm{min})}\bigl\|\nabla \nlog_n(\TransformMat\hat{\bm{\beta}}^{(\A,\B,k')}_n)\bigr\|_2.
\end{equation}
\end{lemma}
\begin{proof}

By  $\nabla \nlog_n(\check{\bm{\beta}}_n)=\bm{0}$, Inequality~\eqref{eq_stcvx2}, and the Cauchy-Schwarz inequality,
\begin{align*}
    \lambda_n^{(\textrm{min})}\|\TransformMat\hat{\bm{\beta}}^{(\A,\B,k')}_n-\check{\bm{\beta}}_n\|_2^2
     \leq& \bigl(\nabla \nlog_n(\TransformMat\hat{\bm{\beta}}^{(\A,\B,k')}_n)
    - \nabla \nlog_n(\check{\bm{\beta}}_n)\bigr)^\T(\TransformMat\hat{\bm{\beta}}^{(\A,\B,k')}_n - \check{\bm{\beta}}_n)\\
  =& \bigl(\nabla \nlog_n(\TransformMat\hat{\bm{\beta}}^{(\A,\B,k')}_n)
 \bigr)^\T(\TransformMat\hat{\bm{\beta}}^{(\A,\B,k')}_n - \check{\bm{\beta}}_n)\\
    \leq &  \bigl\|\nabla \nlog_n(\TransformMat\hat{\bm{\beta}}^{(\A,\B,k')}_n)
 \bigr\|_2\|\TransformMat\hat{\bm{\beta}}^{(\A,\B,k')}_n - \check{\bm{\beta}}_n\|_2.
\end{align*}
This concludes the proof.

\end{proof}

\begin{lemma}[Lipschitz continuity of the derivatives]\label{lem_lipContDG}
 We have $0<\lambda^{(\textrm{max})}_n<\infty$,
and for each $\bm{\beta}^{(1)}, \bm{\beta}^{(2)}\in\ThetaSet$,
\begin{equation}\label{eq_lipContDG}
 \biggl\|\nabla \nlog_n(\bm{\beta}^{(1)}) - \nabla \nlog_n(\bm{\beta}^{(2)})\biggr\|_2\leq \lambda_n^{(\textrm{max})}\|\bm{\beta}^{(1)} - \bm{\beta}^{(2)}\|_2.
\end{equation}
\end{lemma}
\begin{proof}
By Lemma~\ref{lem_ctsEigen} and the boundedness of $\ThetaSet$, we have $\lambda^{(\textrm{max})}_n<\infty$. 
Using the strong convexity of $\nlog_n(\bm{\beta})$ in Lemma~\ref{lem_strCVXg}, we have $\lambda^{(\textrm{max})}_n\geq \lambda^{(\textrm{min})}_n>0$. We complete the proof by applying the mean value theorem to $\nabla \nlog_n(\bm{\beta})$.
\end{proof}

\begin{lemma}[Bounding the error by the estimated parameters]\label{lem_bdLPP}
For $k'\geq 1$, we have
$$
\|\TransformMat\hat{\bm{\beta}}^{(\A,\B,k')}_n - \check{\bm{\beta}}_n\|_2\leq 
\frac{\sqrt{2}\lambda^{(\textrm{max})}_n}{\lambda^{(\textrm{min})}_n}\cdot\|\TransformMat\hat{\bm{\beta}}^{(\A,\B,k'+1)}_n - \TransformMat\hat{\bm{\beta}}^{(\A,\B,k')}_n\|_2.
$$
\end{lemma}
\begin{proof}
We first assume that \A updates in the $(k'+1)$th round. By Lemma~\ref{lem_zeroDeri}, 
\begin{align}
\nabla_{\bm{\beta}^{(\B)}}\nlog_n^{(\A,\B)}(\hat{\bm{\beta}}^{(\A,\B,k')}_n)&= \bm{0},\label{eq_d01}\\
\nabla_{\bm{\beta}^{(\A)}}\nlog_n^{(\A,\B)}(\hat{\bm{\beta}}^{(\A,\B,k'+1)}_n)&= \bm{0},\label{eq_d02}
\end{align}
where recall that $\nabla_{\bm{\beta}^{(\A)}}$ and $\nabla_{\bm{\beta}^{(\B)}}$ denote gradients  respect to the first $p_\A$ and last $p_\B$ arguments of the function, respectively. 
Thus, by Equality~\eqref{eq_lipContDG}, Inequalities \eqref{eq_d01} and \eqref{eq_d02}, we have
\begin{align*}
\biggl\|\nabla\nlog_n^{(\A,\B)}(\hat{\bm{\beta}}^{(\A,\B,k')}_n)\biggr\|_2
=&\biggl\|\nabla_{\bm{\beta}^{(\A)}}\nlog_n^{(\A,\B)}(\hat{\bm{\beta}}^{(\A,\B,k')}_n)\biggr\|_2\\
=&\biggl\|\nabla_{\bm{\beta}^{(\A)}}\nlog_n^{(\A,\B)}(\hat{\bm{\beta}}^{(\A,\B,k')}_n) - \nabla_{\bm{\beta}^{(\A)}}\nlog_n^{(\A,\B)}(\hat{\bm{\beta}}^{(\A,\B,k'+1)}_n)\biggr\|_2\\
\leq& \biggl\|\nabla\nlog_n^{(\A,\B)}(\hat{\bm{\beta}}^{(\A,\B,k')}_n) - \nabla\nlog_n^{(\A,\B)}(\hat{\bm{\beta}}^{(\A,\B,k'+1)}_n)\biggr\|_2\\
=& \biggl\|\TransformMat^\T\biggl(\nabla\nlog_n(\TransformMat\hat{\bm{\beta}}^{(\A,\B,k')}_n) - \nabla\nlog_n(\TransformMat\hat{\bm{\beta}}^{(\A,\B,k'+1)}_n)\biggr)\biggr\|_2\\
\leq&\sqrt{2}\biggl\|\nabla\nlog_n(\TransformMat\hat{\bm{\beta}}^{(\A,\B,k')}_n) - \nabla\nlog_n(\TransformMat\hat{\bm{\beta}}^{(\A,\B,k'+1)}_n)\biggr\|_2\\
\leq&\sqrt{2}\lambda^{(\textrm{max})}_n\|\TransformMat\hat{\bm{\beta}}^{(\A,\B,k')}_n - \TransformMat\hat{\bm{\beta}}^{(\A,\B,k'+1)}_n\|_2,
\end{align*}
where the second inequality is by the fact that the absolute eigenvalue of $\TransformMat\TransformMat^\T$ are upper bounded by $2$.

Next, we will find out the connection between $\nabla\nlog_n^{(\A,\B)}(\hat{\bm{\beta}}^{(\A,\B,k')}_n)$ of dimension $p_\A + p_\B$ and $\nabla \nlog_n(\TransformMat\hat{\bm{\beta}}^{(\A,\B,k')}_n)$ of dimension $p$ and apply Inequality~\eqref{eq_bdPG}. Recall that
$$
\TransformMat^{(\A)} \de 
\begin{pmatrix}
\bm{I}_{p_a\times p_a}& \bm{0}_{p_a\times p_c}\\
 \bm{0}_{p_b\times p_a} &\bm{0}_{p_c\times p_c}\\
 \bm{0}_{p_c\times p_a} &\bm{I}_{p_c\times p_c}
\end{pmatrix}, \quad
\TransformMat^{(\B)} \de 
\begin{pmatrix}
\bm{0}_{p_a\times p_b}& \bm{0}_{p_a\times p_c}\\
 \bm{I}_{p_b\times p_b} &\bm{0}_{p_c\times p_c}\\
 \bm{0}_{p_c\times p_b} &\bm{I}_{p_c\times p_c}
\end{pmatrix}, 
$$
which correspond to the first $p_\A$ and last $p_\B$ columns of $\TransformMat$.
It follows from Equality~\eqref{eq_d01} that
$$
(\TransformMat^{(\B)})^\T\nabla \nlog_n(\TransformMat\hat{\bm{\beta}}^{(\A,\B,k')}_n) = \nabla_{\bm{\beta}^{(\B)}}\nlog_n^{(\A,\B)}(\hat{\bm{\beta}}^{(\A,\B,k')}_n) = \bm{0}.
$$
 Therefore, 
\begin{align*}
    \bigl\|\nabla\nlog_n^{(\A,\B)}(\hat{\bm{\beta}}^{(\A,\B,k')}_n)\bigr\|_2 =&\ \bigl\|\nabla_{\bm{\beta}^{(\A)}}\nlog_n^{(\A,\B)}(\hat{\bm{\beta}}^{(\A,\B,k')}_n)\bigr\|_2\\
        =&\ \bigl\|(\TransformMat^{(\A)})^\T\nabla \nlog_n(\TransformMat\hat{\bm{\beta}}^{(\A,\B,k')}_n)\bigr\|_2\\
    =&\ \bigl\|\TransformMat^\T\nabla \nlog_n(\TransformMat\hat{\bm{\beta}}^{(\A,\B,k')}_n)\bigr\|_2 \\
   =&\ \bigl\|\nabla \nlog_n(\TransformMat\hat{\bm{\beta}}^{(\A,\B,k')}_n)\bigr\|_2.
\end{align*}
By similar arguments, the same inequality holds when \B updates in the $k'+1$th round.
Thus, we finish the proof by applying the Inequality~\eqref{eq_bdPG} to the above results.
\end{proof}

\begin{lemma}[Convergence of $\nlog_n(\TransformMat\hat{\bm{\beta}}^{(\A,\B,k')}_n)$]\label{lem_linConvNlog}
We have 
\begin{equation}\label{eq_linCOnvf}
    \nlog_n(\TransformMat\hat{\bm{\beta}}^{(\A,\B,k'+1)}_n) - 
\nlog_n(\check{\bm{\beta}}_n) \leq \eta_n
(\nlog_n(\TransformMat\hat{\bm{\beta}}^{(\A,\B,k')}_n) - 
\nlog_n(\check{\bm{\beta}}_n)),
\end{equation}
where
$$\eta_n = 1 - \frac{(\lambda_n^{(\textrm{min})})^3 }{4(\lambda_n^{(\textrm{max})})^3}$$ is positive and less than $1$. 
\end{lemma}
\begin{proof}
By the mean value theorem, there exists a
$\bm{\beta}'$ on the line segment connecting $\bm{T}\hat{\bm{\beta}}^{(\A,\B,k')}_n$ and $\check{\bm{\beta}}_n$, such that 
$$
\nlog_n(\bm{T}\hat{\bm{\beta}}^{(\A,\B,k')}_n) - \nlog_n(\check{\bm{\beta}}_n) = \biggl(\nabla\nlog_n(\bm{\beta}') \biggr)^\T(\bm{T}\hat{\bm{\beta}}^{(\A,\B,k')}_n - \check{\bm{\beta}}_n).
$$
Together with the fact that $\nabla\nlog_n(\check{\bm{\beta}}_n) = \bm{0}$, we have
\begin{align*}
    \nlog_n(\bm{T}\hat{\bm{\beta}}^{(\A,\B,k')}_n) - 
    \nlog_n(\check{\bm{\beta}}_n) 
    =& \nlog_n(\bm{T}\hat{\bm{\beta}}^{(\A,\B,k')}_n) - \nlog_n(\check{\bm{\beta}}_n) - \biggl(\nabla\nlog_n(\check{\bm{\beta}}_n)\biggr)^\T (\bm{T}\hat{\bm{\beta}}^{(\A,\B,k')}_n - \check{\bm{\beta}}_n)\\
    =& \biggl(\nabla\nlog_n(\bm{\beta}') - \nabla\nlog_n(\check{\bm{\beta}}_n)\biggr)^\T(\bm{T}\hat{\bm{\beta}}^{(\A,\B,k')}_n - \check{\bm{\beta}}_n)\\
    \leq& \lambda_n^{(\textrm{max})}\|\bm{\beta}' - \check{\bm{\beta}}_n\|_2\|\bm{T}\hat{\bm{\beta}}^{(\A,\B,k)}_g - \check{\bm{\beta}}_n\|_2\\
     \leq& \lambda_n^{(\textrm{max})}\|\bm{T}\hat{\bm{\beta}}^{(\A,\B,k)}_g - \check{\bm{\beta}}_n\|_2^2\\
    \leq& \frac{2(\lambda_n^{(\textrm{max})})^3}{(\lambda_n^{(\textrm{min})})^2}\|\bm{T}\hat{\bm{\beta}}^{(\A,\B,k'+1)}_n - \bm{T}\hat{\bm{\beta}}^{(\A,\B,k')}_n\|_2^2\\
    \leq&\frac{4(\lambda_n^{(\textrm{max})})^3}{(\lambda_n^{(\textrm{min})})^3}\left(\nlog(\bm{T}\hat{\bm{\beta}}^{(\A,\B,k')}_n) - \nlog_n(\bm{T}\hat{\bm{\beta}}^{(\A,\B,k'+1)}_n)\right),
\end{align*}
where 
the first inequality follows from the Cauchy-Schwarz inequality and Inequality~\eqref{eq_lipContDG}, the third inequality follows from Lemma~\ref{lem_bdLPP}, and the last inequality follows from Lemma~\ref{lem_bdPLF}.
Next, we let $\tilde{\eta}\de  \frac{4(\lambda_n^{(\textrm{max})})^3}{(\lambda_n^{(\textrm{min})})^3}$ and rearange the above inequality as follows.
\begin{align*}
        \nlog_n(\bm{T}\hat{\bm{\beta}}^{(\A,\B,k')}_n) - 
    \nlog_n(\check{\bm{\beta}}_n) 
    \leq&\ \tilde{\eta}\left(\nlog(\bm{T}\hat{\bm{\beta}}^{(\A,\B,k')}_n) - \nlog_n(\bm{T}\hat{\bm{\beta}}^{(\A,\B,k'+1)}_n)\right), \\
  \tilde{\eta}\nlog_n(\bm{T}\hat{\bm{\beta}}^{(\A,\B,k'+1)}_n)  - \nlog_n(\check{\bm{\beta}}_n) 
    \leq&\  (\tilde{\eta}-1)\nlog_n(\bm{T}\hat{\bm{\beta}}^{(\A,\B,k')}_n), \\
 \tilde{\eta}\left(\nlog_n(\bm{T}\hat{\bm{\beta}}^{(\A,\B,k'+1)}_n)  - \nlog_n(\check{\bm{\beta}}_n)\right)  - (1 -\tilde{\eta} )\nlog_n(\check{\bm{\beta}}_n)
  \leq&\ (\tilde{\eta}-1)\nlog_n(\bm{T}\hat{\bm{\beta}}^{(\A,\B,k')}_n),\\
  \tilde{\eta}\left(\nlog_n(\bm{T}\hat{\bm{\beta}}^{(\A,\B,k'+1)}_n)  - \nlog_n(\check{\bm{\beta}}_n)\right)
  \leq&\  (\tilde{\eta}-1)\left(\nlog_n(\bm{T}\hat{\bm{\beta}}^{(\A,\B,k')}_n) - \nlog_n(\check{\bm{\beta}}_n)\right),\\
  \nlog_n(\bm{T}\hat{\bm{\beta}}^{(\A,\B,k'+1)}_n)  - \nlog_n(\check{\bm{\beta}}_n)
   \leq&\  \frac{\tilde{\eta}-1}{\tilde{\eta}}\left(\nlog_n(\bm{T}\hat{\bm{\beta}}^{(\A,\B,k')}_n) - \nlog_n(\check{\bm{\beta}}_n)\right).
\end{align*}
Therefore, by taking $$\eta_n =  \frac{\tilde{\eta}-1}{\tilde{\eta}} = 1 - \frac{(\lambda_n^{(\textrm{min})})^3 }{4(\lambda_n^{(\textrm{max})})^3},$$ we obtain Inequality~\eqref{eq_linCOnvf}. Since $\lambda_n^{(\textrm{max})}\geq \lambda_n^{(\textrm{min})}$, and they are bounded away from 0, we have $0<\eta_n<1$.

\end{proof}

\begin{lemma}[Convergence of the AE-AL model for a fixed dataset]\label{lem_convG}
When the initial value $\hat{\bm{\beta}}^{(\A,\B,0)}_n\in \InitSet$,
\begin{equation}\label{eq_betaConvqRes}
    \|\bm{T}\hat{\bm{\beta}}^{(\A,\B,k')}_n - \check{\bm{\beta}}_n\|_2\leq (q_n^{(1)})^{k'}\cdot q_n^{(2)},
\end{equation}
for   all $k'\in \mathbb N$,
where 
\begin{align}
    q_n^{(1)} &= \sqrt{\eta_n}\quad \text{is positive and less than 1},\label{eq_q1Def}\\
    q_n^{(2)} &= \sqrt{2\csthes_n\cdot(1 + \eta_n)/\lambda_{n}^{(\textrm{min})} } \cdot \sum_{h = 0}^{\infty} \eta_n^{h/2}\quad \text{is positive and finite},\label{eq_q2Def}\\
    \csthes_n &= \sup_{\bm{\beta}^{(1)}, \bm{\beta}^{(2)}\in\ThetaSet}\left|\nlog_n(\bm{\beta}^{(1)}) - \nlog_n(\bm{\beta}^{(2)})\right|.\label{eq_csthes_def}
\end{align}
\end{lemma}
\begin{proof}
We will first show that the sequence $\{\bm{T}\hat {\bm{\beta}}^{(\A,\B,k')}_n\}$ converges as $k'\rightarrow\infty$. Then, we will finish the proof by obtaining Inequality~\eqref{eq_betaConvqRes}.

 By the boundedness of $\ThetaSet$, $\csthes_n$ in \eqref{eq_csthes_def} exists. Therefore, 
$$
    \nlog_n(\bm{T}\hat{\bm{\beta}}^{(\A,\B,0)}_n) - 
\nlog_n(\check{\bm{\beta}}_n) \leq \csthes_n,
$$
where recall that we have assumed $\check{\bm{\beta}}_n\in\ThetaSet$ at the begining of this subsection.
Combining the above result and Lemma~\ref{lem_linConvNlog}, we have that for each $k'\in\mathbb N$,
\begin{equation}\label{eq_gAB_gabc_converge}
\left|\nlog_n(\bm{T}\hat{\bm{\beta}}^{(\A,\B,k')}_n) - 
\nlog_n(\check{\bm{\beta}}_n)\right| \leq \csthes_n\cdot  \eta_n^{k'}.
\end{equation}
Therefore, by the triangular inequality,
\begin{equation}\label{eq_gAB_gabd_bound}
    \left|\nlog_n(\bm{T}\hat {\bm{\beta}}^{(\A,\B,k')}_n) - \nlog_n(\bm{T}\hat{\bm{\beta}}^{(\A,\B,k'+1)}_n)\right|\leq \csthes_n(1 + \eta_n) \eta_n^{k'}.
\end{equation}
Next, we  show that the sequence $\{\hat {\bm{T}\bm{\beta}}^{(\A,\B,k')}_n\}$ converges by verifying that it is a Cauchy sequence.
By Inequality~\eqref{eq_gAB_gabd_bound} and Lemma~\ref{lem_bdPLF},
\begin{align*}
  \left\|\bm{T}\hat {\bm{\beta}}^{(\A,\B,k')}_n - \bm{T}\hat{\bm{\beta}}^{(\A,\B,k'+1)}_n\right\|^2_2
  \leq&  \frac{2}{\lambda_n^{(\textrm{min})}}\cdot \left|\nlog_n(\bm{T}\hat {\bm{\beta}}^{(\A,\B,k')}_n) - \nlog_n(\bm{T}\hat{\bm{\beta}}^{(\A,\B,k'+1)}_n)\right|\\
  \leq&  \frac{2\csthes_n }{\lambda_n^{(\textrm{min})}}\cdot(1 + \eta_n)\cdot  \eta_n^{k'}.
\end{align*}
Therefore, for each pair  $k_{(1)}'< k_{(2)}'$  from $\mathbb N$,
\begin{align}
  \bigl\|\TransformMat\hat{\bm{ \beta}}^{(\A,\B,k_{(1)}')}_n - \TransformMat\hat{\bm{ \beta}}^{(\A,\B,k_{(2)}')}_n\bigr\|_2 
  \leq&
  \sum_{h = k_{(1)}'}^{k_{(2)}'-1}\|\TransformMat\hat{\bm{\beta}}^{(\A,\B,h)}_n - \TransformMat\hat{\bm{\beta}}^{(\A,\B,h+1)}_n\|_2\label{eq_ineqBLCONV}\\
  \leq&  \sum_{h = k_{(1)}'}^{\infty} \|\TransformMat\hat{\bm{\beta}}^{(\A,\B,h)}_n - \TransformMat\hat{\bm{\beta}}^{(\A,\B,h+1)}_n\|_2\nonumber\\
  \leq&  \sqrt{2\csthes_n\cdot(1 + \eta_n)/\lambda_n^{(\textrm{min})}} \cdot \sum_{h = k_{(1)}'}^{\infty} \eta_n^{h/2}.\nonumber
\end{align}
Since $\eta_n<1$, the power series $\sum_{h = k_{(1)}'}^{\infty} (\sqrt{\eta_n})^{h}$ converges, and the right-hand side of the above inequality goes to 0 as $k_{(1)}'\rightarrow\infty$. Thus, $\{\TransformMat\hat {\bm{\beta}}^{(\A,\B,k')}_n\}$ is a Cauchy sequence in $\mathbb R$ and it converges. By the convergence of $\nlog_n(\TransformMat\hat{\bm{\beta}}^{(\A,\B,k')}_n)$ to $\nlog_n(\check{\bm{\beta}}_n)$ from Inequality~\eqref{eq_gAB_gabc_converge}, the fact that $\nlog_n(\bm{\beta})$ is a continuous function, and the uniqueness of $\check{\bm{\beta}}_n$, we have   $\TransformMat\hat{\bm{\beta}}^{(\A,\B,k')}_n\rightarrow \check{\bm{\beta}}_n$ as $k'\rightarrow\infty$. 

It remains to show the linear convergence rate with factor $q_n^{(1)}$ in Inequality~\eqref{eq_betaConvqRes}.
Take $k_{(1)}' = k'$ and let $k_{(2)}'$ on the left-hand side of  Inequality~\eqref{eq_ineqBLCONV} go to infinity, we have 
\begin{align*}
    \|\TransformMat\hat {\bm{\beta}}^{(\A,\B,k')}_n - \check{\bm{\beta}}_n\|_2 
  \leq &
  \sqrt{2\csthes_n\cdot(1 + \eta_n)/\lambda_n^{(\textrm{min})} } \cdot \sum_{h = k'}^{\infty} \eta_n^{h/2}\\
  =& \sqrt{2\csthes_n\cdot(1 + \eta_n)/\lambda_n^{(\textrm{min})} } \cdot \sum_{h = 0}^{\infty} \eta_n^{h/2} \cdot \eta_n^{k'/2}.
\end{align*}
Therefore, we obtain Inequality~\eqref{eq_betaConvqRes}.
\end{proof}

\subsection{Uniform convergence of the estimated parameters}\label{subsec_uniform_convergence_parameters}
The results in the previous two subsections show that when $\TransformMat\hat{\bm{\beta}}^{(\A,\B,0)} \in\InitSet$,
\begin{equation}\label{eq_betaCOnvqUniform}
 \|\TransformMat\hat{\bm{\beta}}^{(\A,\B,k')}_n - \check{\bm{\beta}}_n\|_2\leq (q_n^{(1)})^{k'}\cdot q_n^{(2)}
\end{equation}
holds for all $k'\in \mathbb N$ with probability going to 1 as $n\rightarrow\infty$, where  $q_{n}^{(1)}$ and  $q_{n}^{(2)}$ are defined in \eqref{eq_q1Def} and \eqref{eq_q2Def}, respectively. 
To prove $\P(A_n(q_1,q_2))\rightarrow 1$ as $n\rightarrow\infty$, it remains to show that $q_{n}^{(1)}$ is bounded away from 0 and 1, and $q_{n}^{(2)}$ is positive and upper bounded with probability going to 1 as $n\rightarrow\infty$. We show these results by the convergence of components in $q_{n}^{(1)}$ and  $q_{n}^{(2)}$. 

Let $\lambda^{(\textrm{min})}_* \de \inf_{\bm{\beta}\in \ThetaSet}\lambda^{(\textrm{min})}\bigl(\nabla^2\nlog_*(\bm{\beta}) \bigr)$, $
    \lambda^{(\textrm{max})}_* \de \sup_{\bm{\beta}\in \ThetaSet}\lambda^{(\textrm{max})}\bigl(\nabla^2\nlog_*(\bm{\beta}) \bigr)$, and $\csthes_* \de \sup_{\bm{\beta}^{(1)}, \bm{\beta}^{(2)}\in\ThetaSet}|\nlog_*(\bm{\beta}^{(1)}) - \nlog_*(\bm{\beta}^{(2)})|$. 
    By Condition~\ref{cond_sUconv} and Lemma~\ref{lem_ctsEigen}, 
    \begin{equation}\label{eq_lambdaCconv}
         |\lambda^{(\textrm{min})}_n - \lambda^{(\textrm{min})}_*|\limp 0,\ |\lambda^{(\textrm{max})}_n - \lambda^{(\textrm{max})}_*|\limp 0,\ |\csthes_n - \csthes_*|\limp 0
    \end{equation}
as $n\rightarrow\infty$. 
Let
\begin{equation}\label{eq_etaStar_def}
\eta_* \de 1 - \frac{(\lambda_*^{(\textrm{min})})^3 }{4(\lambda_*^{(\textrm{max})})^3}.
\end{equation}
By Condition~\ref{cond_pdHessian}, we have $\lambda_*^{(\textrm{min})}>0$. Therefore, $0<\eta_*<1$. We define constants
$$
     \epsilon_2 \de \min\{|1-\eta_*|/2, \eta_*/2\},\quad
     \epsilon_3  \de \lambda^{(\textrm{min})}_*/2,\quad
     \epsilon_4  \de \csthes_*/2,
$$ 
and the event
$$\mathcal E_n \de \biggl\{|\eta_n - \eta_*| <\epsilon_2\ \text{ and } |\lambda^{(\textrm{min})}_n-\lambda^{(\textrm{min})}_*|<\epsilon_3\   \text{ and } |\csthes_n - \csthes_*|<\epsilon_4 \biggr\}.$$ 
By the results in \eqref{eq_lambdaCconv},
\begin{equation}\label{eq_EProbgoesToOne}
    \P(\mathcal E_n)\rightarrow 1\text{ as }n\rightarrow\infty.
\end{equation} 
Next, we let $q_1 \de \sqrt{\eta_*+\epsilon_2}$ and 
\begin{align*}
    q_2 =&\ \sqrt{2(\csthes_* + \epsilon_4)\cdot(1 + \eta_* + \epsilon_2)/(\lambda_{*}^{(\textrm{min})} -\epsilon_3)} \cdot \sum_{h = 0}^{\infty} (\eta_* + \epsilon_2)^{h/2}\\
    \leq &\ \sqrt{ 3\csthes_*\cdot(1 + \eta_* + \epsilon_2)/\lambda_*^{(\textrm{min})} } \cdot \sum_{h = 0}^{\infty} (\eta_* + \epsilon_2)^{h/2}.
\end{align*}
Then, by the definitions of $q_{n}^{(1)}$ and  $q_{n}^{(2)}$, it can be verified that
$$\P\biggl(\{q_n^{(1)}<q_1<1\},\\\{q_n^{(1)}> \sqrt{\eta_* - \epsilon_2} >0\}, \text{ and }  \{q_n^{(2)} < q_2\} \biggr)\rightarrow 1,$$
as $n\rightarrow\infty$. Therefore, we obtain the result that $\P(A_n(q_1,q_2))\rightarrow 1$ as $n\rightarrow\infty$.

\subsection{Asymptotic normality}\label{sec_proofCorollary1}
We have
$\|\hat{\bm{\beta}}^{(k)}_n - \check{\bm{\beta}}_n\|_2 = O_p(q_1^k),$
as $n\rightarrow\infty$ and $k\rightarrow\infty$. Together with Condition~\ref{cond_AsymNorm}, it suffices to show that 
$\hat{\bm{\beta}}^{(k)}_n = \check{\bm{\beta}}_n + o_p(1/\sqrt{n})$, which is equivalent to
\begin{align*}
  & q_1^k\sqrt{n}\rightarrow0,\\
    \Longleftrightarrow& k \cdot\log q_1+ \frac{1}{2}\log n\rightarrow -\infty,\\
    \Longleftrightarrow& \log n\cdot\biggl(\frac{k}{\log n} \cdot\log q_1+ \frac{1}{2}\biggr) \rightarrow -\infty,
\end{align*}
as $n\rightarrow\infty$. Accordingly, it suffices to take a $k$ such that  $$\frac{k}{\log n} \cdot\log q_1+ \frac{1}{2}<0.$$
Since $0<q_1<1$, the above inequality holds when $k/\log n$ is sufficiently large. Therefore, we complete the proof by taking an arbitrary $q_3>-1/(2\log q_1)$ and let $k\geq q_3\log n$.

\section{Finite Sample Convergence Results}\label{sec_prof_coro_finite_sample}
We first prove Corollary~\ref{coro_finite_sample_convergence} in Section~\ref{sec_prof_coro_finite_sample_Detal}. Then, we provide a discussion about extending the theoretical result to the scenario with a changing $p$ in Section~\ref{subsec_changingp}.
\subsection{Proof of Corollary~\ref{coro_finite_sample_convergence}}\label{sec_prof_coro_finite_sample_Detal}
We prove Corollary~\ref{coro_finite_sample_convergence} by bounding the probabilities of the events in \eqref{eq_OnNotEmpty}, \eqref{eq_OinOn}, \eqref{eq_Mnbeta_strictly_convex}, and \eqref{eq_EProbgoesToOne}.

First, for \eqref{eq_OnNotEmpty} and \eqref{eq_OinOn}, we take $\Mb_1 = \bigl(\min_{\bm{\beta}\in\BDset} \nlog_*(\bm{\beta}) - \nlog_*(\bm{\beta}_*)\bigr)/4$. In the following proof, we show $\InitSet_n\neq\emptySet$  by $\bm{\beta}_*\in \InitSet_n$.
When 
$$
\sup_{\bm{\beta}\in \ThetaSet}\left|\nlog_n(\bm{\beta})-\nlog_*(\bm{\beta})\right|< \Mb_1, 
$$
we have 
\begin{align*}
\nlog_n(\bm{\beta}_*) \leq &\ \nlog_*(\bm{\beta}_*) + \bigl( \min_{\bm{\beta}\in\BDset}\nlog_*(\bm{\beta}) - \nlog_*(\bm{\beta}_*)\bigr)/4\\ 
=&\ (3/4)\cdot \nlog_*(\bm{\beta}_*) + (1/4)\cdot \min_{\bm{\beta}\in\BDset}\nlog_*(\bm{\beta})\\
\leq&\ 
(1/4)\cdot \nlog_*(\bm{\beta}_*) + (3/4)\cdot \min_{\bm{\beta}\in\BDset}\nlog_*(\bm{\beta})\\
= &\ \min_{\bm{\beta}\in\BDset}\nlog_*(\bm{\beta})-\bigl( \min_{\bm{\beta}\in\BDset}\nlog_*(\bm{\beta}) - \nlog_*(\bm{\beta}_*)\bigr)/4.
\end{align*}
Therefore, $\InitSet_n\neq\emptySet$ holds.
Also, for each $\bm{\beta}\in\InitSet$ we have
\begin{align*}
\nlog_n(\bm{\beta})\leq&\  \nlog_*(\bm{\beta}) + \Mb_1\\
\leq&\ \min_{\bm{\beta}\in\BDset}\nlog_*(\bm{\beta})-\bigl( \min_{\bm{\beta}\in\BDset}\nlog_*(\bm{\beta}) - \nlog_*(\bm{\beta}_*)\bigr)/2 + \Mb_1\\
\leq&\ \min_{\bm{\beta}\in\BDset}\nlog_*(\bm{\beta})-\bigl( \min_{\bm{\beta}\in\BDset}\nlog_*(\bm{\beta}) - \nlog_*(\bm{\beta}_*)\bigr)/4.
\end{align*}
Therefore, $\InitSet\subseteq\InitSet_n$ holds. Accordingly, we have 
\begin{equation}\label{eq_finite_res1}
\sup_{\bm{\beta}\in \ThetaSet}\left|\nlog_n(\bm{\beta})-\nlog_*(\bm{\beta})\right|< \Mb_1 \implies \InitSet_n\neq\emptySet\text{ and }\InitSet\subseteq\InitSet_n.
\end{equation}

Second, for \eqref{eq_Mnbeta_strictly_convex}, we have
\begin{equation}\label{eq_finite_res2}
\bm{X} \text{ has full column rank}\implies \nlog_n(\bm{\beta})\text { is strictly convex}.
\end{equation}

Finally, for \eqref{eq_EProbgoesToOne},
we take $\Mb_2>0$ such that 
$$
\Mb_2\leq \frac{\lambda^{(\textrm{min})}_*}{\cst p},
$$
where $\cst = 8$.
Assume
$
\sup_{\bm{\beta}\in \ThetaSet}\left\|\nabla^2\nlog_n(\bm{\beta})-\nabla^2\nlog_*(\bm{\beta})\right\|_2< \Mb_2. 
$
By Lemma~\ref{lem_ctsEigen}, 
\begin{align*}
|\lambda_n^{(\textrm{min})} - \lambda_*^{(\textrm{min})}|\leq&\ p\cdot\Mb_2,\\
|\lambda_n^{(\textrm{max})} - \lambda_*^{(\textrm{max})}|\leq&\ p\cdot\Mb_2.
\end{align*}
Therefore, since we also have $\Mb_2\leq \lambda^{(\textrm{max})}_*/\cst p$, we have
\begin{align*}
|\eta_n - \eta_*|=&
\frac{1}{4} \biggl|\biggl(\frac{\lambda_n^{(\textrm{min})}}{\lambda_n^{(\textrm{max)}}}\biggr)^3 - \biggl(\frac{\lambda_*^{(\textrm{min})}}{\lambda_*^{(\textrm{max)}}}\biggr)^3\biggr|\\
=&\ \frac{1}{4} \biggl|\frac{(\lambda_n^{(\textrm{min})})^3 (\lambda_*^{(\textrm{max)}})^3 - (\lambda_*^{(\textrm{min})})^3 (\lambda_n^{(\textrm{max)}})^3}{(\lambda_n^{(\textrm{max)}})^3 (\lambda_*^{(\textrm{max)}})^3} \biggr|\\
\leq &\ 
\frac{1}{4} \frac{(\lambda_*^{(\textrm{min})} + p\Mb_2)^3 (\lambda_*^{(\textrm{max)}}+ p\Mb_2))^3 - (\lambda_*^{(\textrm{min})}- p\Mb_2))^3 (\lambda_*^{(\textrm{max)}}- p\Mb_2))^3}{(\lambda_*^{(\textrm{max)}}- p\Mb_2))^3 (\lambda_*^{(\textrm{max)}})^3} \\
\leq &\ 
\frac{1}{4} \frac{((\cst + 1)^6/\cst^6)(\lambda_*^{(\textrm{min})} )^3 (\lambda_*^{(\textrm{max)}})^3 - ((\cst-1)^6/\cst^6)(\lambda_*^{(\textrm{min})})^3 (\lambda_*^{(\textrm{max)}})^3}{((\cst-1)^3/\cst^3)(\lambda_*^{(\textrm{max)}})^6} \\
\leq&\ 0.03\frac{(\lambda_*^{(\textrm{min})})^3}{(\lambda_*^{(\textrm{max)}})^3}.
\end{align*}
Accordingly, we have 
$$
|\eta_n - \eta_*|  < \frac{1}{8}\frac{(\lambda_*^{(\textrm{min})})^3}{(\lambda_*^{(\textrm{max)}})^3}
=(1-\eta_*)/2
$$
and 
$$
\eta_*/2 - |\eta_n - \eta_*|\geq 1/2 -  0.155\frac{(\lambda_*^{(\textrm{min})})^3}{(\lambda_*^{(\textrm{max)}})^3} >0.
$$
Thus, 
\begin{equation}\label{eq_finite_res3}
\sup_{\bm{\beta}\in \ThetaSet}\left|\nabla^2\nlog_n(\bm{\beta})-\nabla^2\nlog_*(\bm{\beta})\right|< \Mb_2, \implies |\eta_n - \eta_*| <\epsilon_2\ \text{ and } |\lambda^{(\textrm{min})}_n-\lambda^{(\textrm{min})}_*|<\epsilon_3.
\end{equation}
Next, we take 
$$
\Mb_3 = \csthes_*/4.
$$
Assuming
$
\sup_{\bm{\beta}\in \ThetaSet}\left|\nlog_n(\bm{\beta})-\nlog_*(\bm{\beta})\right|< \Mb_3, 
$
we have
\begin{align*}
\csthes_n =&\ \sup_{\bm{\beta}^{(1)}, \bm{\beta}^{(2)}\in\ThetaSet}|\nlog_n(\bm{\beta}^{(1)}) - \nlog_n(\bm{\beta}^{(2)})|\\
=&\ \sup_{\bm{\beta}^{(1)}, \bm{\beta}^{(2)}\in\ThetaSet}|\nlog_n(\bm{\beta}^{(1)}) - \nlog_*(\bm{\beta}^{(1)}) - \nlog_n(\bm{\beta}^{(2)}) +\nlog_*(\bm{\beta}^{(2)}) + \nlog_*(\bm{\beta}^{(1)})- \nlog_*(\bm{\beta}^{(2)})|  \\ 
\leq&\  \sup_{\bm{\beta}^{(1)}, \bm{\beta}^{(2)}\in\ThetaSet}|\nlog_n(\bm{\beta}^{(1)}) - \nlog_*(\bm{\beta}^{(1)})| + \sup_{\bm{\beta}^{(1)}, \bm{\beta}^{(2)}\in\ThetaSet}|\nlog_n(\bm{\beta}^{(2)}) -\nlog_*(\bm{\beta}^{(2)})| + \\
&\sup_{\bm{\beta}^{(1)}, \bm{\beta}^{(2)}\in\ThetaSet}|\nlog_*(\bm{\beta}^{(1)})- \nlog_*(\bm{\beta}^{(2)})| \\
\leq&\ \sup_{\bm{\beta}^{(1)}, \bm{\beta}^{(2)}\in\ThetaSet}|\nlog_*(\bm{\beta}^{(1)})- \nlog_*(\bm{\beta}^{(2)})|  + 2\Mb_3\\
=&\ \csthes_* + 2\Mb_3.
\end{align*}
Since we can similarly show that $\csthes_* \leq \csthes_n + 2\Mb_3$, 
\begin{equation}\label{eq_finite_res4}
\sup_{\bm{\beta}\in \ThetaSet}\left|\nlog_n(\bm{\beta})-\nlog_*(\bm{\beta})\right|< \Mb_3 \implies |\csthes_n - \csthes_*|<\epsilon_4.
\end{equation}
We take $\Mb = \min(\Mb_1, \Mb_2, \Mb_3)$ and complete the proof by the results in \eqref{eq_finite_res1}, \eqref{eq_finite_res2}, \eqref{eq_finite_res3}, and \eqref{eq_finite_res4}.

\subsection{Remarks on the extension to a changing $p$}\label{subsec_changingp}
When $p$ is allowed to change with $n$ but smaller than $n$, the required modifications involve the proofs in Sections~\ref{subsec_initSet} and \ref{subsec_uniform_convergence_parameters}.

In Section~\ref{subsec_initSet}, the set $\ThetaSet\in\R^p$ will change as the dimension $p$ changes. Also, the difference between
$$\min_{\bm{\beta}\in\BDset}\nlog_*(\bm{\beta})-\bigl( \min_{\bm{\beta}\in\BDset}\nlog_*(\bm{\beta}) - \nlog_*(\bm{\beta}_*)\bigr)/4
$$
and 
$$
\min_{\bm{\beta}\in\BDset}\nlog_*(\bm{\beta})-\bigl( \min_{\bm{\beta}\in\BDset}\nlog_*(\bm{\beta}) - \nlog_*(\bm{\beta}_*)\bigr)/2 $$
needs to be bounded away from zero.  Additionally, the rates of 
$\sup_{\bm{\beta}\in \ThetaSet}\left|\nlog_n(\bm{\beta})-\nlog_*(\bm{\beta})\right|$ converging to zero and 
$\P(\bm{X}\text{ has full rank})$
converging to one
will be dependent on $p$.

In Section~\ref{subsec_uniform_convergence_parameters},
$\lambda^{(\textrm{min})}_*$, $\lambda^{(\textrm{max})}_*$, and $\csthes_*$ will be dependent on $p$. To guarantee the convergence property, we need $\lambda^{(\textrm{min})}_*$ to be bounded away from zero and $\lambda^{(\textrm{max})}_*$ and $\csthes_*$ to be upper bounded as $p$ changes.
Moreover, the rate $\cRate$ will depend on $p$.

\section{Privacy Protection of AE-AL testing}\label{sec_AEAL_local_priv}
We present the definition of $\epsilon$-local differential privacy
in Section~\ref{sec_localdiff_def} and prove Proposition~\ref{prop_local_diff_privacy} in Section~\ref{sec_prop1Proof}.

\subsection{Definition of $\epsilon$-local differential privacy}\label{sec_localdiff_def}

Let $\mathcal X\subseteq \R^{p}$ denote the support of $X$ and $\epsilon$ be a positive constant.
The definition of $\epsilon$-local differential privacy \citep{kasiviswanathan2011can, duchi2013local, duchi2018minimax} is as follows:
\begin{definition}
Let a privacy mechanism be a set of conditional distributions $\P(X^{(\text {priv})}\mid X^{(\mathrm{raw})} = \bm{x}^{(\mathrm{raw})})$ for all $\bm{x}^{(\mathrm{raw})}\in\mathcal X$ that transforms  $\bm{x}^{(\mathrm{raw})}$ to privatized data $\bm{x}^{(\text {priv})}$.
We say that 
$\bm{x}^{(\text {priv})}$ is a $\epsilon$-local differential private view of  $\bm{x}^{(\mathrm{raw})}$ if for all $\bm{x}^{(\text {raw},1)}, \bm{x}^{(\text {raw},2)}\in\mathcal X$,
$$
\forall S \in \sigma\bigl(X^{(\text {priv})}\bigr): \frac{\P(X^{(\text {priv})}\in S\mid X^{(\mathrm{raw})} = \bm{x}^{(\mathrm{raw,1})})}{\P(X^{(\text {priv})}\in S\mid X^{(\mathrm{raw})} = \bm{x}^{(\mathrm{raw,2})})} \leq \exp(\epsilon)
,$$
 where $ \sigma\bigl(X^{(\text {priv})}\bigr)$ is a $\sigma$-filed on $X^{(\text {priv})}$.
\end{definition}

\subsection{Proof of Proposition~\ref{prop_local_diff_privacy}}\label{sec_prop1Proof}
Let $\noiseMat$ be a $\id$-dimensional random vector with its elements i.i.d.\ from the Laplace distribution with mean zero and scale parameter $2\id\cdot\cstt/\epsilon$. Thus, the rows of the matrix $\bm{\noiseMat}$ can be considered as i.i.d.\ sample from the distribution of $\noiseMat$.
The idea of the proof is to consider $\noiseMat$ as an additional covariate held by \B  and following a proof similar to the one of Theorem~\ref{thm_Wn}. In this way, we are replacing the old assisted learning problem with
\begin{equation}\label{eq_originalProb}
    \nlog_*(\bm{\beta}) = \E\bigl(m(Y,X^\T\bm{\beta})\bigr),\quad \nlog_*^{(\bm{U})}(\bm{\beta}^{(\bm{U})})
=\E\bigl(\mest\bigl(Y,{X^{(\bm{U})}}^\T \bm{\beta}^{(\bm{U})} \bigr)\mid\bm{U}\bigr),
\end{equation}
by a new assisted learning problem with the two losses replaced by
\begin{equation}\label{eq_new_Prob}
\nlog_*^{(\noise)}(\bm{\beta}^{(\noise)})\de \E\bigl(m(Y,(X^\T,\noiseMat^\T )\bm{\beta}^{(\noise)})\bigr), \quad \nlog_*^{(\bm{U}')}(\bm{\beta}^{(\bm{U}')})\de \E\bigl(\mest\bigl(Y,{X^{(\bm{U}')}}^\T \bm{\beta}^{(\bm{U}')} \bigr)\mid\bm{U}\bigr),
\end{equation}
respectively, where 
$X^{(\bm{U}')}\de (X^\T, \noiseMat^\T)\MUp$,
\begin{equation*}
\MUp \de 
\begin{pmatrix}
\bm{I}_{p_a\times p_a} & \bm{0}& \bm{0}\\
\bm{0}& \bm{0} &\bm{U}^{(p_b)}\\
\bm{0}&\bm{I}_{p_c\times p_c}&\bm{U}^{(p_c)}\\
\bm{0} & \bm{0} & \bm{I}_{\id\times\id}
\end{pmatrix},
\end{equation*}
$\bm{I}_{\id\times\id}$ is an $\id\times\id$ identity matrix, and $\bm{\beta}^{(\bm{U}')}$ is the coefficient vector of $X^{(\bm{U}')}$.

For $\textrm{H}_0$, recall that we have $\bm{\beta}^{(b)}_* = \bm{0}$ 
in the original assisted learning problem~\eqref{eq_originalProb}.  To generalize the proof to the new assisted learning problem~\eqref{eq_new_Prob}, it suffices to show that the components in $\bm{\beta}^{(\noise)}_*\de \argmin_{\bm{\beta}^{(\noise)}\in\mathbb R^{p+\id}}\E\bigl(m(Y,(X^\T,\noiseMat^\T) \bm{\beta}^{(\noise)})\bigr)$ that correspond to $X^{(\B)}$ and $\noiseMat$ are all $0$. 
Let $\bm{\beta}^{(\noiseMat)}$ denote the coefficients of $\bm{\beta}^{(\noise)}$ that correspond to $\noiseMat$.
By our assumption, the function $\lp \mapsto \mest(Y, \lp)$ is almost surely convex conditional on $Y$. Therefore, it can be shown that $\E\bigl(m(Y,X^\T\bm{\beta} + \noiseMat^\T \bm{\beta}^{(\noiseMat)})\mid Y, X\bigr)$ is almost surely a convex function of $\noiseMat$ for each fixed $\bm{\beta}^{(\noise)} = (\bm{\beta}, \bm{\beta}^{(\noiseMat)})$. Also, by the independence between $\noiseMat$ and $(Y,X)$, we have $\E(\noiseMat\mid Y, X) = \E(\noiseMat) = \bm{0}$. 
Therefore, by Jensen's inequality,
\begin{align*}
    \E\bigl(m(Y,X^\T\bm{\beta} + \noiseMat^\T \bm{\beta}^{(\noiseMat)})\bigr) = &\
\E\bigl(\E\bigl(m(Y,X^\T\bm{\beta} + \noiseMat^\T \bm{\beta}^{(\noiseMat)})\mid Y, X\bigr) \bigr)\\
\geq &\  \E\bigl(m(Y,X^\T\bm{\beta} + \E\bigl(\noiseMat^\T \bm{\beta}^{(\noiseMat)}\mid Y, X)\bigr) \bigr)\\=&\ \E\bigl(m(Y,X^\T\bm{\beta} )\bigr).
\end{align*}
Thus, $\bm{\beta}^{(\noise)}_* = ({\bm{\beta}^{(a)}_*}^\T,{\bm{\beta}^{(b)}_*}^\T,{\bm{\beta}^{(c)}_*}^\T,\bm{0}^\T)^\T$ is a minimizer of $ \E\bigl(m(Y,X^\T\bm{\beta} + \noiseMat^\T \bm{\beta}^{(\noiseMat)})\bigr)$, where recall that $({\bm{\beta}^{(a)}_*}^\T,{\bm{\beta}^{(b)}_*}^\T,{\bm{\beta}^{(c)}_*}^\T)^\T$ is the minimizer of $\nlog_*(\bm{\beta})$ from the original assisted learning problem~\eqref{eq_originalProb}. Together with the strict convexity of $\mest(Y, \lp)$ with respect to $\lp$, it can be shown that
$$\bm{\beta}^{(\noiseMat)}_* = ({\bm{\beta}^{(a)}_*}^\T,{\bm{\beta}^{(b)}_*}^\T,{\bm{\beta}^{(c)}_*}^\T,\bm{0}^\T)^\T = ({\bm{\beta}^{(a)}_*}^\T,\bm{0}^\T,{\bm{\beta}^{(c)}_*}^\T,\bm{0}^\T)^\T $$
is the unique minimizer.  Accordingly, the components of 
$\bm{\beta}^{(\noiseMat)}_*$ corresponding to $X^{(\B)}$ and $\noise$ are all $0$. 

For $\textrm{H}_1$,
let $\bm{\beta}^{(\bm{U}')}_* = ({\bm{\beta}^{(\bm{U}',\A)}_*}^\T, {\bm{\beta}^{(\bm{U}',\id)}_*}^\T)^\T$ be the minimizer of $\nlog_*^{(\bm{U}')}(\bm{\beta}^{(\bm{U}')})$ where let $\bm{\beta}^{(\bm{U}',\id)}_*$  and $\bm{\beta}^{(\bm{U}',\id)}_*$ are the coefficients corresponding to $X^{(\A)}$ and $X^{(\bm{U}')}$, respectively. 
We follow a proof similar to the one of Theorem~\ref{thm_Wn} in Section~\ref{sec_H1} by showing $\bm{\beta}^{(\bm{U}',\id)}_*$, which is the coefficient in $\bm{\beta}^{(\bm{U}')}_*$ that corresponds to the sketchy data, is not $\bm{0}$. 
To prove this, we proceed by contradiction, with the following outline:
\begin{enumerate}
\item When 
$\bm{\beta}^{(\bm{U}',\id)}_*=\bm{0}$, we first show that $\nabla\nlog_*^{(\noise)}(\MUp \bm{\beta}^{(\bm{U}')}_*) \neq \bm{0}$.
\item Based on the above result, we show that $\nabla \nlog^{(\bm{U}')}_*( \bm{\beta}^{(\bm{U}')}_* )\neq \bm{0}$, which contradicts with the definition of $\bm{\beta}^{(\bm{U}')}_*$.
\end{enumerate}

When 
$\bm{\beta}^{(\bm{U}',\id)}_*=\bm{0}$, we have 
\begin{equation}\label{eq_betaUprimeEq}
\MUp \bm{\beta}^{(\bm{U}')}_* =
\MUp
\begin{pmatrix}
\bm{\beta}^{(\bm{U}',a)}_*\\
\bm{\beta}^{(\bm{U}',c)}_*\\
\bm{0}_{\id}
\end{pmatrix}
=
\begin{pmatrix}
\bm{\beta}^{(\bm{U}',a)}_*\\
\bm{0}_{p_b+\id}\\
\bm{\beta}^{(\bm{U}',c)}_*
\end{pmatrix},
\end{equation}
where $\bm{\beta}^{(\bm{U}',a)}_*$ and  $\bm{\beta}^{(\bm{U}',c)}_*$ correspond to $\bm{X}^{(a)}$ and $\bm{X}^{(c)}$, respectively, and the subscripts of $\bm{0}$ denote the number of zeros in the vector.
Recall that under $\textrm{H}_1$, $\bm{\beta}^{(b)}_* \neq \bm{0}$, where $ \bm{\beta}_*=({\bm{\beta}^{(a)}_*}^\T,{\bm{\beta}^{(b)}_*}^\T,{\bm{\beta}^{(c)}_*}^\T)^\T\de \underset{\bm{\beta}\in\mathbb R^{p}}{\argmin}\  \nlog_*(\bm{\beta})$.
Therefore, 
$$
\nlog_*(\bm{\beta}_*) <\nlog_*\bigl(({\bm{\beta}^{(\bm{U}',a)}_*}^\T, \bm{0}_{p_b}^\T, {\bm{\beta}^{(\bm{U}',c)}_*}^\T)^\T). 
$$
Since 
\begin{align*}
\nlog_*(\bm{\beta}_*)  =&\  \nlog_*^{(\noise)}\bigl(({\bm{\beta}_*}^\T,\bm{0}_{\id}^\T)^\T\bigr),\\
\nlog_*\bigl(({\bm{\beta}^{(\bm{U}',a)}_*}^\T, \bm{0}_{p_b}^\T, {\bm{\beta}^{(\bm{U}',c)}_*}^\T)^\T) =&\ \nlog_*^{(\noise)}\bigl(({\bm{\beta}^{(\bm{U}',a)}_*}^\T, \bm{0}_{p_b + \id}^\T, {\bm{\beta}^{(\bm{U}',c)}_*}^\T)^\T) = \nlog_*^{(\noise)}\bigl(\MUp \bm{\beta}^{(\bm{U}')}_*\bigr),
\end{align*}
we have 
$$
\nlog_*^{(\noise)}\bigl(({\bm{\beta}_*}^\T,\bm{0}_{\id}^\T)^\T\bigr)< \nlog_*^{(\noise)}\bigl(\MUp \bm{\beta}^{(\bm{U}')}_*\bigr). 
$$
Thus, $\MUp\bm{\beta}^{(\bm{U}')}_*$ is not the minimizer of $\nlog_*^{(\noise)}(\bm{\beta}^{(\noise)})$ and $\nabla\nlog_*^{(\noise)}(\MUp \bm{\beta}^{(\bm{U}')}_*) \neq \bm{0}$. Additionally, 
we have
\begin{align*}
&\nabla \nlog^{(\bm{U}')}_*( \bm{\beta}^{(\bm{U}')}_*  ) \\
=& {\MUp}^\T\nabla \nlog_*^{(\noise)}(\MUp \bm{\beta}^{(\bm{U}')}_*) \\
=&
\begin{pmatrix}
	\nabla_{\bm{\beta}^{(a)}}  \nlog_*^{(\noise)}(\MUp \bm{\beta}^{(\bm{U}')}_*) \\
		\nabla_{\bm{\beta}^{(c)}}   \nlog_*^{(\noise)}(\MUp \bm{\beta}^{(\bm{U}')}_*) \\
{(\bm{U}}^{(p_b)})^\T \nabla_{\bm{\beta}^{(b)}}   \nlog_*^{(\noise)}(\MUp \bm{\beta}^{(\bm{U}')}_*) +
{(\bm{U}}^{(p_c)})^\T \nabla_{\bm{\beta}^{(c)}}   \nlog_*^{(\noise)}(\MUp \bm{\beta}^{(\bm{U}')}_*) +
\nabla_{\bm{\beta}^{(\noiseMat)}}   \nlog_*^{(\noise)}(\MUp \bm{\beta}^{(\bm{U}')}_*)
\end{pmatrix}.
\end{align*}
When $\nabla_{\bm{\beta}^{(a)}}  \nlog_*^{(\noise)}(\MUp \bm{\beta}^{(\bm{U}')}_*)\neq \bm{0}$ or $\nabla_{\bm{\beta}^{(a)}}  \nlog_*^{(\noise)}(\MUp \bm{\beta}^{(\bm{U}')}_*)\neq \bm{0}$, we obtain a contraction with the definition of $\bm{\beta}^{(\bm{U}')}_*$. Otherwise, we have 
\begin{align*}
&\nabla \nlog^{(\bm{U}')}_*( \bm{\beta}^{(\bm{U}')}_*  ) \\
=&
\begin{pmatrix}
	\bm{0}\\
		\bm{0} \\
  \begin{pmatrix}
  {\bm{U}}^{(p_b)}\\
  \bm{I}_{\id\times\id}
  \end{pmatrix}^\T
  \begin{pmatrix}
  \nabla_{\bm{\beta}^{(b)}}   \nlog_*^{(\noise)}(\MUp \bm{\beta}^{(\bm{U}')}_*) \\
  \nabla_{\bm{\beta}^{(\noiseMat)}}   \nlog_*^{(\noise)}(\MUp \bm{\beta}^{(\bm{U}')}_*)
  \end{pmatrix}
\end{pmatrix}.
\end{align*}
By Equation~\eqref{eq_betaUprimeEq}, the vector $\bigl((\nabla_{\bm{\beta}^{(b)}}   \nlog_*^{(\noise)}(\MUp \bm{\beta}^{(\bm{U}')}_*))^\T, (\nabla_{\bm{\beta}^{(\noiseMat)}}   \nlog_*^{(\noise)}(\MUp \bm{\beta}^{(\bm{U}')}_*))^\T\bigr)^\T$ does not depend on $\bm{U}'$. Together with
\begin{equation*}
\P\biggl(\exists \bm{z}\in\mathbb R^{p_b + \id}, \text{ s.t. }\bm{z} \neq \bm{0}\text{ and }\bigl({\bigl(\bm{U}}^{(p_b)}\bigr)^\T,\bm{I}_{\id\times\id}\bigr)\bm{z} = \bm{0}\biggr) = 0,
\end{equation*} 
we have 
$$
\begin{pmatrix}
	\bm{0}\\
		\bm{0} \\
  \begin{pmatrix}
  {\bm{U}}^{(p_b)}\\
  \bm{I}_{\id\times\id}
  \end{pmatrix}^\T
  \begin{pmatrix}
  \nabla_{\bm{\beta}^{(b)}}   \nlog_*^{(\noise)}(\MUp \bm{\beta}^{(\bm{U}')}_*) \\
  \nabla_{\bm{\beta}^{(\noiseMat)}}   \nlog_*^{(\noise)}(\MUp \bm{\beta}^{(\bm{U}')}_*)
  \end{pmatrix}
\end{pmatrix}\neq \bm{0}
$$
holds almost surely. Thus, we obtain a contradiction and
complete the proof.

\section{Proof of 
Corollary~\ref{prop_suffGLM}}\label{sec_ProofOfCo1}
The sketch of the proof is as follows. First, we prove Condition~\ref{cond_sUconv} by Lemmas~\ref{lem_bddExp} and \ref{lem_sUconv},  where we take $\ThetaSet$ to be $\ThetaSett$. Second, we show Condition~\ref{cond_exchInteDiff} and the differentiability of $\nlog_*^{(\bm{U})}(\bm{\beta}^{(\bm{U})})$ by Lemmas~\ref{lem_exch} and \ref{lem_diffbetaU}.
Third, we obtain Conditions~\ref{A2} and \ref{cond_oraExist} by Lemmas~\ref{lem_snConvex}-\ref{lem_Betaconv}. Fourth, we prove the result of Condition~\ref{A3} in Lemma~\ref{lem_proofOfCond4}. Fourth, we obtain the result of Condition~\ref{cond_AsymNorm} in Lemma~\ref{lem_proofOfCond10}.
Finally, we show Condition~\ref{cond_finiteSample} in Lemma~\ref{lem_additional_assumptions}.


Recall that $\|\cdot\|_{\infty}$ denotes the sup norm of a vector or matrix.
\begin{lemma}[Bounded expectations]\label{lem_bddExp}
By Conditions~\ref{dc} and 
\ref{cond_Fbounded}, we have
\begin{align}
  \E\biggl(\sup_{\bm{\beta}\in  \ThetaSet}\biggl\|\log f(Y\mid X^\T\bm{\beta}) \biggr\|_{\infty}\biggr) <& \infty,\label{eq_boundDeri0}\\
    \E\biggl(\sup_{\bm{\beta}\in  \ThetaSet}\biggl\| \nabla_{\bm{\beta}} \log f(Y\mid X^\T\bm{\beta}) \biggr\|_{\infty}\biggr) <& \infty,\label{eq_boundDeri1}\\
    \E\biggl(\sup_{\bm{\beta}\in  \ThetaSet}\biggl\|\nabla_{\bm{\beta}}^2  \log f(Y\mid X^\T\bm{\beta}) \biggr\|_{\infty} \biggr) <& \infty, \label{eq_boundDeri2}
\end{align}
where $\nabla_{\bm{\beta}}$ and $\nabla_{\bm{\beta}}^2$ denote the gradient and Hessian matrix with respect to $\bm{\beta}$, respectively.
\end{lemma}
\begin{proof}
We first prove Inequality~\eqref{eq_boundDeri2}, followed by Inequality~\eqref{eq_boundDeri1}, and finally, Inequality~\eqref{eq_boundDeri0}.
Fix an arbitrary $\bm{\beta}'\in\ThetaSet$. By the triangle inequality, the mean value theorem,  and the Cauchy-Schwarz inequality, we have
\begin{align*}
\sup_{\bm{\beta}'\in  \ThetaSet}\biggl\|\nabla_{\bm{\beta}}^2  \log f(y\mid \bm{x}^\T\bm{\beta}')\biggr\|_{\infty}
\leq&\
    \sup_{\bm{\beta}'\in  \ThetaSet}\biggl\|\nabla_{\bm{\beta}}^2  \log f(y\mid \bm{x}^\T\bm{\beta}') - \nabla_{\bm{\beta}}^2  \log f(y\mid \bm{x}^\T\bm{\beta}_*)\biggr\|_{\infty} +\\ &\biggl\|\nabla_{\bm{\beta}}^2  \log f(y\mid \bm{x}^\T\bm{\beta}_*)\biggr\|_{\infty}\\
    \leq&\  \sqrt{p}\sup_{\bm{\beta}'\in  \ThetaSet}\biggl\| \nabla_{\bm{\beta}}^3 \log f(y\mid \bm{x}^\T\bm{\beta}) \biggr\|_{\infty}\cdot \sup_{\bm{\beta}'\in  \ThetaSet}\|\bm{\beta}' - \bm{\beta}_*\|_2 +\\
    &\biggl\|\nabla_{\bm{\beta}}^2  \log f(y\mid \bm{x}^\T\bm{\beta}_*)\biggr\|_{\infty},
\end{align*}
where $\nabla_{\bm{\beta}}^3$ denotes the tensor consisting of the third order derivatives with respect to $\bm{\beta}$.
By Conditions~\ref{dc} and 
\ref{cond_Fbounded}
and the boundedness of $\ThetaSet$, the right-hand side of the above inequality has a finite expectation. Therefore, we obtain Inequality~\eqref{eq_boundDeri2}. 

For Inequality~\eqref{eq_boundDeri1}, we can similarly show 
\begin{align*}
\sup_{\bm{\beta}'\in  \ThetaSet}\biggl\|\nabla_{\bm{\beta}}  \log f(y\mid \bm{x}^\T\bm{\beta}')\biggr\|_{\infty}
\leq&  \sqrt{p}\sup_{\bm{\beta}'\in  \ThetaSet}\biggl\| \nabla_{\bm{\beta}}^2 \log f(y\mid \bm{x}^\T\bm{\beta}) \biggr\|_{\infty}\cdot \sup_{\bm{\beta}'\in  \ThetaSet}\|\bm{\beta}' - \bm{\beta}_*\|_2 +\\
    &\biggl\|\nabla_{\bm{\beta}}  \log f(y\mid \bm{x}^\T\bm{\beta}_*)\biggr\|_{\infty}.
\end{align*}
We obtain Inequality~\eqref{eq_boundDeri1} by combining the above inequality, 
Condition~\ref{cond_Fbounded}, and Inequality~\eqref{eq_boundDeri2}.

Inequality~\eqref{eq_boundDeri0} follows from a similar proof given Inequality~\eqref{eq_boundDeri1}.
\end{proof}

 \begin{lemma}[Uniform convergence of the log-likelihood and its derivatives]\label{lem_sUconv}
 Under Conditions~\ref{dc} and 
 \ref{cond_Fbounded},
 \begin{align}
  \sup_{\bm{\beta}\in \ThetaSet}\left|\nlog_n(\bm{\beta})-\nlog_*(\bm{\beta})\right|  \limp& 0,\label{eq_snConv1}\\
  \sup_{\bm{\beta}\in \ThetaSet}\left\|\nabla\nlog_n(\bm{\beta})-\nabla\nlog_*(\bm{\beta})\right\|_{\infty}  \limp& 0\label{eq_snConv2},\\
     \sup_{\bm{\beta}\in \ThetaSet}\left\|\nabla^2\nlog_n(\bm{\beta})-\nabla^2\nlog_*(\bm{\beta})\right\|_{\infty}  \limp& 0,\label{eq_snConv3}
 \end{align}
as $n\rightarrow\infty$.
 \end{lemma}
\begin{proof}
The above results follow from Lemma~\ref{lem_bddExp} and the uniform law of large numbers.
 
 \end{proof}

\begin{lemma}[Exchange of integrations and differentiations]\label{lem_exch}
By Conditions~\ref{dc} and 
\ref{cond_Fbounded}, for each $\bm{\beta}\in\ThetaSet$,
\begin{align}
   \nabla \nlog_*(\bm{\beta})&= -\E\biggl(\nabla_{\bm{\beta}}\log f (Y \mid X^\T \bm{\beta})\biggr),\label{eq_exch1}\\
    \nabla^2 \nlog_*(\bm{\beta})&= -\E\biggl(\nabla_{\bm{\beta}}^2\log f (Y \mid X^\T \bm{\beta})\biggr),\label{eq_exch2}
\end{align}
where recall that $\nlog_*(\bm{\beta})\de  -\E(\log f (y \mid {\bm{x}}^\T \bm{\beta}))$.
\end{lemma}
\begin{proof}
We first show Equality~\eqref{eq_exch1}.
By the mean value theorem and the Cauchy-Schwarz inequality,
$$\bigl|\log f (y \mid {\bm{x}}^\T \bm{\beta}) - \log f (y \mid {\bm{x}}^\T (\bm{\beta} + \bm{h})\bigr|/\|\bm{h}\|_2\leq \sup_{\bm{\beta}\in\ThetaSet} \biggl\|\nabla_{\bm{\beta}}\log f (y \mid {\bm{x}}^\T \bm{\beta})\biggr\|_2,$$
for all $\bm{\beta}$ and $\bm{\beta} + \bm{h}$ in the interior of $\ThetaSet$.
 Thus, Equation~\eqref{eq_exch1} follows from Inequality~\eqref{eq_boundDeri1} and the dominated convergence theorem. 

For Equality~\eqref{eq_exch2}, it suffices to show that
$$\nabla_{\bm{\beta}}  \E\biggl(\nabla_{\bm{\beta}} \log f (Y \mid X^\T \bm{\beta})\biggr)= \E\biggl(\nabla_{\bm{\beta}}^2\log f (Y \mid X^\T \bm{\beta})\biggr),$$
which can be obtained by Inequality~\eqref{eq_boundDeri2}, mean value theorem, and the dominated convergence theorem.

\end{proof}
\begin{lemma}[Differentiability of $\nabla\nlog_*^{(\bm{U})}(\bm{\beta}^{(\bm{U})})$]\label{lem_diffbetaU}
By Condition~\ref{A1}, $\nlog_*^{(\bm{U})}(\bm{\beta}^{(\bm{U})})$ is differentiable for all $\bm{\beta}^{(\bm{U})}\in \R^{p_\A + \id}$. Together with Conditions~ \ref{dc}, 
\ref{cond_Fbounded},
and \ref{lem_sbb_inm}, we have
\begin{align}    \nabla\nlog_*^{(\bm{U})}(\bm{\beta}^{(\bm{U})})&=-\E\biggl(\nabla_{\bm{\beta}^{(\bm{U})}}\log f\bigl(Y, {X^{(\bm{U})}}^\T\bm{\beta}^{(\bm{U})})\mid\bm{U}\biggr),\ \forall \bm{\beta}^{(\bm{U})}\in \ThetaSetU, \label{eq_exch3}
\end{align}
where $\ThetaSetU\de \{\bm{\beta}\in\R^p : \|\bm{\beta} - \MU\bm{\beta}^{(\bm{U})}_*\|_2\leq \cst\}$, which is a subset of $\ThetaSet$,  and recall that $\nlog_*^{(\bm{U})}(\bm{\beta}^{(\bm{U})})
\de -\E\bigl(\log f\bigl(Y, {X^{(\bm{U})}}^\T \bm{\beta}^{(\bm{U})} \bigr)\mid\bm{U}\bigr)$.
\end{lemma}
\begin{proof}
Recall that 
$X^{(\bm{U})} = X^\T\MU$, where $\MU$ was defined in \eqref{Mdef}.
The differentiability of  $\nabla\nlog_*^{(\bm{U})}(\bm{\beta}^{(\bm{U})})$ is obtained from the differentiability of $M_*(\bm{\beta})$ and the fact that $$\nabla \nlog^{(\bm{U})}_*( \bm{\beta}^{(\bm{U})}_*  ) 
= {\MU}^\T\nabla \nlog_*(\MU \bm{\beta}^{(\bm{U})}_*).$$

For Equality~\eqref{eq_exch3}, by Inequality~\eqref{eq_boundDeri2} and Condition~\ref{lem_sbb_inm}, we have 
$$
-\E\biggl(\nabla_{\bm{\beta}}\log f\bigl(Y,X^\T\MU\bm{\beta}^{(\bm{U})})\mid\bm{U}\biggr)<\infty,\ \text{ almost surely }\forall \bm{\beta}^{(\bm{U})}\in\ThetaSetU.
$$
Therefore,
\begin{align*}
   -\E\biggl(\nabla_{\bm{\beta}^{(\bm{U})}}\log f\bigl(Y, {X^{(\bm{U})}}^\T\bm{\beta}^{(\bm{U})})\mid\bm{U}\biggr)
    =&\ -\E\biggl({\MU}^\T\nabla_{\bm{\beta}}\log f\bigl(Y,X^\T\MU\bm{\beta}^{(\bm{U})})\mid\bm{U}\biggr)\\
    =&\
    {\MU}^\T \nabla \nlog_*\bigl(\MU\bm{\beta}^{(\bm{U})}\bigr)\\
    =&\     \nabla \nlog_*^{(\bm{U})}(\bm{\beta}^{(\bm{U})}).
\end{align*}
\end{proof}

\begin{lemma}[Convexity of $\nlog_n(\bm{\beta})$]\label{lem_snConvex}
By Conditions~\ref{cond_designMatFullrank} and \ref{cond_strictConvexloss}, $\nlog_n(\bm{\beta})$ is strictly convex with probability going to one as $n\rightarrow\infty$. 
\end{lemma}
\begin{proof}
This result is obtained from the proof of \eqref{eq_Mnbeta_strictly_convex}.
\end{proof}

\begin{lemma}[Convexity of  $\nlog^{(\bm{U})}_n(\bm{\beta}^{(\bm{U})})$]\label{lem_snAsnUconvex}
By Conditions~\ref{cond_designMatFullrank} and \ref{cond_strictConvexloss},  $\nlog^{(\bm{U})}_n\allowbreak(\bm{\beta}^{(\bm{U})})$ is strictly convex with probability going to one. 
\end{lemma}
\begin{proof}
We have 
\begin{align*}
    \nabla^2\nlog^{(\bm{U})}_n(\bm{\beta}^{(\bm{U})}) =& {\MU}^\T\nabla^2\nlog_n(\MU \bm{\beta}^{(\bm{U})}) \MU.
\end{align*}
Since the elements of $\bm{U}$ are sampled from a continuous distribution, it can be shown that $\MU$ has full column rank almost surely. Together with Lemma~\ref{lem_snConvex}, we obtain the result of $\nlog^{(\bm{U})}_n(\bm{\beta}^{(\bm{U})})$. 

\end{proof}

\begin{lemma}[Existence and convergence of $\hat{\bm{\beta}}^{(\bm{U})}_n$ and $\check{\bm{\beta}}_n$]\label{lem_Betaconv}
Under Conditions~\ref{cond_designMatFullrank}, \ref{cond_strictConvexloss}, and \ref{dc}, $\hat{\bm{\beta}}^{(\bm{U})}_n$ and $\check{\bm{\beta}}_n$ uniquely exist with probability going to one as $n\rightarrow\infty$. Additionally, 
\begin{align}
    \|\hat{\bm{\beta}}^{(\bm{U})}_n - \bm{\beta}^{(\bm{U})}_*\|_2\limp & 0,\label{eq_betaUconv}\\
    \|\check{\bm{\beta}}_n - \bm{\beta}_*\|_2\limp & 0, \label{eq_betaAconv}
\end{align}
as $n\rightarrow\infty$.
\end{lemma}
\begin{proof}
We first present the proof for $\hat{\bm{\beta}}^{(\bm{U})}_n$. 
Recall that 
$$
 \ThetaSet^{(\bm{U})}\de
\{\bm{\beta}' \in\mathbb R^{p_\A+\id}: \MU\bm{\beta}'\in\ThetaSet \}.
$$
Let
\begin{align}
\ThetaSet^{(\bm{U},a)}\de& \{\bm{\beta}'\in\R^{p_\A + \id}: \|\MU\bm{\beta}' - \MU\bm{\beta}^{(\bm{U})}_*\|_2= a\},\nonumber
\end{align}
with a positive and sufficiently small $a$ such that  $\ThetaSet^{(\bm{U},a)}\subset \ThetaSet^{(\bm{U})}$ holds almost surely. We first show that for each  $\epsilon>0$, there exists $N_\epsilon$, such that when $n\geq N_\epsilon$, $$\P\left(\nlog_n^{(\bm{U})}(\bm{\beta}') - \nlog_n^{(\bm{U})}(\bm{\beta}^{(\bm{U})}_*)>0,\ \forall \bm{\beta}'\in \ThetaSet^{(\bm{U},a)}\right)>1-\epsilon,$$
which indicates that there exists a local minimum of $\nlog_n^{(\bm{U})}(\bm{\beta}^{(\bm{U})})$  in 
$$
\Omega^{(\bm{U},a)}\de 
\biggl\{\bm{\beta}'\in\R^{p_\A + \id}: \|\MU\bm{\beta}' - \MU\bm{\beta}^{(\bm{U})}_*\|_2< a\biggr\}.$$
By the Taylor expansion,
$$
\nlog_n^{(\bm{U})}(\bm{\beta}') - \nlog_n^{(\bm{U})}(\bm{\beta}^{(\bm{U})}_*)\geq T_1 + T_2 - T_3,
$$
where
\begin{align*}
T_1=& (\bm{\beta}' - \bm{\beta}^{(\bm{U})}_*)^\T \nabla\nlog_n^{(\bm{U})}(\bm{\beta}^{(\bm{U})}_*),\\
T_2 = &\frac{1}{2}(\bm{\beta}' - \bm{\beta}^{(\bm{U})}_*)^\T \nabla^2 \nlog_n^{(\bm{U})}(\bm{\beta}^{(\bm{U})}_*)(\bm{\beta}' - \bm{\beta}^{(\bm{U})}_*),\\
T_3 = &\frac{1}{6}p^3\|\bm{\beta}' - \bm{\beta}^{(\bm{U})}_*\|_2^3 \sup_{\bm{\beta}^{(\bm{U})}\in\Omega^{(\bm{U},a)}}\biggl\|\nabla^3\nlog_n^{(\bm{U})}(\bm{\beta}^{(\bm{U})})\biggr\|_{\infty},
\end{align*}
and $\biggl\|\nabla^3\nlog_n^{(\bm{U})}(\bm{\beta}^{(\bm{U})})\biggr\|_{\infty}$ denotes the sup norm of the 3-D tensor $\nabla^3\nlog_n^{(\bm{U})}(\bm{\beta}^{(\bm{U})})$.
Next, we show that for each $a$, $T_1\limp 0$, $T_2/a^2$ is lower bounded by a positive constant, and $T_3/a^3$ is upper bounded by a positive constant with probability going to one as  $n\rightarrow\infty$.

First, by Equality~\eqref{eq_exch3}, 
\begin{align*}
    \nabla \nlog_*^{(\bm{U})}(\bm{\beta}^{(\bm{U})}) 
=& -\E\biggl(\nabla_{\bm{\beta}^{(\bm{U})}} \log f(Y\mid  {X^{(\bm{U})}}^\T\bm{\beta}^{(\bm{U})})\mid\bm{U}\biggr) \\
=&-\E\biggl({\MU}^\T\nabla_{\bm{\beta}}\log f(Y\mid X^\T\MU\bm{\beta}^{(\bm{U})})\mid\bm{U}\biggr)\\
=&-{\MU}^\T\E\biggl(\nabla_{\bm{\beta}}\log f(Y\mid X^\T\MU\bm{\beta}^{(\bm{U})})\mid\bm{U}\biggr),
\end{align*}
where recall that $\nabla_{\bm{\beta}^{(\bm{U})}}$ and $\nabla_{\bm{\beta}}$ denote the derivatives of $\bm{\beta}^{(\bm{U})}$ and $\MU\bm{\beta}^{(\bm{U})}$, respectively, and $\nlog_*^{(\bm{U})}(\bm{\beta}^{(\bm{U})})
\de -\E\bigl(\log f (Y, {X^{(\bm{U})}}^\T\bm{\beta}^{(\bm{U})} \bigr)\mid\bm{U}\bigr)$.
Since $\nabla \nlog_*^{(\bm{U})}(\bm{\beta}^{(\bm{U})}_*) = {\MU}^\T\nabla \nlog_*(\MU\allowbreak\bm{\beta}^{(\bm{U})}_*)$, by Condition~\ref{lem_sbb_inm} and Lemma~\ref{lem_sUconv}, $\bigl\|\nabla \nlog_*^{(\bm{U})}(\bm{\beta}^{(\bm{U})}_*)- \nabla \nlog_n^{(\bm{U})}(\bm{\beta}^{(\bm{U})}_*)\bigr\|_{\infty}\limp 0$.
Together with \eqref{eq_snConv2} and $\nabla \nlog_*^{(\bm{U})}(\bm{\beta}^{(\bm{U})}_*) =\bm{0}$, we have 
$$\nabla \nlog_n^{(\bm{U})}(\bm{\beta}^{(\bm{U})}_*) \limp\bm{0},
$$ as $n\rightarrow\infty$. Therefore, 
$T_1\limp 0$ uniformly for all $\bm{\beta}'\in \ThetaSet^{(\bm{U},a)}$ as $n\rightarrow\infty$.

Second, 
according to Equality~\eqref{eq_exch2}, we have
$$
\nabla^2 \nlog_*^{(\bm{U})}(\bm{\beta}^{(\bm{U})}) = {\MU}^\T\nabla^2\nlog_*\biggl(\MU\bm{\beta}^{(\bm{U})}\biggr) \MU.
$$
Therefore, by  the fact that $\MU$ has full column rank almost surely, $T_2/a^2$ is lower bounded for all $\bm{\beta}'\in \ThetaSet^{(\bm{U},a)}$ with probability going to one as $n\rightarrow\infty$.

Third, since the elements of $\MU$ are between $-1$ and $1$, according to Condition~\ref{dc}, we have: 
$$
\sup_{\bm{\beta}^{(\bm{U})})\in\ThetaSet^{(\bm{U})}}\biggl\|\nabla^3\nlog_n^{(\bm{U})}(\bm{\beta}^{(\bm{U})})\biggr\|_{\infty}\leq p^3\sup_{\bm{\beta}\in\ThetaSet}\biggl\|\nabla^3\nlog_n(\bm{\beta})\biggr\|_{\infty}\leq \frac{p^3}{n}\sum_{i=1}^n\bdd(y_i,\bm{x}_i),
$$
which converges to $p^3\E(\bdd(y,\bm{x}))$ in probability as $n\rightarrow\infty$. Therefore, $T_3/a^3$ is uniformly  upper bounded by a positive constant for all $\bm{\beta}'\in \ThetaSet^{(\bm{U},a)}$ with probability going to one.
Thus, for a $a$ small enough, such that $T_2 > T_3$ with probability going to one as $n\rightarrow\infty$, 
we have $\nlog_n^{(\bm{U})}(\bm{\beta}') - \nlog_n^{(\bm{U})}(\bm{\beta}^{(\bm{U})}_*\allowbreak)>0$ holds with probability going to one.

According to the above results, there exists a local minimum of $\nlog_n^{(\bm{U})}(\bm{\beta}^{(\bm{U})})$ in $\Omega^{(\bm{U},a)}$ with probability going to 1  as $n\rightarrow\infty$.
By the convexity of $\nlog_n^{(\bm{U})}(\bm{\beta}^{(\bm{U})})$ from Lemma~\ref{lem_snAsnUconvex}, this local minimum is $\hat{\bm{\beta}}^{(\bm{U})}_n$, and it is unique.
Since the result holds for an arbitrarily small $a$, we obtain the result in \eqref{eq_betaUconv}. 
For  $\check{\bm{\beta}}_n$, it can be shown in a similar way by the Taylor expansion of $\nlog_n(\bm{\beta})$.
\end{proof}

\begin{lemma}[Proof of Condition~\ref{A3}]\label{lem_proofOfCond4}
By Conditions~\ref{cond_pdHessian}, \ref{dc}, \ref{cond_pdI}, 
and \ref{cond_Fbounded}
and Lemmas~\ref{lem_bddExp}, \ref{lem_sUconv}, and \ref{lem_Betaconv}, the matrices $\bm{V}_1$ and $\bm{V}_2$ are positive definite, and $\|\hat{\bm{V}}_1 - \bm{V}_1\|_{\infty}\limp 0$ and $\|\hat{\bm{V}}_2 - \bm{V}_2\|_{\infty}\limp 0$ as $n\rightarrow\infty$
, where $\|\cdot\|_2$ represents the spectrum norm for matrices and  recall that
\begin{align*}
    \bm{V}_1\de&\ -\E \bigl(\nabla_{\bm{\beta}^{(\bm{U})}}^2\log f(Y,{X^{(\bm{U})}}^\T\bm{\beta}^{(\bm{U})}_*)\mid\bm{U}\bigr),\\
      \bm{V}_2\de&\ \E \biggl(\nabla_{\bm{\beta}^{(\bm{U})}} \log f(Y,{X^{(\bm{U})}}^\T\bm{\beta}^{(\bm{U})}_*)\bigl(\nabla_{\bm{\beta}^{(\bm{U})}} \log f(Y,{X^{(\bm{U})}}^\T\bm{\beta}^{(\bm{U})}_*)\bigr)^\T\mid\bm{U}\biggr),\\
    \hat{\bm{V}}_1\de&\ \nabla^2\nlog^{(\bm{U})}_n(\hat{\bm{\beta}}^{(\bm{U})}_n),\\
   \hat{\bm{V}}_2\de&\ \frac{1}{n}\sum_{i=1}^n\biggl(\nabla_{\bm{\beta}^{(\bm{U})}} \log f(y_i,{\bm{x}_i^{(\bm{U})}}^\T\bm{\beta}^{(\bm{U})}_*)\bigl(\nabla_{\bm{\beta}^{(\bm{U})}} \log f(y_i,{\bm{x}_i^{(\bm{U})}}^\T\bm{\beta}^{(\bm{U})}_*)\bigr)^\T\mid\bm{U}\biggr)
\end{align*}
Additionally, under $\textrm{H}_0$, $\sup_{\bm{\beta}_n \in \mathcal{L}_n} \|\nabla^2\nlog^{(\bm{U})}_n(\bm{\beta}_n) - \bm{V}_1\|_{\infty}\limp 0$ as $n\rightarrow \infty$, where $\mathcal{L}_n$ is the line segment connecting $\hat{\bm{\beta}}^{(\bm{U})}_n$ and $\bm{\beta}^{(\bm{U})}_*$. 

\end{lemma}
\begin{proof}
The positive definiteness of $\bm{V}_1$ follows from Condition~\ref{cond_pdHessian} and Lemma~\ref{lem_exch}.
 By Condition~\ref{cond_pdI}, $\bm{V}_2$ is positive definite. 
 For the convergence of $\hat{\bm{V}}_1$, we have
\begin{equation}\label{eq_nablaMUconv1}
 \|\hat{\bm{V}}_1 - \bm{V}_1\|_{\infty}\leq \|\hat{\bm{V}}_1 - \nabla^2\nlog^{(\bm{U})}_*(\hat{\bm{\beta}}^{(\bm{U})}_n)\|_{\infty} + \|\nabla^2\nlog^{(\bm{U})}_*(\hat{\bm{\beta}}^{(\bm{U})}_n)-\bm{V}_1\|_{\infty}.
 \end{equation}
Since
 \begin{align*}
 \hat{\bm{V}}_1 =&\ -{\MU}^\T\frac{1}{n}\sum_{i=1}^n \bigl(\nabla^2\log f(y_i,\bm{x}_i^\T\MU\hat{\bm{\beta}}_n^{(\bm{U})})\mid\bm{U}\bigr)\MU,\\
  \nabla^2\nlog^{(\bm{U})}_*(\hat{\bm{\beta}}^{(\bm{U})}_n)=&\ -{\MU}^\T\E\bigl(\nabla^2\log f(Y,X^\T\MU\hat{\bm{\beta}}_n^{(\bm{U})})\mid\bm{U}\bigr)\MU,
 \end{align*}
 by Lemmas~\ref{lem_sUconv} and \ref{lem_Betaconv} and the fact that the elements of $\MU$ are bounded between $-1$ and $1$, we have
 \begin{equation}\label{eq_nablaMUconv2}
 \|\hat{\bm{V}}_1 - \nabla^2\nlog^{(\bm{U})}_*(\hat{\bm{\beta}}^{(\bm{U})}_n)\|_{\infty}\limp 0, \text{ as }n\rightarrow\infty.
 \end{equation}
 By Inequality~\eqref{eq_boundDeri2}, the boundedness of elements in $\MU$, and applying the bounded convergence theorem to the difference
 $$
 \nabla_{\bm{\beta}^{(\bm{U})}}^2\log f(Y,{X^{(\bm{U})}}^\T\bm{\beta}^{(\bm{U})}_*) - \nabla_{\bm{\beta}^{(\bm{U})}}^2\log f(Y,{X^{(\bm{U})}}^\T\bm{\beta}^{(\bm{U})}_n),
 $$
 we have 
 \begin{equation}\label{eq_nablaMUconv3}
 \|\nabla^2\nlog^{(\bm{U})}_*(\hat{\bm{\beta}}^{(\bm{U})}_n)-\bm{V}_1\|_{\infty}\limp 0, \text{ as }n\rightarrow\infty.
 \end{equation}
By combining the results in \eqref{eq_nablaMUconv1}, \eqref{eq_nablaMUconv2}, \eqref{eq_nablaMUconv3}, we obtain $\|\hat{\bm{V}}_1 - \bm{V}_1\|_{\infty}\limp 0$ as $n\rightarrow \infty$.

The convergence of $\|\hat{\bm{V}}_2 - \bm{V}_2\|_{\infty}$ and $\sup_{\bm{\beta}_n \in \mathcal{L}_n}\|\nabla^2\nlog^{(\bm{U})}_n(\bm{\beta}_n) - \bm{V}_1\|_{\infty}$ can be shown in a similar way with Condition~\ref{cond_Fbounded}
\end{proof}

\begin{lemma}[Proof of Condition~\ref{cond_AsymNorm}]\label{lem_proofOfCond10}
The oracle estimator 
$\sqrt{n}(\check{\bm{\beta}}_n - \bm{\beta}_*)\rightarrow_d\mathcal N(\bm{0},\Sigma)$
 as $n\rightarrow\infty$, where $\Sigma=\bigl(\E \bigl(\nabla_{\bm{\beta}}\log f(Y\mid X^\T\bm{\beta})(\nabla_{\bm{\beta}}\log f(Y\mid X^\T\bm{\beta}))^\T\bigr)^{-1}$.
\end{lemma}
\begin{proof}
This result is obtained by Conditions~\ref{cond_diffu}, \ref{dc}, and \ref{cond_pdI}, and  Theorem~7.5.2 of \cite{lehmann2004elements}.
\end{proof}

\begin{lemma}[Proof of Condition~\ref{cond_finiteSample}]\label{lem_additional_assumptions}
When $\E\bigl(\bigl(\sup_{\bm{\beta}\in \ThetaSet}|\mest(Y,X^\T \bm{\beta} )|\bigr)^2\bigr)$, $\E\bigl(\bigl(\sup_{\bm{\beta}\in \ThetaSet}\|\nabla^2\mest(\allowbreak Y,X^\T \bm{\beta} )\|_{\infty}\bigr)^2\bigr)$, and $\E\bigl(\|XX^\T\|_{\infty}^2\bigr)$ exist, 
Condition~\ref{cond_finiteSample} holds with
\begin{align*}
\bd(\Mb) =&\  4\Mb^{-2}\cdot \max\biggl(\E\biggl(\bigl(\sup_{\bm{\beta}\in \ThetaSet}|\mest(Y,X^\T \bm{\beta} )|\bigr)^2\biggr), \E\biggl(\bigl(\sup_{\bm{\beta}\in \ThetaSet}\|\nabla^2\mest(Y,X^\T \bm{\beta} )\|_{\infty}\bigr)^2\biggr)\biggr),\\
\Cst =&\ \frac{16p^2\E\bigl(\|XX^\T\|_{\infty}^2\bigr)}{\bigl(\lambda^{(\textrm{min})}(\E(XX^\T))\bigr)^2}, \quad \text{and $\cRate = 1/n$.}
\end{align*}
\end{lemma}
\begin{proof}
We first prove the result of $\bd(\Mb)$.
We have
\begin{align*}
&\E\biggl(\bigl(\sup_{\bm{\beta}\in \ThetaSet}|\mest(Y,X^\T \bm{\beta} ) - \E(\mest(Y,{X}^\T \bm{\beta} ))|\bigr)^2\biggr)\\
\leq &
\E\biggl(\bigl(\sup_{\bm{\beta}\in \ThetaSet}|\mest(Y,X^\T \bm{\beta} )| + \sup_{\bm{\beta}\in \ThetaSet}|\E(\mest(Y,{X}^\T \bm{\beta} ))|\bigr)^2\biggr)\\
\leq &2\E\biggl(\bigl(\sup_{\bm{\beta}\in \ThetaSet}|\mest(Y,X^\T \bm{\beta} )|\bigr)^2 +\bigl( \sup_{\bm{\beta}\in \ThetaSet}|\E(\mest(Y,{X}^\T \bm{\beta} ))|\bigr)^2\biggr)\\
\leq &\ 4\E\biggl(\bigl(\sup_{\bm{\beta}\in \ThetaSet}|\mest(Y,X^\T \bm{\beta} )|\bigr)^2\biggr).
\end{align*}
Denote right-hand side by $\bdcoeff_1$.
Since $|\nlog_n(\bm{\beta})-\nlog_*(\bm{\beta})| = n^{-1}\sum_{i=1}^n |\mest(Y_i,X_i^\T \bm{\beta} ) - \E(\mest(Y,X^\T \bm{\beta} ))|$,
\begin{align*}
   \E\biggl(\bigl(\sup_{\bm{\beta}\in \ThetaSet}|\nlog_n(\bm{\beta})-\nlog_*(\bm{\beta})|\bigr)^2\biggr)\leq &\ \E\biggl(\biggl(n^{-1}\sum_{i=1}^n\sup_{\bm{\beta}\in \ThetaSet}|\mest(Y_i,X_i^\T \bm{\beta} ) - \E(\mest(Y,{X}^\T \bm{\beta} ))|\biggr)^2\biggr)\\
   \leq&\ \bdcoeff_1/n.
\end{align*}
Therefore, by Chebyshev's inequality, 
\begin{equation}\label{eq_finitesampleInqbound1}
\P\biggl(\sup_{\bm{\beta}\in \ThetaSet}\left|\nlog_n(\bm{\beta})-\nlog_*(\bm{\beta})\right|\geq\Mb\biggr)\leq \frac{\bdcoeff_1}{\Mb^2 \cdot n}.
\end{equation}
 We can similarly show that 
\begin{equation}\label{eq_finitesampleInqbound2}
\P\biggl(\sup_{\bm{\beta}\in \ThetaSet}\left\|\nabla^2\nlog_n(\bm{\beta})-\nabla^2\nlog_*(\bm{\beta})\right\|_{\infty}\geq\Mb\biggr) \leq \frac{\bdcoeff_2}{\Mb^2 \cdot n}, 
\end{equation}
where 
$$\bdcoeff_2\de 4\E\biggl(\bigl(\sup_{\bm{\beta}\in \ThetaSet}\|\nabla^2\mest(Y,X^\T \bm{\beta} )\|_{\infty}\bigr)^2\biggr).$$
Thus, for $\cRate = 1/n$, we may take
$$
\bd(\Mb) = 4\Mb^{-2}\cdot \max\biggl(\E\biggl(\bigl(\sup_{\bm{\beta}\in \ThetaSet}|\mest(Y,X^\T \bm{\beta} )|\bigr)^2\biggr), \E\biggl(\bigl(\sup_{\bm{\beta}\in \ThetaSet}\|\nabla^2\mest(Y,X^\T \bm{\beta} )\|_{\infty}\bigr)^2\biggr)\biggr).
$$

Next, we show the result of $\Cst$. By Chebyshev's inequality, for each $\Mb'>0$, 
 \begin{equation}\label{eq_finitesampleInqbound3}
 \P\bigl(\|XX^\T- \E(XX^\T)\|_{\infty}>\Mb'\bigr) \leq \frac{\bdcoeff_3}{(\Mb')^2\cdot n},
\end{equation}
 where
 $$\bdcoeff_3\de 4\E\bigl(\|XX^\T\|_{\infty}^2\bigr).$$
By Lemma~\ref{lem_ctsEigen}, 
$$
|\lambda^{(\textrm{min})}(XX^\T) - \lambda^{(\textrm{min})}(\E(XX^\T))|<p\cdot \|(XX^\T- \E(XX^\T)\|_{\infty},
$$
where recall that  $\lambda^{(\textrm{min})}(\cdot)$  returns the smallest and largest eigenvalues of a matrix. Since we require the covariates in $X$ to be linearly independent, $\lambda^{(\textrm{min})}(\E(XX^\T))>0$. Therefore, when 
$
\|(XX^\T- \E(XX^\T)\|_{\infty}\leq\Mb', 
$
where $\Mb' = (2p)^{-1}\lambda^{(\textrm{min})}(\E(XX^\T))$, 
we have
\begin{equation}\label{eq_deMat_X_full_rank}
\lambda^{(\textrm{min})}(XX^\T)>0.
\end{equation}
Since Inequality~\eqref{eq_deMat_X_full_rank} guarantees that $\bm{X}$ has full rank, we may take
$$\Cst = \frac{16p^2\E\bigl(\|XX^\T\|_{\infty}^2\bigr)}{(\lambda^{(\textrm{min})}(\E(XX^\T))^2}.$$

Combining Inequalities~\eqref{eq_finitesampleInqbound1}, \eqref{eq_finitesampleInqbound2}, \eqref{eq_finitesampleInqbound3}, we complete the proof.

\end{proof}

\section{A Discussion About the Conditions}\label{eq_a_discussion_about_conditions}

A set of useful properties of $m(\cdot)$ and covariates $X$ to guarantee Conditions~\ref{A1}-\ref{cond_exchInteDiff} are as follows:
\begin{align}
&\exists\csttt>0,\ s.t.\ \|\nabla_\lp \mest(Y,\lp)\|_{\infty}\leq \csttt,\text{ almost surely for all } Y\text{ and }\lp,\label{eq_suff_bdd1}\\
&\exists\cstttt>0,\ s.t.\ \|\nabla^2_\lp \mest(Y,\lp)\|_{\infty}\leq \cstttt,\text{ almost surely for all } Y\text{ and }\lp,\label{eq_suff_bdd2}\\
&X \text{ has finite second-order moments},\label{eq_suff_mom2}
\end{align}
where \eqref{eq_suff_bdd1} and \eqref{eq_suff_bdd2} are satisfied for, e.g., logistic regression and robust regressions with the log-cosh loss \citep{neuneier2002train, saleh2022statistical} or the pseudo-Huber loss \citep{charbonnier1997deterministic, hartley2003multiple}.
Condition~\ref{cond_designMatFullrank} holds e.g., under \eqref{eq_suff_mom2} or when the covariates are continuous random variables (recall that we assume linearly independent covariates). 
Condition~\ref{cond_exchInteDiff} and the differentiability 
of $M_*(\bm{\beta})$ in Condition~\ref{A1} can be obtained from \eqref{eq_suff_bdd1}, the mean value theorem, and the dominated convergence theorem. Note that for some other models like Normal regression and Possion regression that do not satisfy \eqref{eq_suff_bdd1}, it is still possible for  $\nlog_*(\bm{\beta})$ and $\nlog_*^{(\bm{U})}(\bm{\beta}^{(\bm{U})}_*)$ to be differentiable if additional constraints are imposed on the data-generating distribution, e.g., $(Y,X)$ have bounded second-order moments for Normal regression, and $Y$ is bounded and the covariates are i.i.d.\ from a uniform distribution for Poisson regression. 
The existence of $\bm{\beta}_*$ in Conditions~\ref{A1} holds, e.g., when $m(\cdot)$ is a negative log-likelihood (of GLM) that is correctly specified with respect to the data-generating model.
The uniqueness of $\bm{\beta}_*$ and $\bm{\beta}^{(\bm{U})}_*$  can be obtained given Condition~\ref{cond_pdHessian}. Sufficient conditions for the convergence of $\hat{\bm{\beta}}^{(\bm{U})}_n$ in Condition~\ref{A3} can be found in Theorem~5.7  of \citet{van2000asymptotic}. The convergence of matrices in Condition~\ref{A3} can be verified under \eqref{eq_suff_bdd1}, \eqref{eq_suff_bdd2}, \eqref{eq_suff_mom2}, and the uniform law of large number. 
Condition~\ref{cond_strictConvexloss} is satisfied for, e.g., GLMs with canonical link functions, robust regressions with the log-cosh loss or the pseudo-Huber loss.
Given \eqref{eq_suff_bdd1}, \eqref{eq_suff_bdd2}, and \eqref{eq_suff_mom2}, Condition~\ref{cond_pdHessian} can be derived from Condition~\ref{cond_strictConvexloss}, the mean value theorem, and the dominated convergence theorem, and Condition~\ref{cond_sUconv} is satisfied by the uniform law of large numbers. 
Conditions~\ref{cond_oraExist} and \ref{cond_AsymNorm} can be obtained by verifying the conditions required by Theorem~5.7 and 5.21 of \citet{van2000asymptotic}, respectively.

\section{An Alternative to the Wald Test}\label{sec_proofOfTh3}

In this section, we first explain the reason for selecting the Wald test for the ``initializing connection'' stage. Subsequently, we explore the likelihood ratio test, which can serve as an alternative in certain scenarios.
\subsection{Comments on other tests}\label{subseq_otherTestscomment}
Recall that we consider a Wald-type test statistic for the initializing connection step. A curious reader may wonder why we prefer it over other competitive tests like the likelihood test and the Rao test. Rather, it is because other alternative tests have specific requirements that are not suitable for our setting.
 First, the Rao test needs knowledge of the marginal distribution of the covariates. However, this information is not obtainable for \A in constructing the test statistic. 
Second, the likelihood ratio test has a stronger requirement, in the sense that 
\begin{equation}\label{eq_correctly_specified_likelihood}
f(Y| X=\bm{x}) \in \bigl\{-\exp\bigl(\mest(Y,\bm{x}^\T \bm{\beta} )\bigr):\bm{\beta}\in \R^p\bigr\},\quad \text{for }\bm{x}\in\R^p,
\end{equation}
where $f(Y|X)$ is the conditional probability density/mass function of $Y$ given $X$.

\subsection{Likelihood ratio test}
For GLM that satisfies \eqref{eq_correctly_specified_likelihood},  we can construct $W_{n,\id}$ by the likelihood ratio test statistic
$$
    W_{n,\id} = 2n\bigl(\nlog_n^{(\AC)}(\hat{\bm{\beta}}^{(\AC)}_n)-\nlog_n^{(\bm{U})}(\hat{\bm{\beta}}^{(\bm{U})}_n)\bigr),
$$
where $\nlog_n^{(\AC)}(\bm{\beta}^{(\AC)})\de -n^{-1}\sum_{i=1}^n \log f\bigl(y_i\mid {\bm{x}^{(\A)}_i}^\T\bm{\beta}^{(\AC)}\bigr)$, $\hat{\bm{\beta}}^{(\AC)}_n \de \argmin_{\bm{\beta}^{(\AC)}\in\mathbb R^{p_\A}} \nlog_n^{(\AC)}(\bm{\beta}^{(\AC)})$, 
and recall that $\nlog_n^{(\bm{U})}(\bm{\beta}^{(\bm{U})})$ and $\hat{\bm{\beta}}^{(\bm{U})}_n$ were defined on Page~\pageref{page_betaU_def} of the main text.
\begin{condition}\label{cond_identifiability}
For $\bm{\beta}^{(1)}, \bm{\beta}^{(2)}\in\R^p$, $\mest(y,\bm{x}^\T\bm{\beta}^{(1)}) = \mest(y,\bm{x}^\T\bm{\beta}^{(2)})$ indicates that $\bm{\beta}^{(1)} = \bm{\beta}^{(2)}$.
\end{condition}
\label{page_thm_prop2}
\begin{theorem}\label{prop2} Assume that 
Conditions~\ref{cond_designMatFullrank}, \ref{A1},  \ref{cond_strictConvexloss}, \ref{cond_pdHessian},  \ref{cond_diffu}-\ref{lem_sbb_inm}, and \ref{cond_identifiability} hold, and $\id <p_b$. Also, assume that $\MA \bm{\beta}^{(\AC)}_*\in\ThetaSett$ (defined in Condition~\ref{dc}), where 
$\MA\de (\bm{I}_{p_\A\times p_\A}, \bm{0}_{p_b\times p_\A})^\T$
and
$$ \bm{\beta}^{(\AC)}_* \de \argmin_{\bm{\beta}^{(\AC)}\in\mathbb R^{p_\A}} \nlog_*^{(\AC)}(\bm{\beta}^{(\AC)}),\quad    \nlog_*^{(\AC)}(\bm{\beta}^{(\AC)})\de -\E\bigl( \log f\bigl(Y\mid {X^{(\A)}}^\T\bm{\beta}^{(\AC)}\bigr)\bigr).$$
 Then, there exists a set $\Uset$ that is subset of the support of $\bm{U}$. Furthermore,
 $\P(\bm{U}\in \Uset) = 1$ and for each $\Uele\in\Uset$,
 \begin{description}
\item[(1)] $W_{n,\id}\mid \bm{U}=\Uele$ converges weakly to a chi-squared distribution with degrees of freedom~$\id$ under $\textup{H}_0$;
\item[(2)] $W_{n,\id}\mid\bm{U}=\Uele$ goes to infinity in probability under $\textup{H}_1$.
\end{description}
Additionally, if \B sends an additional vector $\bm{X}^{(\B)}\bm{u}^{(\id + 1)}$ to \A and
Condition~\ref{lem_sbb_inm} also holds for $t+1$, then
	$W_{n,\id+1} - W_{n,\id} \limp \infty$ as $n \rightarrow \infty$ under $\textup{H}_1$.
\end{theorem}

\begin{proof}
We first show the result under $\textrm{H}_0$.
Let
\begin{align*}
    \nlog_n^{(\AC)}(\bm{\beta}^{(\AC)})\de& -\frac{1}{n}\sum_{i=1}^n \log f\bigl(y_i\mid {\bm{x}^{(\A)}_i}^\T\bm{\beta}^{(\AC)}\bigr),\\
     \nlog_*^{(\AC)}(\bm{\beta}^{(\AC)})\de&
     -\E\bigl(\log f\bigl(y\mid {\bm{x}^{(\A)}}^\T\bm{\beta}^{(\AC)}\bigr)\bigr),\\
    \nlog_n^{(\bm{U})}(\bm{\beta}^{(\bm{U})})\de& -\frac{1}{n}\sum_{i=1}^n \log f\bigl(y_i\mid ({\bm{x}^{(\A)}_i}^\T, {\bm{x}^{(\bm{U})}_i}^\T)\bm{\beta}^{(\bm{U})}\bigr).
\end{align*}
Let $\hat{\bm{\beta}}^{(\AC)}_n$ and $\hat{\bm{\beta}}^{(\bm{U})}_n$ be the minimizers of $\nlog_n^{(\AC)}(\bm{\beta}^{(\AC)})$ and $\nlog_n^{(\bm{U})}(\bm{\beta}^{(\bm{U})})$, respectively. The likelihood ratio test statistic is calculated by
\begin{equation}\label{eq_likeTest}
    W_{n,\id} = 2n\bigl(\nlog_n^{(\AC)}(\hat{\bm{\beta}}^{(\AC)}_n)-\nlog_n^{(\bm{U})}(\hat{\bm{\beta}}^{(\bm{U})}_n)\bigr).
\end{equation}
It can be verified by Theorem~7.7.4 of \cite{lehmann2004elements} that $W_{n,\id}$ converges weakly to a chi-squared distribution with degrees freedom $\id$ as $n\rightarrow\infty$ under $\textrm{H}_0$. 

Next, we derive the result under $\textrm{H}_1$. By Lemmas~\ref{lem_sUconv} and \ref{lem_Betaconv}, $$|\nlog_n^{(\bm{U})}(\hat{\bm{\beta}}^{(\bm{U})}_n)- \nlog_*^{(\bm{U})}(\bm{\beta}^{(\bm{U})}_*)|\limp 0,$$ as $n\rightarrow\infty$. 
We can similarly show that $|\nlog_n^{(\AC)}(\hat{\bm{\beta}}^{(\AC)}_n)- \nlog_*^{(\AC)}(\bm{\beta}^{(\AC)}_*)|\limp0$. Therefore, it suffices to prove
\begin{equation}\label{eq_nlogNeq}
    \nlog_*^{(\bm{U})}(\bm{\beta}^{(\bm{U})}_*)< \nlog_*^{(\AC)}(\bm{\beta}^{(\AC)}_*)
\end{equation}
under $\textrm{H}_1$. By the proof in Section~\ref{sec_H1}, $\bm{\beta}^{(\bm{U},\id)}_*$, which corresponds to the sketchy data $\bm{X}^{(\B)}\bm{U}$ transmitted from \B in the larger model, is almost surely not zero under $\textrm{H}_1$. Therefore, the above inequality holds.

Next, we show that
\begin{equation}
    W_{n,\id+1} - W_{n,\id}\limp \infty
\end{equation}
as $n\rightarrow\infty$ under $\textrm{H}_1$.
Let $\bm{\beta}^{(\id + 1)}$ denote the coefficient corresponding to the new vector ${\bm{X}^{(\B)}}^\T\bm{u}^{(\id+1)}$.
Let $\breve{\bm{U}}  \de (\bm{U}, \bm{u}^{(\id + 1)})$.
Define $\nlog_n^{(\breve{\bm{U}})}(\cdot)$, $\nlog^{(\breve{\bm{U}})}_*(\cdot)$,
 $\hat{\bm{\beta}}^{(\breve{\bm{U}})}_n$, and $\bm{\beta}^{(\breve{\bm{U}})}_*$ for  $\breve{\bm{U}}$ in a similar way as we define $\nlog_n^{(\bm{U})}(\cdot)$, $\nlog^{(\bm{U})}_*(\cdot)$,
 $\hat{\bm{\beta}}^{(\bm{U})}_n$, and $\bm{\beta}^{(\bm{U})}_*$. We have
\begin{align*}
 W_{n,\id+1} - W_{n,\id} = &\ 2n\left(\nlog_n^{(\AC)}(\hat{\bm{\beta}}^{(\AC)}_n) - \nlog_n^{(\breve{\bm{U}})}(\hat{\bm{\beta}}^{(\breve{\bm{U}})}_n)\right)- 2n\left(\nlog_n^{(\AC)}(\hat{\bm{\beta}}^{(\AC)}_n) - \nlog_n^{(\bm{U})}(\hat{\bm{\beta}}^{(\bm{U})}_n) \right)\\
=&\ 2n\left(\nlog_n^{(\bm{U})}(\hat{\bm{\beta}}^{(\bm{U})}_n) - \nlog_n^{(\breve{\bm{U}})}(\hat{\bm{\beta}}^{(\breve{\bm{U}})}_n)\right).
\end{align*}
It can be shown that $|\nlog_n^{(\breve{\bm{U}})}(\hat{\bm{\beta}}^{(\breve{\bm{U}})}_n)- \nlog_*^{(\breve{\bm{U}})}(\bm{\beta}^{(\breve{\bm{U}})}_*)|\limp0$.
Therefore, remains to show that
\begin{equation*}\label{eq_diffElog2} \nlog^{(\breve{\bm{U}})}_*(\bm{\beta}^{(\breve{\bm{U}})}_*)<\nlog^{(\bm{U})}_*(\bm{\beta}^{(\bm{U})}_*)
\end{equation*}
almost surely. The above statement is equivalent to 
\begin{equation}\label{eq_newbeta_not_zero}
    \bm{\beta}^{(\breve{\bm{U}},\id+1)}_*\neq 0\text{ holds almost surely},
\end{equation}
where $\bm{\beta}^{(\breve{\bm{U}},\id+1)}_*$ is the component of $\bm{\beta}^{(\breve{\bm{U}})}_*$ corresponding to the new vector ${\bm{X}^{(\B)}}^\T\bm{u}^{(\id+1)}$.  
To prove the above result, we first show 
\begin{equation}\label{eq_betaU_not_equal_betaStar}
\nabla\nlog_*(\MU \bm{\beta}^{(\bm{U})}_*)\neq \bm{0}
 \text{ holds almost surely},
\end{equation}
which is indicated by 
$\MU \bm{\beta}^{(\bm{U})}_*\neq \bm{\beta}_*$.
Recall that
$$
\MU \bm{\beta}^{(\bm{U})}_* =
\MU
\begin{pmatrix}
\bm{\beta}^{(\bm{U},a)}_*\\
\bm{\beta}^{(\bm{U},c)}_*\\
\bm{\beta}^{(\bm{U},\id)}_*
\end{pmatrix}
=
\begin{pmatrix}
\bm{\beta}^{(\bm{U},a)}_*\\
\bm{U}^{(p_b)} \bm{\beta}^{(\bm{U},\id)}_*\\
\bm{\beta}^{(\bm{U},c)}_* + \bm{U}^{(p_c)} \bm{\beta}^{(\bm{U},\id)}_*
\end{pmatrix},
$$
where $\bm{\beta}^{(\bm{U},a)}_*$, $\bm{\beta}^{(\bm{U},c)}_*$, $\bm{\beta}^{(\bm{U},\id)}_*$ correspond to $\bm{X}^{(a)}$, $\bm{X}^{(c)}$, and  $\bm{X}^{(\bm{U})}$, respectively.
Therefore, $\MU \bm{\beta}^{(\bm{U})}_*= \bm{\beta}_*$ indicates that $\bm{U}^{(p_b)} \bm{\beta}^{(\bm{U},\id)}_* = \bm{\beta}^{(b)}_*$,
which can be rewritten as
\begin{equation}\label{eq_betab_in_U_column_space}
\bigl(\bm{U}^{(p_b)}, \bm{\beta}^{(b)}_*\bigr)
\begin{pmatrix}
	\bm{\beta}^{(\bm{U},\id)}_*\\
	-1
\end{pmatrix}
=\bm{0}.
\end{equation}
Since the $\id<p_b$ columns of $\bm{U}$ are i.i.d.\ sampled from a continuous distribution and $\bm{\beta}^{(b)}_*\neq \bm{0}$, it can be verified that the determinant of the square matrix that consists of the first $\id + 1$ rows of $\bigl(\bm{U}^{(p_b)}, \bm{\beta}^{(b)}_*\bigr)$ does not equal zero almost surely.
Therefore, this matrix is invertible, which contradicts with the fact that 
$$
\begin{pmatrix}
	\bm{\beta}^{(\bm{U},\id)}_*\\
	-1
\end{pmatrix}\neq \bm{0}.
$$
Accordingly, we obtain the result in \eqref{eq_betaU_not_equal_betaStar}. 
Next, we show the result in \eqref{eq_newbeta_not_zero}
by contradiction. When $ \bm{\beta}^{(\breve{\bm{U}},\id+1)}_*= 0$, 
we have
$\bm{\beta}^{(\breve{\bm{U}})}_* = \bigl(\bigl(\bm{\beta}^{(\bm{U})}_*\bigr)^\T, 0\bigr)^\T$. Let 
\begin{equation*}
\MUb \de 
\begin{pmatrix}
\bm{I}_{p_a\times p_a} & \bm{0}& \bm{0}\\
\bm{0}& \bm{0} &\bigl(\bm{U}^{(p_b)}, \bm{u}^{(\id+1, p_b)}\bigr)\\
\bm{0}&\bm{I}_{p_c\times p_c}&\bigl(\bm{U}^{(p_c)}, \bm{u}^{(\id+1, p_c)}\bigr)
\end{pmatrix},
\end{equation*}
where $\bm{u}^{(\id+1, p_b)}$ and $\bm{u}^{(\id+1, p_c)}$ are the elements of $\bm{u}^{(\id+1)}$ corresponding to $\bm{X}^{(b)}$ and $\bm{X}^{(c)}$, respectively. Therefore, 
\begin{align*}
&\nabla \nlog_*^{(\breve{\bm{U}})}\biggl(\bigl(\bigl(\bm{\beta}^{(\bm{U})}_*\bigr)^\T, 0\bigr)^\T\biggr) \\
=&\  
{\MUb}^\T\nabla \nlog_*\biggl(\MUb\bigl(\bigl(\bm{\beta}^{(\bm{U})}_*\bigr)^\T, 0\bigr)^\T\biggr)\\
=&\ {\MUb}^\T\nabla \nlog_*\bigl(\MU\bm{\beta}^{(\bm{U})}_*\bigr)\\
= &\ 
\begin{pmatrix}
\nabla_{\bm{\beta}^{(a)}} \nlog_*\bigl(\MU\bm{\beta}^{(\bm{U})}_*\bigr)\\
\nabla_{\bm{\beta}^{(c)}} \nlog_*\bigl(\MU\bm{\beta}^{(\bm{U})}_*\bigr)\\
\bigl(\bm{U}^{(p_b)}, \bm{u}^{(\id+1, p_b)}\bigr)^\T \nabla_{\bm{\beta}^{(b)}}   \nlog_*(\MU \bm{\beta}^{(\bm{U})}_*) +
\bigl(\bm{U}^{(p_c)}, \bm{u}^{(\id+1, p_c)}\bigr)^\T \nabla_{\bm{\beta}^{(c)}}   \nlog_*(\MU \bm{\beta}^{(\bm{U})}_*)
\end{pmatrix}.
\end{align*}
Following a proof similar to the one of \eqref{eq_cont_rand_linear}, 
we have
\begin{equation*}
\P(\exists \bm{z}\in\mathbb R^{p_b}, \text{ s.t. }\bm{z} \neq \bm{0}\text{ and }\bigl(\bm{U}^{(p_b)}, \bm{u}^{(\id+1, p_b)}\bigr)^\T \bm{z} = \bm{0}) = 0.
\end{equation*}
According to the above results, $\nabla \nlog_*^{(\breve{\bm{U}})}\biggl(\bigl(\bigl(\bm{\beta}^{(\bm{U})}_*\bigr)^\T, 0\bigr)^\T\biggr) \neq \bm{0}$ almost surely, which contradicts with the assumption that $\nlog^{(\breve{\bm{U}})}_*(\bm{\beta}^{(\breve{\bm{U}})})$ is minimized at 
$\bm{\beta}^{(\breve{\bm{U}})} = \bigl(\bigl(\bm{\beta}^{(\bm{U})}_*\bigr)^\T, 0\bigr)^\T$.
As a result, we arrive at the conclusion stated in \eqref{eq_newbeta_not_zero}, and this completes the proof.



\end{proof}

\section{AE-AL with Penalties}\label{sec_penal_th_sim}

In this section, we extend the AE-AL  by incorporating it with a penalty term.

For testing, 
let 
\begin{align}
\nlog_n^{(\lam, \bm{U})}(\bm{\beta}^{(\lam, \bm{U})})
\de&\ \frac{1}{n}\sum_{i=1}^n\mest(y_i, ({\bm{x}^{(\bm{U})}_i})^\T\bm{\beta}^{(\lam, \bm{U})}) + \lam\cdot \sum_{j=1}^p\penala(\beta_j^{(\lam, \bm{U})}), \label{eq_test_penal_loss}\\
\hat{\bm{\beta}}^{(\lam, \bm{U})}_n\de&\  \arg\min_{\bm{\beta}^{(\lam, \bm{U})}\in \mathbb R^{p_\A+\id}}\nlog_n^{(\lam, \bm{U})}(\bm{\beta}^{(\lam, \bm{U})}),\label{eq_test_penal_beta}
\end{align}
where $\bm{\beta}^{(\lam, \bm{U})} = \bigl(\beta^{(\lam, \bm{U})}_1,\dots,\beta^{(\lam, \bm{U})}_{p_\A + \id}\bigr)$, $\penala: \R \rightarrow \R^+$ denotes the penalty function, and  $\lam\geq 0$ is the penalty parameter.
The penalty parameter $\lam$ can be selected from cross-validation (CV).
The Wald test statistic is calculated by replacing $\hat{\bm{\beta}}^{(\bm{U})}_n$ with $\hat{\bm{\beta}}^{(\lam, \bm{U})}_n$, namely, 
$$
W_{n,\id}^{(\lam)}\de 
n\bigl(\hat{\bm{\beta}}^{(\lam, \bm{U},\id)}_n\bigr)^\T \bigl(\hat{\bm{V}}_{\id}^{(\lam)}\bigr)^{-1}\hat{\bm{\beta}}^{(\lam, \bm{U},\id)}_n,$$
where  $\hat{\bm{\beta}}^{(\lam, \bm{U},\id)}_n$ denotes the $\id$ elements of  $\hat{\bm{\beta}}^{(\lam, \bm{U})}_n$ corresponding to $\bm{X}^{(\B)}\bm{U}$, $\hat{\bm{V}}_{\id}^{(\lam)}$ denotes the lower right
 $\id\times \id$ block of $\hat{\bm{V}}^{(\lam)}\de \bigl(\hat{\bm{V}}_1^{(\lam)}\bigr)^{-1}\hat{\bm{V}}_2^{(\lam)}\bigl(\hat{\bm{V}}_1^{(\lam)}\bigr)^{-1}$ with
 $\hat{\bm{V}}_1^{(\lam)}\de \nabla^2\nlog^{(\bm{U})}_n(\hat{\bm{\beta}}^{(\lam, \bm{U})}_n)$,
 and 
$\hat{\bm{V}}_2^{(\lam)}\de n^{-1}\sum_{i=1}^n\nabla\mest_i^{(\bm{U})}(\hat{\bm{\beta}}^{(\lam,\bm{U})}_n)\bigl(\nabla\mest_i^{(\bm{U})}(\hat{\bm{\beta}}^{(\lam, \bm{U})}_n)\bigr)^\T$. 
Note that the penalty is not included in $\hat{\bm{V}}_1^{(\lam)}$ and $\hat{\bm{V}}_2^{(\lam)}$.

For training, \A  targets the performance of an M-estimator that minimizes
\begin{equation}\label{eq_penLoss}
   \nlog_n^{(\lam)}(\bm{\beta}) \de \frac{1}{n}\sum_{i=1}^n \mest(y_i,{\bm{x}_i}^\T \bm{\beta} ) + \lam\cdot \sum_{j=1}^p\penala(\beta_j).
\end{equation}
Denote $\sum_{j=1}^p\penala(\beta_j)$ by $\penalb(\bm{\beta})$.
The AE-AL training follows the same procedure in Section~\ref{almeth} 
except that \eqref{eq_assistedLossB} and \eqref{eq_assistedLossA} are replaced by
\begin{align*}
	\nlog_n^{(\B,k,\lam)}(\bm{\beta}^{(\B)}) &\de \frac{1}{n}\sum_{i=1}^n \mest(y_i, {\bm{x}^{(\B)}_i}^\T\bm{\beta}^{(\B)} + {\bm{x}^{(\A)}_i}^\T\hat{\bm{\beta}}^{(\A,\lam,k\mi 1)}_n) + \lam\cdot\penalb(\bm{\beta}^{(\B)}),\\
		\nlog_n^{(\A,k,\lam)}(\bm{\beta}^{(\A)}) &\de \frac{1}{n}\sum_{i=1}^n \mest(y_i, {\bm{x}^{(\A)}_i}^\T\bm{\beta}^{(\A)} + {\bm{x}^{(\B)}_i}^\T\hat{\bm{\beta}}^{(\B,\lam,k)}_n) + \lam\cdot\penalb(\bm{\beta}^{(\A)}),
\end{align*}
respectively, where $ \hat{\bm{\beta}}^{(\A,\lam,k)}_n$ and $ \hat{\bm{\beta}}^{(\B,\lam,k)}_n$ are the estimated coefficients held by \A and \B in the $k$th assistance round, respectively. 
The penalty parameter $\lam$ can be selected by training the AE-AL estimator for a grid of $\lam$ values and choosing the one with the best prediction performance on a hold-out dataset.

In the following Section~\ref{subsec_penal_thm}, we show that the introduction of a penalty term leads to the inclusion of bias terms that diminish as $\lam\rightarrow 0$, for both the test statistic in the ``initializing connection'' stage and the estimated model coefficients in the ``assisted training'' stage. 
In Section~\ref{subsec_ridge}, we present simulation studies for ridge regression with both AE-AL testing and training.
In Section~\ref{subsec_Enet}, we provide additional simulation studies for the training of Lasso and elastic net under high-dimensional settings.

\subsection{Theoretical properties}\label{subsec_penal_thm}
We assume that the penalty function satisfies the following property, which, for example,  holds for the $l_2$ penalty. 
\begin{condition}\label{eq_penaltyCondition}
    The function $\beta\mapsto \penala(\beta)$ is three times continuously differentiable. Additionally, $d^2\penala(\beta)\allowbreak/d\beta^2 >0$ for all $\beta\in\R$. 
\end{condition}
 For testing, we have the following result for GLMs:
\begin{corollary}\label{coro_test_penal}  

Assume that Conditions~\ref{cond_designMatFullrank}, \ref{A1},  \ref{cond_strictConvexloss}, \ref{cond_pdHessian}, \ref{cond_diffu}-\ref{lem_sbb_inm}, and \ref{eq_penaltyCondition} hold. For almost surely all $\bm{U}$ (recall that both $W_{n,\id}^{(\lam)}$ and $W_{n,\id}$ depend on $\bm{U}$),
$$
|W_{n,\id}^{(\lam)} - W_{n,\id}|\limp 0, \text{ as } \lambda\rightarrow 0\text{ and  }n\rightarrow\infty.
$$
\end{corollary}
The above corollary shows that the penalty will introduce a bias to the test statistic that diminishes as $\lam\rightarrow 0$.
\begin{proof}
By a proof similar to the one of Lemma~\ref{lem_Betaconv}, we 
can show that 
$\|\hat{\bm{\beta}}^{(\lam, \bm{U})}_n - \hat{\bm{\beta}}^{(\bm{U})}_n\|_2\limp 0$ as $\lambda\rightarrow 0$ and $n\rightarrow\infty$.
Therefore, we complete the proof by a derivation similar to the one of Lemma~\ref{lem_proofOfCond4}. 
\end{proof}

Let $\hat{\bm{\beta}}^{(\lam, k)}_n\de \TransformMat \cdot \bigl((\hat{\bm{\beta}}^{(\a,\lam,k)}_n)^\T, (\hat{\bm{\beta}}^{(\b,\lam,k)}_n)^\T\bigr)^\T$.
Denote the oracle estimator under the penalized loss by
$
  \check{\bm{\beta}}_n^{(\lam)} \de \argmin_{\bm{\beta}\in \R^p}\ \nlog_n^{(\lam)}(\bm{\beta}). 
$
Recall that to facilitate our derivation, we distinguish between \a's update and \b's update by introducing $k'=\floor{k'/2}$ in addition to the original assistance round number $k$. In this way, $k'$ is updated by $k'+1$ after \a's or \b's update. 
For training, we have the following result:


\begin{corollary}\label{coro_penal}
    Assume that Conditions~\ref{cond_designMatFullrank}, \ref{A1}, \ref{cond_strictConvexloss}-\ref{cond_sUconv}, and \ref{eq_penaltyCondition} hold. 
    There exist a set $\InitSetN\subseteq\InitSet$, 
 and constants $0<q_5<1$, $0<q_6<\infty$, $0<q_7<\infty$, and $0<\Lam<\infty$, such that when the initial value $\hat{\bm{\beta}}^{(\lam, 0)}_n\in \InitSetN$ and $\lam < \Lam$,
\begin{equation}\label{eq_convergence_in_M}
 \P(  |\nlog_n(\hat{\bm{\beta}}^{(\lam, k')}_n) - 
\nlog_n(\check{\bm{\beta}}_n)| \leq q_6\cdot q_5^{k'} + \lam \cdot q_7,\ \forall k'\in\mathbb N) \rightarrow 1,
\end{equation}
as $n\rightarrow\infty$.
\end{corollary}
The above corollary shows that the sample loss evaluated at $\hat{\bm{\beta}}^{(\lam, k')}_n$ converges to the one evaluated at $\check{\bm{\beta}}_n$ as $k'\rightarrow\infty$ and $\lam\rightarrow 0$ with probability going to one. When the iteration number $k'$ is small, the upper bound of $|\nlog_n(\hat{\bm{\beta}}^{(\lam, k')}_n) - 
\nlog_n(\check{\bm{\beta}}_n)|$ is primarily influenced by $q_6\cdot q_5^{k'}$.  However, as $k'$ becomes larger, it becomes dominated by $\lam\cdot q_7$, which depends on the penalty parameter.
\begin{proof}
Let 
$
   \nlog_n^{(\lam)}(\bm{\beta}) \de \nlog_n(\bm{\beta})  + \lam\cdot \penalb(\bm{\beta})
$
and
$\hat{\bm{\beta}}^{(\lam, k')}_n\de \TransformMat \cdot \bigl((\hat{\bm{\beta}}^{(\a,\lam,k')}_n)^\T, (\hat{\bm{\beta}}^{(\b,\lam,k')}_n)^\T\bigr)^\T$.
Recall that
\begin{align*}
  \check{\bm{\beta}}_n^{(\lam)} \de&\ \argmin_{\bm{\beta}\in \R^p}\ \nlog_n^{(\lam)}(\bm{\beta}),\\
  \check{\bm{\beta}}_n \de&\ \argmin_{\bm{\beta}\in \R^p}\ \nlog_n(\bm{\beta}),\\
  \hat{\bm{\beta}}^{( k')}_n\de  &\ \TransformMat \cdot \bigl((\hat{\bm{\beta}}^{(\a,k')}_n)^\T, (\hat{\bm{\beta}}^{(\b,k')}_n)^\T\bigr)^\T.
\end{align*}

The idea of the proof is based on Lemma~\ref{lem_linConvNlog} and showing $\hat{\bm{\beta}}^{(\lam, 1)}_n$ and $\hat{\bm{\beta}}^{(1)}_n$ will be close to each other when $\lam$ is small.
We prove Corollary~\ref{coro_penal} by showing the following results.
\begin{enumerate}
    \item There exists a positive constant $\Lam$, such that when $\lam <\Lam$, $\check{\bm{\beta}}_n^{(\lam)}$ exists with probability going to one as $n\rightarrow\infty$. 
    \item There exists a set $\InitSetN\subseteq\InitSet$, such that when $\hat{\bm{\beta}}^{(\lam, 0)}_n\in \InitSetN$ and  $\lam < \Lam$, 
\begin{equation}\label{eq_InitSetNN_requirements}
    \P(\text{for each $k'\in\mathbb N$, $\hat{\bm{\beta}}^{(\lam, k')}_n$ exists and is in $\InitSet$ } )\rightarrow 1\text{ as $n\rightarrow\infty$. }
\end{equation}
    \item There exists a
positive constant $q_8$, such that when $\lam <\Lam$,
\begin{equation}\label{eq_betalam_bound}
\P(\forall \hat{\bm{\beta}}^{(\lam, 0)}_n\in \InitSetN,\ \|\hat{\bm{\beta}}^{(\lam, 1)}_n - \hat{\bm{\beta}}^{(1)}_n\|_2\leq \lam\cdot q_8)\rightarrow 1,
\end{equation}
 as $n\rightarrow\infty$.
    \item There exist constants $0<q_5<1$, $0<q_6<\infty$, and $0<q_7<\infty$, such that when $\hat{\bm{\beta}}^{(\lam, 0)}_n\in \InitSetN$ and $\lam < \Lam$,
\begin{equation}
 \P(  |\nlog_n(\hat{\bm{\beta}}^{(\lam, k')}_n) - 
\nlog_n(\check{\bm{\beta}}_n)| \leq q_6\cdot q_5^{k'} + \lam \cdot q_7,\ \forall k'\in\mathbb N) \rightarrow 1,
\end{equation}
as $n\rightarrow\infty$.
\end{enumerate}

 First,  the existence of $\check{\bm{\beta}}_n^{(\lam)}$
 can be shown by a proof similar to the one of Lemma~\ref{lem_Betaconv} by third order Taylor expansion.

Second, we show the existence of the set $\InitSetN$. We take an open ball $\ThetaSetN\subseteq\InitSet$ that is centered at $\bm{\beta}_*$. Based on $\ThetaSetN$, define the set
$$
\InitSetNn \de 
\biggl\{\bm{\beta}\in\R^p : \nlog_n^{(\lam)}(\bm{\beta}) \leq \min_{\bm{\beta}\in\ThetaSetN}\nlog_*(\bm{\beta})-\bigl( \min_{\bm{\beta}\in\ThetaSetN}\nlog_*(\bm{\beta}) - \nlog_*(\bm{\beta}_*)\bigr)/4 \biggr\}.
$$
By Condition~\ref{cond_sUconv}, $\nlog_n^{(\lam)}(\bm{\beta}_*) = \nlog_n(\bm{\beta}_*) + \lam\penalb(\bm{\beta}_*)\limp \nlog_*(\bm{\beta}_*) + \lam\penalb(\bm{\beta}_*)$ as $n\rightarrow \infty$ and $\lam\rightarrow 0$. Since $\nlog_*(\bm{\beta}_*)<\min_{\bm{\beta}\in\ThetaSetN}\nlog_*(\bm{\beta})-\bigl( \min_{\bm{\beta}\in\ThetaSetN}\nlog_*(\bm{\beta}) - \nlog_*(\bm{\beta}_*)\bigr)/4$, there exists a positive constant $\Lam$ such that when $\lam<\Lam$, $\InitSetNn$ is not empty with probability going to one as $n\rightarrow \infty$. Also, since $\nlog_n^{(\lam)}(\hat{\bm{\beta}}^{(\lam, k' + 1)}_n) \leq \nlog_n^{(\lam)}(\hat{\bm{\beta}}^{(\lam, k')}_n)$ ($k'\in\mathbb N$), we have $\hat{\bm{\beta}}^{(\lam, k')}_n\in \InitSetNn$ ($k'\in\mathbb N$) when $\hat{\bm{\beta}}^{(\lam, 0)}_n\in \InitSetNn$. Next, let
$$
\InitSetN \de 
\biggl\{\bm{\beta}\in\InitSet: \nlog_*^{(\lam)}(\bm{\beta})  \leq \min_{\bm{\beta}\in\ThetaSetN}\nlog_*(\bm{\beta})-\bigl( \min_{\bm{\beta}\in\ThetaSetN}\nlog_*(\bm{\beta}) - \nlog_*(\bm{\beta}_*)\bigr)/2  \biggr\}.
$$
Following a proof similar to the one in Section~\ref{subsec_initSet}, we can show that $\InitSetN\subseteq\InitSetNn\subseteq\InitSet$. Thus, $\InitSetN$ satisfies the requirement in \eqref{eq_InitSetNN_requirements}.

Third, by the mean value theorem, for each $\bm{\beta}\in \ThetaSet$, 
\begin{align}
    \nabla_{\bm{\beta}^{(b)}}\nlog^{(\lam)}_n(\bm{\beta}) =&\ \nabla_{\bm{\beta}^{(b)}}\nlog^{(\lam)}_n(\hat{\bm{\beta}}^{(1)}_n) + \nabla_{\bm{\beta}^{(b)}}^2\nlog^{(\lam)}_n(\tilde{\bm{\beta}})(\bm{\beta} - \hat{\bm{\beta}}^{(1)}_n)\nonumber\\
    =&\ 0 + \lambda\cdot \nabla_{\bm{\beta}^{(b)}}\penalb(\hat{\bm{\beta}}^{(1)}_n) + \nabla_{\bm{\beta}^{(b)}}^2\nlog^{(\lam)}_n(\tilde{\bm{\beta}})(\bm{\beta} - \hat{\bm{\beta}}^{(1)}_n),\label{eq_beta_n_lam_mean_value}
\end{align}
where $\tilde{\bm{\beta}}$ is on the line segment between $\bm{\beta}$ and $\hat{\bm{\beta}}^{(1)}_n$. By Lemma~\ref{lem_ctsEigen}, when $\bm{\beta}\in\ThetaSet$ and $\lam$ is upper bounded, the eigenvalues of $\nabla_{\bm{\beta}^{(b)}}^2\nlog^{(\lam)}_*({\bm{\beta}})$ take values in a closed interval in $(0,\infty)$. Together with Condition~\ref{cond_sUconv},
the same conclusion holds for $\nabla_{\bm{\beta}^{(b)}}^2\nlog^{(\lam)}_n({\bm{\beta}})$ with probability going to one as $n\rightarrow\infty$.  Plug in $\bm{\beta}$ by $\hat{\bm{\beta}}^{(\lam, 1)}_n $. When $\nabla_{\bm{\beta}^{(b)}}^2\nlog^{(\lam)}_n(\tilde{\bm{\beta}})$ is positive definite, by Equation~\eqref{eq_beta_n_lam_mean_value}
we have
\begin{equation}\label{eq_beta_betahat_bound}
    \hat{\bm{\beta}}^{(\lam, 1)}_n =  \hat{\bm{\beta}}^{(1)}_n  + \bigl(\nabla_{\bm{\beta}^{(b)}}^2\nlog^{(\lam)}_n(\tilde{\bm{\beta}})\bigr)^{-1}\cdot \lam\nabla_{\bm{\beta}^{(b)}}\cdot \penalb(\hat{\bm{\beta}}^{(1)}_n ).
\end{equation}
Therefore, we obtain the result in \eqref{eq_betalam_bound} by the uniform convergence of  $\nabla^2\nlog^{(\lam)}_n(\tilde{\bm{\beta}})$ to $\nabla^2\nlog^{(\lam)}_*(\tilde{\bm{\beta}})$. Note that this result also holds when \A instead of \B updates first.


Finally, we show the result in \eqref{eq_convergence_in_M}.
By Lemma~\ref{lem_linConvNlog},
\begin{equation}\label{eq_M_bound}
  \nlog_n(\hat{\bm{\beta}}^{(1)}_n) - 
\nlog_n(\check{\bm{\beta}}_n) \leq \eta_n
(\nlog_n(\hat{\bm{\beta}}^{(0)}_n) - 
\nlog_n(\check{\bm{\beta}}_n)).  
\end{equation}
We take $\hat{\bm{\beta}}^{(\lam,0)}_n = \hat{\bm{\beta}}^{(0)}_n\in\InitSetN$. Then,  $\hat{\bm{\beta}}^{(\lam,1)}_n\in\InitSetN$. Let $q_9^{(n)}\de \sup_{\bm{\beta}\in\InitSetN}\|\nabla \nlog_n(\bm{\beta})\|_2$, which is upper bounded by the continuity of $\nabla \nlog_n(\bm{\beta})$. By mean value theorem and \eqref{eq_betalam_bound}, 
\begin{align}
|\nlog_n(\hat{\bm{\beta}}^{(1)}_n) -\nlog_n(\hat{\bm{\beta}}^{(\lam, 1)}_n)| 
\leq &\ q_9^{(n)}\cdot\|\hat{\bm{\beta}}^{(1)}_n - \hat{\bm{\beta}}^{(\lam,1)}_n\|_2\nonumber\\
\leq &\ \lam \cdot q_8\cdot q_9^{(n)}.\label{eq_M_and_beta_bound}
\end{align}
By the uniform convergence of $\nabla \nlog_n(\bm{\beta})$ to $\nabla \nlog_*(\bm{\beta})$ from Condition~\ref{cond_sUconv}, there exists a positive constant $q_9$, such that $q_9^{(n)}\leq q_9$ with probability going to one as $n\rightarrow\infty$.  Since 
$\eta_n\limp\eta_*$ as $n\rightarrow\infty$, where $0<\eta_*<1$ (defined in Equation~\eqref{eq_etaStar_def}), there exists a positive constant $q_5<1$, such that $\eta_n\leq q_5$ with probability going to one as $n\rightarrow\infty$.
Combining Inequalities~\eqref{eq_M_bound} and \eqref{eq_M_and_beta_bound} (also recall that the initial values $\nlog_n(\hat{\bm{\beta}}^{(\lam, 0)}_n) = \nlog_n(\hat{\bm{\beta}}^{(0)}_n)$), we have
$$
  |\nlog_n(\hat{\bm{\beta}}^{(\lam, 1)}_n) - 
\nlog_n(\check{\bm{\beta}}_n)| \leq q_5\cdot |
(\nlog_n(\hat{\bm{\beta}}^{(\lam, 0)}_n) - 
\nlog_n(\check{\bm{\beta}}_n))| + \lam \cdot q_8\cdot q_9, 
$$
holds 
for all $\hat{\bm{\beta}}^{(\lam, 0)}_n\in\InitSetN$
with probability going to one as $n\rightarrow\infty$.
By replacing $\hat{\bm{\beta}}^{(\lam, 0)}_n\in\InitSetN$ in the above inequality by $\hat{\bm{\beta}}^{(\lam, 1)}_n, \hat{\bm{\beta}}^{(\lam, 2)}_n, \dots \hat{\bm{\beta}}^{(\lam, k'-1)}_n$, we have
\begin{align*}
 |\nlog_n(\hat{\bm{\beta}}^{(\lam, 2)}_n) - 
\nlog_n(\check{\bm{\beta}}_n)| \leq &\ q_5\cdot |
(\nlog_n(\hat{\bm{\beta}}^{(\lam, 1)}_n) - 
\nlog_n(\check{\bm{\beta}}_n))| + \lam \cdot q_8\cdot q_9,\\
 |\nlog_n(\hat{\bm{\beta}}^{(\lam, 3)}_n) - 
\nlog_n(\check{\bm{\beta}}_n)| \leq &\ q_5\cdot |
(\nlog_n(\hat{\bm{\beta}}^{(\lam, 2)}_n) - 
\nlog_n(\check{\bm{\beta}}_n))| + \lam \cdot q_8\cdot q_9,\\
&\dots,\\
 |\nlog_n(\hat{\bm{\beta}}^{(\lam, k')}_n) - 
\nlog_n(\check{\bm{\beta}}_n)| \leq &\ q_5\cdot |
(\nlog_n(\hat{\bm{\beta}}^{(\lam, k'-1)}_n) - 
\nlog_n(\check{\bm{\beta}}_n))| + \lam \cdot q_8\cdot q_9.
\end{align*}
By combining the above results, we have
$$
  |\nlog_n(\hat{\bm{\beta}}^{(\lam, k')}_n) - 
\nlog_n(\check{\bm{\beta}}_n)| \leq q_5^{k'}\cdot
\sup_{\bm{\beta}^{(1)}, \bm{\beta}^{(2)}\in\ThetaSet}|\nlog_n(\bm{\beta}^{(1)}) - \nlog_n(\bm{\beta}^{(2)})| + \lam \cdot q_8\cdot q_9\cdot \sum_{h=0}^{k'} q_5^h. 
$$
Based on Condition~\ref{cond_sUconv} and the fact that
 $\sup_{\bm{\beta}^{(1)}, \bm{\beta}^{(2)}\in\ThetaSet}|\nlog_*(\bm{\beta}^{(1)}) - \nlog_*(\bm{\beta}^{(2)})|< \infty$ and $\hat{\bm{\beta}}^{(\lam, 0)}_n\in \InitSetN\subseteq \ThetaSet$, there exists a positive constant $q_6$ that
$$
\P\bigl(\sup_{\bm{\beta}^{(1)}, \bm{\beta}^{(2)}\in\ThetaSet}|\nlog_n(\bm{\beta}^{(1)}) - \nlog_n(\bm{\beta}^{(2)})|<q_6\bigr)\rightarrow 1,
$$
as $n\rightarrow\infty$.
Additionally, since $\sum_{h=0}^{k'} q_5^h$ converges to a positive constant as $k'\rightarrow\infty$, it is upper bounded. Let $q_7$ be an upper bound of $q_8\cdot q_9\cdot \sum_{h=0}^{k'} q_5^h$. Accordingly,
we obtain the result in \eqref{eq_convergence_in_M} and finish the proof.

\end{proof}

\subsection{Simulation study: ridge regression}\label{subsec_ridge}
In this subsection, we present simulation examples of ridge regression.
The data are generated by linear regression models with the following two settings.
\begin{description}
\item[Setting 1: ]    $y =  \beta_1x_1 +\cdots+  \beta_{12}x_{12} +\varepsilon$, where $\A$ holds $x_1,\dots,x_6$ and $\B$ holds $x_7,\dots, x_{12}$,
\item[Setting 2: ]    $y =  \beta_1x_1 +\cdots+  \beta_{24}x_{24}+\varepsilon$, where $\A$ holds $x_1,\dots,x_{20}$ and $\B$ holds $x_5,\dots, x_{24}$,
\end{description}
where $\varepsilon\sim \mathcal N(0,1)$. The covariate observations are generated from the same procedure with the one in Section~\ref{subsec_type1}  where we take $\rho=0.9$. 

For the test under $\textrm{H}_0$, we take $\beta_1,\dots \beta_6$ in Setting~1 and $\beta_1\dots,\beta_{20}$ in Setting~2 to be one and other coefficients to be zero. The scale of the Laplace noise (refer to Section~\ref{sec_priv} in the main text) is either 0.1 or 0.5. The penalty parameter $\lam$ is selected from 10-fold CV. We repeat the experiment 100 times and plot the Q-Q plots of the p-values of the test statistic with $\id=1,\dots,5$. The results of Settings~1 and 2 are presented in Figures~\ref{ridgeQQ_setting1} and \ref{ridgeQQ_setting2}, respectively. It shows that when the sample size is small ($n=200$), the test statistic may fail to control the type-one error (blue dots may fall below the red straight lines) due to the added penalty in model fitting. Nevertheless, as the sample size increases, the performance of the test statistic improves. We also plot the box plots of the selected penalty parameters ($\lam$) in Figure~\ref{ridgeH0_lam}. It shows that in the considered settings, the selected $\lam$ tends to get smaller when $n$ gets larger, which aligns with the requirement of Corollary~\ref{coro_test_penal}.
\begin{figure}
\centering
\begin{subfigure}[b]{0.9\textwidth}
   \includegraphics[width=1\linewidth]{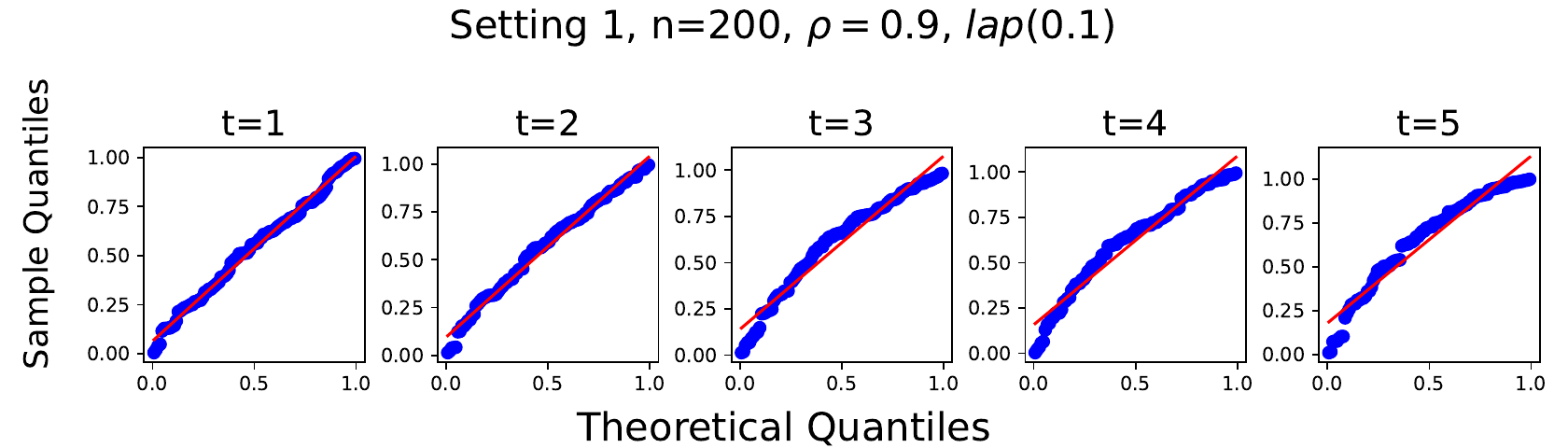}
\end{subfigure}\vspace{0.5em}
\begin{subfigure}[b]{0.9\textwidth}
   \includegraphics[width=1\linewidth]{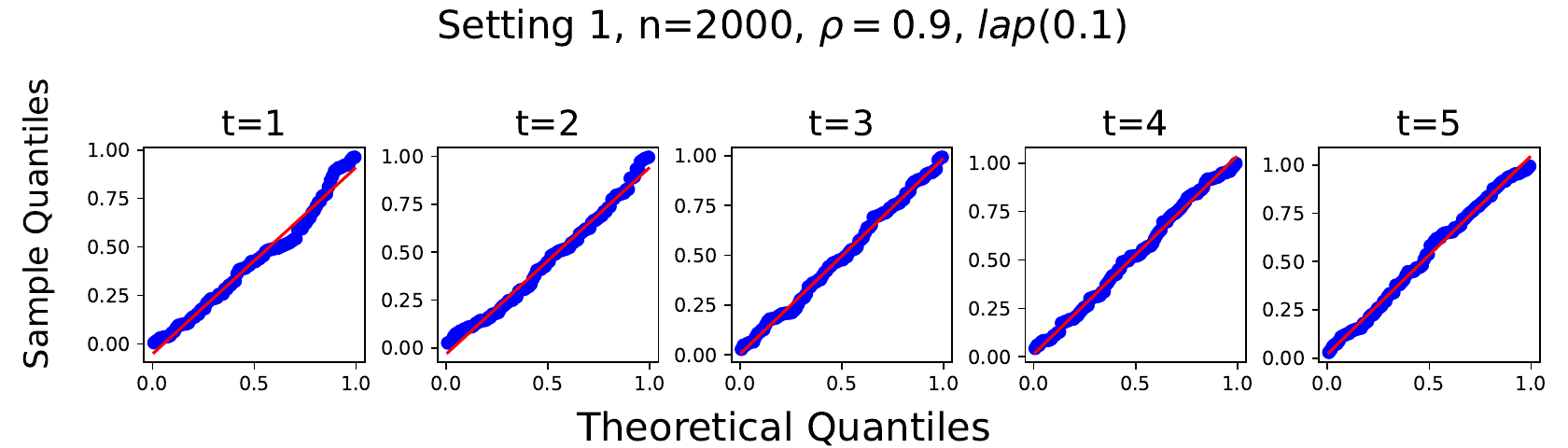}
\end{subfigure}\vspace{0.5em}
\begin{subfigure}[b]{0.9\textwidth}
   \includegraphics[width=1\linewidth]{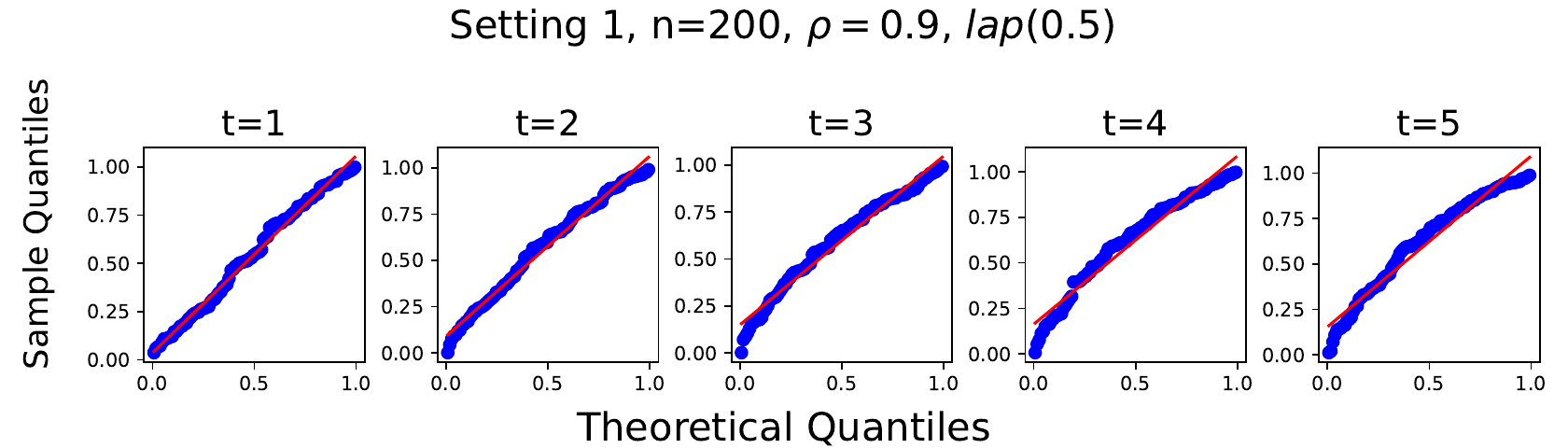}
\end{subfigure}\vspace{0.5em}
\begin{subfigure}[b]{0.9\textwidth}
   \includegraphics[width=1\linewidth]{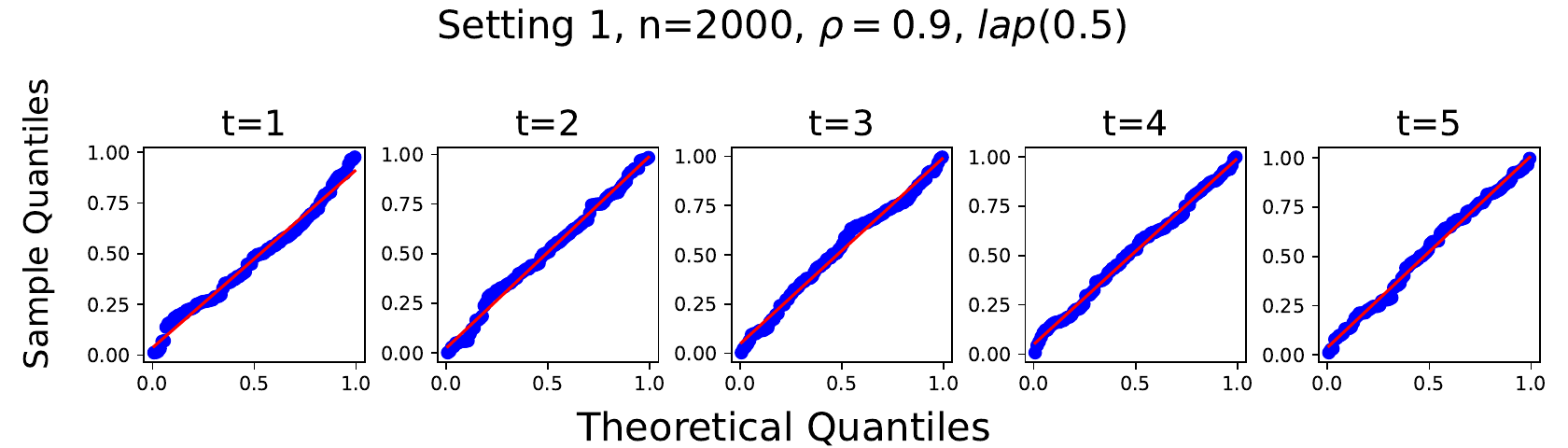}
\end{subfigure}
\caption{The Q-Q plot of the p-values of $W_{n,\id}$ from the ridge regression in Setting~1. The red straight line stands for the perfect match between the sample quantiles and theoretical quantiles, and $\textrm{Lap}(0.1)$ and $\textrm{Lap}(0.5)$ denote adding Laplace noise with scale $0.1$ and $0.5$, 
respectively.}\label{ridgeQQ_setting1}
\end{figure}

\begin{figure}
\centering
\begin{subfigure}[b]{0.9\textwidth}
   \includegraphics[width=1\linewidth]{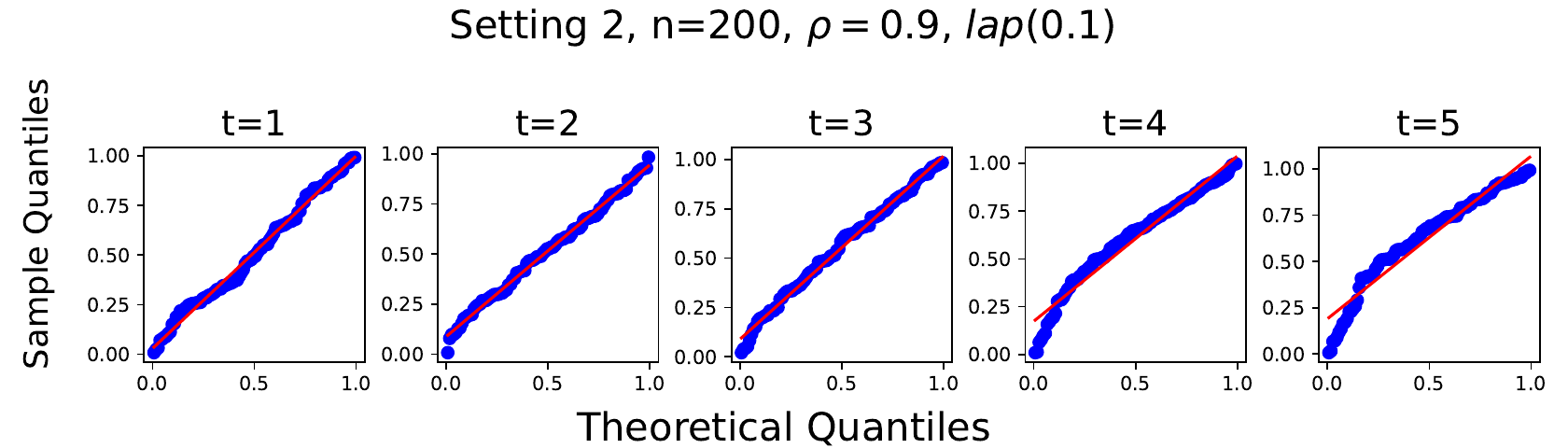}
\end{subfigure}\vspace{0.5em}
\begin{subfigure}[b]{0.9\textwidth}
   \includegraphics[width=1\linewidth]{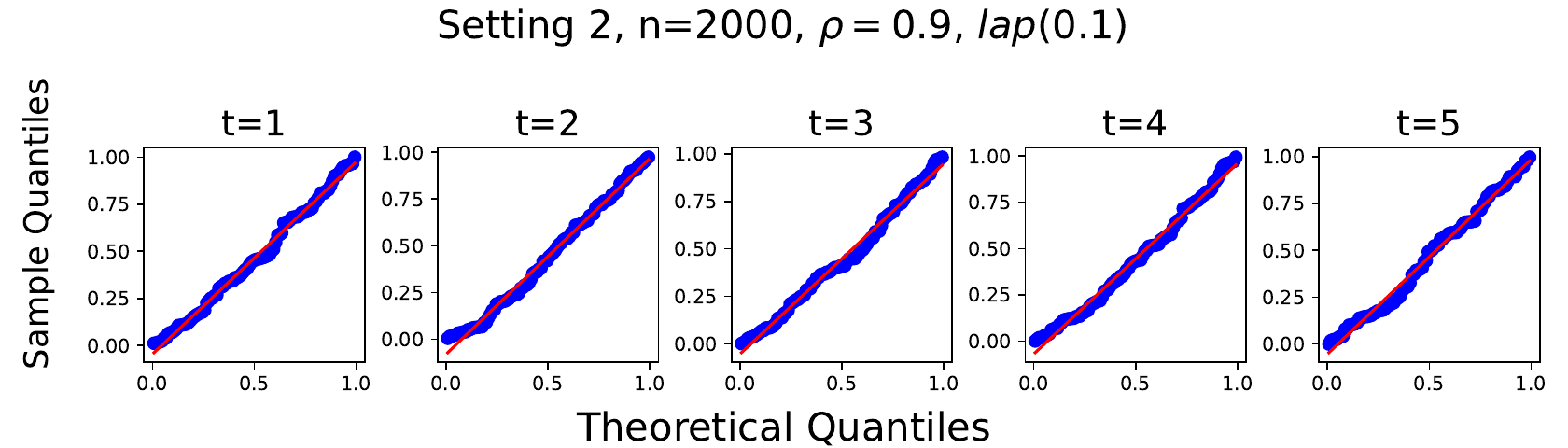}
\end{subfigure}\vspace{0.5em}
\begin{subfigure}[b]{0.9\textwidth}
   \includegraphics[width=1\linewidth]{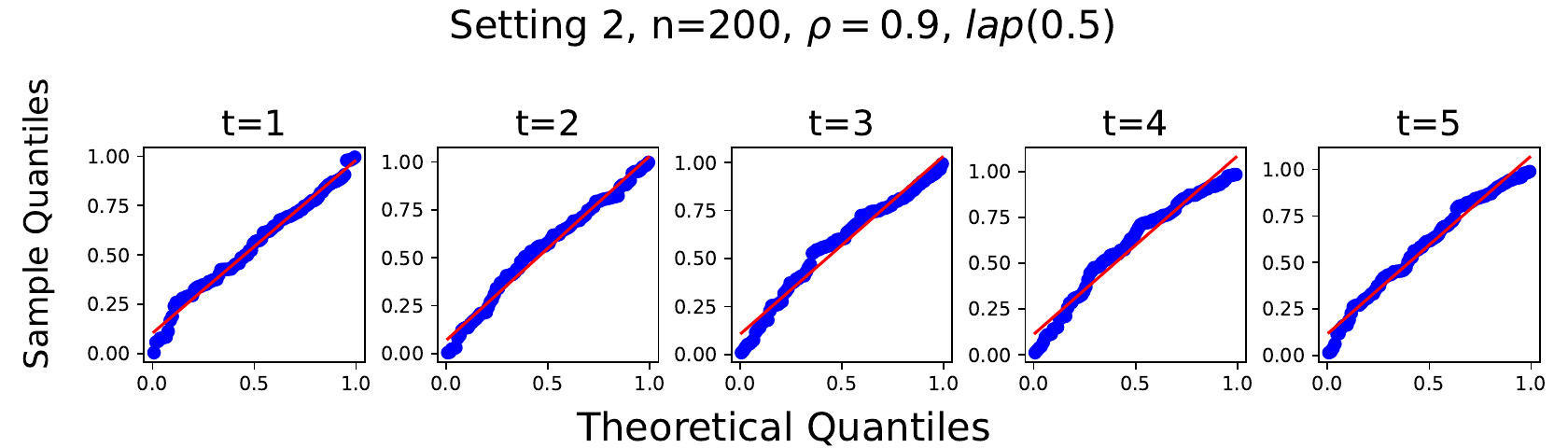}
\end{subfigure}\vspace{0.5em}
\begin{subfigure}[b]{0.9\textwidth}
   \includegraphics[width=1\linewidth]{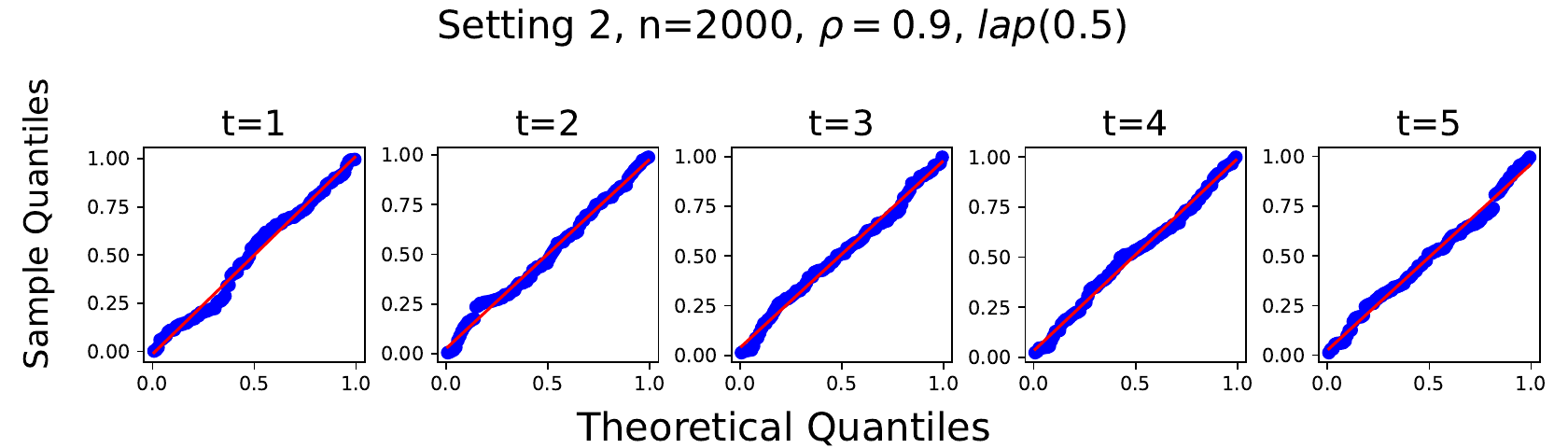}
\end{subfigure}
\caption{The Q-Q plot of the p-values of $W_{n,\id}$ from the ridge regression in Setting 2. The notations are the same with Figure~\ref{ridgeQQ_setting1}.}\label{ridgeQQ_setting2}
\end{figure}
\begin{figure}[!htb]
	\centering
	\vspace{-0.1cm}
	\includegraphics[width=1\linewidth]{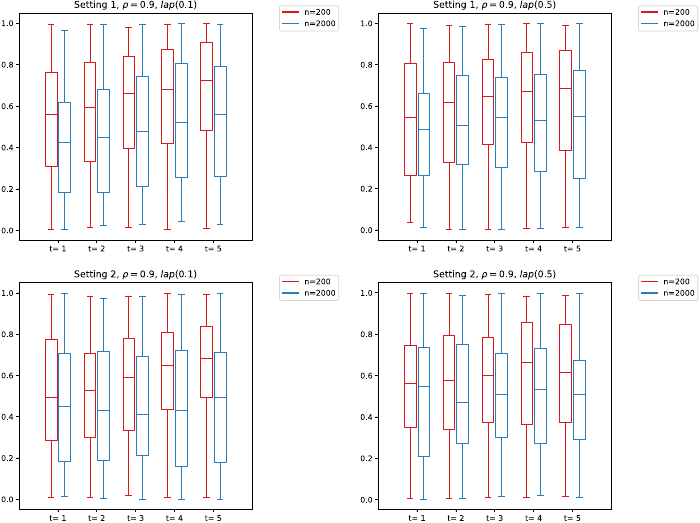}
	\vspace{-0.0cm}
	\caption{The selected penalty parameter values of the ridge regression model under $\textrm{H}_0$.
	}
	\label{ridgeH0_lam}
\end{figure}
For the test under $\textrm{H}_1$, the coefficients are i.i.d.\ generated from the standard normal distribution. We calculate the rejection rates of the test statistic at the significance level of 0.05.  The remaining settings are the same as those for $\textrm{H}_0$.  Figure~\ref{ridgeH1_rej} demonstrates that the addition of the penalty term has little impact on the rejection rate of the test statistic. In Figure~\ref{ridgeH1_lam}, the boxplot of $\lam$ illustrates that as the sample size increases, the selected $\lam$ tends to become smaller.
\begin{figure}[!htb]
	\centering
	\vspace{-0.1cm}
	\includegraphics[width=1\linewidth]{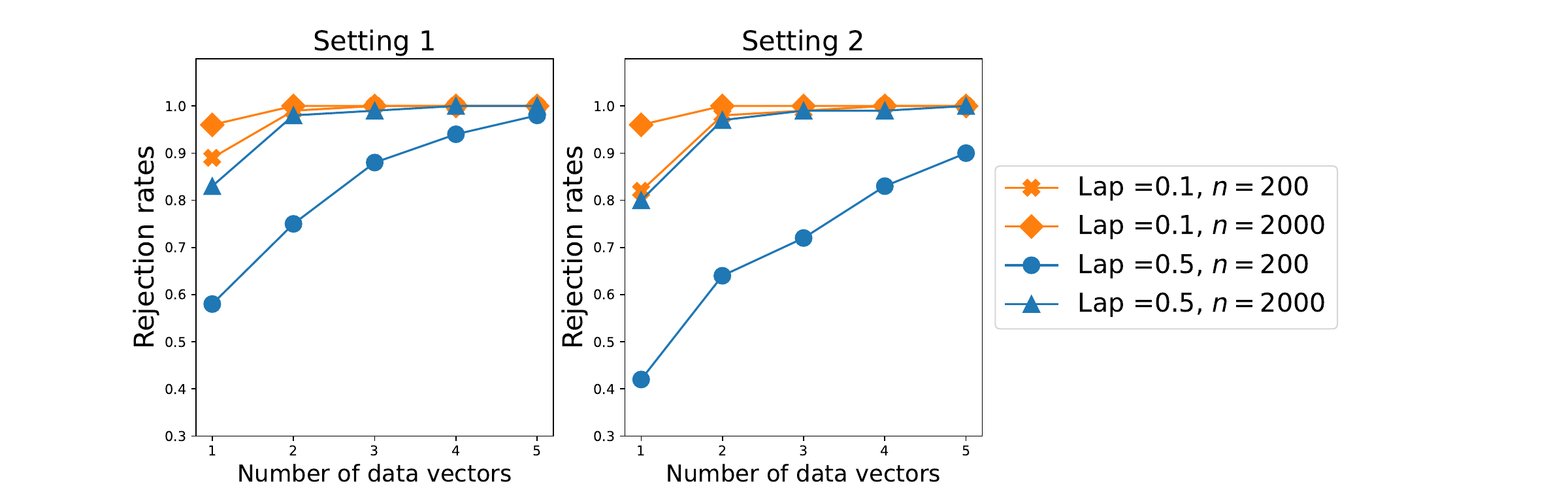}
	\vspace{-0.0cm}
	\caption{The rejection rates of $W_{n,\id}^{(\lam)}$ at the significance level of $0.05$ versus the number of vectors sent in 100 replications.
	}
	\label{ridgeH1_rej}
\end{figure}

\begin{figure}[!htb]
	\centering
	\vspace{-0.1cm}
	\includegraphics[width=1\linewidth]{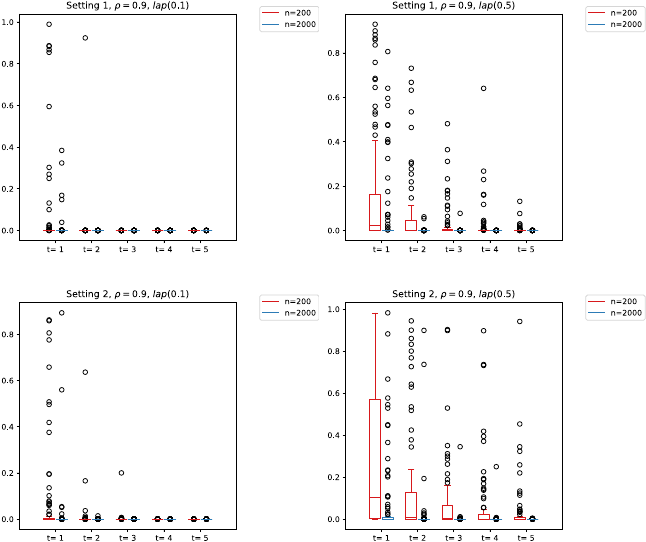}
	\vspace{-0.0cm}
	\caption{The selected penalty parameter values of the ridge regression model under $\textrm{H}_1$.
	}
	\label{ridgeH1_lam}
\end{figure}

For model training performance, we apply the same data generating procedure as the one for testing under $\textrm{H}_1$. The penalty parameters $\lam$ for both the AE-AL model and the oracle model are obtained through 10-fold CVs. In the case of AE-AL, we select the $\lam$ value with the smallest CV-prediction mean squared error (MSE) after five iteration rounds.
Figures~\ref{ridgeMSE} and \ref{ridgeTraining_lam} show that AE-AL performs well in the iterative training, due to the relatively small penalty parameters.

\begin{figure}
\centering
\begin{subfigure}[b]{1\textwidth}
   \includegraphics[width=1\linewidth]{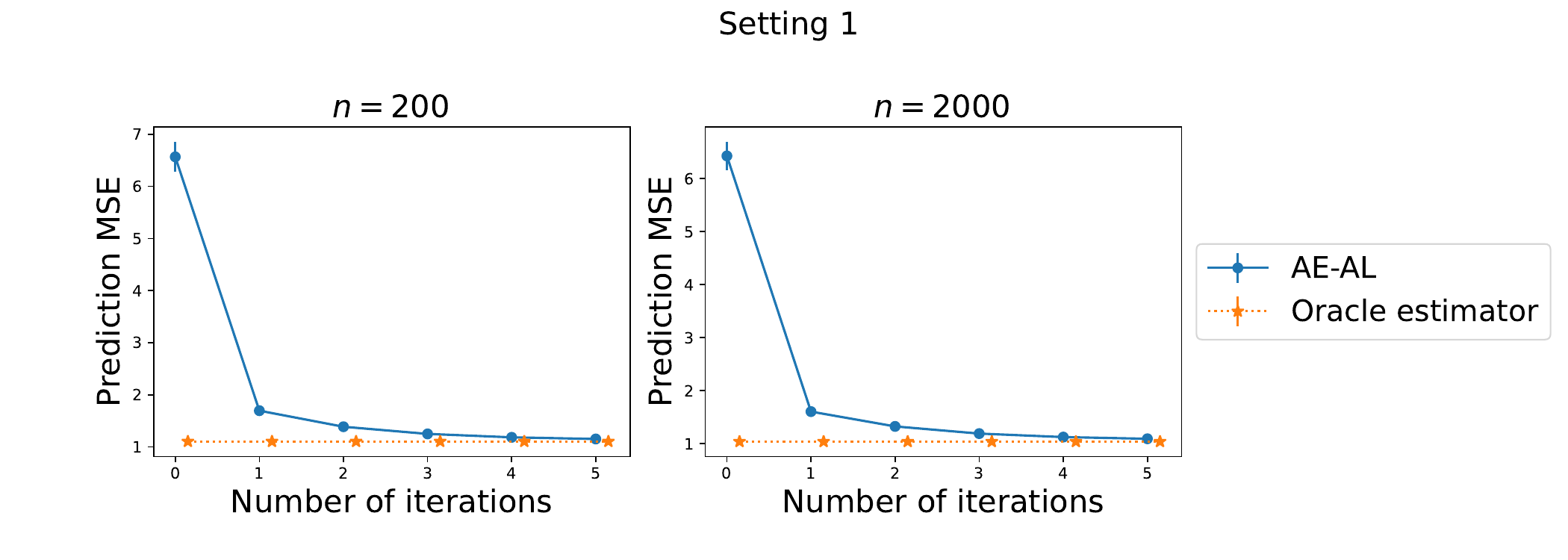}
\end{subfigure}\vspace{0.8em}
\begin{subfigure}[b]{1\textwidth}
   \includegraphics[width=1\linewidth]{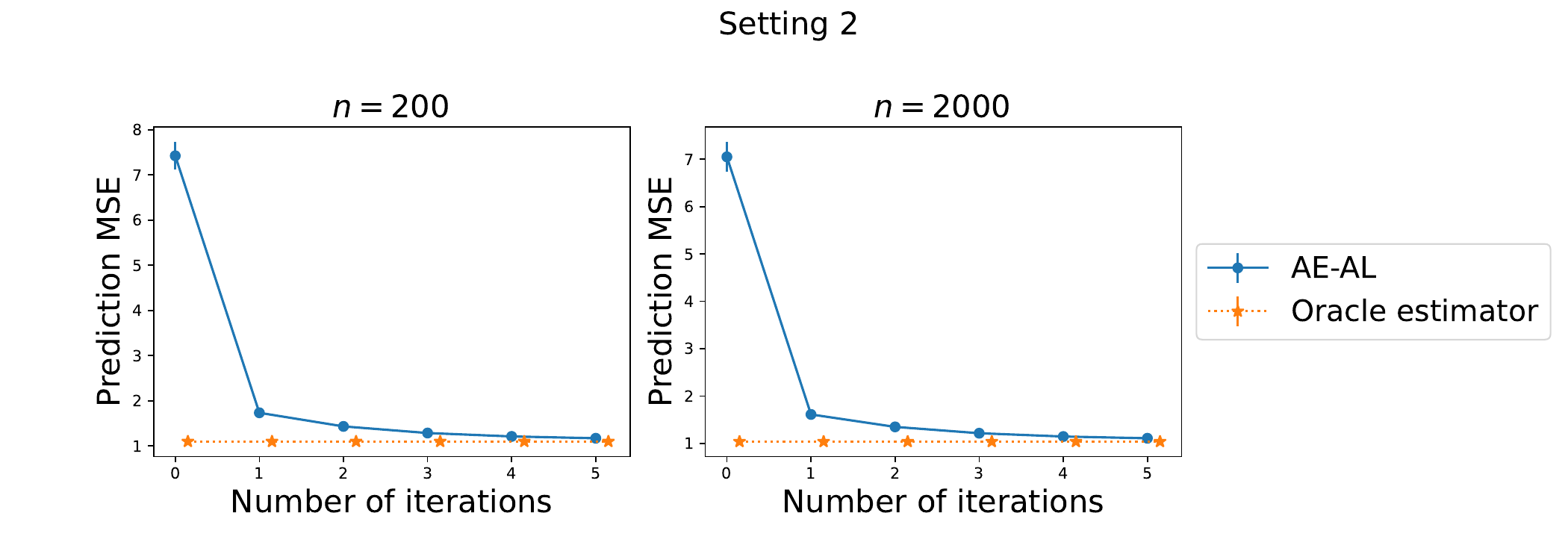}
\end{subfigure}
\caption{The average prediction MSE of the AE-AL estimator and the oracle estimator for ridge regression on the evaluation data. The vertical bars represent the standard errors from 100 replications.}\label{ridgeMSE}
\end{figure}

\begin{figure}[!htb]
	\centering
	\vspace{-0.1cm}
	\includegraphics[width=1\linewidth]{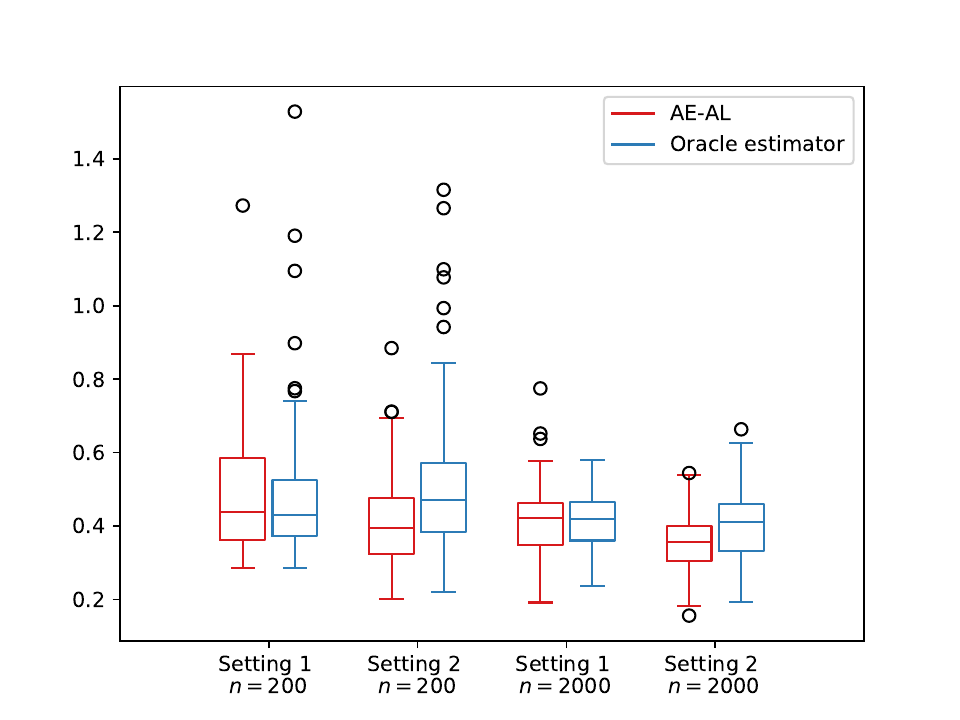}
	\vspace{-0.0cm}
	\caption{The selected penalty parameter values of AE-AL and the oracle model in training the ridge regression models.
	}
	\label{ridgeTraining_lam}
\end{figure}

\subsection{Simulation study: Lasso regression and elastic net regression}\label{subsec_Enet}

Although Corollary~\ref{coro_test_penal} does not directly apply to Lasso regression and Elastic Net regression, it is still worth considering their application when the number of covariates $p$ exceeds the sample size $n$. 
In this subsection, we conduct simulation studies to assess the performance of AE-AL training for these two models in high-dimensional settings.

We generate data by sparse linear regression models under the following two settings.
\begin{description}
\item[Setting 1: ]    
\begin{align*}
y =&\   0.3 x_1 + 0.2 x_2 + 0.1 x_3 + 0.3 x_{201} + 0.2 x_{202} + 0.1 x_{203} + \\
&\ 0.3 x_{401} + 0.2 x_{402} + 0.1 x_{403} + \varepsilon.
\end{align*}
Let $x^{(a)} = (x_1,\dots,x_{200})^\T, x^{(b)} =(x_{201}\cdots, x_{400})^\T, x^{(c)} =(x_{401}\cdots, x_{600})^\T$. \A holds the observations of $x^{(a)}$ and $x^{(c)}$, and \B holds the observations of $x^{(b)}$ and $x^{(c)}$. 
\item[Setting 2: ]    \begin{align*}
y =&\   0.3 x_1 + 0.2 x_2 + 0.1 x_3 + 0.3 x_{301} + 0.2 x_{302} + 0.1 x_{303} + \varepsilon.
\end{align*}
\A holds the observations of $x_1,\dots,x_{300}$ and \B holds the observations of $x_{301},\dots,x_{600}$.
\end{description}
The random noise $\varepsilon\sim \mathcal N(0,1)$. 
 We apply the same procedures to generate covariate observations and select tuning parameters as the previous section. The averaged prediction MSEs from 100 replications are shown in Figure~\ref{EnetMSE}. It can be seen that the predicted MSE of the AE-AL may fail to converge to the one of the oracle estimator. Nevertheless, it still brings an improvement to the original model based on \a's local data only.

\begin{figure}
\centering
\begin{subfigure}[b]{1\textwidth}
   \includegraphics[width=1\linewidth]{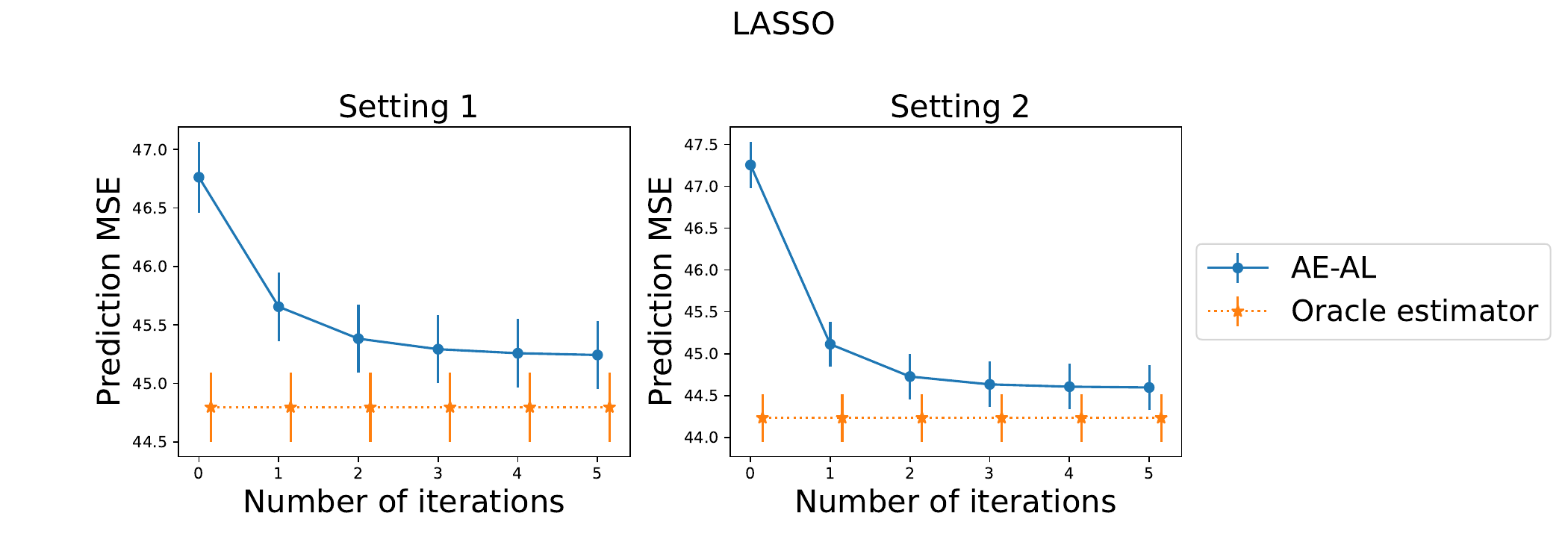}
\end{subfigure}\vspace{0.8em}
\begin{subfigure}[b]{1\textwidth}
   \includegraphics[width=1\linewidth]{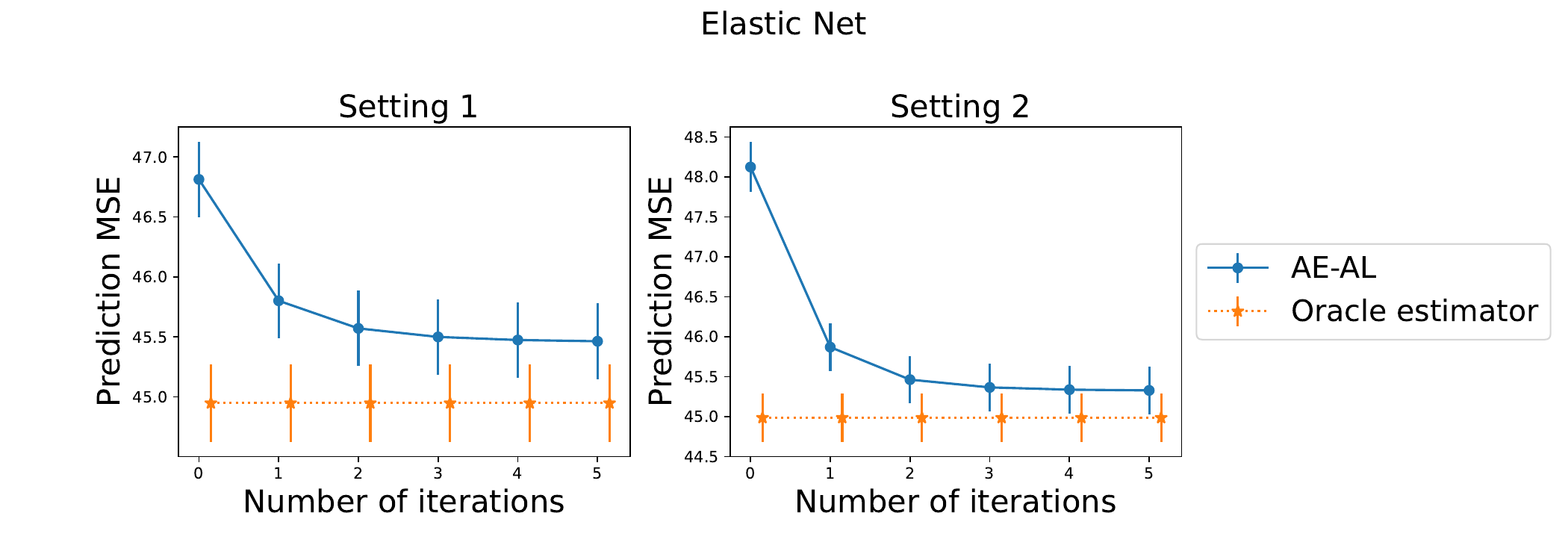}
\end{subfigure}
\caption{The average prediction MSE of the AE-AL estimator and the oracle estimator for Lasso regression and elastic net regression on the evaluation data. The vertical bars represent the standard errors from 100 replications.}\label{EnetMSE}
\end{figure}

\clearpage

\section{Protecting $Y$ in AE-AL Training}\label{sec_masking the response}
Recall that in the ``assisted training'' stage, \B needs to obtain the response from \a.
To protect the response observations, \A may instead send $Y'$ that is obtained from a random transformation of $Y$. 
Then, after jointly fitting the model of $Y'$, \A obtains predictions according to the transformation mechanism. In this section, we provide an example for binary classifications.

For binary classifications, we may apply the randomized response method by \citet{warner1965randomized}.
Specifically, 
\A generates $Y'$ by 
\begin{equation*}
\begin{cases}
\P(Y'=1\mid Y=1) = 1-p',\\
\P(Y'=1\mid Y=0) = p',
\end{cases}
\end{equation*}
where $p'\in (0,0.5)$.
Since 
\begin{align*}
\frac{\P(Y'=1\mid Y=1)}{\P(Y'=1\mid Y=0)} =&\ \frac{\P(Y'=0\mid Y=0)}{\P(Y'=0\mid Y=1)} = \frac{1-p'}{p'},\\
\frac{\P(Y'=1\mid Y=0)}{\P(Y'=1\mid Y=1)} =&\ \frac{\P(Y'=0\mid Y=1)}{\P(Y'=0\mid Y=0)} = \frac{p'}{1-p'},
\end{align*}
this mechanism
satisfies the $\epsilon$-local differential privacy defined in Section~\ref{sec_localdiff_def}
when 
$$
p'\geq \frac{1}{1 + \exp(\epsilon)}.
$$
After fitting $Y'$ and obtaining the joint prediction of $\P(Y'=1\mid \bm{x})$,
\A will calculate the predicted probability for $Y=1$ using the equation:
$$
P(Y=1\mid\bm{x})=\frac{\P(Y'=1\mid \bm{x}) -p'}{(1-2p')},
$$
which is obtained from 
$$
P(Y'=1\mid\bm{x}) = (1-p')\cdot P(Y=1\mid\bm{x}) + p'\cdot \bigl(1-P(Y=1\mid\bm{x})\bigr).
$$

We conduct simulation studies to examine 
the relationship between $p'$ and the predictive performance of the AE-AL estimator. For the experiments, we apply
Setting~2 from 
Section~\ref{subsec_type1}
for AE-AL training where $n=2000$, $\rho=0.1$. The transformation probability $p'$ is taken to be $0.1$, $0.2$, $0.3$, or $0.4$. The AUCs are shown in Figure~\ref{MaskingY_results}. It can be seen that both the AE-AL estimator and the oracle estimator exhibit declining performance as $p'$ increases. Nevertheless, the AE-AL still improves upon the initial model trained solely with the data from \a, which corresponds to $k=0$ in Figure~\ref{MaskingY_results}.
\begin{figure}[!htb]
	\centering
	\vspace{-0.1cm}
	\includegraphics[width=1\linewidth]{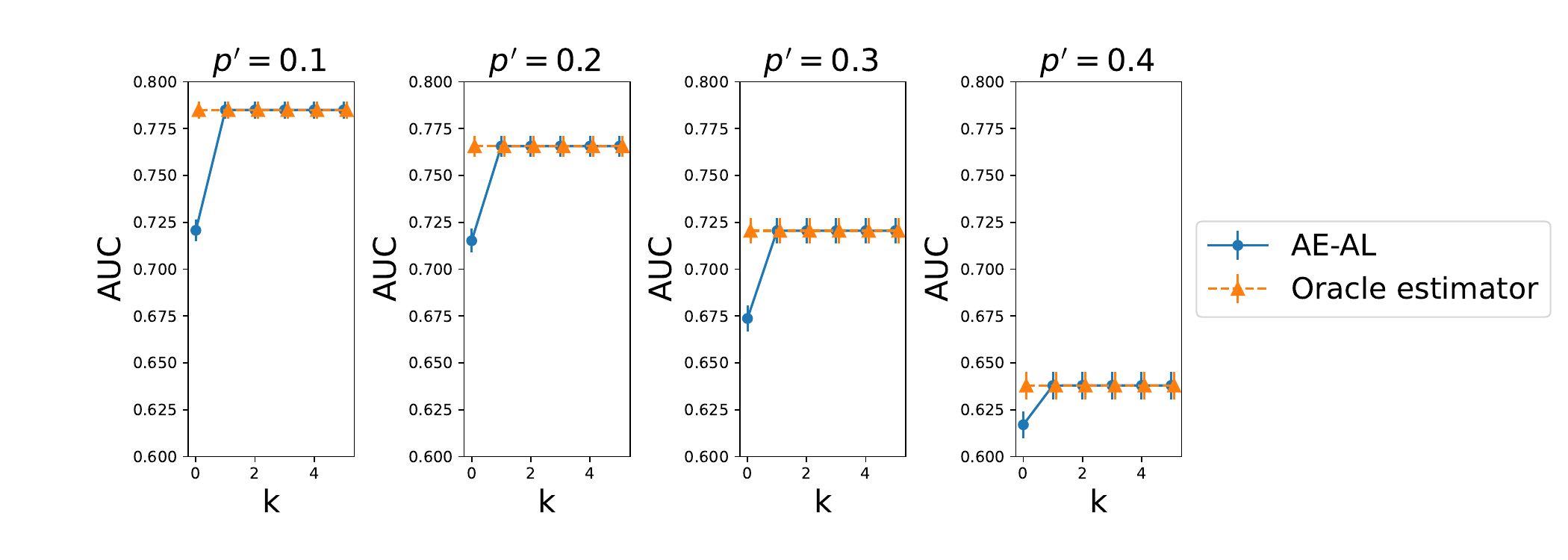}
	\vspace{-0.0cm}
	\caption{AUCs of the AE-AL estimator for logistic regression. The x-axis represents the number of transmisson rounds $k$.
	}
	\label{MaskingY_results}
\end{figure}

\section{Details of the Hospital Length of Stay data}\label{sec_hosp_data_detail}
We follow the data pre-processing procedure of \citet{wang2022parallel}.
The resulting dataset consists 16 covariate measurements of 6916 patients, collected from intensive care units (ICU) and laboratories. 
These covariates are height, weight, temperature, heart rate, oxygen saturation, respiratory rate, diastolic blood pressure, systolic blood pressure, fraction inspired oxygen, mean blood pressure, glascow coma scale total, glascow coma scale eye, glascow coma scale motor, glascow coma scale verbal, pH, and glucose. We assume that the first 10 covariates are from the ICU and the remaining ones are from the labotarary.

\section{Simulation Studies}\label{sec_addexp}

\subsection{The remaining Q-Q plots of the logistic regression}\label{subsec_QQlogistic}
The remaining Q-Q plots from Section~\ref{subsec_type1}  are shown in Figures~\ref{logisticQQ_setting1_sig0_lap0}-\ref{logisticQQ_setting3_sig0.25_lap0.5}.

			\begin{figure}[tb]
				\centering
			\begin{subfigure}[b]{0.9\textwidth}
			\includegraphics[width=1\linewidth]{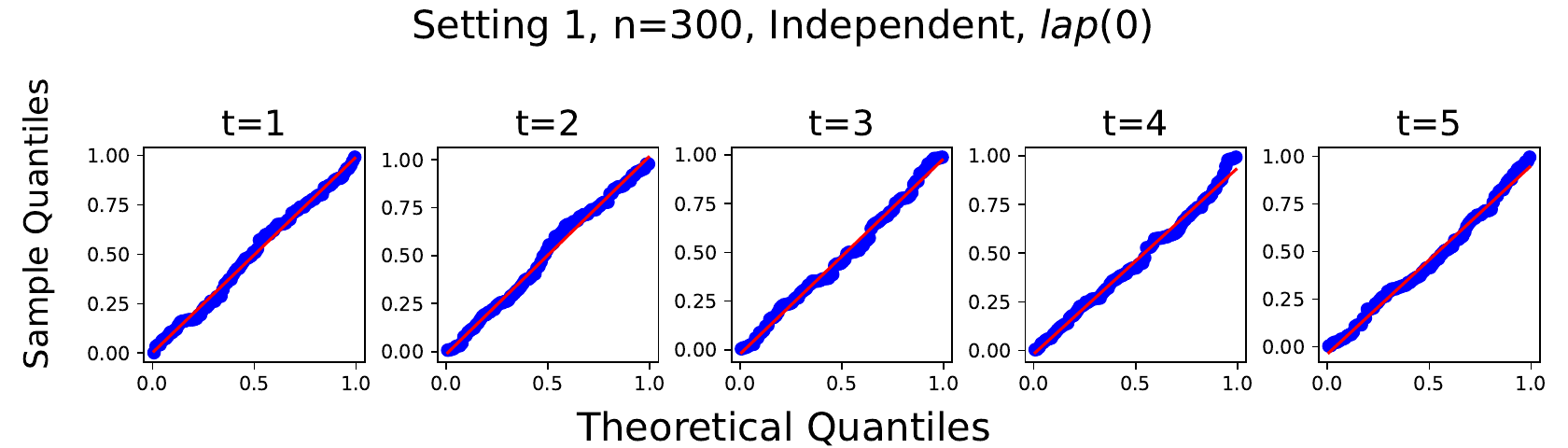}
			\end{subfigure}\vspace{0.8em}
			\begin{subfigure}[b]{0.9\textwidth}
			\includegraphics[width=1\linewidth]{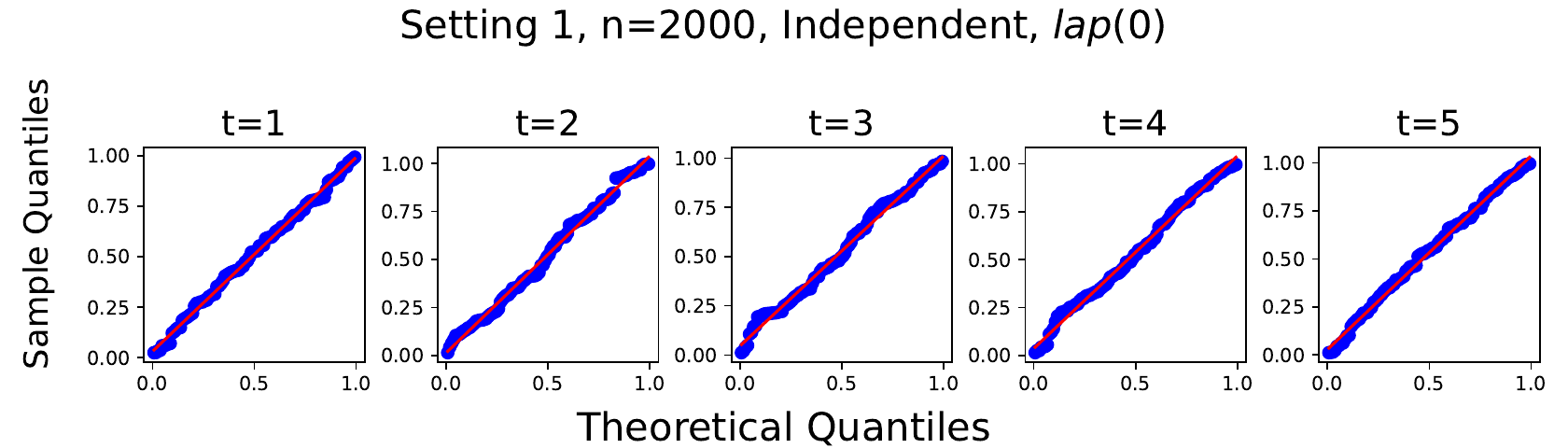}
			\end{subfigure}
				\caption{The Q-Q plot of $W_{n,\id}$ for logistic regression from Setting 1 with Independent covariates,  and without the Laplace noise. 
				}
				\label{logisticQQ_setting1_sig0_lap0}
			\end{figure}

			\begin{figure}[tb]
				\centering
			\begin{subfigure}[b]{0.9\textwidth}
			\includegraphics[width=1\linewidth]{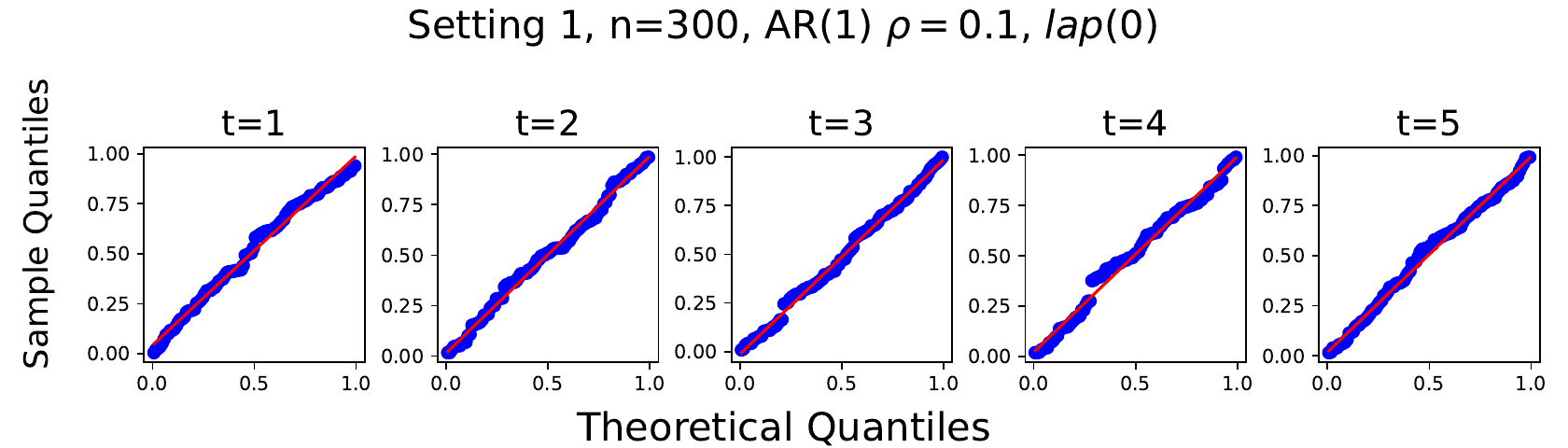}
			\end{subfigure}\vspace{0.8em}
			\begin{subfigure}[b]{0.9\textwidth}
			\includegraphics[width=1\linewidth]{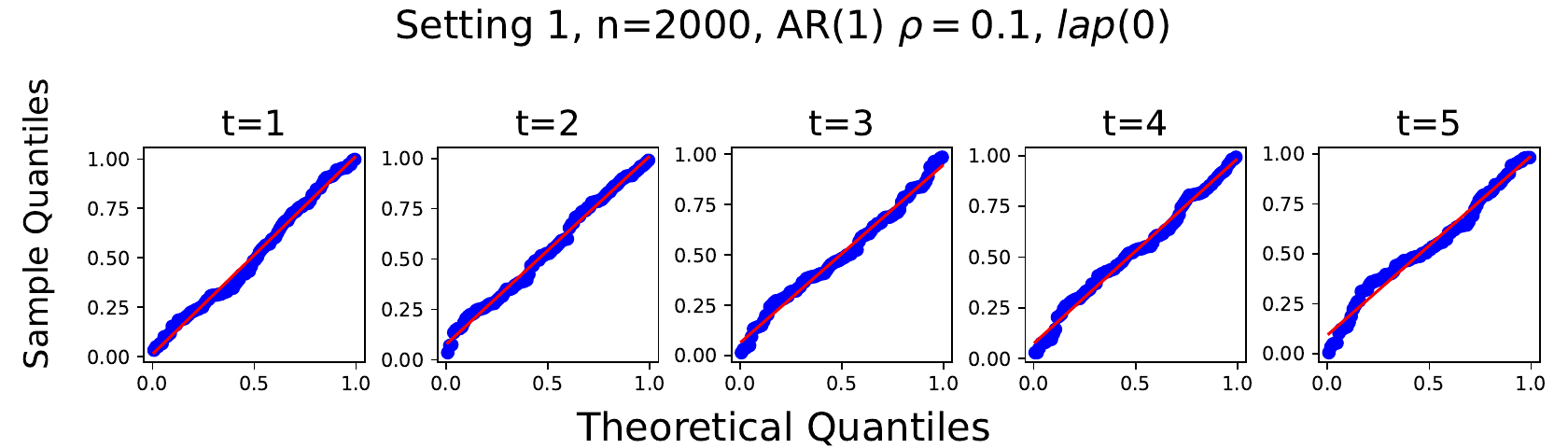}
			\end{subfigure}
				\caption{The Q-Q plot of $W_{n,\id}$ for logistic regression from Setting 1 with AR(1) covariates, $\rho =0.1$,  and without the Laplace noise. 
				}
				\label{logisticQQ_setting1_sig0.1_lap0}
			\end{figure}

			\begin{figure}[tb]
				\centering
			\begin{subfigure}[b]{0.9\textwidth}
			\includegraphics[width=1\linewidth]{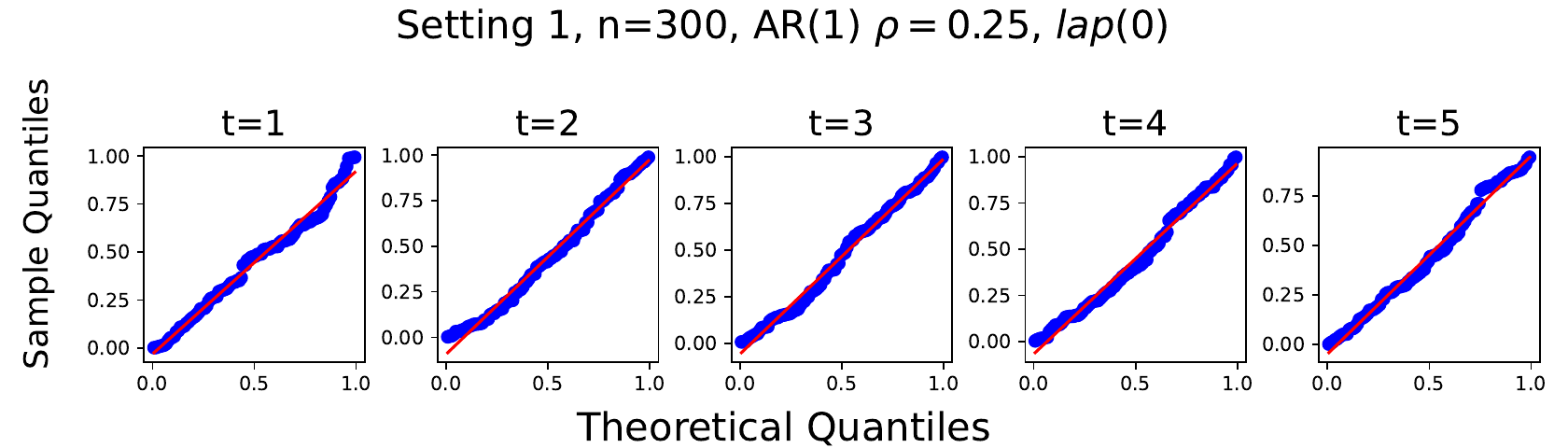}
			\end{subfigure}\vspace{0.8em}
			\begin{subfigure}[b]{0.9\textwidth}
			\includegraphics[width=1\linewidth]{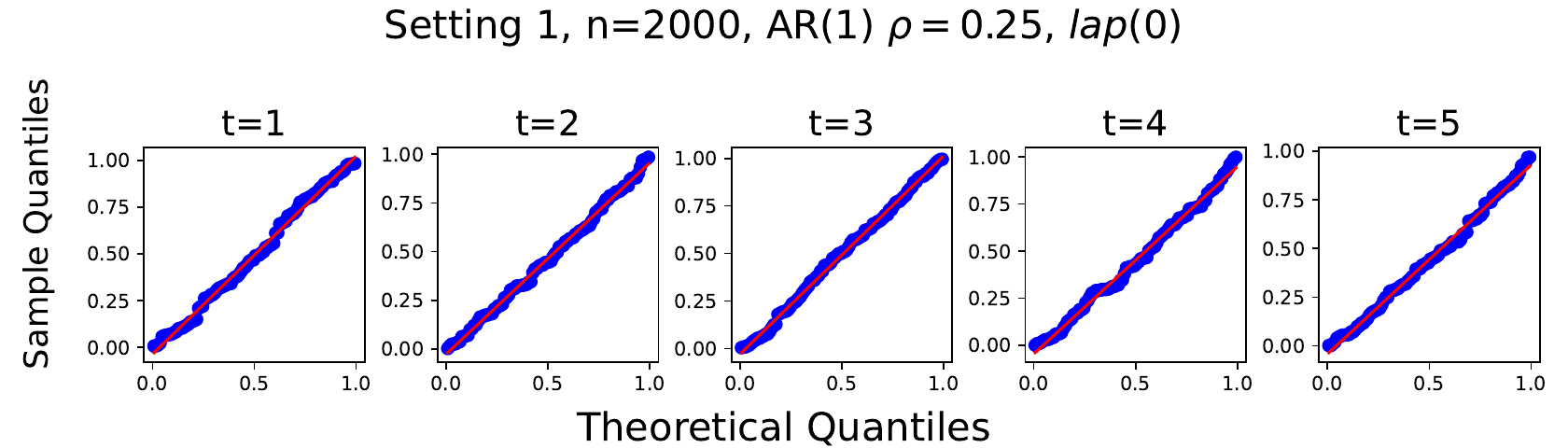}
			\end{subfigure}
				\caption{The Q-Q plot of $W_{n,\id}$ for logistic regression from Setting 1 with AR(1) covariates, $\rho =0.25$,  and without the Laplace noise. 
				}
				\label{logisticQQ_setting1_sig0.25_lap0}
			\end{figure}

			\begin{figure}[tb]
				\centering
			\begin{subfigure}[b]{0.9\textwidth}
			\includegraphics[width=1\linewidth]{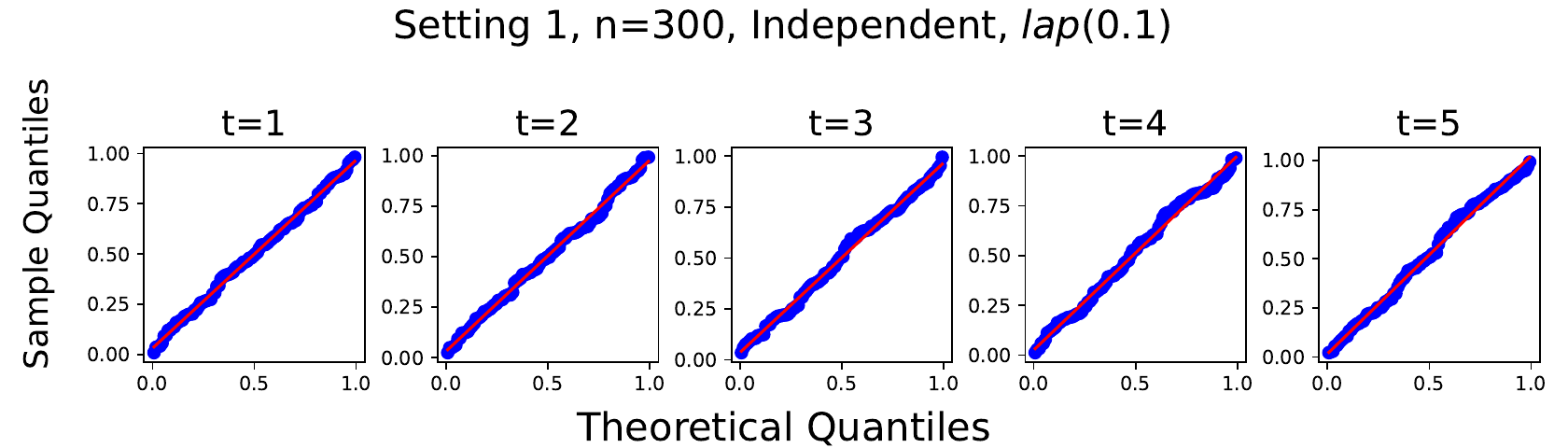}
			\end{subfigure}\vspace{0.8em}
			\begin{subfigure}[b]{0.9\textwidth}
			\includegraphics[width=1\linewidth]{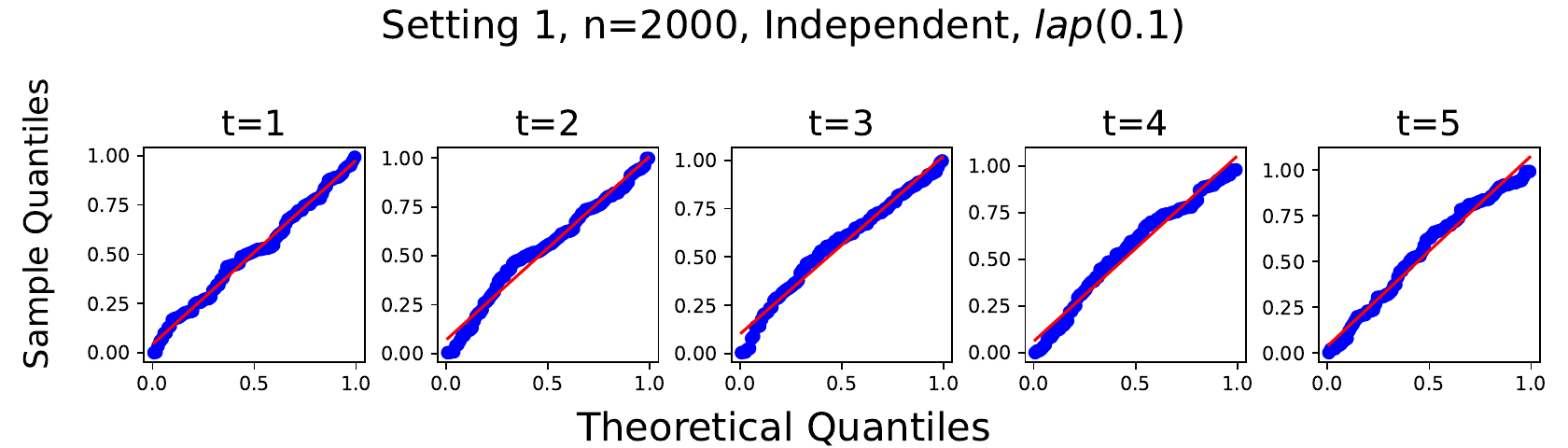}
			\end{subfigure}
				\caption{The Q-Q plot of $W_{n,\id}$ for logistic regression from Setting 1 with Independent covariates,  and the scale of the Laplace noise is 0.1. 
				}
				\label{logisticQQ_setting1_sig0_lap0.1}
			\end{figure}

			\begin{figure}[tb]
				\centering
			\begin{subfigure}[b]{0.9\textwidth}
			\includegraphics[width=1\linewidth]{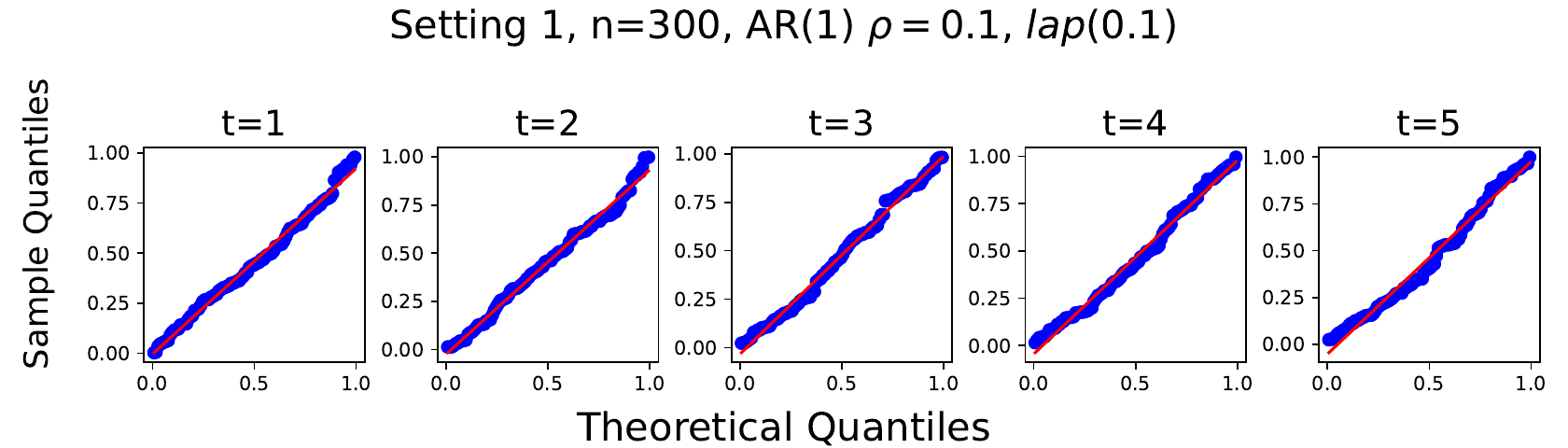}
			\end{subfigure}\vspace{0.8em}
			\begin{subfigure}[b]{0.9\textwidth}
			\includegraphics[width=1\linewidth]{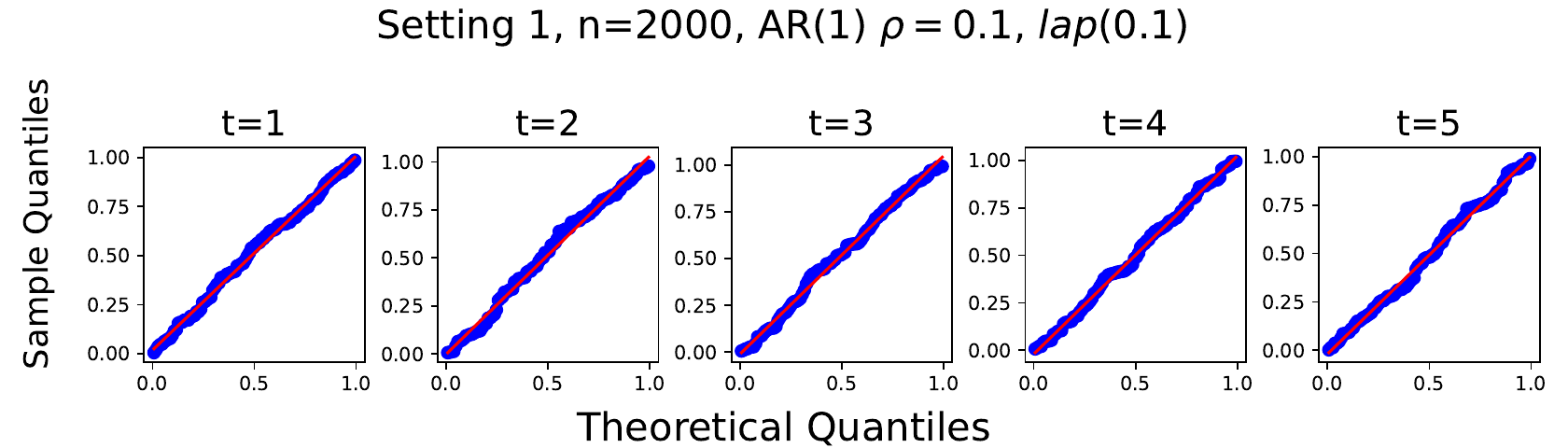}
			\end{subfigure}
				\caption{The Q-Q plot of $W_{n,\id}$ for logistic regression from Setting 1 with AR(1) covariates, $\rho =0.1$,  and the scale of the Laplace noise is 0.1. 
				}
				\label{logisticQQ_setting1_sig0.1_lap0.1}
			\end{figure}

			\begin{figure}[tb]
				\centering
			\begin{subfigure}[b]{0.9\textwidth}
			\includegraphics[width=1\linewidth]{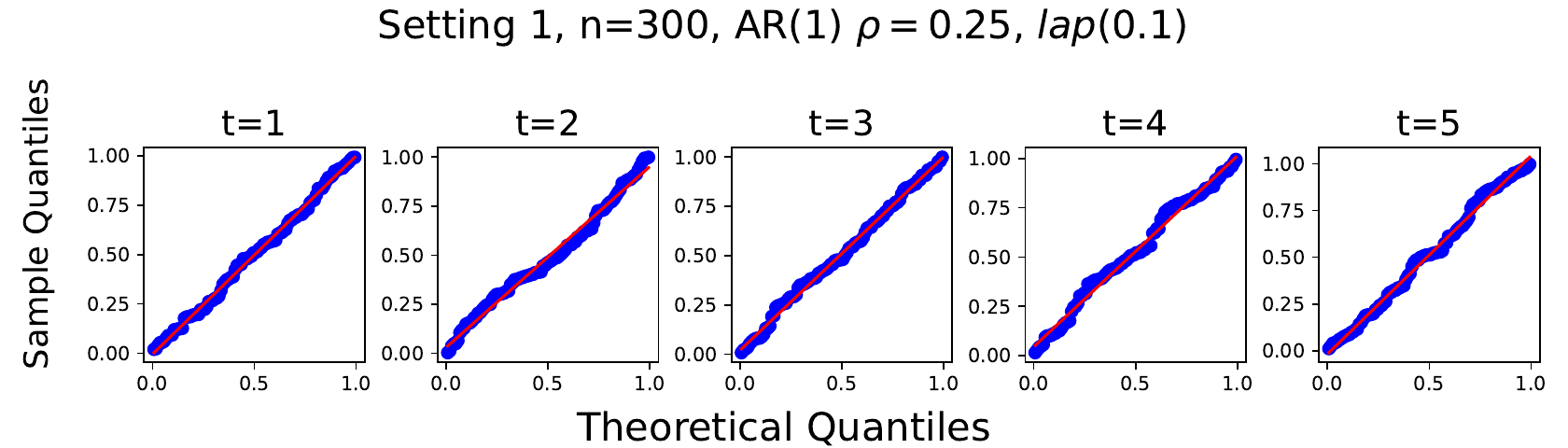}
			\end{subfigure}\vspace{0.8em}
			\begin{subfigure}[b]{0.9\textwidth}
			\includegraphics[width=1\linewidth]{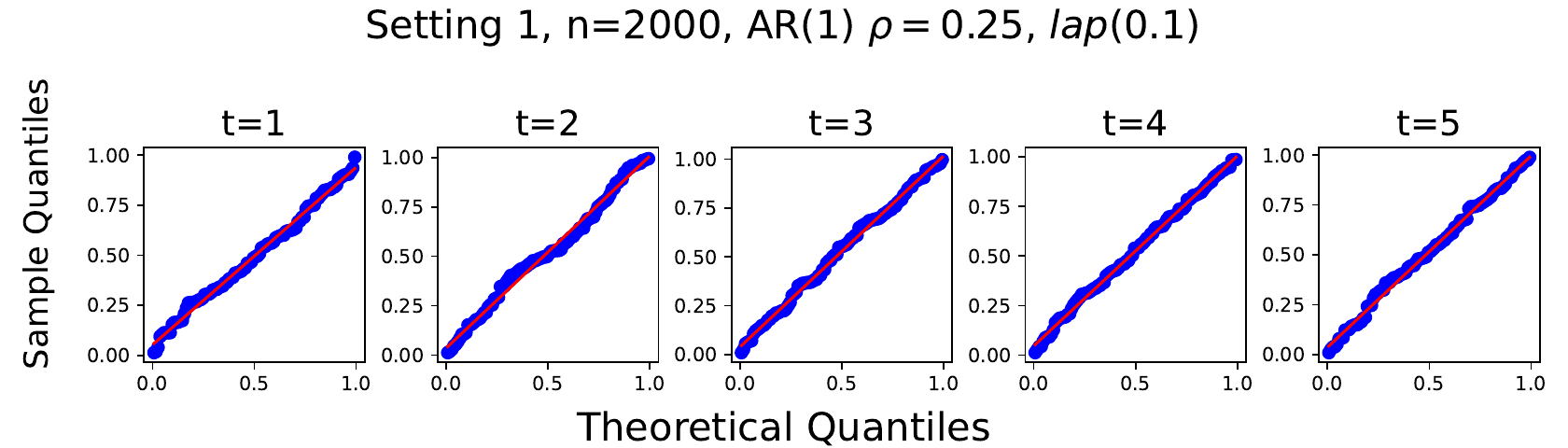}
			\end{subfigure}
				\caption{The Q-Q plot of $W_{n,\id}$ for logistic regression from Setting 1 with AR(1) covariates, $\rho =0.25$,  and the scale of the Laplace noise is 0.1. 
				}
				\label{logisticQQ_setting1_sig0.25_lap0.1}
			\end{figure}

			\begin{figure}[tb]
				\centering
			\begin{subfigure}[b]{0.9\textwidth}
			\includegraphics[width=1\linewidth]{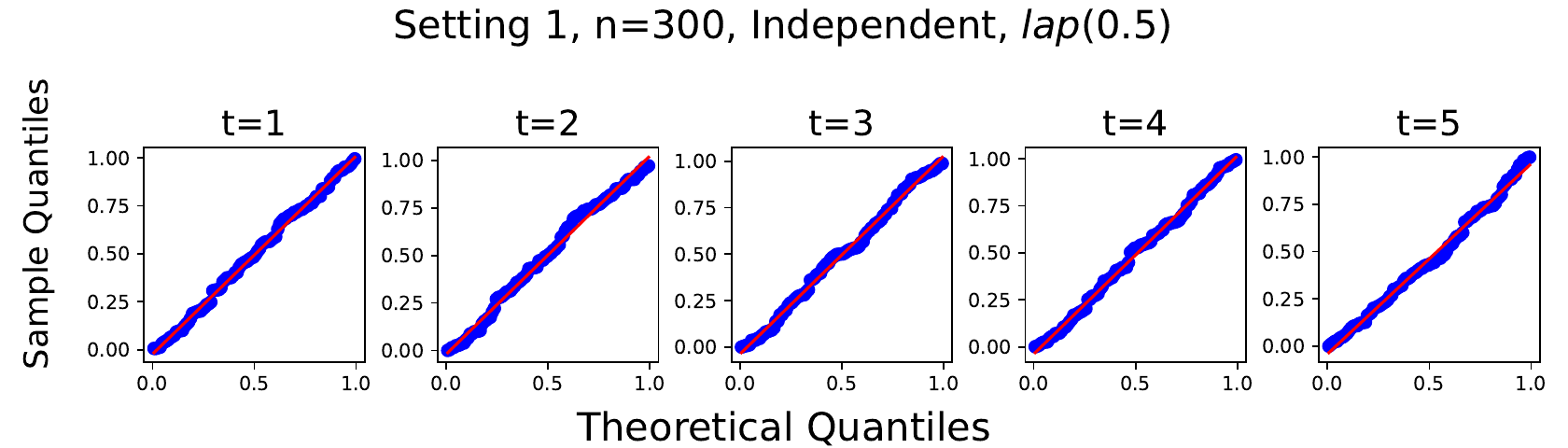}
			\end{subfigure}\vspace{0.8em}
			\begin{subfigure}[b]{0.9\textwidth}
			\includegraphics[width=1\linewidth]{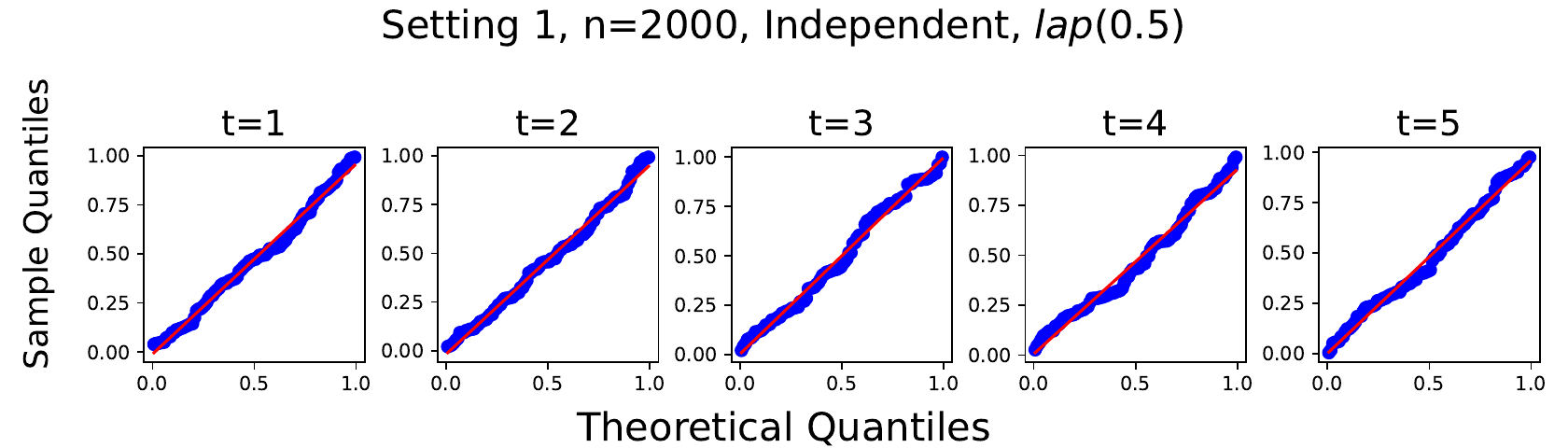}
			\end{subfigure}
				\caption{The Q-Q plot of $W_{n,\id}$ for logistic regression from Setting 1 with Independent covariates,  and the scale of the Laplace noise is 0.5. 
				}
				\label{logisticQQ_setting1_sig0_lap0.5}
			\end{figure}

			\begin{figure}[tb]
				\centering
			\begin{subfigure}[b]{0.9\textwidth}
			\includegraphics[width=1\linewidth]{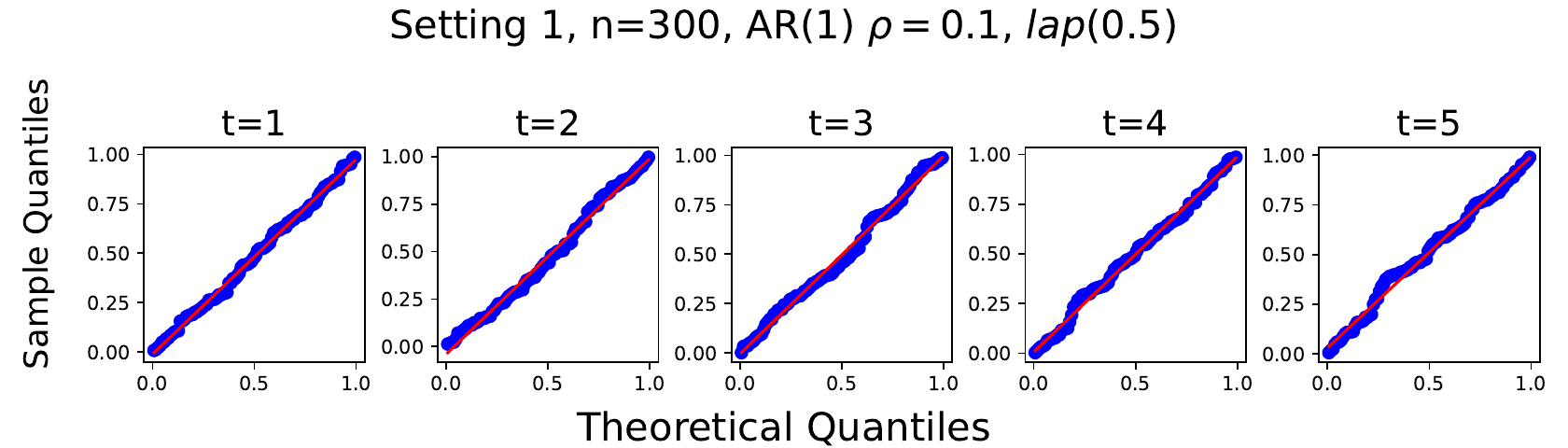}
			\end{subfigure}\vspace{0.8em}
			\begin{subfigure}[b]{0.9\textwidth}
			\includegraphics[width=1\linewidth]{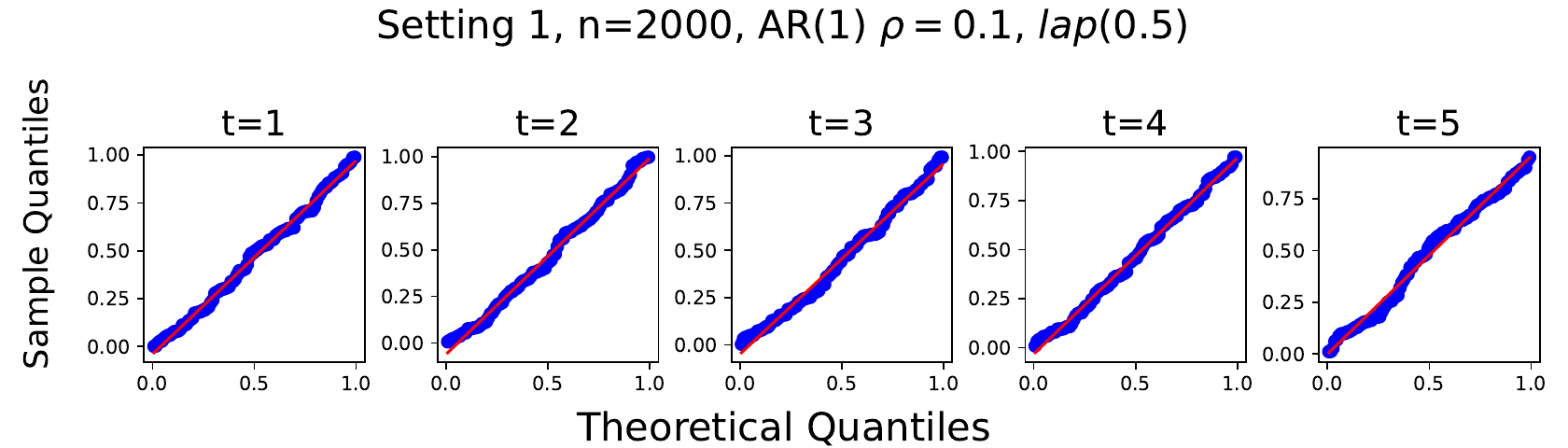}
			\end{subfigure}
				\caption{The Q-Q plot of $W_{n,\id}$ for logistic regression from Setting 1 with AR(1) covariates, $\rho =0.1$,  and the scale of the Laplace noise is 0.5. 
				}
				\label{logisticQQ_setting1_sig0.1_lap0.5}
			\end{figure}

			\begin{figure}[tb]
				\centering
			\begin{subfigure}[b]{0.9\textwidth}
			\includegraphics[width=1\linewidth]{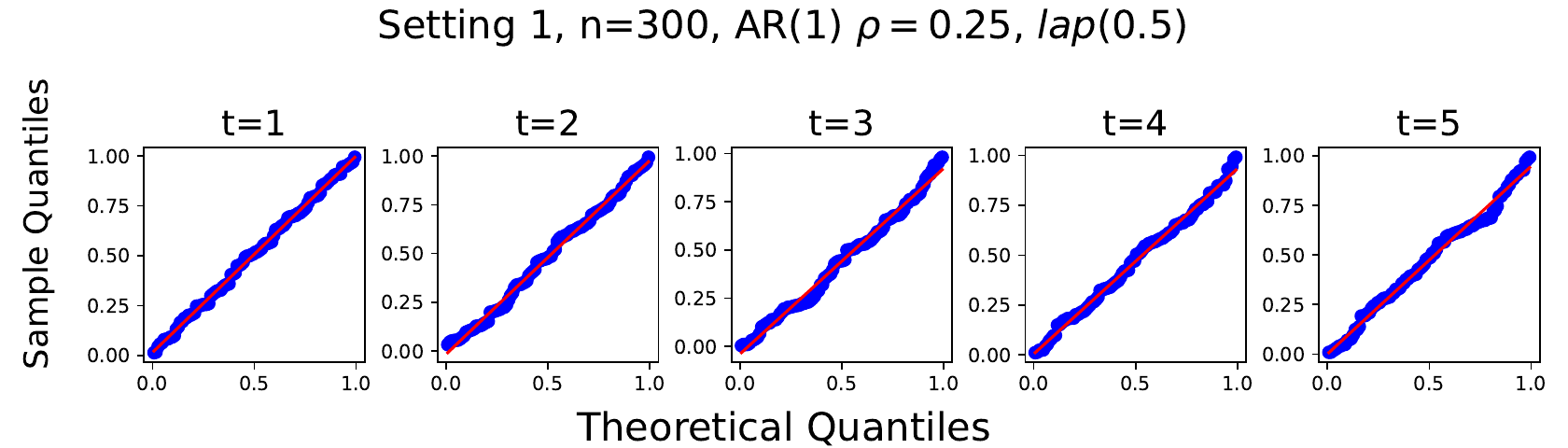}
			\end{subfigure}\vspace{0.8em}
			\begin{subfigure}[b]{0.9\textwidth}
			\includegraphics[width=1\linewidth]{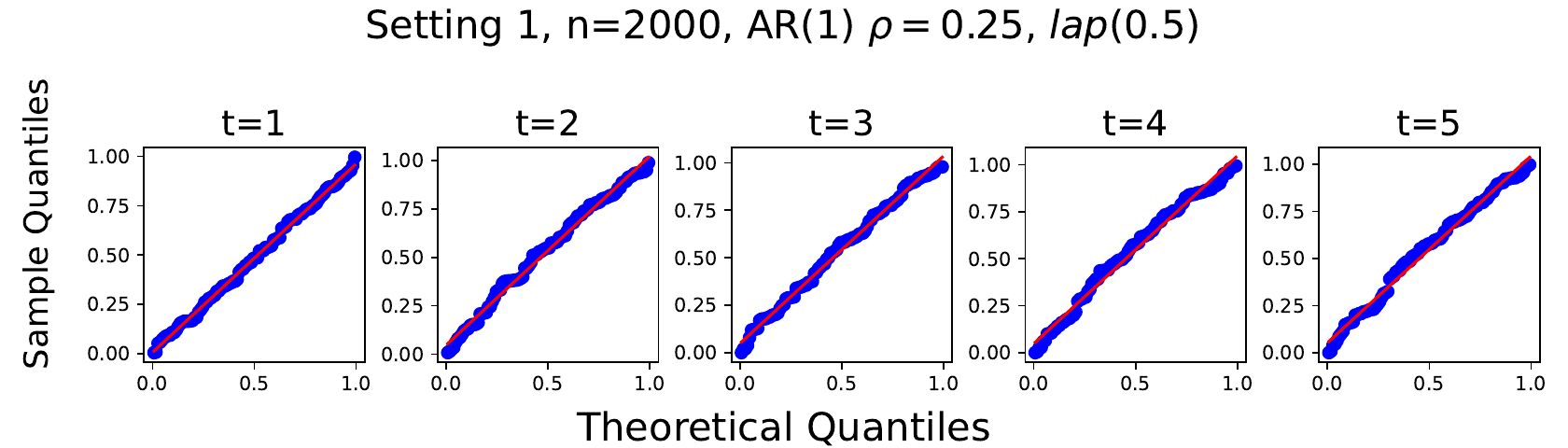}
			\end{subfigure}
				\caption{The Q-Q plot of $W_{n,\id}$ for logistic regression from Setting 1 with AR(1) covariates, $\rho =0.25$,  and the scale of the Laplace noise is 0.5. 
				}
				\label{logisticQQ_setting1_sig0.25_lap0.5}
			\end{figure}

			\begin{figure}[tb]
				\centering
			\begin{subfigure}[b]{0.9\textwidth}
			\includegraphics[width=1\linewidth]{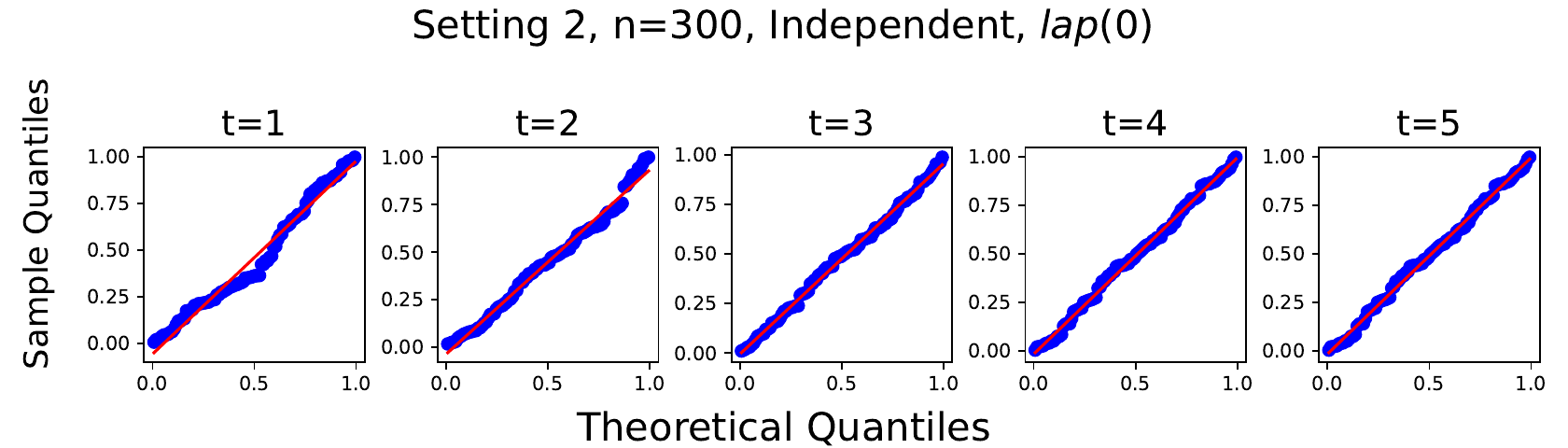}
			\end{subfigure}\vspace{0.8em}
			\begin{subfigure}[b]{0.9\textwidth}
			\includegraphics[width=1\linewidth]{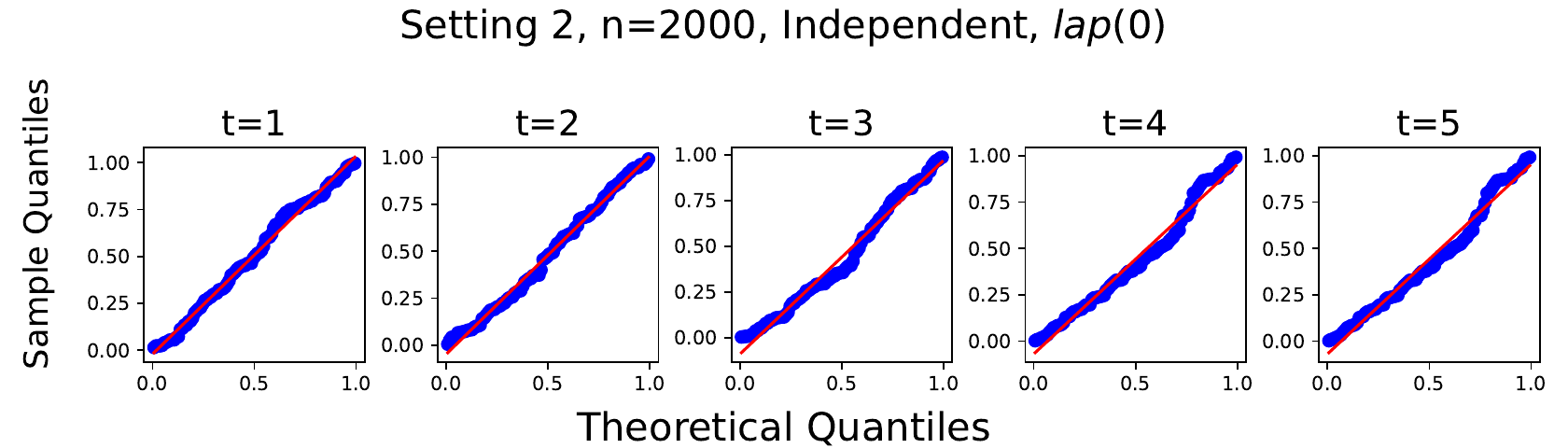}
			\end{subfigure}
				\caption{The Q-Q plot of $W_{n,\id}$ for logistic regression from Setting 2 with Independent covariates,  and without the Laplace noise. 
				}
				\label{logisticQQ_setting2_sig0_lap0}
			\end{figure}

			\begin{figure}[tb]
				\centering
			\begin{subfigure}[b]{0.9\textwidth}
			\includegraphics[width=1\linewidth]{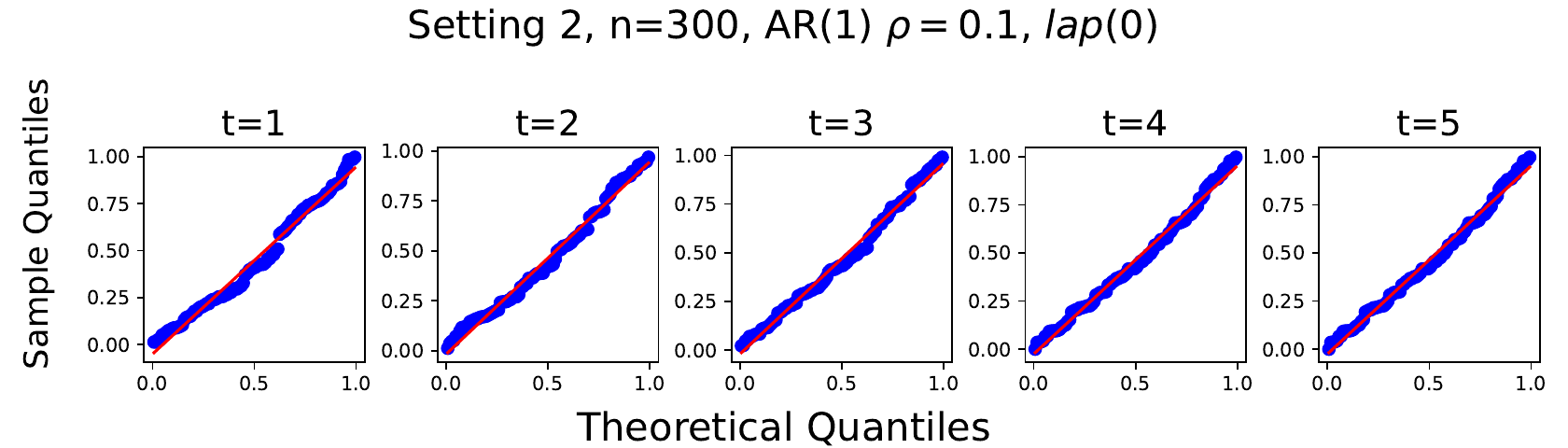}
			\end{subfigure}\vspace{0.8em}
			\begin{subfigure}[b]{0.9\textwidth}
			\includegraphics[width=1\linewidth]{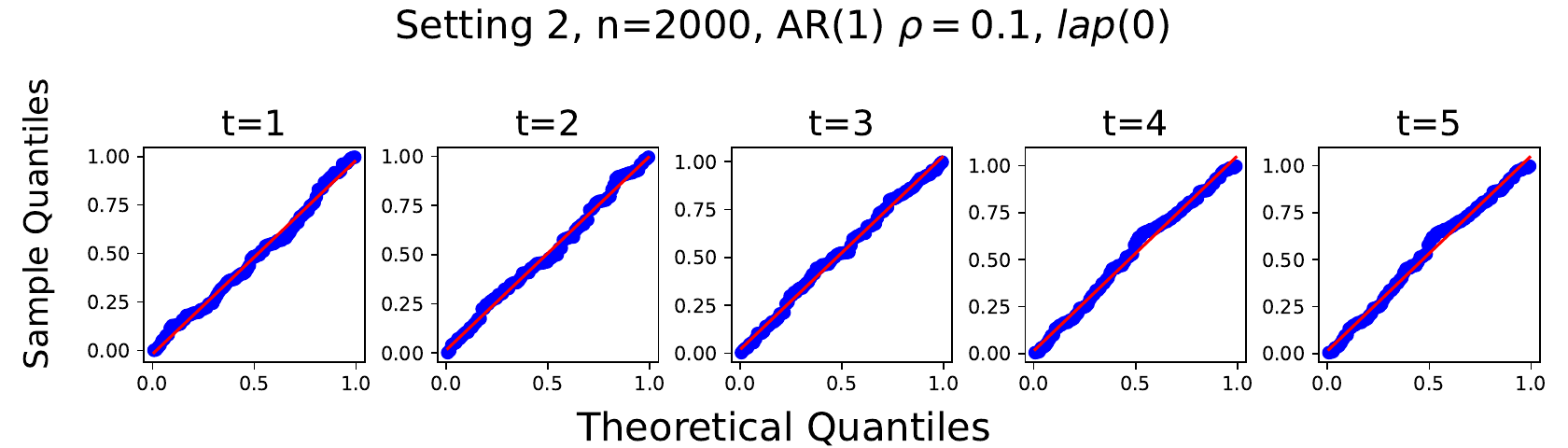}
			\end{subfigure}
				\caption{The Q-Q plot of $W_{n,\id}$ for logistic regression from Setting 2 with AR(1) covariates, $\rho =0.1$,  and without the Laplace noise. 
				}
				\label{logisticQQ_setting2_sig0.1_lap0}
			\end{figure}

			\begin{figure}[tb]
				\centering
			\begin{subfigure}[b]{0.9\textwidth}
			\includegraphics[width=1\linewidth]{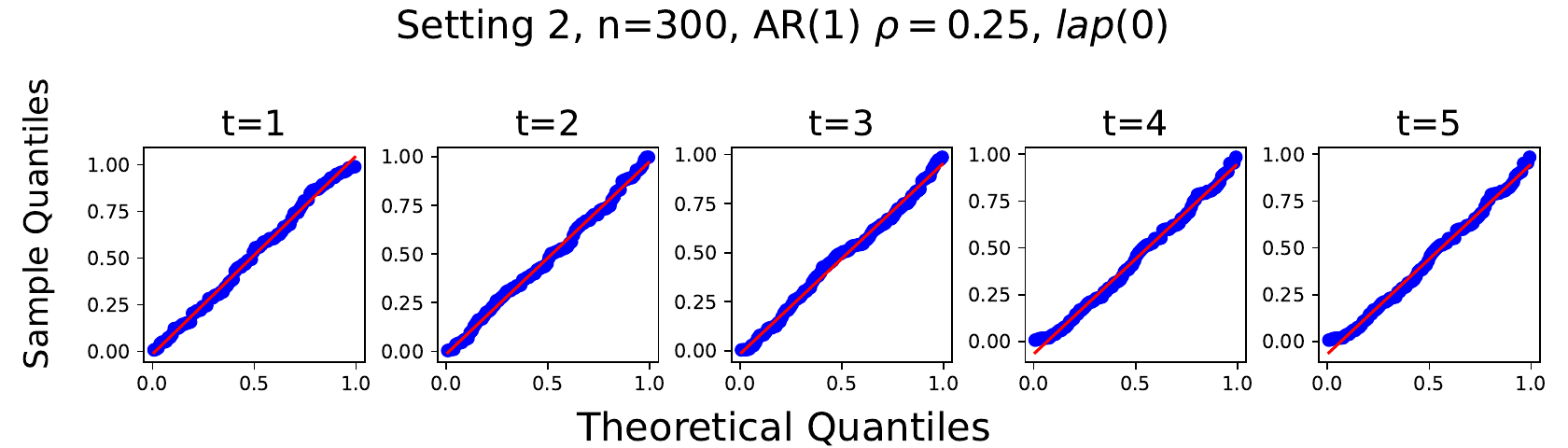}
			\end{subfigure}\vspace{0.8em}
			\begin{subfigure}[b]{0.9\textwidth}
			\includegraphics[width=1\linewidth]{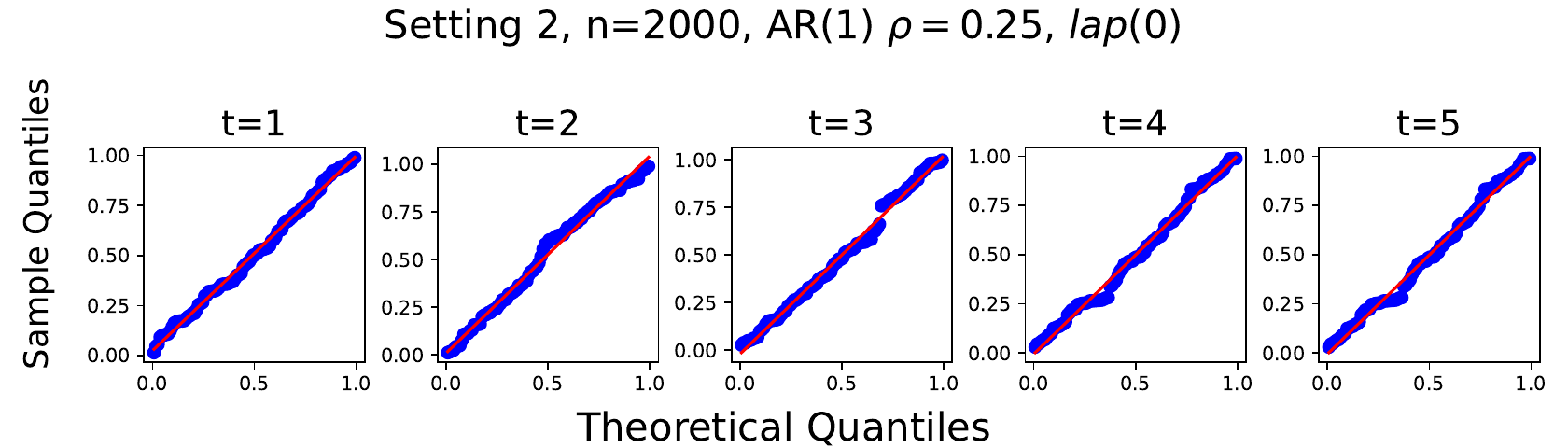}
			\end{subfigure}
				\caption{The Q-Q plot of $W_{n,\id}$ for logistic regression from Setting 2 with AR(1) covariates, $\rho =0.25$,  and without the Laplace noise. 
				}
				\label{logisticQQ_setting2_sig0.25_lap0}
			\end{figure}

			\begin{figure}[tb]
				\centering
			\begin{subfigure}[b]{0.9\textwidth}
			\includegraphics[width=1\linewidth]{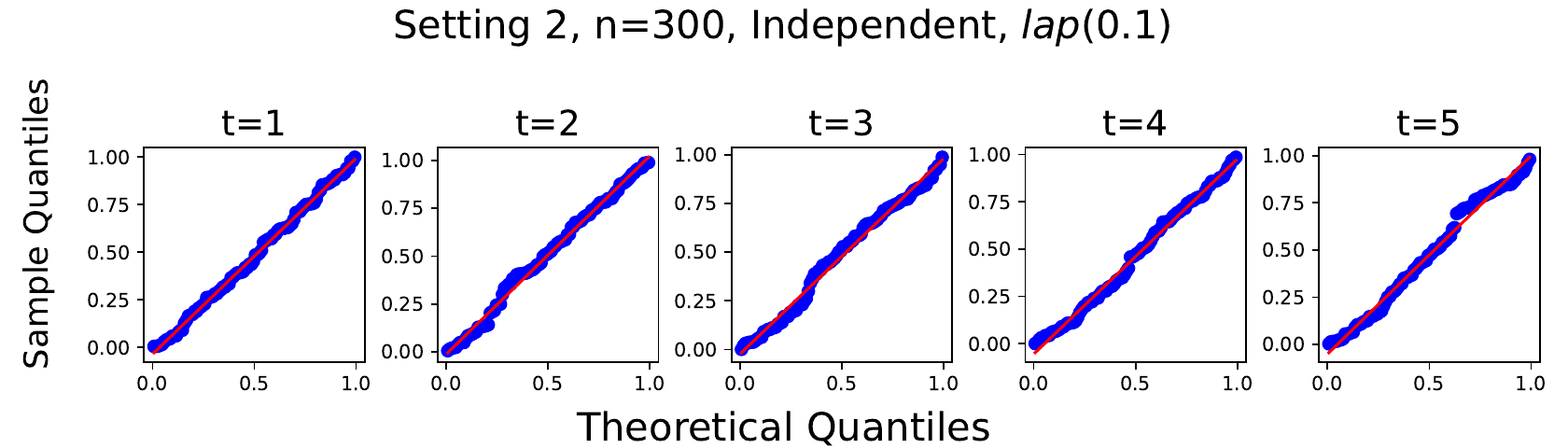}
			\end{subfigure}\vspace{0.8em}
			\begin{subfigure}[b]{0.9\textwidth}
			\includegraphics[width=1\linewidth]{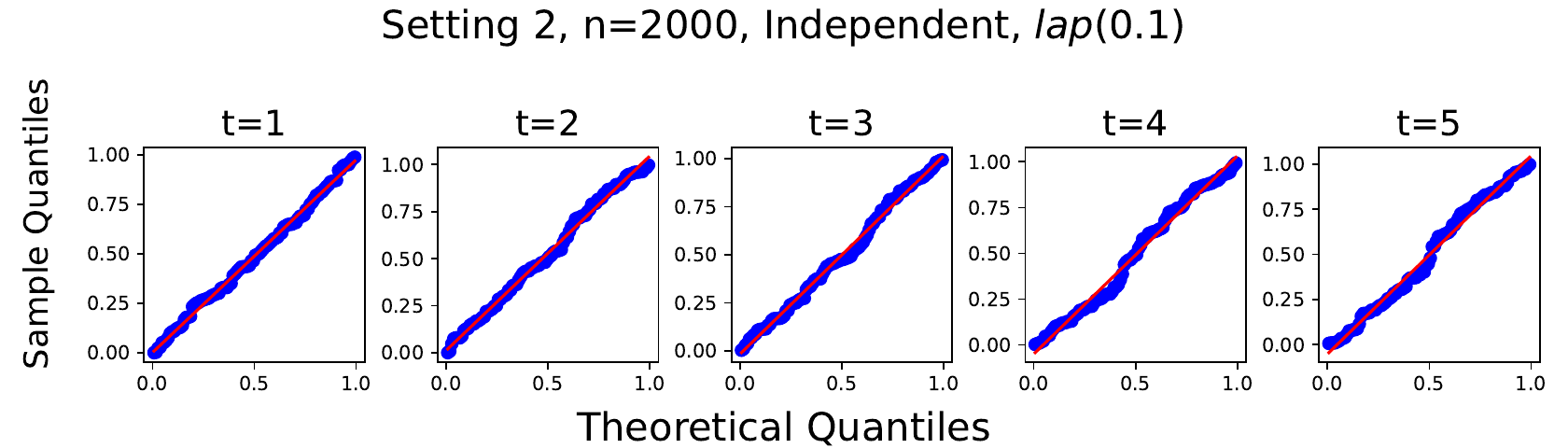}
			\end{subfigure}
				\caption{The Q-Q plot of $W_{n,\id}$ for logistic regression from Setting 2 with Independent covariates,  and the scale of the Laplace noise is 0.1. 
				}
				\label{logisticQQ_setting2_sig0_lap0.1}
			\end{figure}

			\begin{figure}[tb]
				\centering
			\begin{subfigure}[b]{0.9\textwidth}
			\includegraphics[width=1\linewidth]{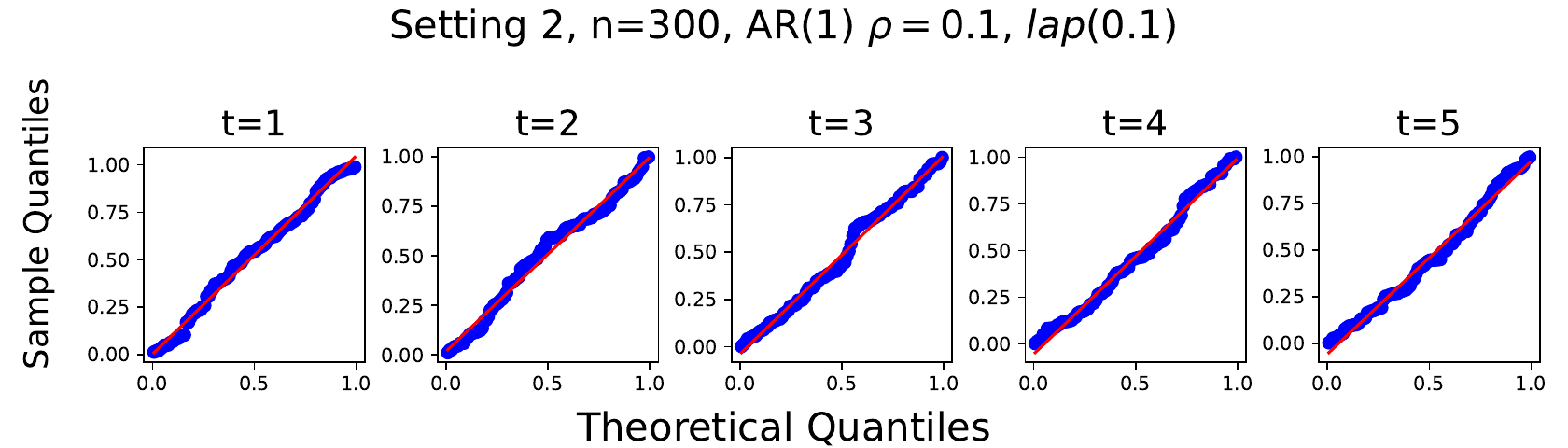}
			\end{subfigure}\vspace{0.8em}
			\begin{subfigure}[b]{0.9\textwidth}
			\includegraphics[width=1\linewidth]{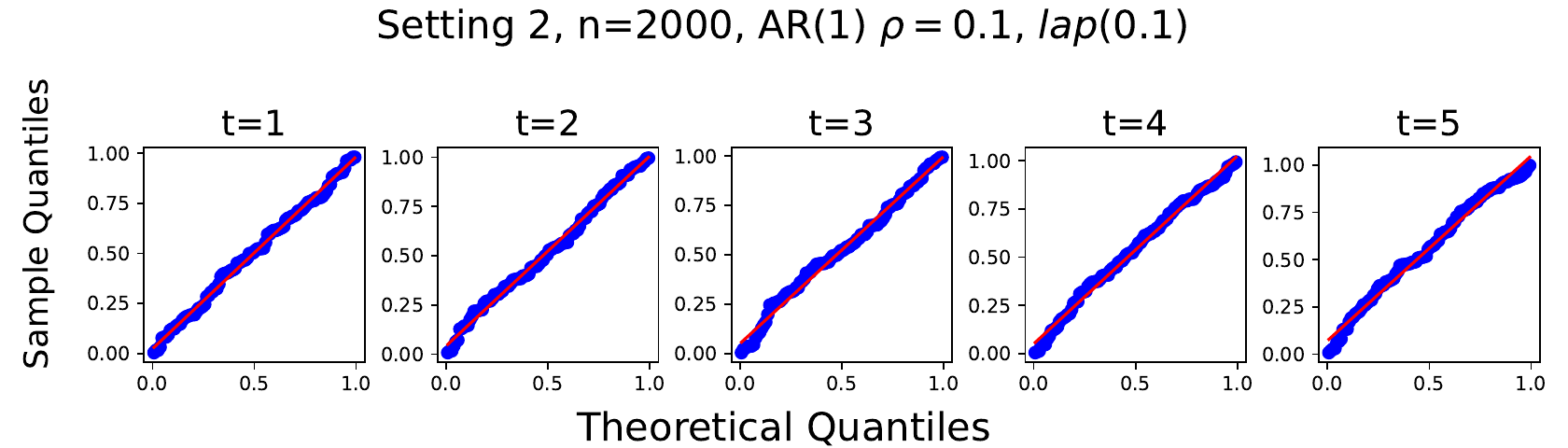}
			\end{subfigure}
				\caption{The Q-Q plot of $W_{n,\id}$ for logistic regression from Setting 2 with AR(1) covariates, $\rho =0.1$,  and the scale of the Laplace noise is 0.1. 
				}
				\label{logisticQQ_setting2_sig0.1_lap0.1}
			\end{figure}

			\begin{figure}[tb]
				\centering
			\begin{subfigure}[b]{0.9\textwidth}
			\includegraphics[width=1\linewidth]{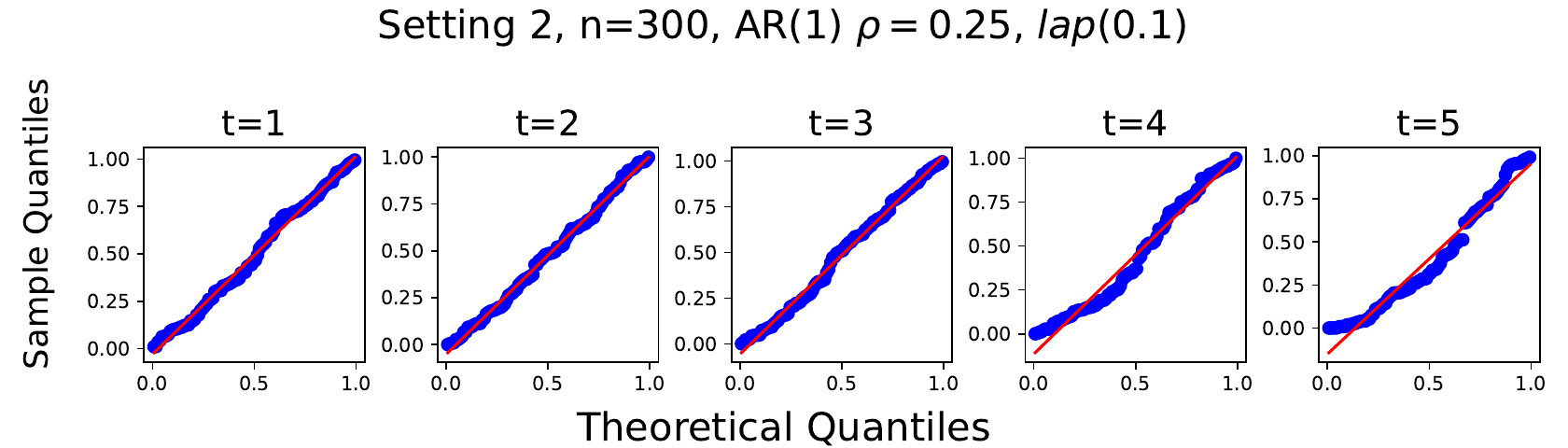}
			\end{subfigure}\vspace{0.8em}
			\begin{subfigure}[b]{0.9\textwidth}
			\includegraphics[width=1\linewidth]{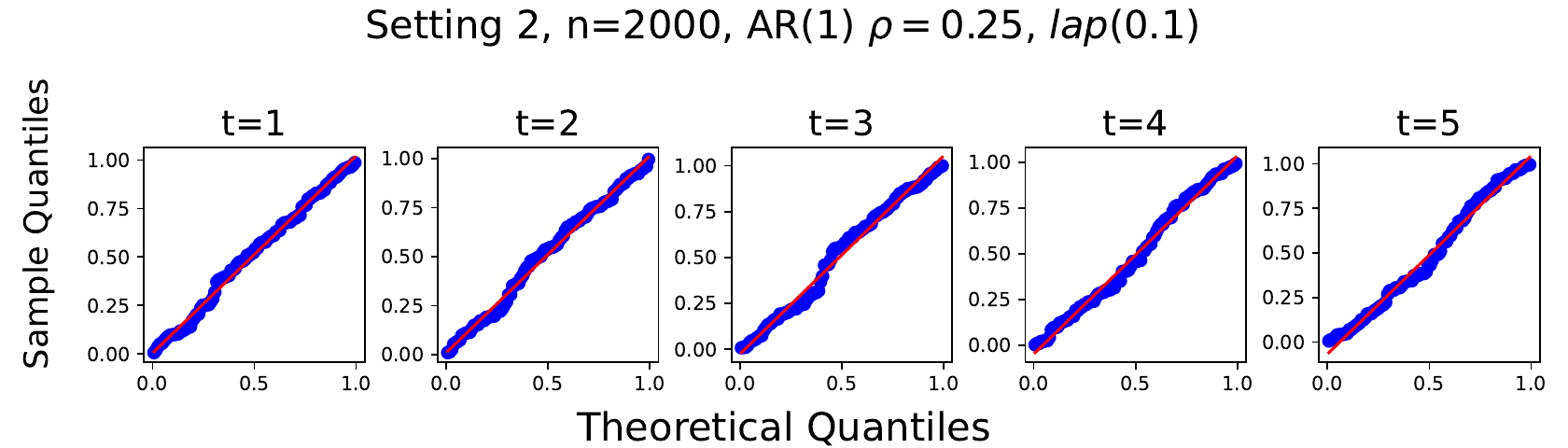}
			\end{subfigure}
				\caption{The Q-Q plot of $W_{n,\id}$ for logistic regression from Setting 2 with AR(1) covariates, $\rho =0.25$,  and the scale of the Laplace noise is 0.1. 
				}
				\label{logisticQQ_setting2_sig0.25_lap0.1}
			\end{figure}

			\begin{figure}[tb]
				\centering
			\begin{subfigure}[b]{0.9\textwidth}
			\includegraphics[width=1\linewidth]{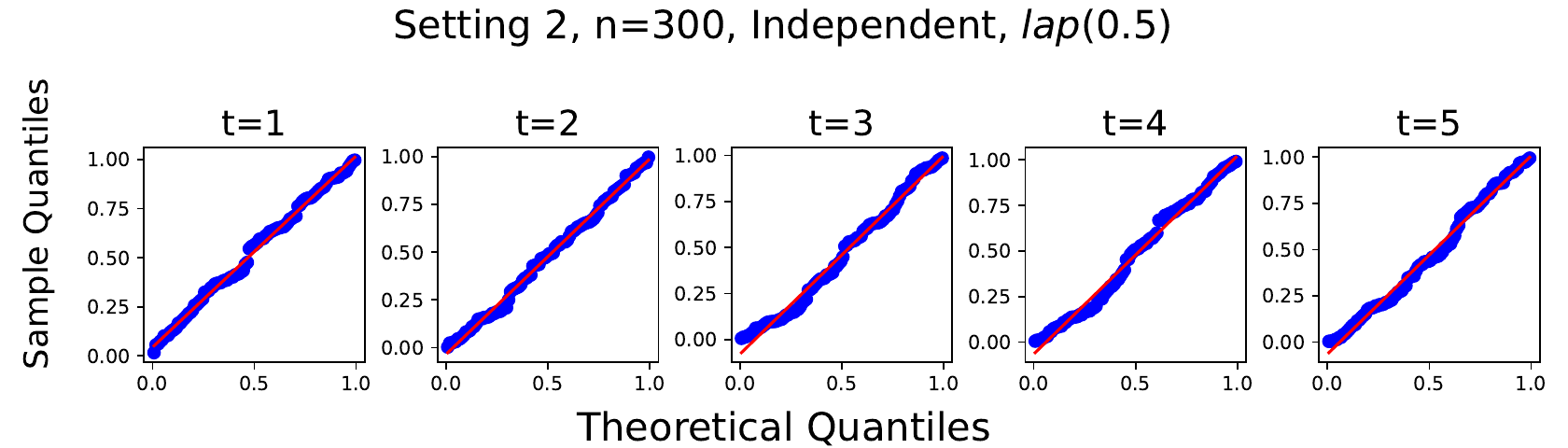}
			\end{subfigure}\vspace{0.8em}
			\begin{subfigure}[b]{0.9\textwidth}
			\includegraphics[width=1\linewidth]{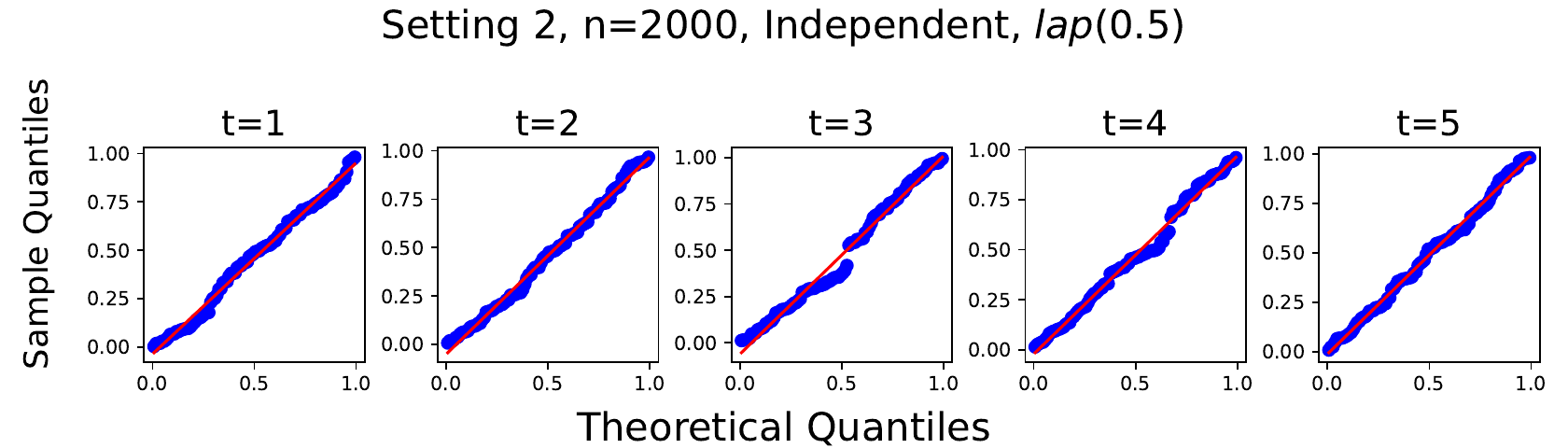}
			\end{subfigure}
				\caption{The Q-Q plot of $W_{n,\id}$ for logistic regression from Setting 2 with Independent covariates,  and the scale of the Laplace noise is 0.5. 
				}
				\label{logisticQQ_setting2_sig0_lap0.5}
			\end{figure}

			\begin{figure}[tb]
				\centering
			\begin{subfigure}[b]{0.9\textwidth}
			\includegraphics[width=1\linewidth]{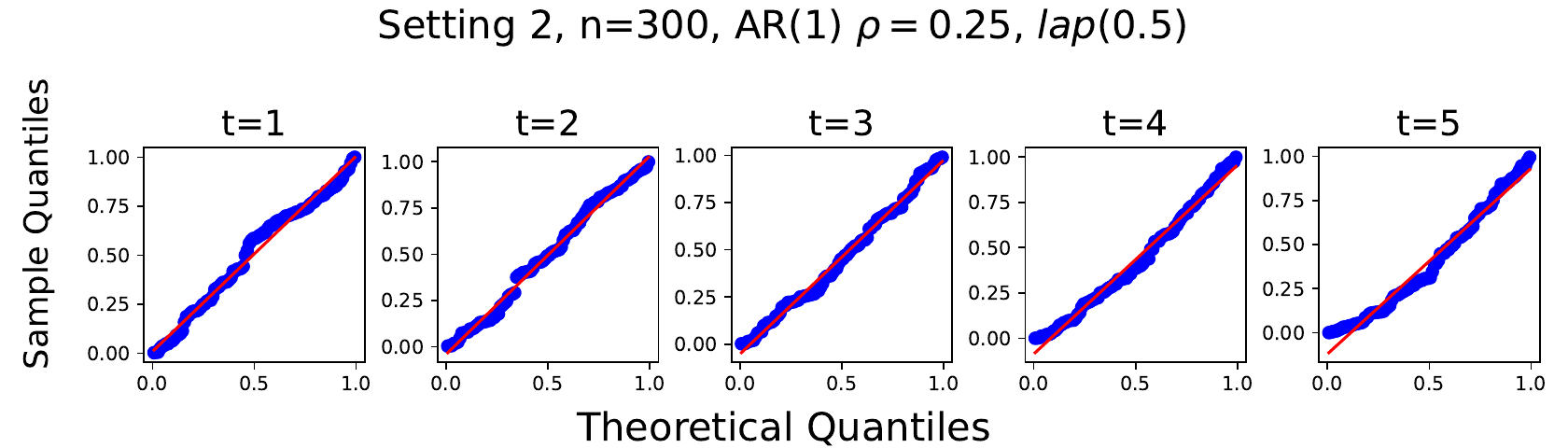}
			\end{subfigure}\vspace{0.8em}
			\begin{subfigure}[b]{0.9\textwidth}
			\includegraphics[width=1\linewidth]{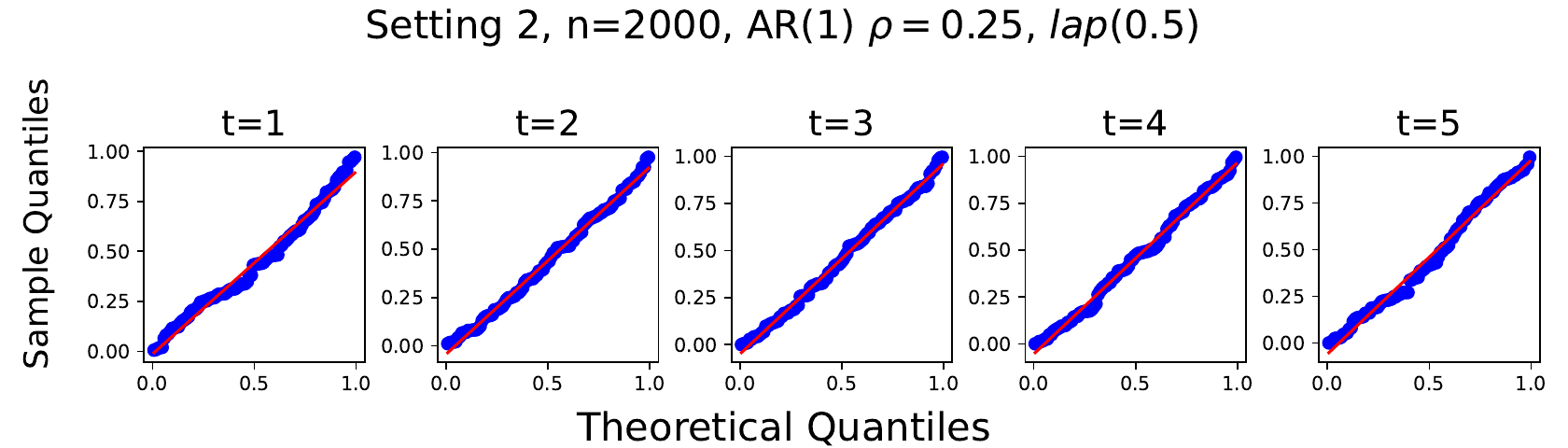}
			\end{subfigure}
				\caption{The Q-Q plot of $W_{n,\id}$ for logistic regression from Setting 2 with AR(1) covariates, $\rho =0.25$,  and the scale of the Laplace noise is 0.5. 
				}
				\label{logisticQQ_setting2_sig0.25_lap0.5}
			\end{figure}

			\begin{figure}[tb]
				\centering
			\begin{subfigure}[b]{0.9\textwidth}
			\includegraphics[width=1\linewidth]{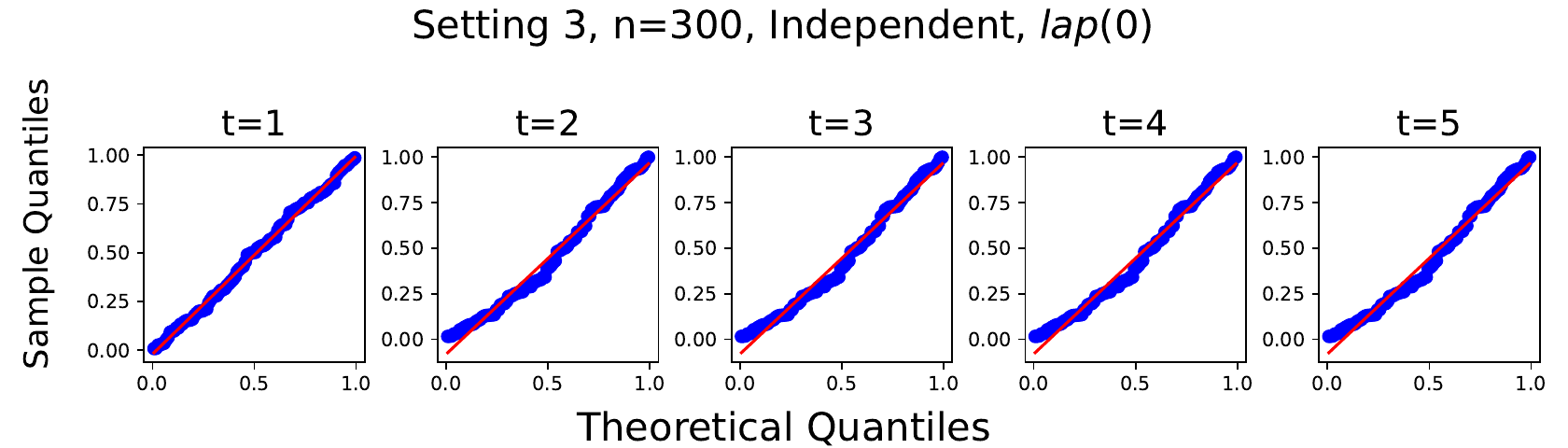}
			\end{subfigure}\vspace{0.8em}
			\begin{subfigure}[b]{0.9\textwidth}
			\includegraphics[width=1\linewidth]{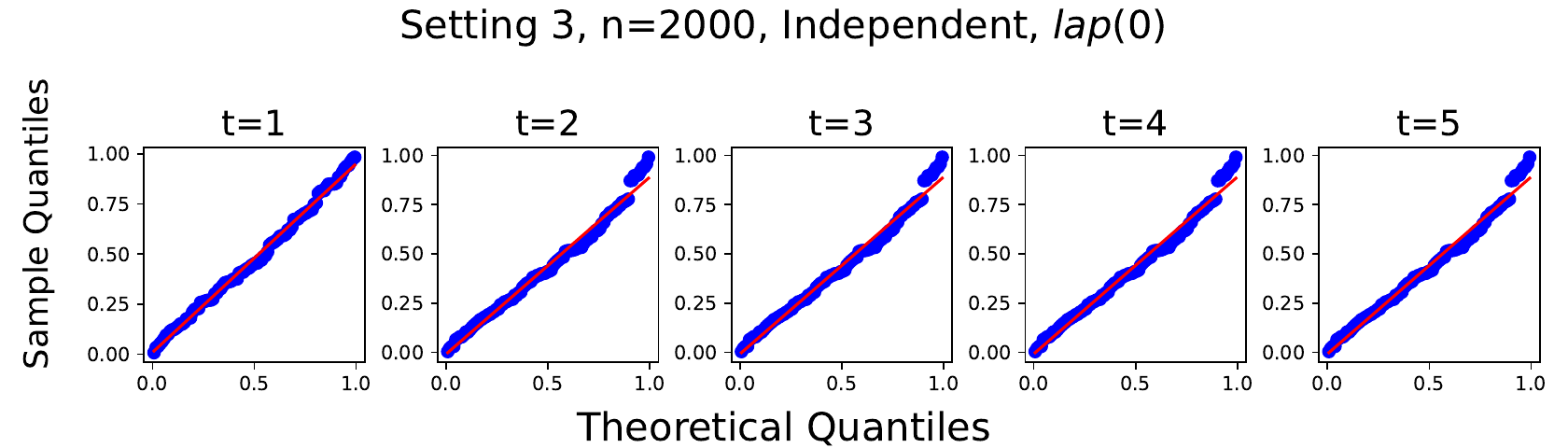}
			\end{subfigure}
				\caption{The Q-Q plot of $W_{n,\id}$ for logistic regression from Setting 3 with Independent covariates,  and without the Laplace noise. 
				}
				\label{logisticQQ_setting3_sig0_lap0}
			\end{figure}

			\begin{figure}[tb]
				\centering
			\begin{subfigure}[b]{0.9\textwidth}
			\includegraphics[width=1\linewidth]{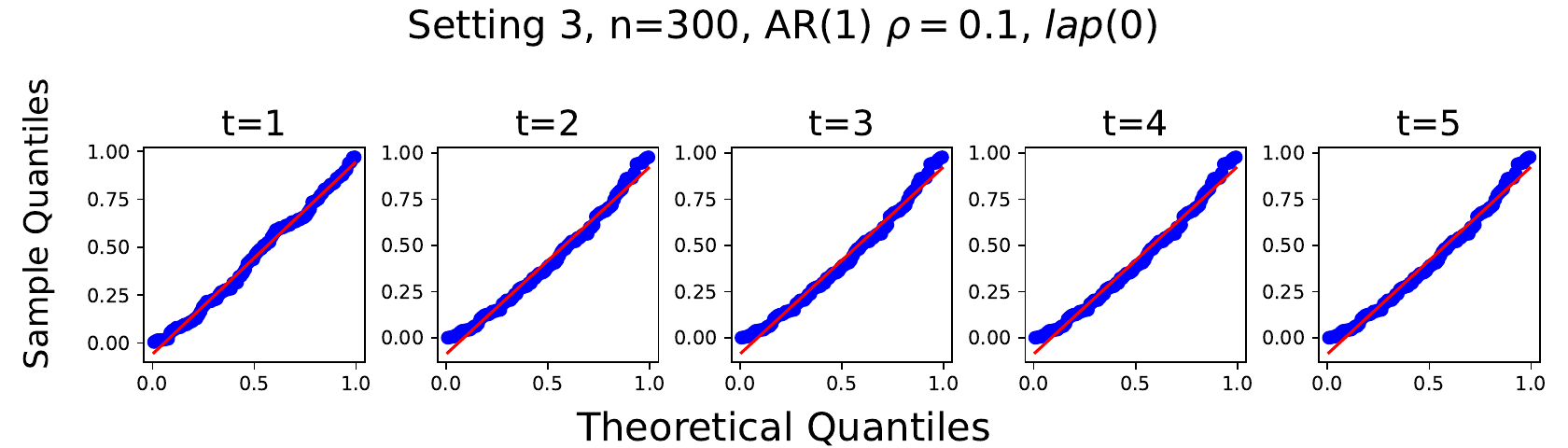}
			\end{subfigure}\vspace{0.8em}
			\begin{subfigure}[b]{0.9\textwidth}
			\includegraphics[width=1\linewidth]{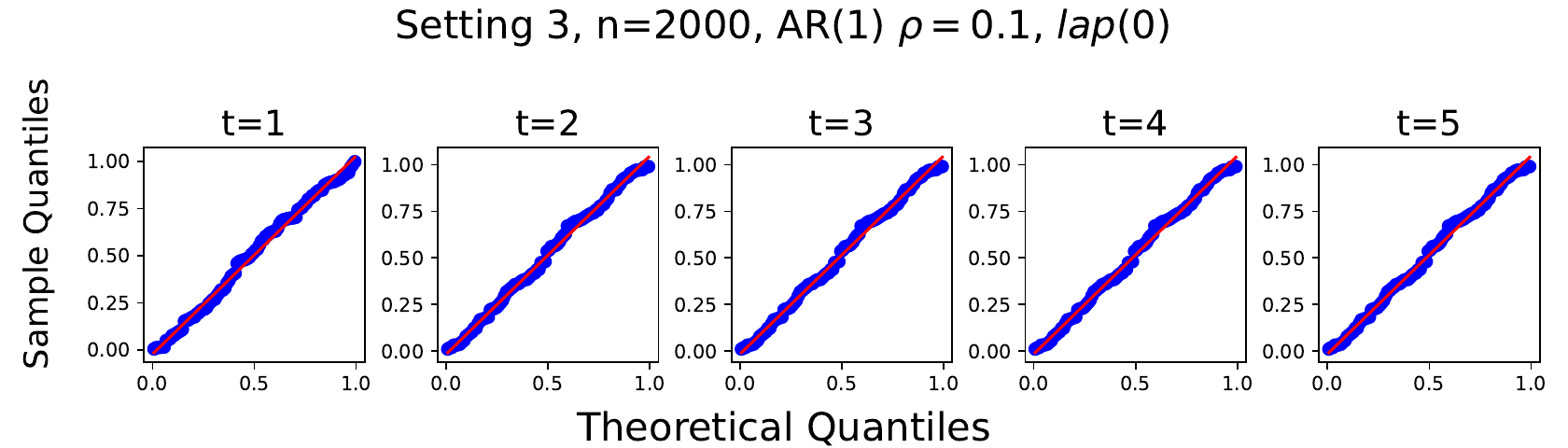}
			\end{subfigure}
				\caption{The Q-Q plot of $W_{n,\id}$ for logistic regression from Setting 3 with AR(1) covariates, $\rho =0.1$,  and without the Laplace noise. 
				}
				\label{logisticQQ_setting3_sig0.1_lap0}
			\end{figure}

			\begin{figure}[tb]
				\centering
			\begin{subfigure}[b]{0.9\textwidth}
			\includegraphics[width=1\linewidth]{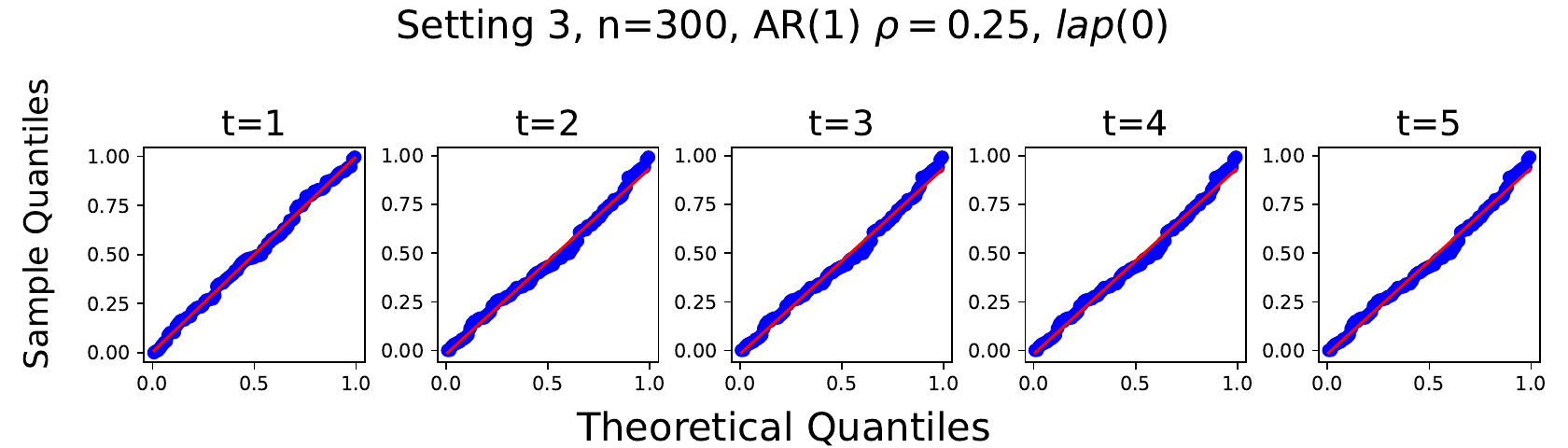}
			\end{subfigure}\vspace{0.8em}
			\begin{subfigure}[b]{0.9\textwidth}
			\includegraphics[width=1\linewidth]{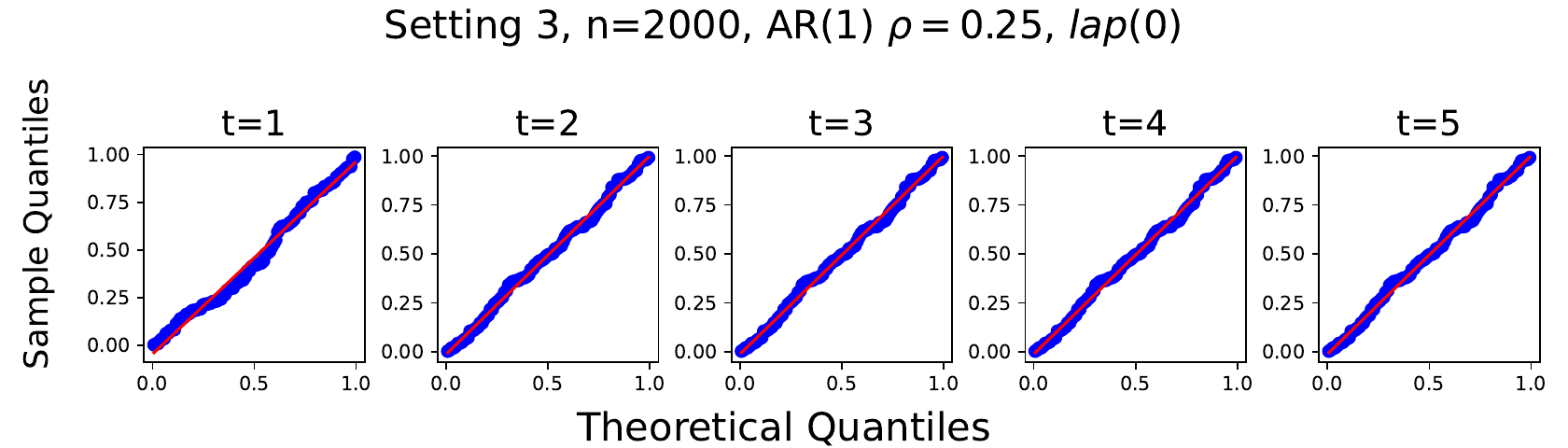}
			\end{subfigure}
				\caption{The Q-Q plot of $W_{n,\id}$ for logistic regression from Setting 3 with AR(1) covariates, $\rho =0.25$,  and without the Laplace noise. 
				}
				\label{logisticQQ_setting3_sig0.25_lap0}
			\end{figure}

			\begin{figure}[tb]
				\centering
			\begin{subfigure}[b]{0.9\textwidth}
			\includegraphics[width=1\linewidth]{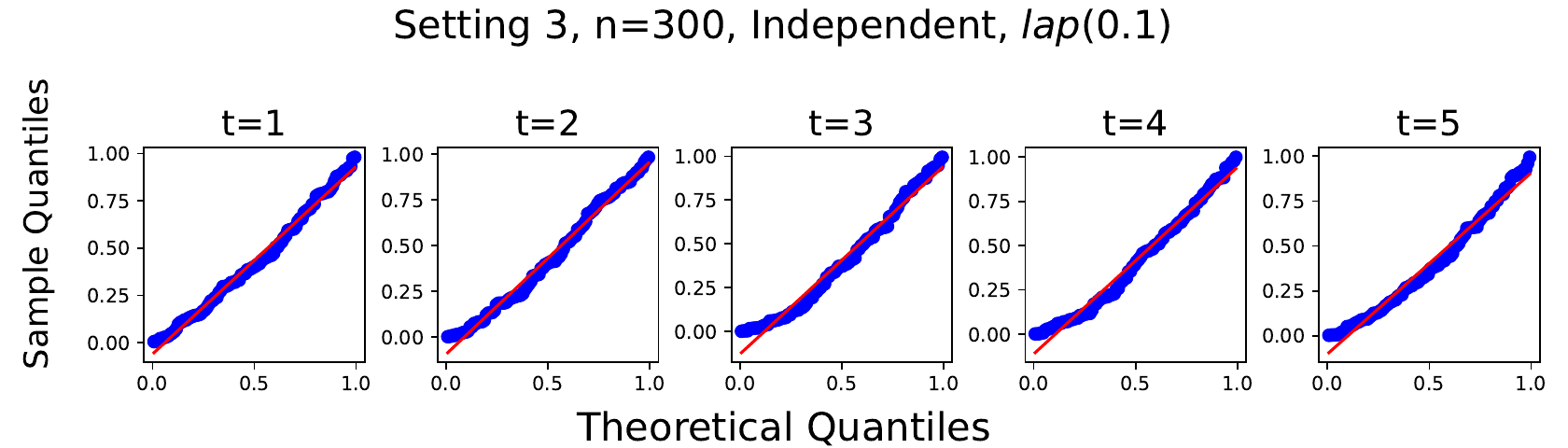}
			\end{subfigure}\vspace{0.8em}
			\begin{subfigure}[b]{0.9\textwidth}
			\includegraphics[width=1\linewidth]{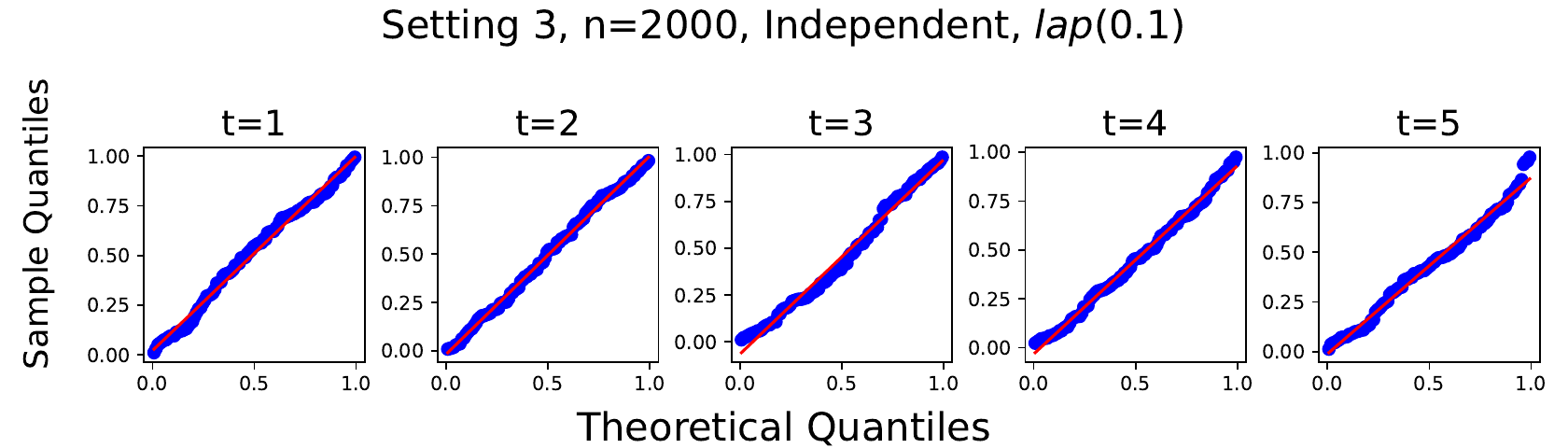}
			\end{subfigure}
				\caption{The Q-Q plot of $W_{n,\id}$ for logistic regression from Setting 3 with Independent covariates,  and the scale of the Laplace noise is 0.1. 
				}
				\label{logisticQQ_setting3_sig0_lap0.1}
			\end{figure}

			\begin{figure}[tb]
				\centering
			\begin{subfigure}[b]{0.9\textwidth}
			\includegraphics[width=1\linewidth]{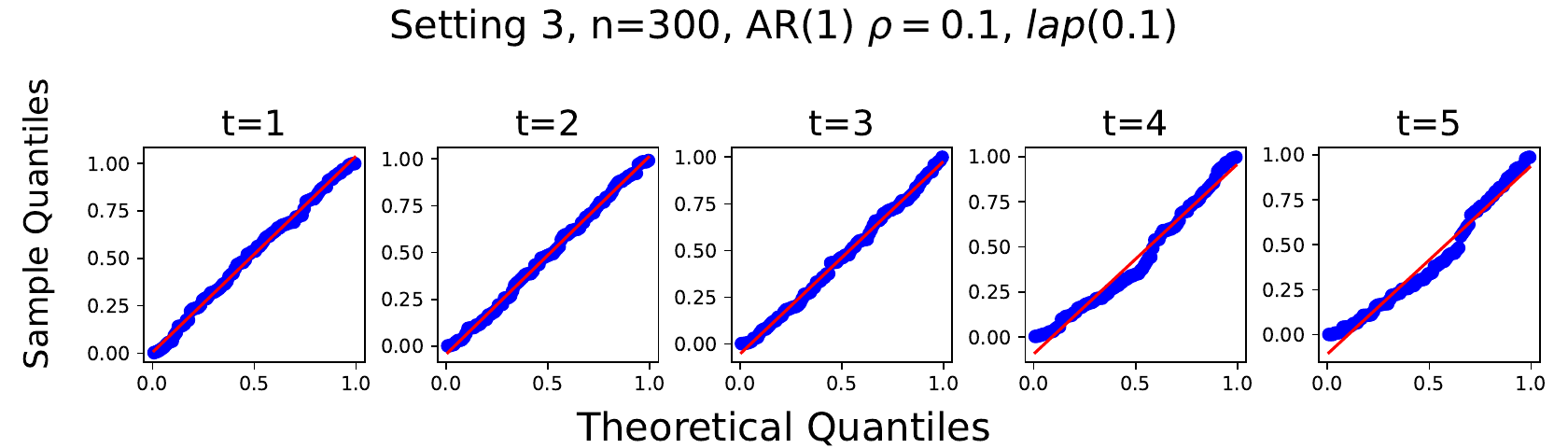}
			\end{subfigure}\vspace{0.8em}
			\begin{subfigure}[b]{0.9\textwidth}
			\includegraphics[width=1\linewidth]{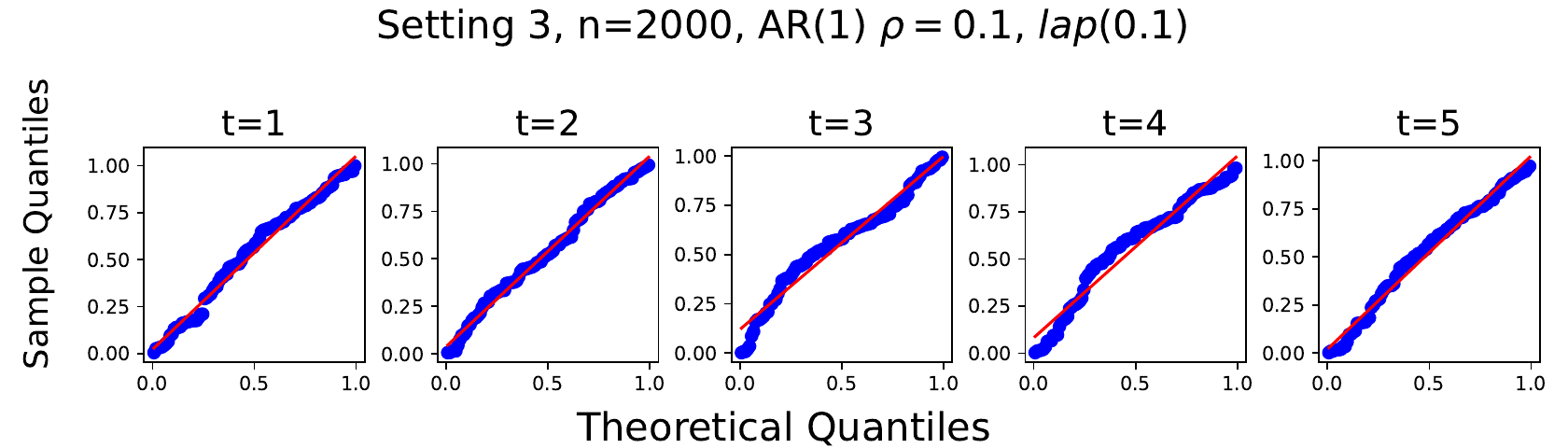}
			\end{subfigure}
				\caption{The Q-Q plot of $W_{n,\id}$ for logistic regression from Setting 3 with AR(1) covariates, $\rho =0.1$,  and the scale of the Laplace noise is 0.1. 
				}
				\label{logisticQQ_setting3_sig0.1_lap0.1}
			\end{figure}

			\begin{figure}[tb]
				\centering
			\begin{subfigure}[b]{0.9\textwidth}
			\includegraphics[width=1\linewidth]{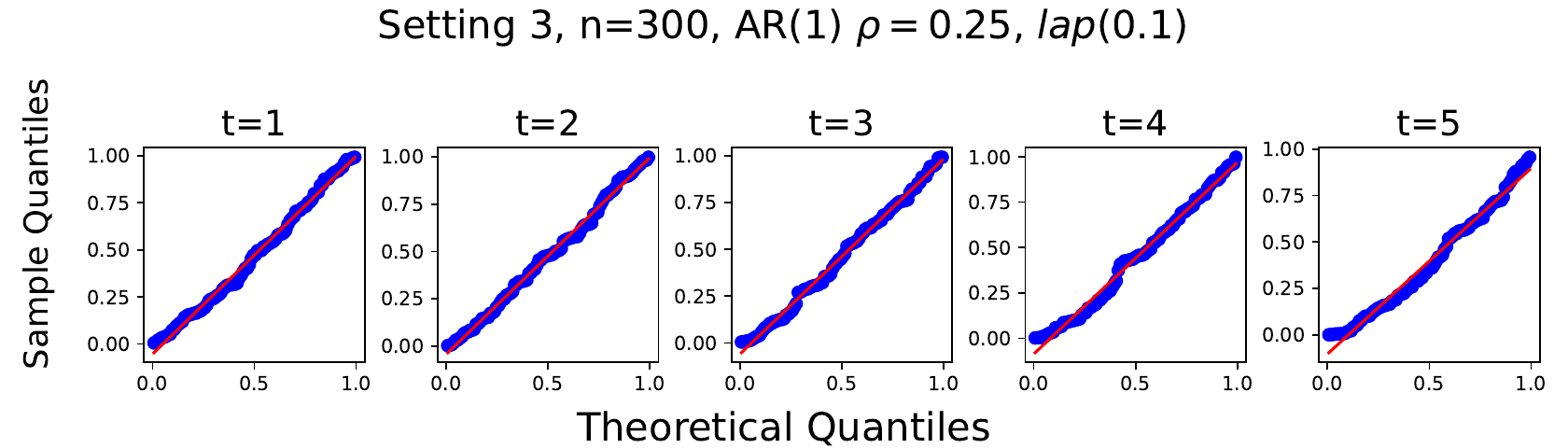}
			\end{subfigure}\vspace{0.8em}
			\begin{subfigure}[b]{0.9\textwidth}
			\includegraphics[width=1\linewidth]{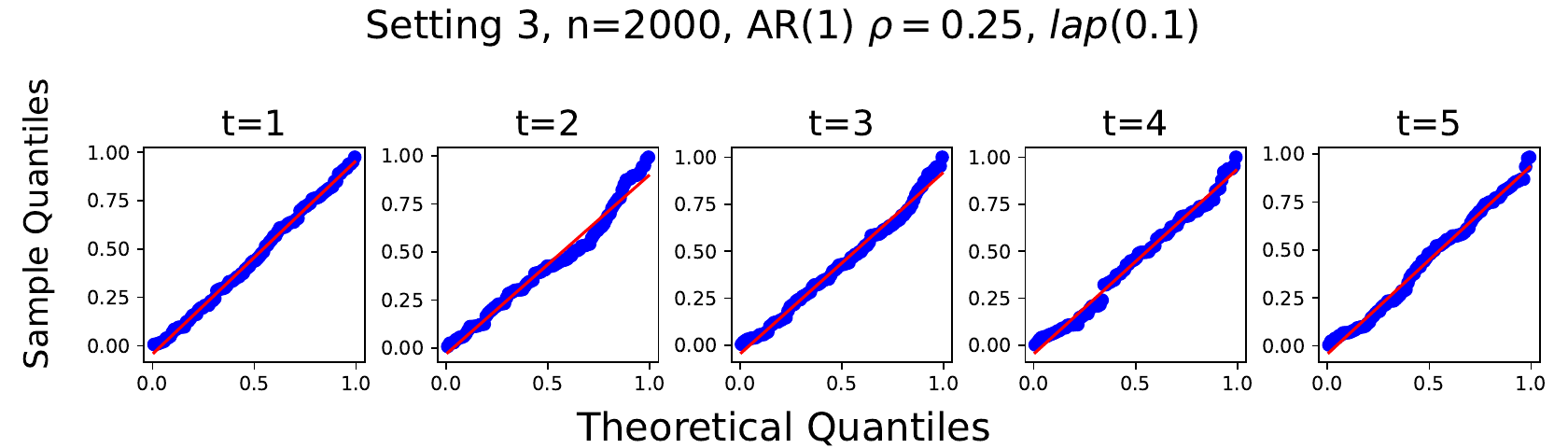}
			\end{subfigure}
				\caption{The Q-Q plot of $W_{n,\id}$ for logistic regression from Setting 3 with AR(1) covariates, $\rho =0.25$,  and the scale of the Laplace noise is 0.1. 
				}
				\label{logisticQQ_setting3_sig0.25_lap0.1}
			\end{figure}

			\begin{figure}[tb]
				\centering
			\begin{subfigure}[b]{0.9\textwidth}
			\includegraphics[width=1\linewidth]{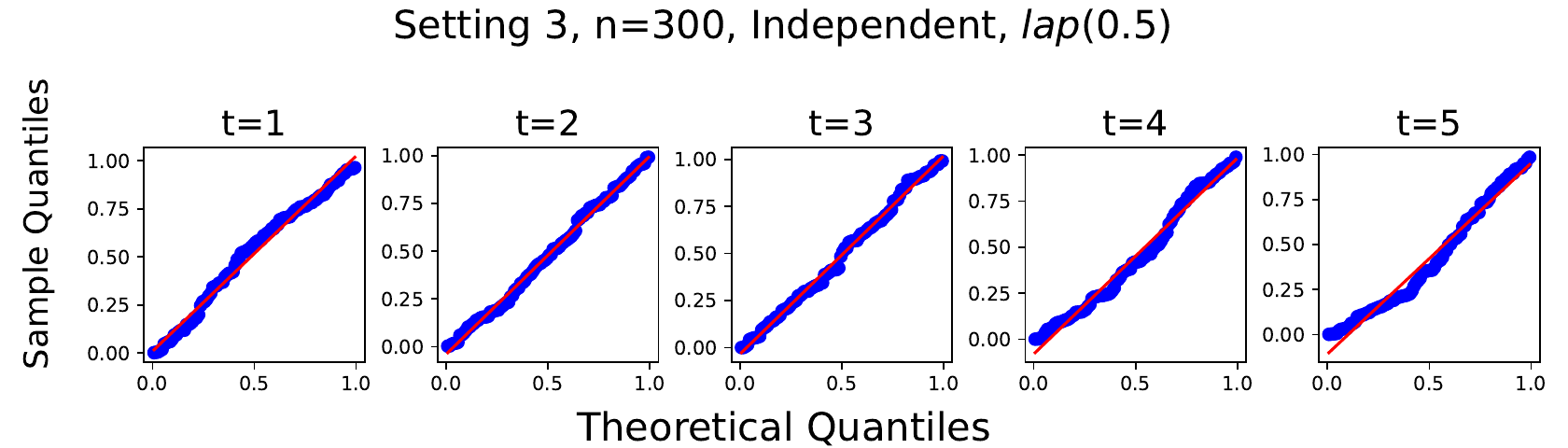}
			\end{subfigure}\vspace{0.8em}
			\begin{subfigure}[b]{0.9\textwidth}
			\includegraphics[width=1\linewidth]{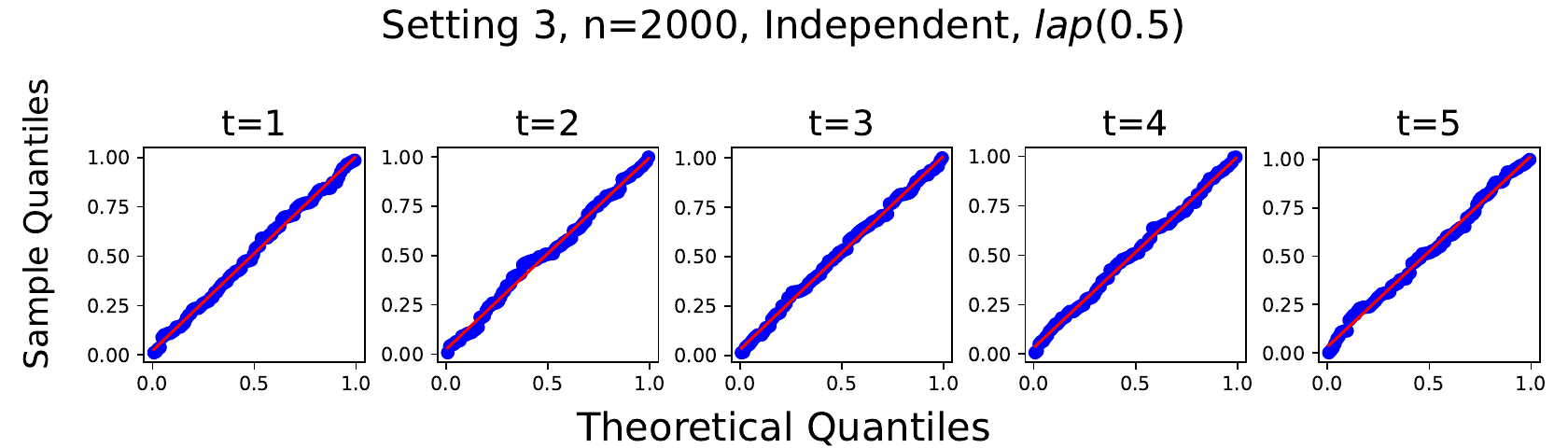}
			\end{subfigure}
				\caption{The Q-Q plot of $W_{n,\id}$ for logistic regression from Setting 3 with Independent covariates,  and the scale of the Laplace noise is 0.5. 
				}
				\label{logisticQQ_setting3_sig0_lap0.5}
			\end{figure}

			\begin{figure}[tb]
				\centering
			\begin{subfigure}[b]{0.9\textwidth}
			\includegraphics[width=1\linewidth]{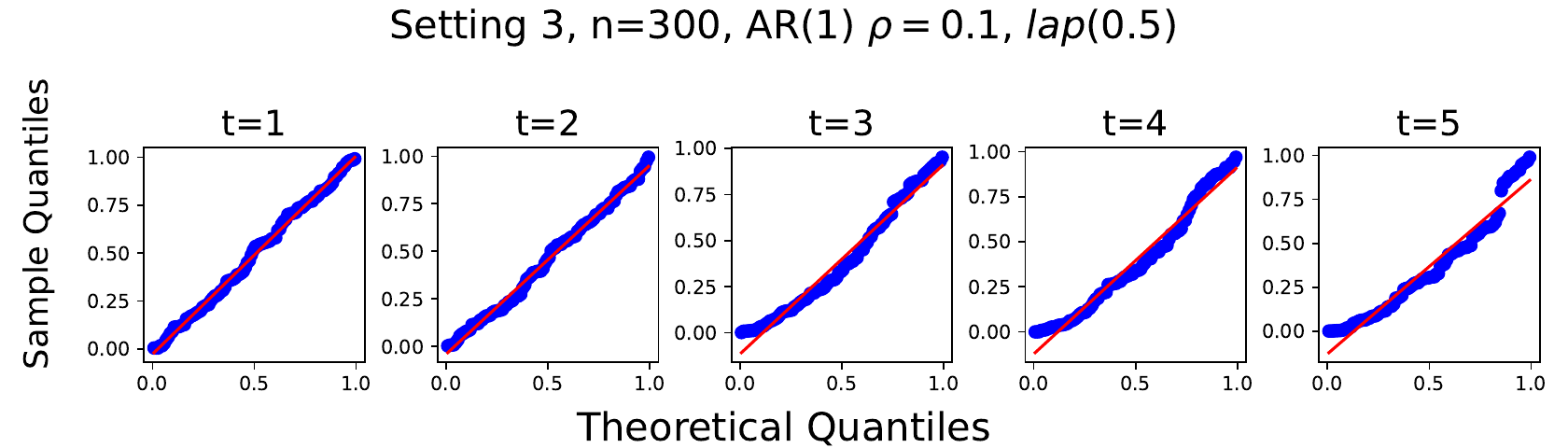}
			\end{subfigure}\vspace{0.8em}
			\begin{subfigure}[b]{0.9\textwidth}
			\includegraphics[width=1\linewidth]{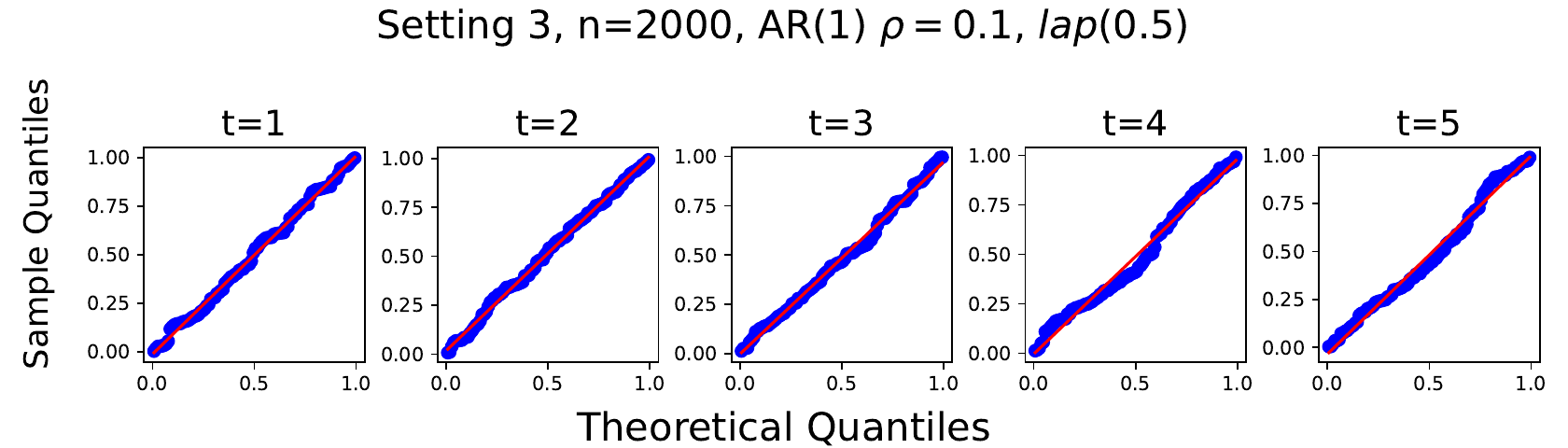}
			\end{subfigure}
				\caption{The Q-Q plot of $W_{n,\id}$ for logistic regression from Setting 3 with AR(1) covariates, $\rho =0.1$,  and the scale of the Laplace noise is 0.5. 
				}
				\label{logisticQQ_setting3_sig0.1_lap0.5}
			\end{figure}

			\begin{figure}[tb]
				\centering
			\begin{subfigure}[b]{0.9\textwidth}
			\includegraphics[width=1\linewidth]{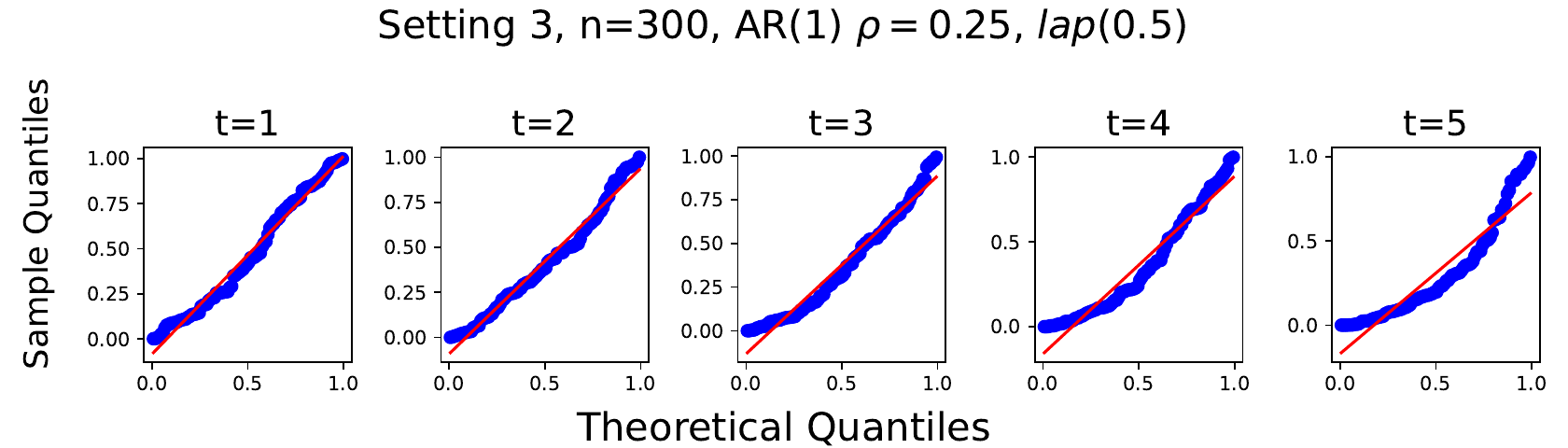}
			\end{subfigure}\vspace{0.8em}
			\begin{subfigure}[b]{0.9\textwidth}
			\includegraphics[width=1\linewidth]{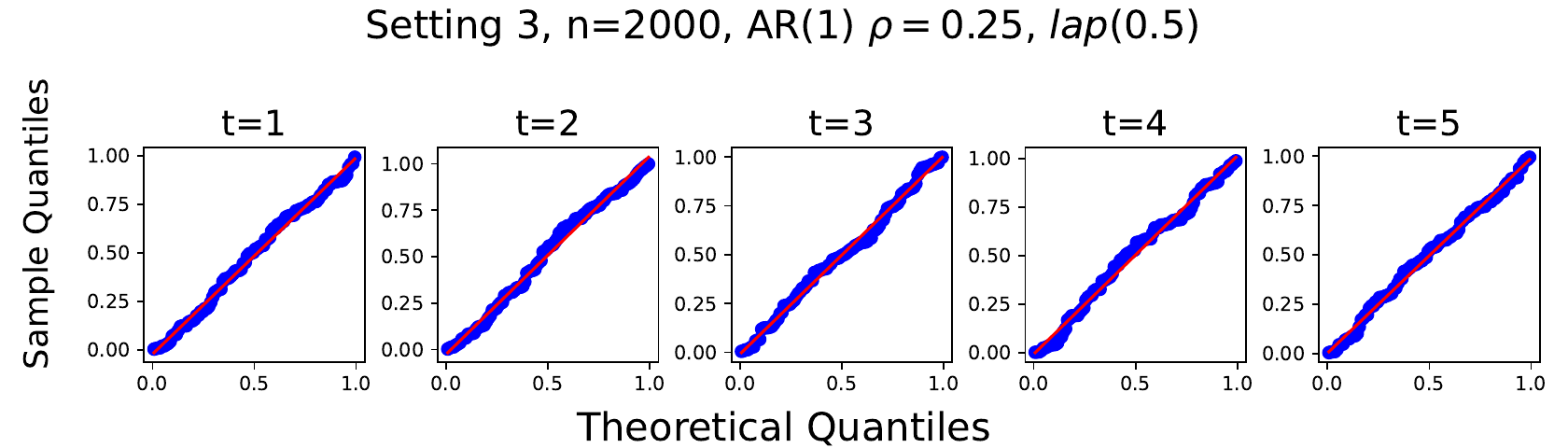}
			\end{subfigure}
				\caption{The Q-Q plot of $W_{n,\id}$ for logistic regression from Setting 3 with AR(1) covariates, $\rho =0.25$,  and the scale of the Laplace noise is 0.5. 
				}
				\label{logisticQQ_setting3_sig0.25_lap0.5}
			\end{figure}

\clearpage

\clearpage
\subsection{Simulation Results of normal regression, Poisson, and robust regression with log-cosh loss}\label{subsec_PoNoH0}
For Normal and Poisson regression, we consider the same parameter and coefficient settings as those for logistic regression, except for Poisson regression,  where we set the nonzero coefficients to 0.1 instead of 0.5 in the experiments for type-I error control under $\textrm{H}_0$. Also, for the experiments on power under $\textrm{H}_1$ and model training, we randomly generate coefficients from a normal distribution with mean 0 and variance 0.01.
In robust regression with log-cosh loss, 
$$
\mest(y,{\bm{x}}^\T \bm{\beta} ) = \frac{1}{\alpha}\log \cosh \bigl(\alpha(y - {\bm{x}}^\T \bm{\beta})\bigr),
$$
where $\alpha>0$ is a tuning parameter that controls the smoothness of the loss function. In the experiments, we set $\alpha$ to be 0.3.
The experiment data are generated in the same way as the Normal regression.
To evaluate the predictive performance of the estimators for Normal regression, Poisson regression, and robust regression with log-cosh loss, we calculate the Euclidean distance between the estimated coefficients to the ones from the oracle estimators.

To keep the presentation concise for Q-Q plots, we only present the results under Setting 2 with AR(1)-correlated covariates, $\rho=0.25$, and Laplace noise with a scale of 0.1. The plots under other settings are similar.
The Q-Q plots for Normal regression, Poisson regression, and robust regression by log-cosh loss under $\textrm{H}_0$ can be found in Figures~\ref{normalQQ_setting2_sig0.25_lap0.1}, \ref{poissonQQ_setting2_sig0.25_lap0.1}, and \ref{logcoshQQ_setting2_sig0.25_lap0.1},respectively. 
The rejection rates under $\textrm{H}_1$ are shown in Figures~\ref{normalPower}, \ref{poissonPower}, and \ref{logcoshPower}. 
The training performance evaluation is presented in Figures~\ref{fig_normal_loglikeli}, \ref{fig_poisson_loglikeli}, and \ref{fig_logcosh_loglikeli}.

			\begin{figure}[tb]
				\centering
			\begin{subfigure}[b]{0.9\textwidth}
			\includegraphics[width=1\linewidth]{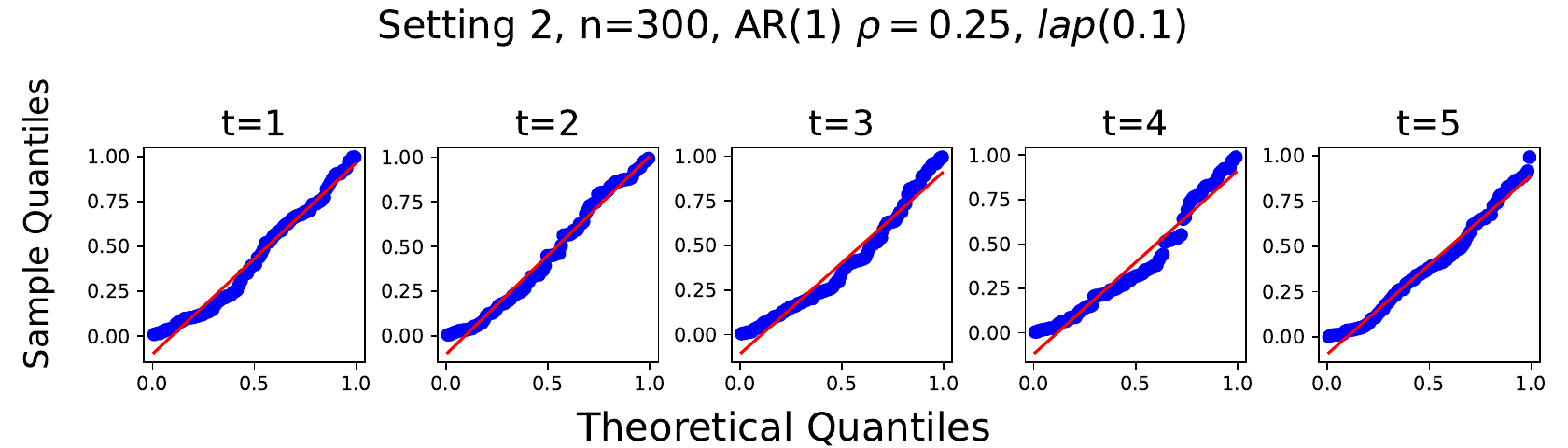}
			\end{subfigure}\vspace{0.8em}
			\begin{subfigure}[b]{0.9\textwidth}
			\includegraphics[width=1\linewidth]{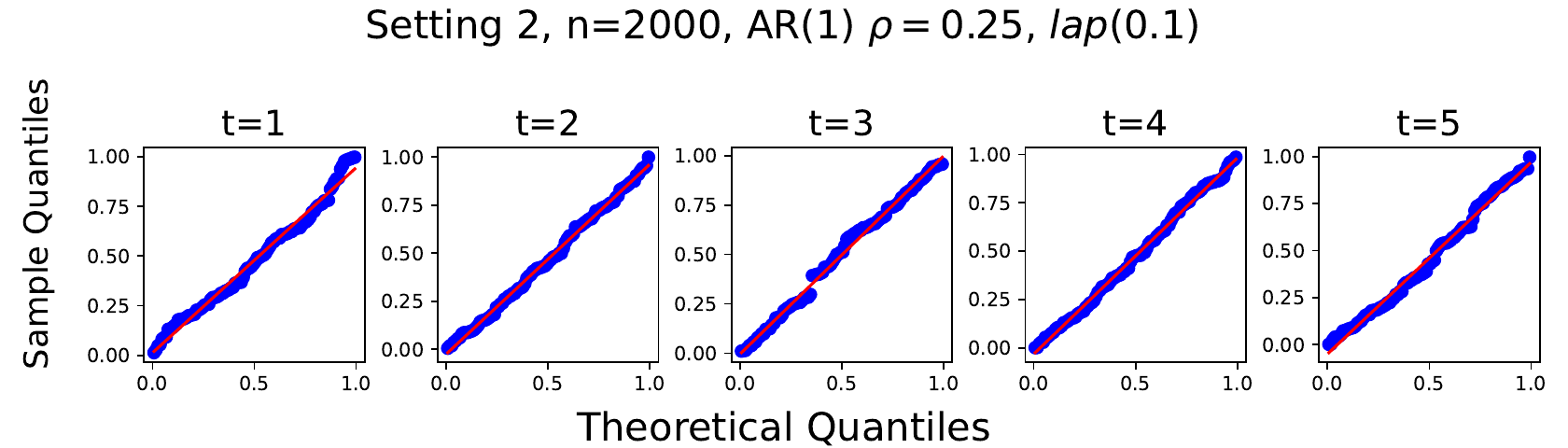}
			\end{subfigure}
				\caption{The Q-Q plot of $W_{n,\id}$ for Normal regression from Setting 2 with AR(1) covariates, $\rho =0.25$,  and the scale of the Laplace noise is 0.1. 
				}
				\label{normalQQ_setting2_sig0.25_lap0.1}
			\end{figure}

			\begin{figure}[tb]
				\centering
			\begin{subfigure}[b]{0.9\textwidth}
			\includegraphics[width=1\linewidth]{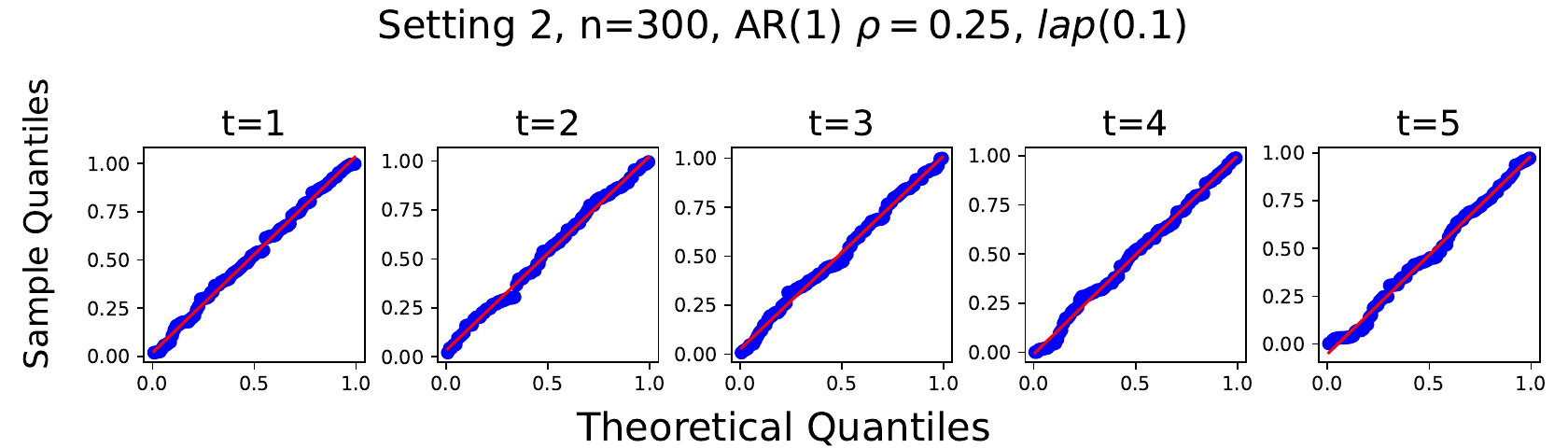}
			\end{subfigure}\vspace{0.8em}
			\begin{subfigure}[b]{0.9\textwidth}
			\includegraphics[width=1\linewidth]{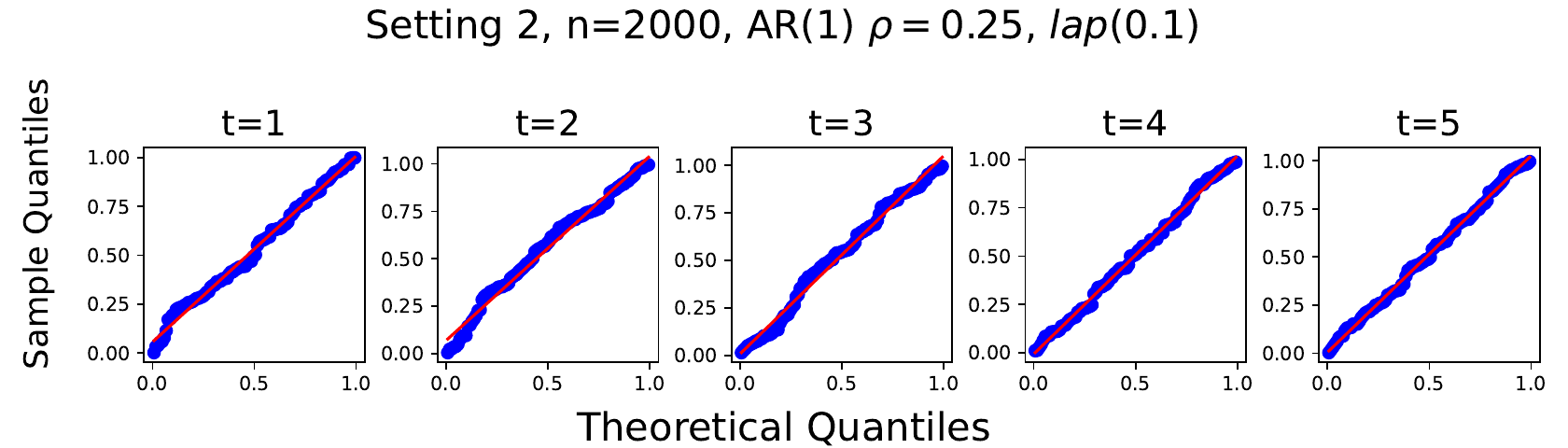}
			\end{subfigure}
				\caption{The Q-Q plot of $W_{n,\id}$ for Poisson regression from Setting 2 with AR(1) covariates, $\rho =0.25$,  and the scale of the Laplace noise is 0.1. We do not fix a set of $U$ for sample sizes 300 and 2000.
				}
				\label{poissonQQ_setting2_sig0.25_lap0.1}
			\end{figure}

   			\begin{figure}[tb]
				\centering
			\begin{subfigure}[b]{0.9\textwidth}
			\includegraphics[width=1\linewidth]{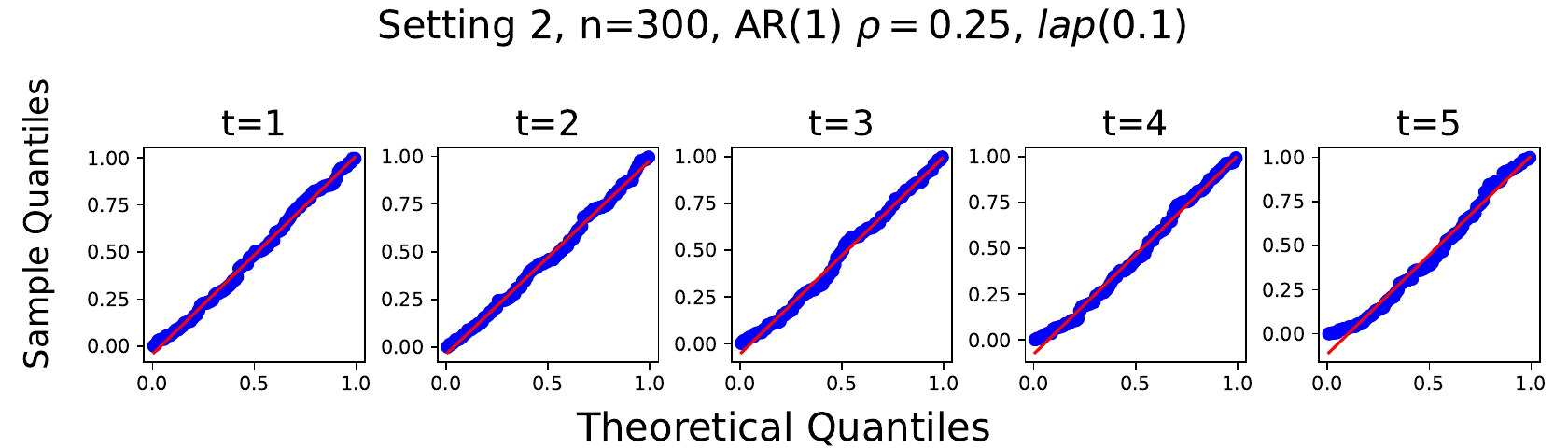}
			\end{subfigure}\vspace{0.8em}
			\begin{subfigure}[b]{0.9\textwidth}
			\includegraphics[width=1\linewidth]{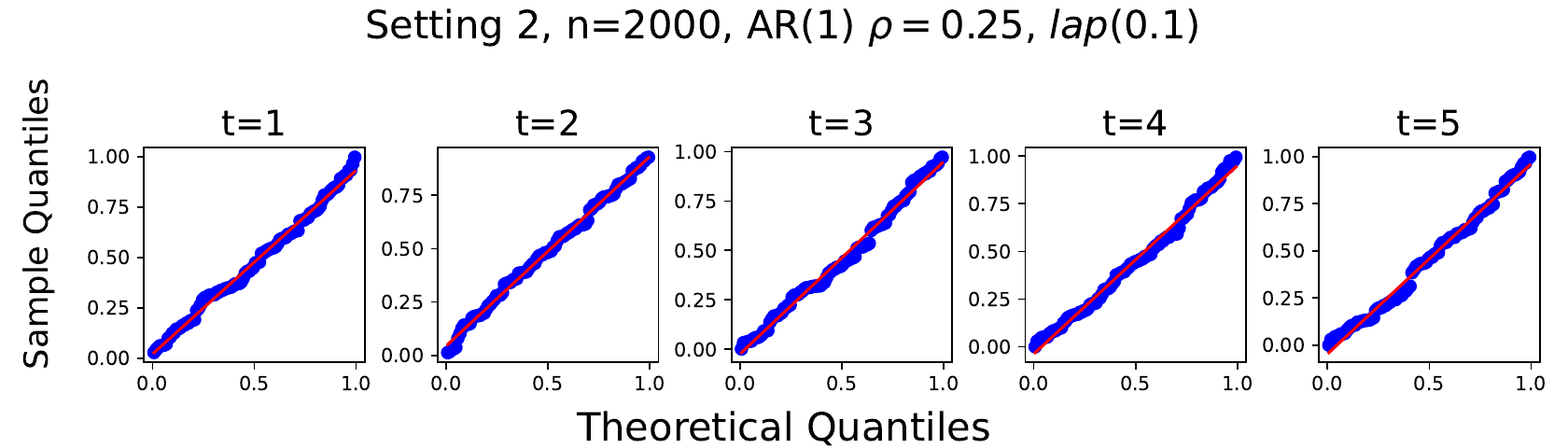}
			\end{subfigure}
				\caption{The Q-Q plot of $W_{n,\id}$ for log-cosh regression from Setting 2 with AR(1) covariates, $\rho =0.25$,  and the scale of the Laplace noise is 0.1. 
				}
				\label{logcoshQQ_setting2_sig0.25_lap0.1}
			\end{figure}


\begin{figure}
\centering
\begin{subfigure}[b]{0.9\textwidth}
   \includegraphics[width=1\linewidth]{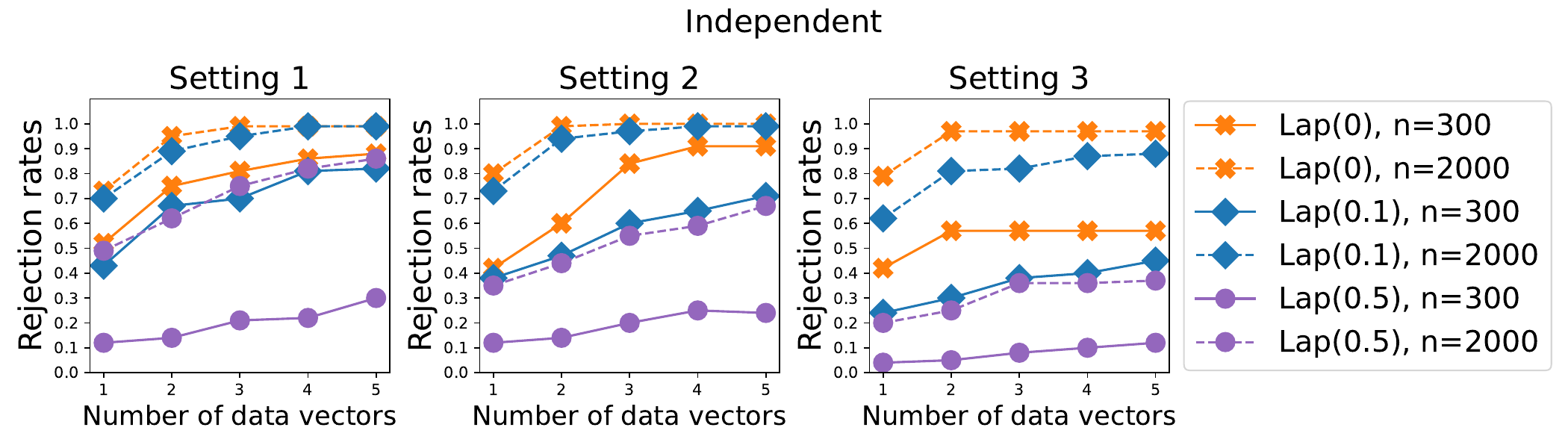}
\end{subfigure}\vspace{0.8em}
\begin{subfigure}[b]{0.9\textwidth}
   \includegraphics[width=1\linewidth]{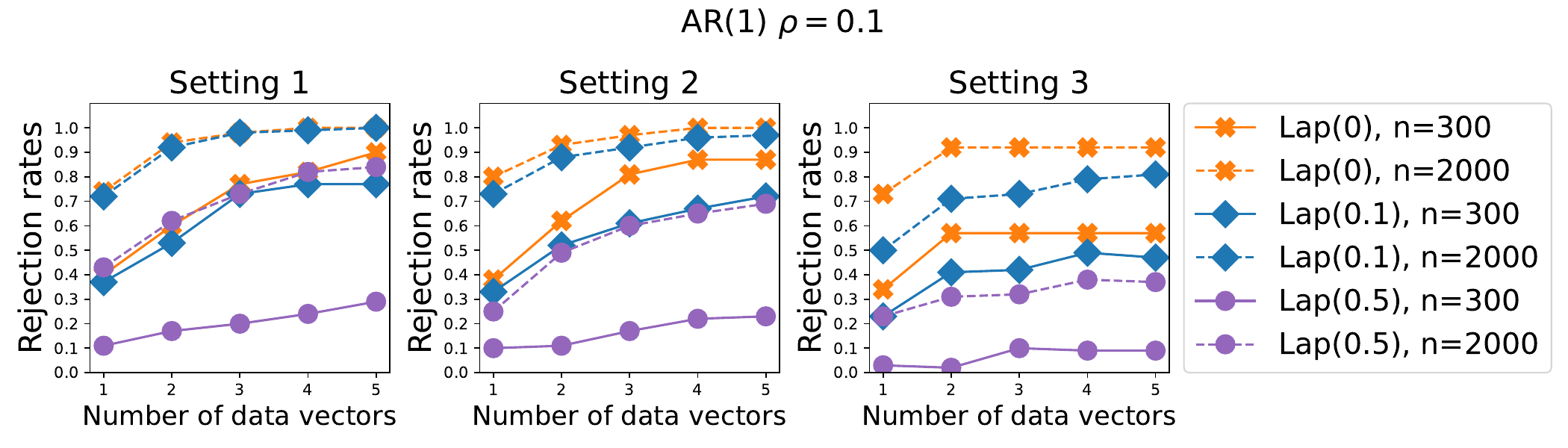}
\end{subfigure}\vspace{0.8em}

\begin{subfigure}[b]{0.9\textwidth}
   \includegraphics[width=1\linewidth]{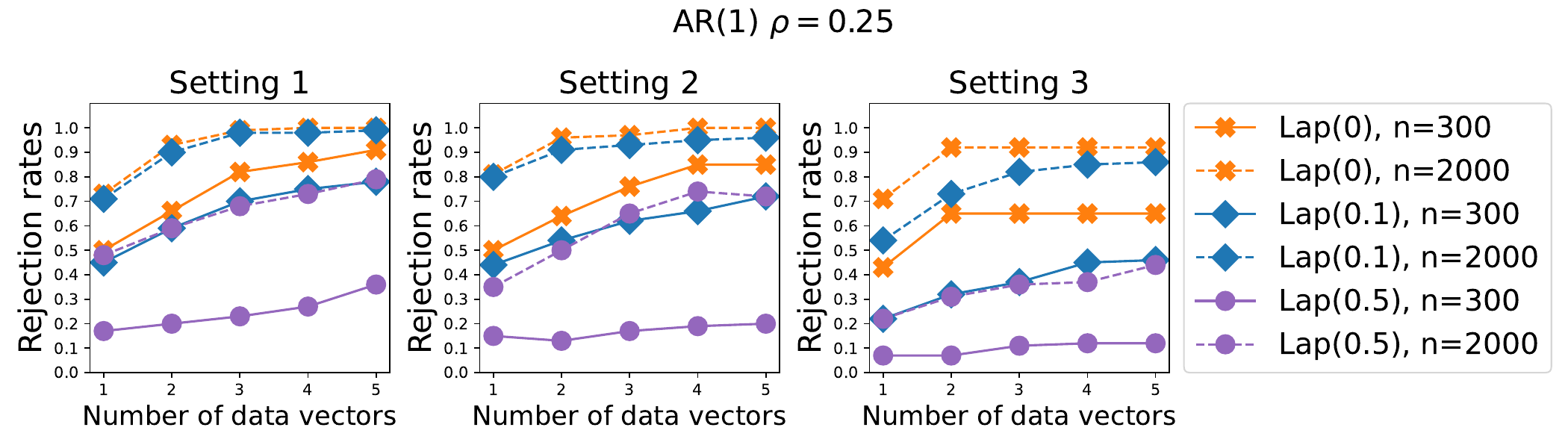}
\end{subfigure}

\caption{The rejection rates of $W_{n,\id}$ for Normal regression versus $\id$ (the number of columns in the sketchy data) from 100 replications, where $\textrm{Lap}(0)$, $\textrm{Lap}(0.1)$, and $\textrm{Lap}(0.5)$ stand for the sketchy data without noise, with Laplace noise of scale 0.1, and with Laplace noise of scale 0.5, respectively.}\label{normalPower}
\end{figure}
   
\begin{figure}
\centering
\begin{subfigure}[b]{0.9\textwidth}
   \includegraphics[width=1\linewidth]{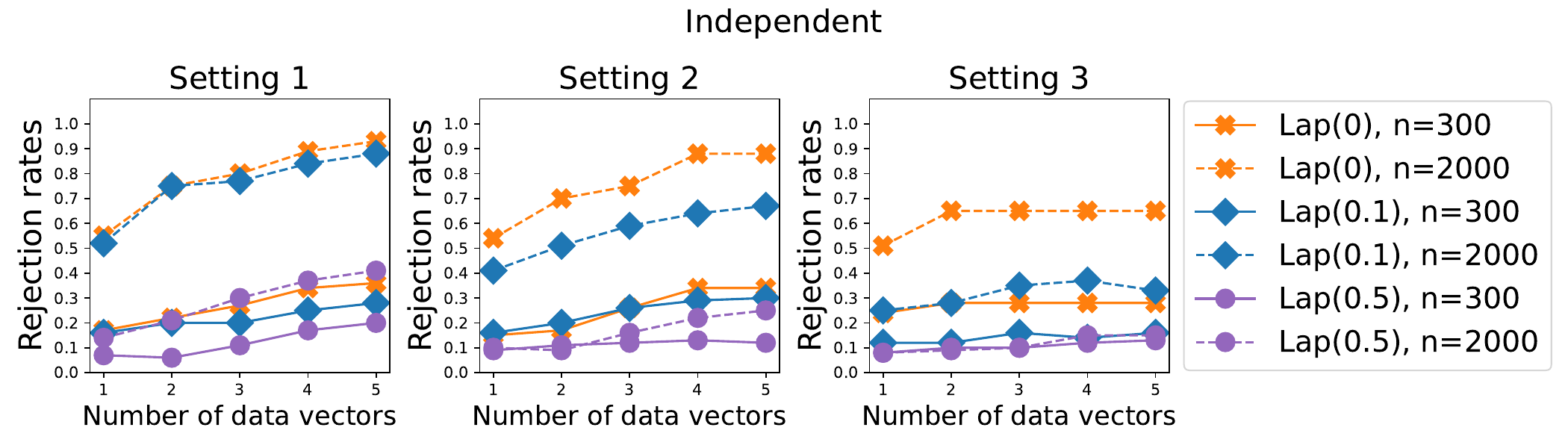}
\end{subfigure}\vspace{0.8em}
\begin{subfigure}[b]{0.9\textwidth}
   \includegraphics[width=1\linewidth]{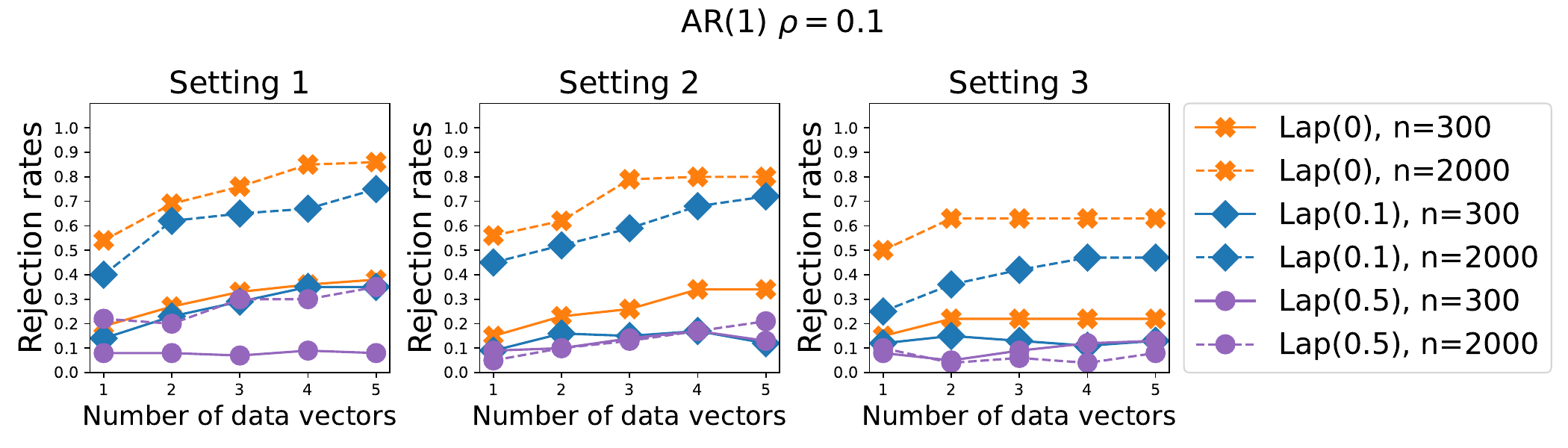}
\end{subfigure}\vspace{0.8em}

\begin{subfigure}[b]{0.9\textwidth}
   \includegraphics[width=1\linewidth]{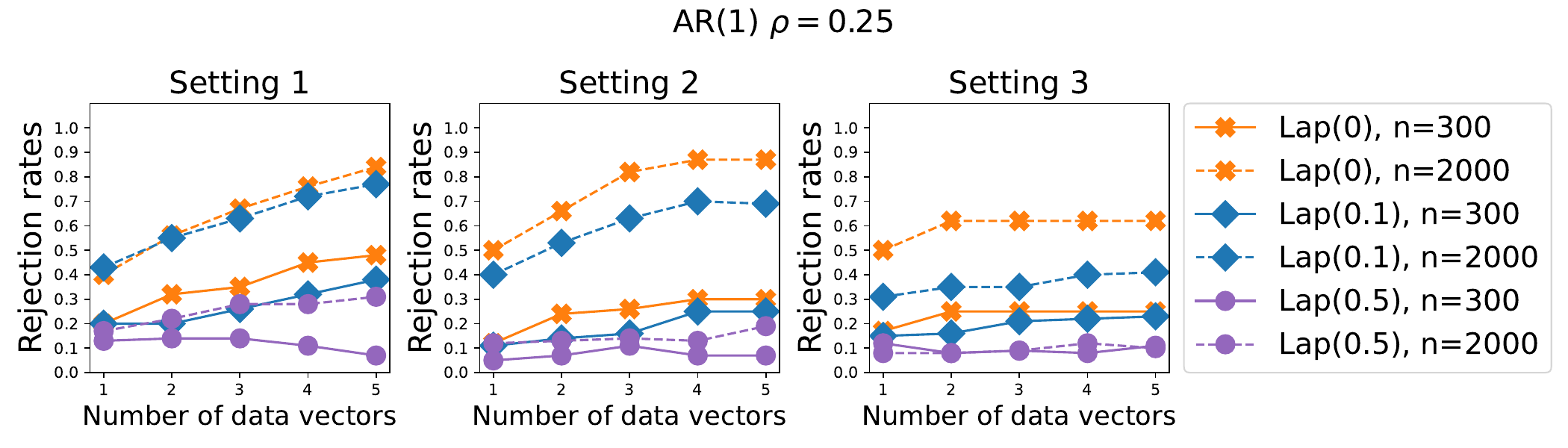}
\end{subfigure}

\caption{The rejection rates of $W_{n,\id}$ for Poisson regression versus $\id$ (the number of columns in the sketchy data) from 100 replications, where $\textrm{Lap}(0)$, $\textrm{Lap}(0.1)$, and $\textrm{Lap}(0.5)$ stand for the sketchy data without noise, with Laplace noise of scale 0.1, and with Laplace noise of scale 0.5, respectively.}\label{poissonPower}
\end{figure}

\begin{figure}
\centering
\begin{subfigure}[b]{0.9\textwidth}
   \includegraphics[width=1\linewidth]{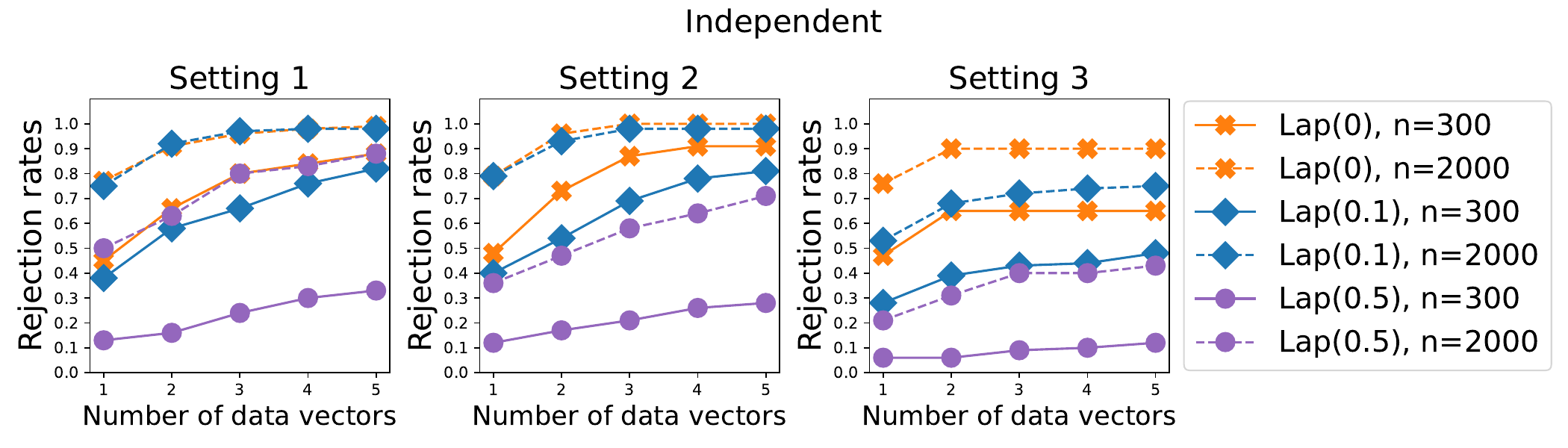}
\end{subfigure}\vspace{0.8em}
\begin{subfigure}[b]{0.9\textwidth}
   \includegraphics[width=1\linewidth]{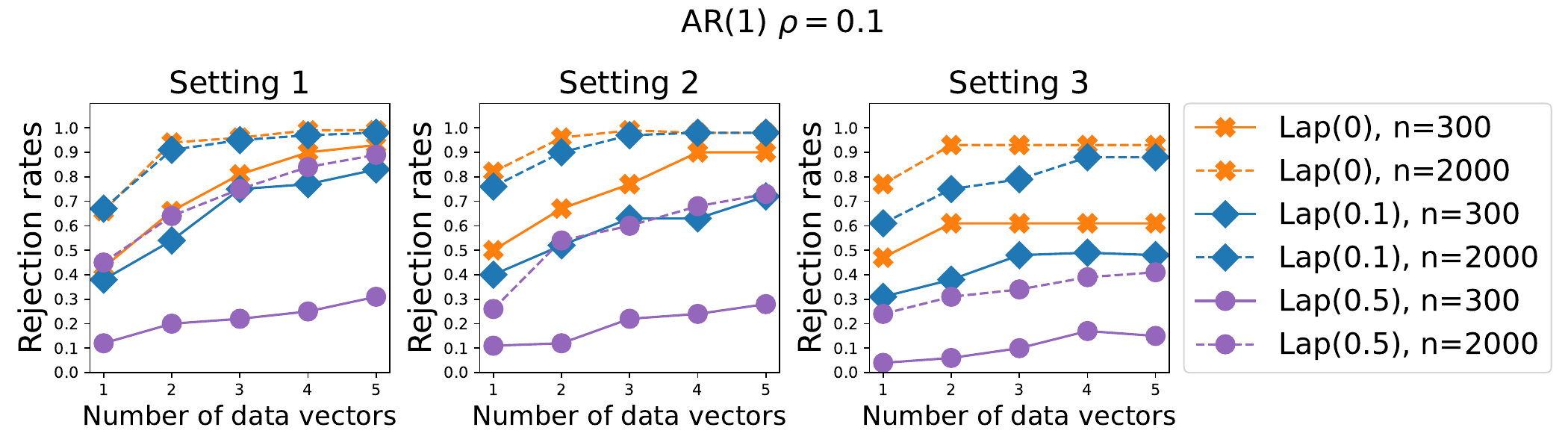}
\end{subfigure}\vspace{0.8em}

\begin{subfigure}[b]{0.9\textwidth}
   \includegraphics[width=1\linewidth]{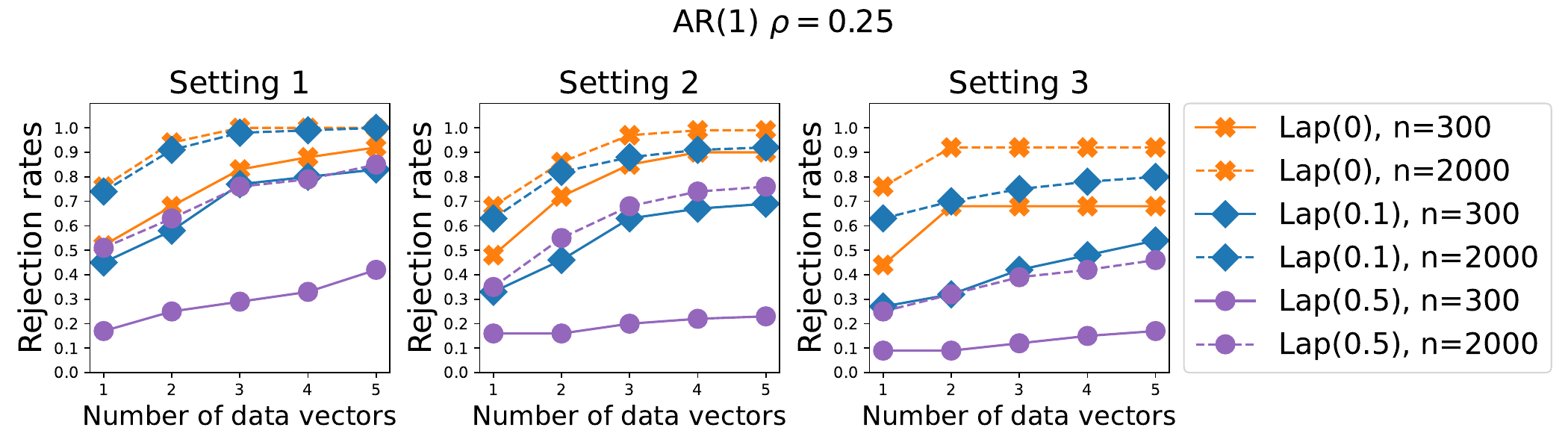}
\end{subfigure}

\caption{The rejection rates of $W_{n,\id}$ for log-cosh regression versus $\id$ (the number of columns in the sketchy data) from 100 replications, where $\textrm{Lap}(0)$, $\textrm{Lap}(0.1)$, and $\textrm{Lap}(0.5)$ stand for the sketchy data without noise, with Laplace noise of scale 0.1, and with Laplace noise of scale 0.5, respectively.}\label{logcoshPower}
\end{figure}

\begin{figure}
\centering
\begin{subfigure}[b]{1\textwidth}
   \includegraphics[width=0.8\linewidth]{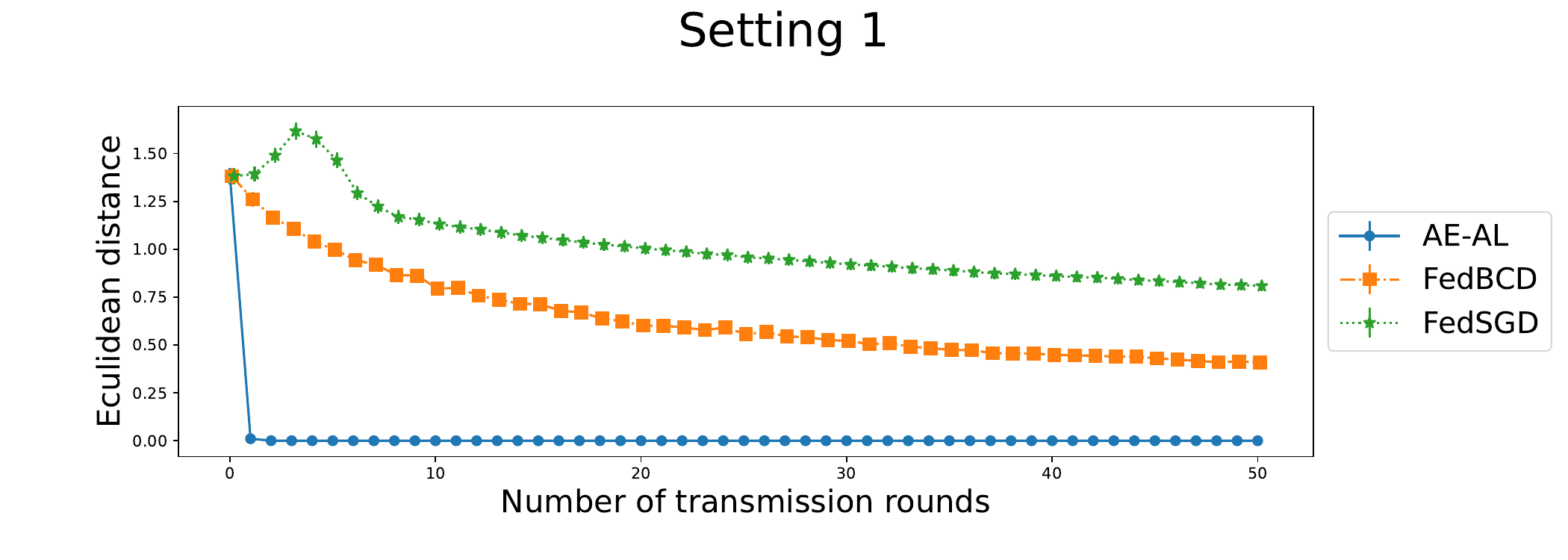}
\end{subfigure}\vspace{0em}
\begin{subfigure}[b]{1\textwidth}
   \includegraphics[width=0.8\linewidth]{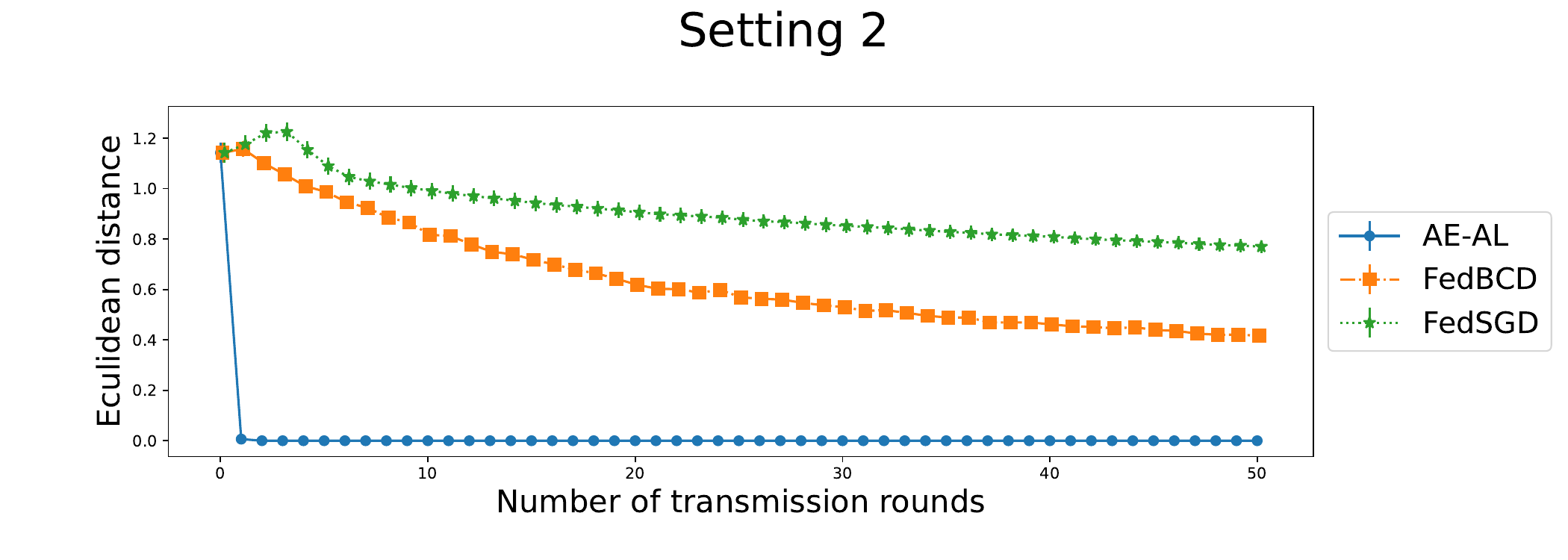}
\end{subfigure}\vspace{0em}

\begin{subfigure}[b]{1\textwidth}
   \includegraphics[width=0.8\linewidth]{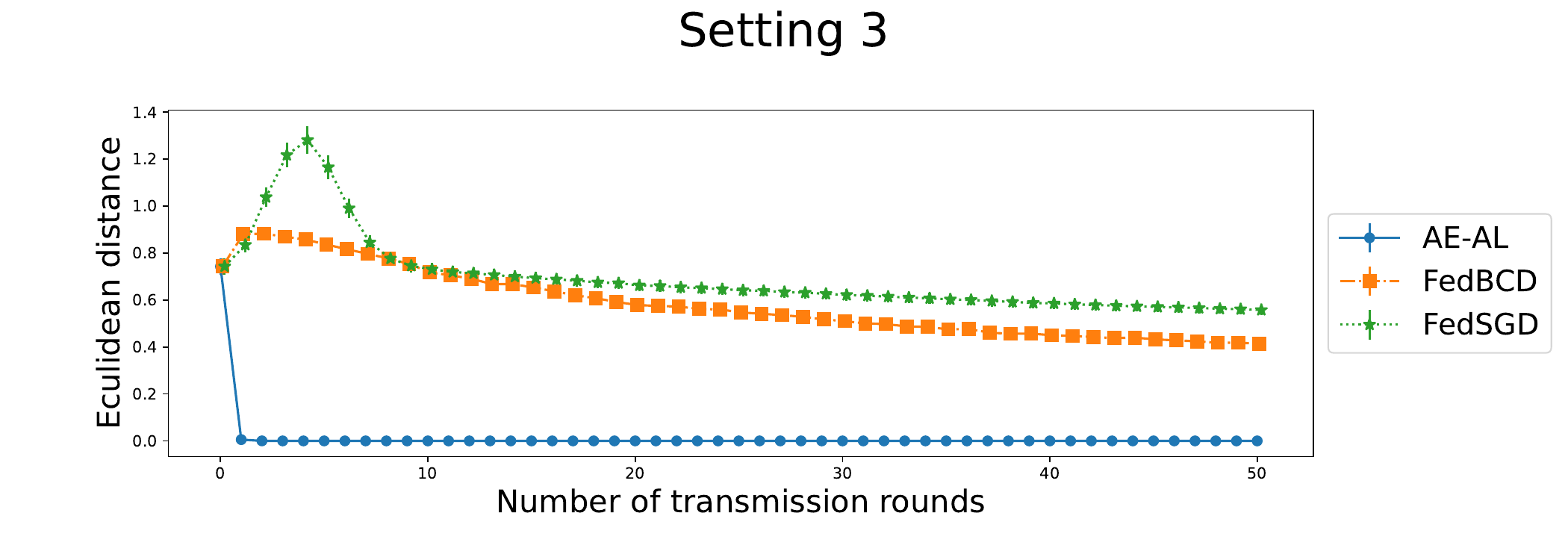}
\end{subfigure}

\caption{
The average Euclidean distance between the estimated coefficients and the ones from the oracle estimator for Normal regression. The vertical bars represent the standard errors based on 100 replications.}\label{fig_normal_loglikeli}
\end{figure}

\begin{figure}
\centering
\begin{subfigure}[b]{1\textwidth}
   \includegraphics[width=0.8\linewidth]{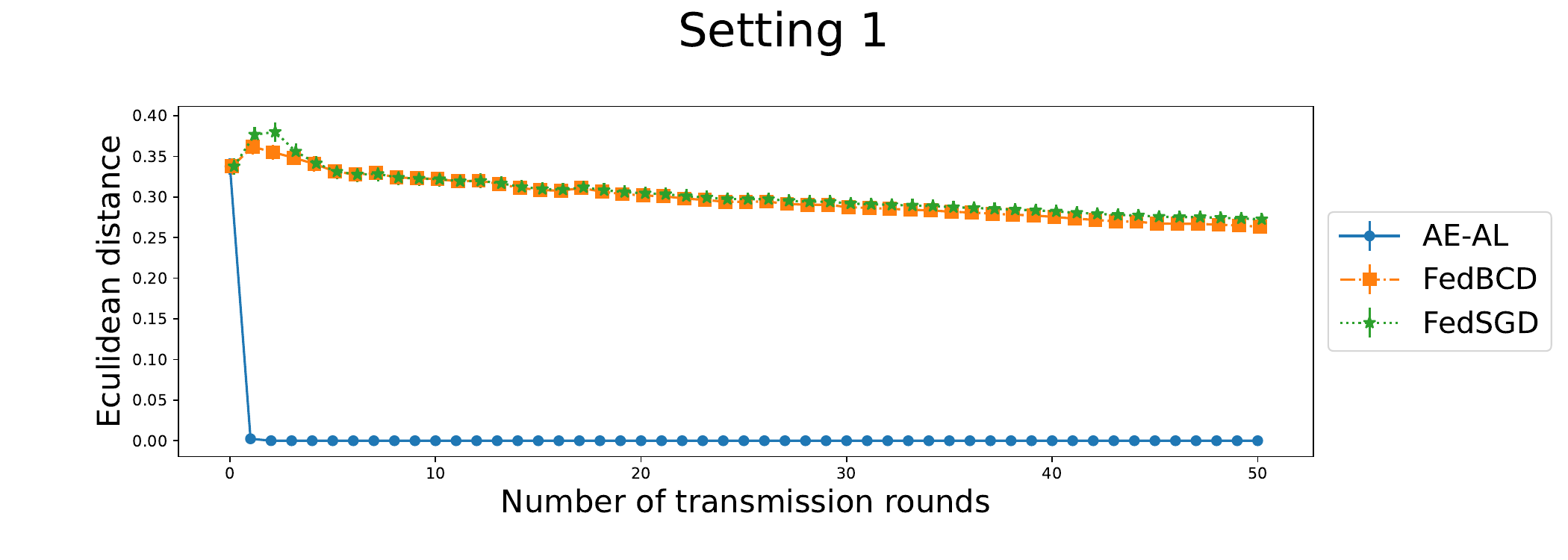}
\end{subfigure}\vspace{0em}
\begin{subfigure}[b]{1\textwidth}
   \includegraphics[width=0.8\linewidth]{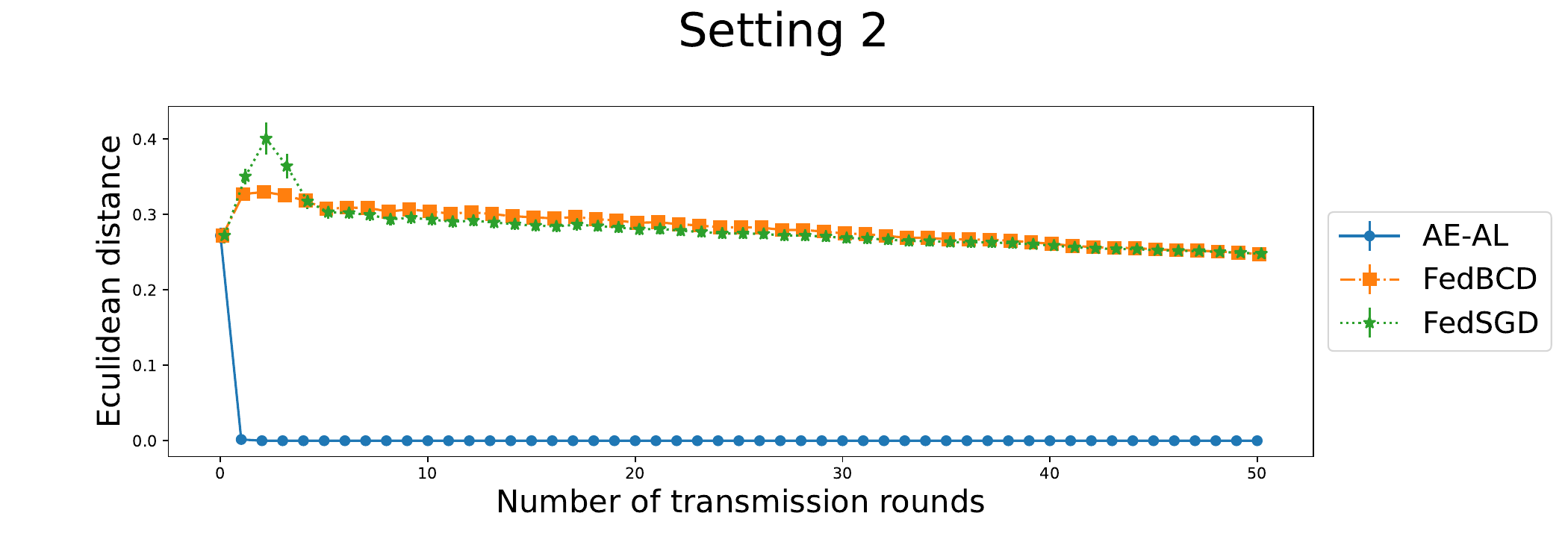}
\end{subfigure}\vspace{0em}

\begin{subfigure}[b]{1\textwidth}
   \includegraphics[width=0.8\linewidth]{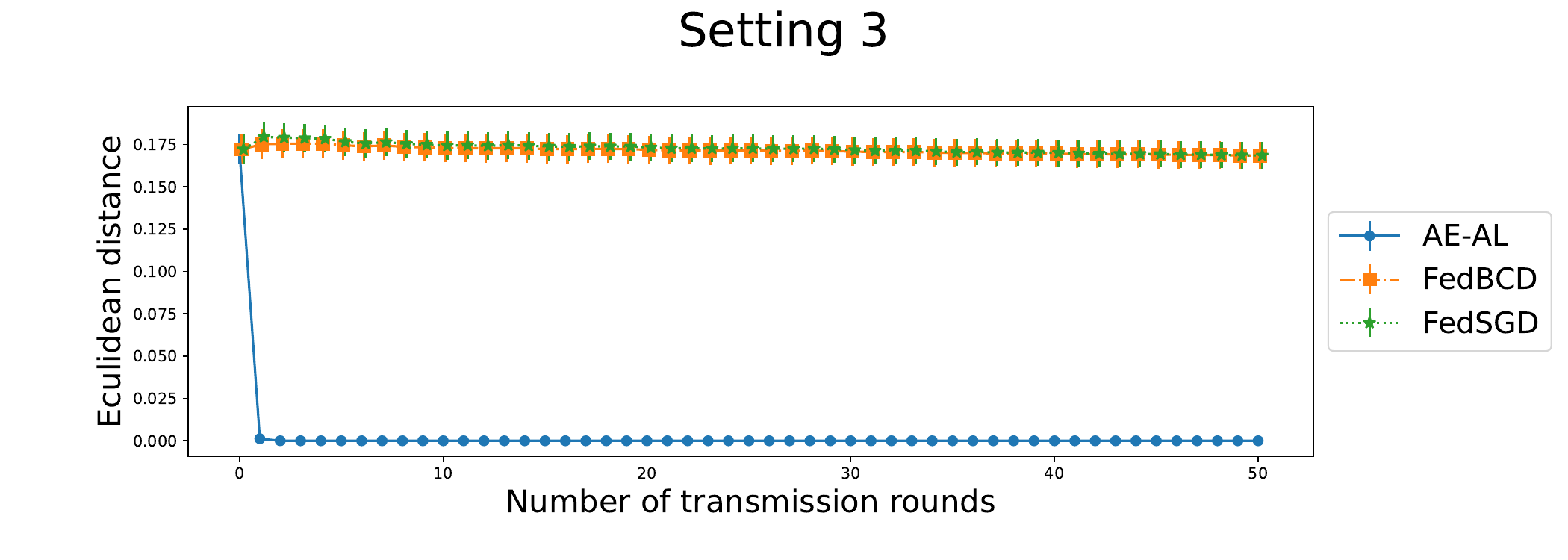}
\end{subfigure}

\caption{
The average Euclidean distance between the estimated coefficients and the ones from the oracle estimator for Poisson regression. The vertical bars represent the standard errors based on 100 replications. }\label{fig_poisson_loglikeli}
\end{figure}

\begin{figure}
\centering
\begin{subfigure}[b]{1\textwidth}
   \includegraphics[width=0.8\linewidth]{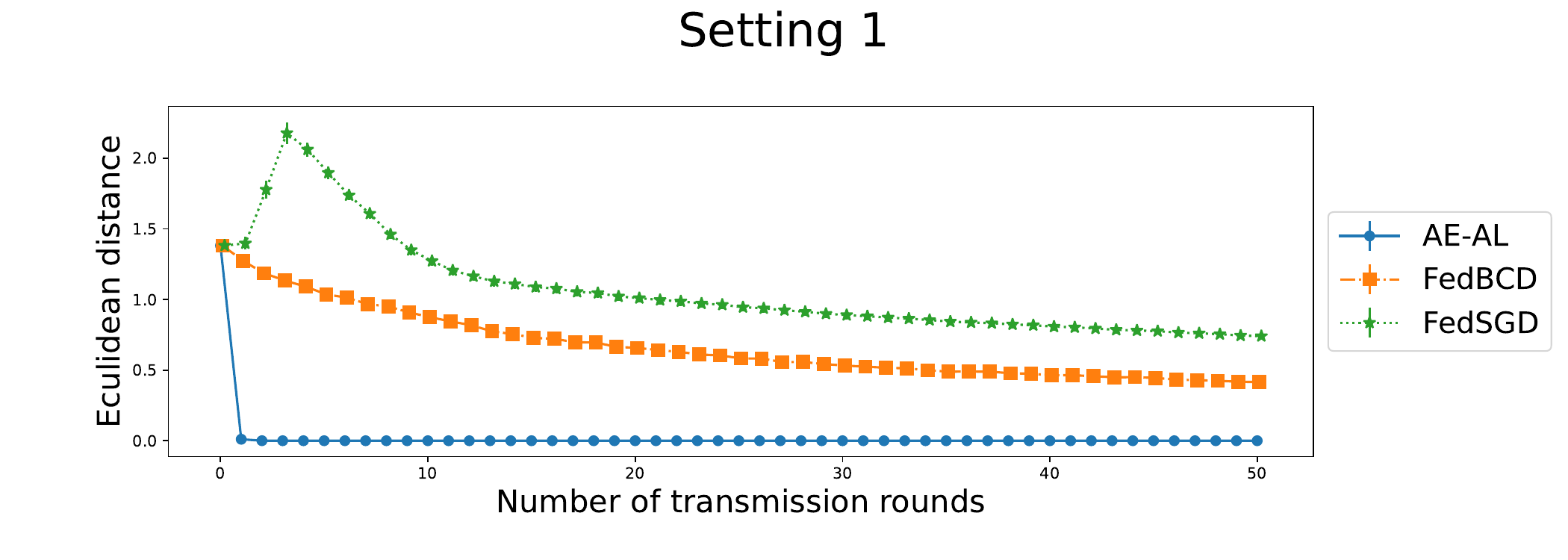}
\end{subfigure}\vspace{0em}
\begin{subfigure}[b]{1\textwidth}
   \includegraphics[width=0.8\linewidth]{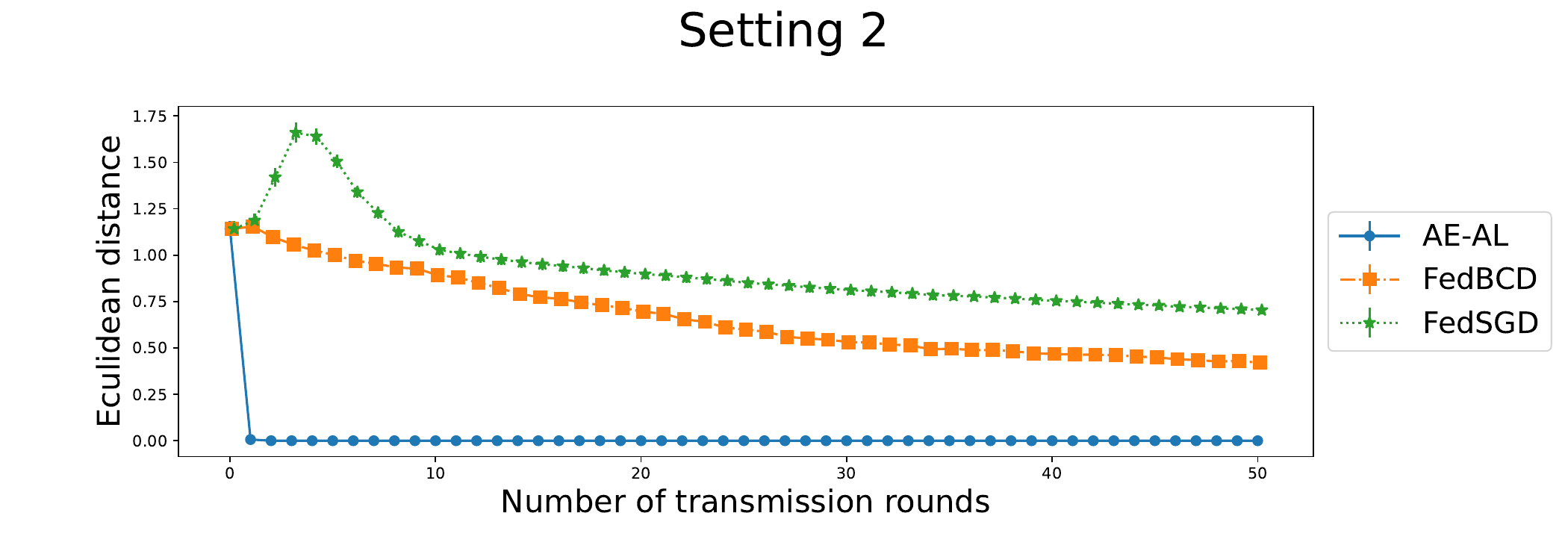}
\end{subfigure}\vspace{0em}

\begin{subfigure}[b]{1\textwidth}
   \includegraphics[width=0.8\linewidth]{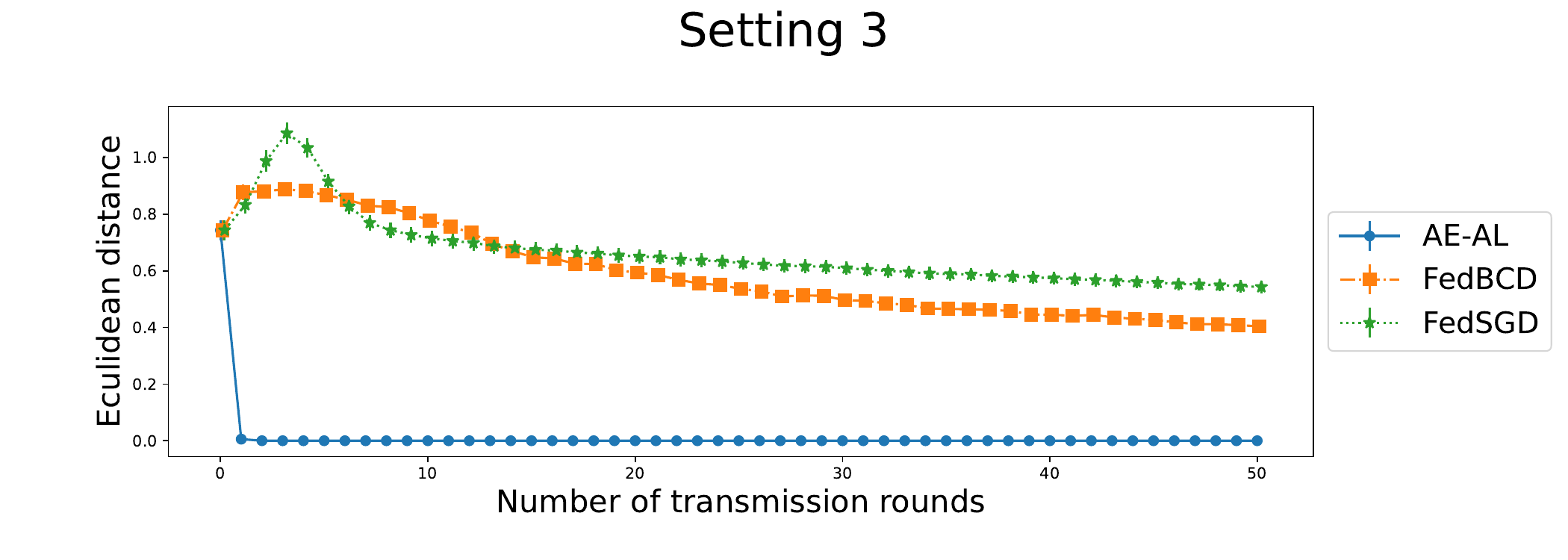}
\end{subfigure}

\caption{
The average Euclidean distance between the estimated coefficients and the ones from the oracle estimator for log-cosh regression. The vertical bars represent the standard errors based on 100 replications. }\label{fig_logcosh_loglikeli}
\end{figure}

\clearpage

\subsection{Robustness of the AE-AL testing with respect to $\bm{U}$}\label{subsec_robustness_U}

In this subsection, we investigate the robustness of the AE-AL testing results with respect to  $\bm{U}$ through two experiments.
In experiment~1, we examine the overall robustness of the AE-AL with respect to the distribution to generate $\bm{U}$. For this experiment, we generate $\bm{U}$ from different distributions and compare their aggregated testing results from multiple replications of experiment datasets.
 In experiment~2, we assess the robustness of the AE-AL with respect to the specific realizations of $\bm{U}$. For this experiment, we compare the testing results obtained on each individual dataset with $\bm{U}$ taken from a set of matrices with different values. For both experiments, we generate data according to the Setting~2 in Section~\ref{subsec_type1} with $n=2000$, $\rho=0.1$.

\subsubsection{Experiment~1: Different distributions to generate $\bm{U}$}\label{subsubsec_robustexp1}
For experiment~1, we replace the standard normal distribution with $t(3)$, $t(5)$, $t(10)$, and $\textrm{Uniform}(-1, 1)$ to generate $\bm{U}$, where $t(\cdot)$ denotes the $t$-distribution. The scale of the Laplace noise is set to be 0.1 and the number of replications is set to be 500. The AE-AL testing results for logistic regression presented in Figures~\ref{Fig_RobustU_dist_H0} and \ref{Fig_RobustU_dist_H1} demonstrate that the choice of distribution for $\bm{U}$ does not have a significant influence on the results.
			\begin{figure}[!ht]
				\centering
			\begin{subfigure}[b]{0.9\textwidth}
			\includegraphics[width=1\linewidth]{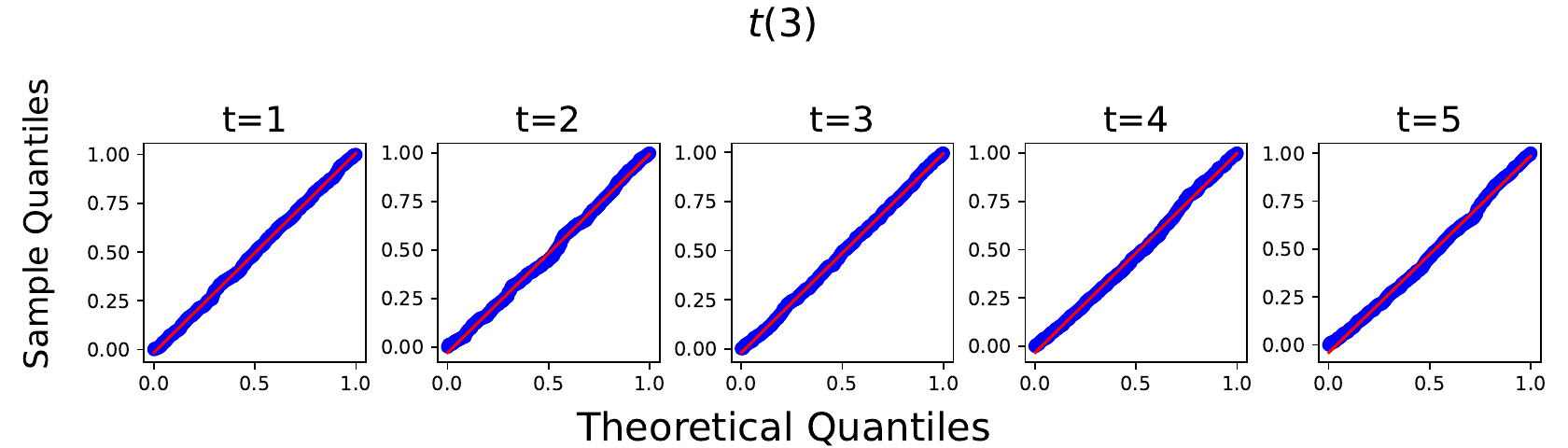}
			\end{subfigure}\vspace{0.8em}
			\begin{subfigure}[b]{0.9\textwidth}
			\includegraphics[width=1\linewidth]{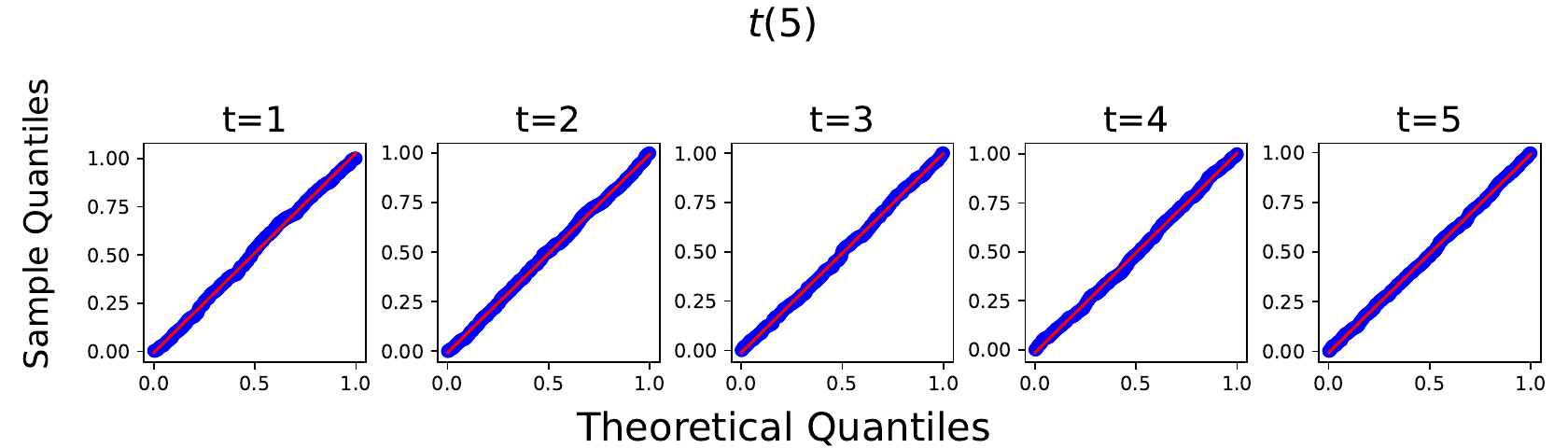}
			\end{subfigure}\vspace{0.8em}
			\begin{subfigure}[b]{0.9\textwidth}
			\includegraphics[width=1\linewidth]{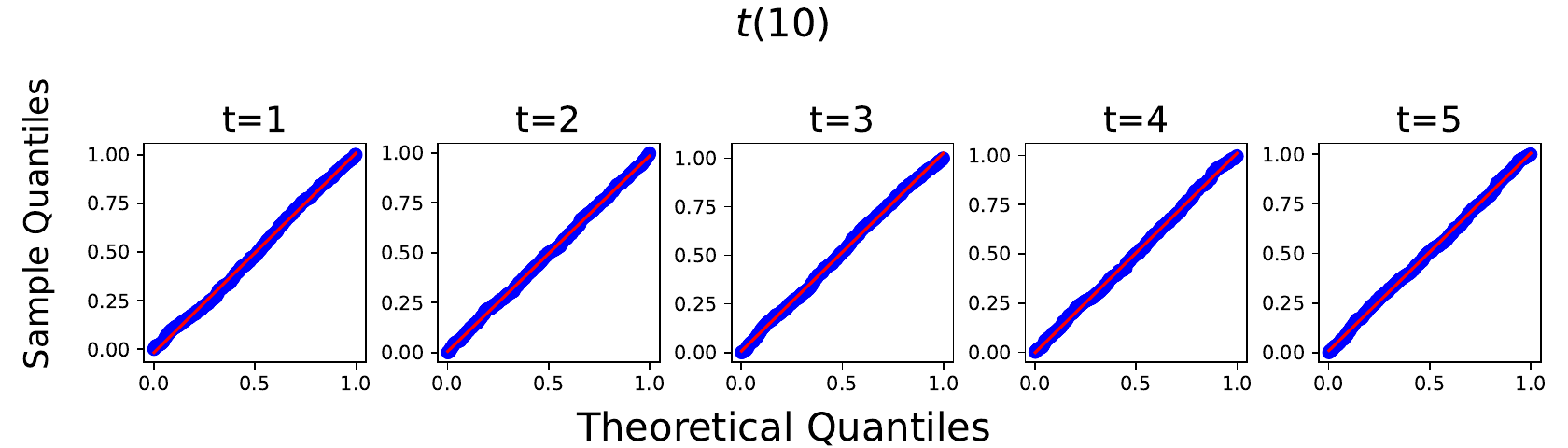}
			\end{subfigure}\vspace{0.8em}
			\begin{subfigure}[b]{0.9\textwidth}
			\includegraphics[width=1\linewidth]{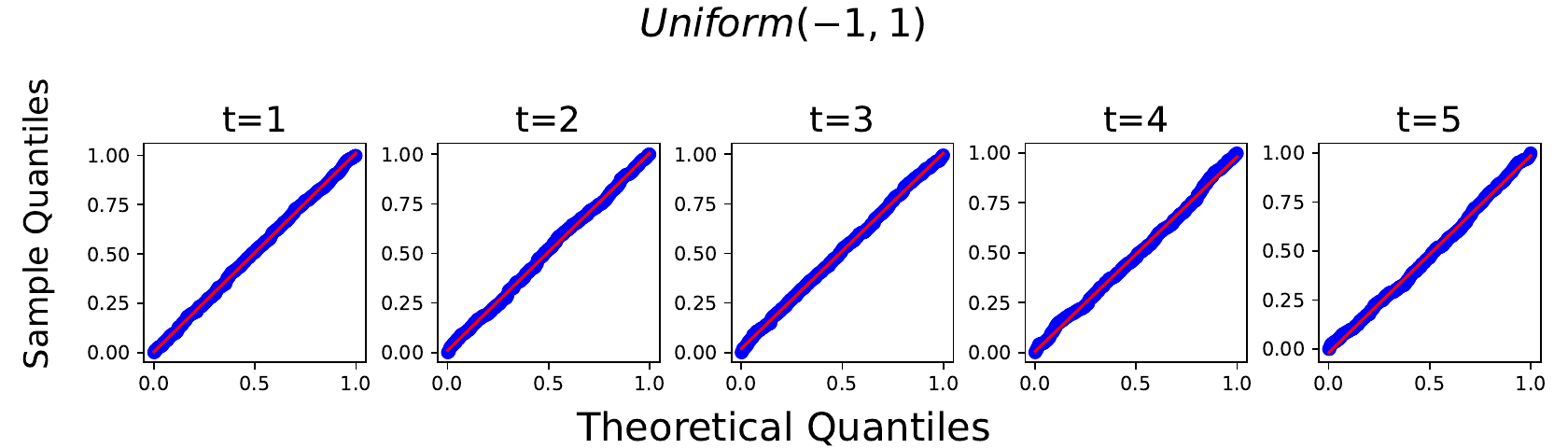}
			\end{subfigure}
				\caption{The Q-Q plot of $W_{n,\id}$ with $\bm{U}$ obtained from different distributions under the setting of $\textrm{H}_0$.  }
		\label{Fig_RobustU_dist_H0}
			\end{figure}

\begin{figure}[!htb]
	\centering
	\vspace{-0.1cm}
	\includegraphics[width=1\linewidth]{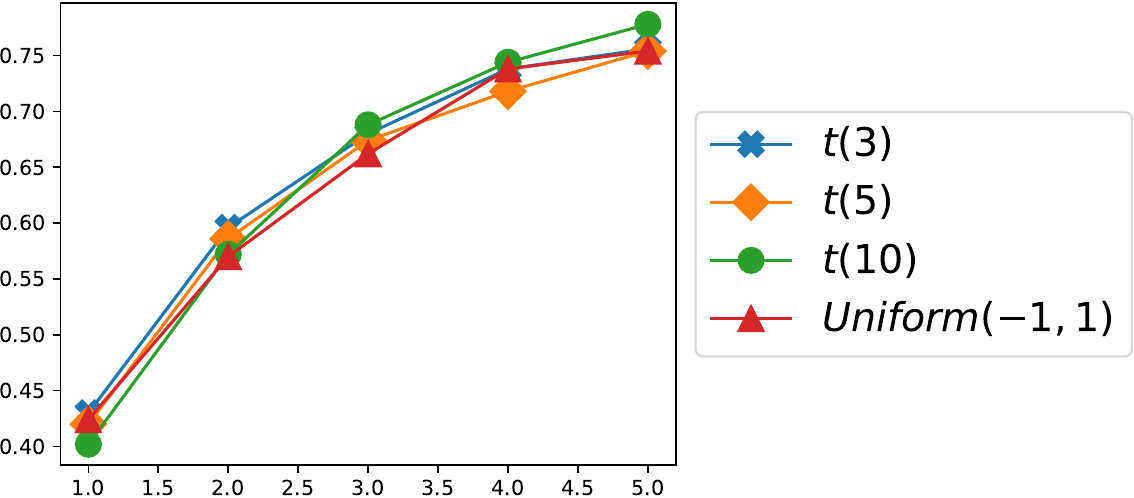}
	\vspace{-0.0cm}
	\caption{The rejection rates of $W_{n,\id}$ with $\bm{U}$ obtained from different distributions under the setting of $\textrm{H}_1$. 
	}
	\label{Fig_RobustU_dist_H1}
\end{figure}

\subsubsection{Experiment~2: Different realizations of $\bm{U}$}\label{subsubsec_robustexp2}
For experiment~2, we first randomly generate six $\bm{U}$ matrices by the standard normal distribution. Then, they are applied to calculate the AE-AL test statistics for 100 independently generated datasets under the settings of $\textrm{H}_0$ and $\textrm{H}_1$, respectively.
We count the number of replications where the results of the six test statistics are in accordance, in the sense that they simultaneously reject or do not reject under the significance level of 0.05.

The results where $\id =1$ (referring to the number of columns of $\bm{U}$) are presented in Table~\ref{tab_Urobust1}. Under $\textrm{H}_0$, the value of $\bm{U}$ does not significantly affect the size control since the number of matches is high. In contrast, under $\textrm{H}_1$, the number of matches is  smaller. 
This is because the improvement from incorporating the data from \B is relatively moderate compared to the sensitivity of the AE-AL test in identifying such improvements in this particular scenario. 
As shown in Section~\ref{subsubsec_power}, AE-AL tests in general have relatively low power under this setting.

We further conduct an experiment to illustrate the relation between the sensitivity of the AE-AL with its robustness with respect to $\bm{U}$.
Recall that the tests tend to have higher power when the number of columns $\id$ of $\bm{U}$ increases. Therefore, we set $\id = 5$ in the new experiment.  The results in Table~\ref{tab_Urobust2} show that the tests without the Laplace noise have the largest increase in the number of matches. This aligns with the results of previous experiments, where the power in this setting has the most significant improvement.
Additionally, the testing results for logistic regression and Poisson regression with Laplace noise of the scale $0.5$ have relatively large number of matches. This is due to the high noise level, which diminishes the power and results in the majority of replications not being rejected.

\begin{table}[!ht]
\centering
\begin{tabular}{llll}
            & $\textrm{Lap}(0)$ & $\textrm{Lap}(0.1)$ & $\textrm{Lap}(0.5)$ \\\hline
logistic $\textrm{H}_0$ & 91     & 87       & 93       \\
logistic $\textrm{H}_1$ & 15     & 16       & 60       \\
normal $\textrm{H}_0$   & 85     & 74       & 94       \\
normal $\textrm{H}_1$   & 27     & 17       & 22       \\
Poisson $\textrm{H}_0$  & 78     & 79       & 92       \\
Poisson $\textrm{H}_1$  & 20     & 18       & 71       \\
log-cosh $\textrm{H}_0$ & 82     & 70       & 90       \\
log-cosh $\textrm{H}_1$ & 27     & 20       & 21      
\end{tabular}
\caption{The counts of replications where the six test statistics have consistent testing results under the setting with $\id=1$. Notations $\textrm{Lap}(0)$, $\textrm{Lap}(0.1)$, $\textrm{Lap}(0.5)$ denote not adding the Laplace noise, adding the Laplace noise with scale 0.1, and adding the Laplace noise with scale 0.5, respectively.
}\label{tab_Urobust1}
\end{table}

\begin{table}[!ht]
\centering
\begin{tabular}{llll}
            & $\textrm{Lap}(0)$ & $\textrm{Lap}(0.1)$ & $\textrm{Lap}(0.5)$ \\\hline
logistic $\textrm{H}_1$ & 100    & 62       & 59       \\
normal $\textrm{H}_1$   & 100    & 94       & 31       \\
Poisson $\textrm{H}_1$  & 100    & 51       & 67       \\
log-cosh $\textrm{H}_1$ & 100    & 96       & 35      
\end{tabular}
\caption{The counts of replications where the six test statistics have consistent testing results under the setting with $\id=5$. Notations are the same with Table~\ref{tab_Urobust1}.
}\label{tab_Urobust2}
\end{table}

%
%
%

\end{appendices}

\bibliographystyle{apalike}
\bibliography{bibliography}

\end{document}